%% file: thesis.tex
\documentclass[a4paper,twoside]{book}
\usepackage[phd]{tugrazthesis}

\usepackage{filecontents}  
\usepackage[backend=bibtex,style=alphabetic,minbibnames=99,maxbibnames=99, maxnames=99,minnames=1]{biblatex}  
\input{header}

\newboolean{includequotes}
\setboolean{includequotes}{true}

\usepackage{makeidx}
\makeindex

\usepackage[refpage]{nomencl}
\makenomenclature

\begin{document}


\thesisauthor[Filip Cano]{Filip Cano C{\'o}rdoba}

\thesistitle[Short Thesis Title]{Towards Responsible AI: Advances in Safety, Fairness, and Accountability of Autonomous Systems}

\thesisdate[ ]{March 2025}

\supervisortitle{\germanenglish{Betreuerin/Betreuer}{Assessors}}


\supervisor{%
\vspace{2em}
\begin{minipage}{0.1\textwidth}
    $\,$
\end{minipage}
\begin{minipage}{0.38\textwidth}
    \centering
    Advisor and examiner\\
  {\large \textsc{Prof. Roderick Bloem}}\\
 {\scriptsize Graz University of Technology}\\
\end{minipage}
\hfill
\begin{minipage}{0.38\textwidth}
    \centering
    Examiner\\
  {\large \textsc{Prof. Ruzica Piskac}}\\
 {\scriptsize Yale University}
\end{minipage}
\begin{minipage}{0.1\textwidth}
    $\,$
\end{minipage}
}




\academicdegree{Doctor of Technical Sciences} 



\printthesistitle

\printaffidavit


\chapter*{Abstract}

Ensuring responsible use of artificial intelligence (AI) has become imperative as autonomous systems increasingly influence critical societal domains. However, the concept of trustworthy AI remains broad and multi-faceted. This thesis advances knowledge in the safety, fairness, transparency, and accountability of AI systems.

In safety, we extend classical deterministic shielding techniques to become resilient against delayed observations, enabling practical deployment in real-world conditions. 
We also implement both deterministic and probabilistic safety shields into simulated autonomous vehicles to prevent collisions with road users, validating the use of these techniques in realistic driving simulators.

We introduce fairness shields, a novel post-processing approach to enforce group fairness in sequential decision-making settings over finite and periodic time horizons. By optimizing intervention costs while strictly ensuring fairness constraints, this method efficiently balances fairness with minimal interference.

For transparency and accountability, we propose a formal framework for assessing intentional behaviour in probabilistic decision-making agents, introducing quantitative metrics of agency and intention quotient. We use these metrics to propose a retrospective analysis of intention, useful for determining responsibility when autonomous systems cause unintended harm.

Finally, we unify these contributions through the ``reactive decision-making'' framework, providing a general formalization that consolidates previous approaches. Collectively, the advancements presented contribute practically to the realization of safer, fairer, and more accountable AI systems, laying the foundations for future research in trustworthy AI.

\chapter*{Kurzfassung}

Die Sicherstellung eines verantwortungsvollen Umgangs mit Künstlicher Intelligenz (KI) ist unabdingbar geworden, da autonome Systeme zunehmend kritische gesellschaftliche Bereiche beeinflussen. Dennoch bleibt das Konzept vertrauenswürdiger KI breit gefächert und facettenreich. Diese Dissertation erweitert das Wissen über Sicherheit, Fairness, Transparenz und Rechenschaftspflicht von KI-Systemen.

Im Bereich der Sicherheit erweitern wir klassische deterministische Shielding-Techniken, sodass sie auch gegenüber verzögerten Beobachtungen widerstandsfähig sind. Dadurch ermöglichen wir deren praktischen Einsatz unter realistischen Bedingungen. Zudem implementieren wir sowohl deterministische als auch probabilistische Sicherheitsschilde in simulierte autonome Fahrzeuge, um Kollisionen mit Verkehrsteilnehmern zu verhindern, und validieren so den Einsatz dieser Techniken in realitätsnahen Fahrsimulatoren.

Wir führen Fairness-Schilde ein, einen neuartigen Post-Processing-Ansatz zur Durchsetzung von Gruppenfairness in sequenziellen Entscheidungssituationen über endliche und periodische Zeithorizonte. Durch die Optimierung der Interventionskosten bei strikter Einhaltung von Fairness-Beschränkungen ermöglicht diese Methode eine effiziente Balance zwischen Fairness und minimalem Eingriff.

Für Transparenz und Rechenschaftspflicht schlagen wir einen formalen Rahmen zur Bewertung intentionalen Verhaltens bei probabilistischen Entscheidungsagenten vor und führen quantitative Maße für Handlungsfähigkeit (Agency) und Intentionsquotienten ein. Diese Maße nutzen wir für eine retrospektive Analyse der Absicht, die hilfreich ist, um Verantwortung festzustellen, wenn autonome Systeme unbeabsichtigte Schäden verursachen.

Schließlich vereinigen wir diese Beiträge im Rahmen der ``reaktiven Entscheidungsfindung'' und bieten so eine allgemeine Formalisierung, die bisherige Ansätze integriert. Die vorgestellten Fortschritte leisten insgesamt einen praktischen Beitrag zur Realisierung sicherer, fairer und verantwortungsvollerer KI-Systeme und bilden eine Grundlage für zukünftige Forschung zu vertrauenswürdiger KI.


\cleardoublepage

\chapter*{\germanenglish{Danksagung}{Acknowledgements}}

\input{05_acknowledgements}

% ...

\cleardoublepage

\tableofcontents

\listoffigures

\listoftables


% ...


\chapter{Introduction}
\label{chap:intro}
\input{10_intro}

\chapter{Preliminaries}
\label{chap:preliminaries}
\input{20_preliminaries}
\chapter{Reactive Decision Making Framework}
\label{chap:reactive_decision_making}
\input{25_reactive_decision_making}

\chapter[Delay-resilient Shielding]{Safety Shielding Resilient to Delayed Observation}
\label{chap:delayed_shields}
\input{30_delayed_safety_shields}

\chapter[Probabilistic Shielding]{Probabilistic Shielding for Autonomous Valet Parking}
\label{chap:foceta}
\input{40_probabilistic_shielding_autonomous_driving}


\chapter[Enforcing Fairness Properties]{Fairness Shields: Enforcing Fairness Properties for Bounded and Periodic Horizons}
\label{chap:fairness}
\input{50_fairness_enforcement}

\chapter[Analyzing Intentional Behaviour]{Analyzing Intentional Behaviour in Autonomous Agents}
\label{chap:intention}
\input{60_intentional_behavior}


\chapter{Conclusion}
\label{chap:conclusion}
\input{80_conclusion}

% ...

\appendix
\chapter*{List of Publications}
\markboth{LIST OF PUBLICATIONS}{LIST OF PUBLICATIONS} 
\addcontentsline{toc}{chapter}{List of Publications} 
\input{90_publications}


\defbibnote{myprenote}{
\ifthenelse{\boolean{includequotes}}{
\begin{quotation}
    \textit{If I have seen further it is by standing on the shoulders of Giants.}
    \textcolor{white}{a}\hfill
    --- Isaac Newton
\end{quotation}
\vspace{2em}
}{}
}
\printbibliography[heading=bibintoc,prenote=myprenote]

\cleardoublepage 

\chapter*{Nomenclature}
\markboth{NOMENCLATURE}{NOMENCLATURE} 
\addcontentsline{toc}{chapter}{Nomenclature} 
\input{95_symbol_usage.tex}



\end{document}

%% file: header.tex
\usepackage[greek,english]{babel}

\usepackage{amsmath, amssymb, amsthm, stmaryrd, mathtools, xfrac}
\usepackage{dsfont, bm, mathrsfs, nicefrac}

\usepackage{fontawesome}
\usepackage{pifont}
\usepackage{siunitx}

\usepackage{xcolor, graphicx, caption, subcaption}

\usepackage{tikz, pgf}
\usetikzlibrary{shapes, automata, positioning, arrows.meta}

\usepackage{booktabs}
\usepackage{multirow}
\usepackage{makecell}
\usepackage{arydshln} 
\makeatletter
\def\adl@drawiv#1#2#3{%
    \hskip.5\tabcolsep
    \xleaders#1\hskip#2\relax
    \hskip.5\tabcolsep
    \hskip\arrayrulewidth}
\makeatother

\usepackage{float, enumitem, calc}

\usepackage[ruled,vlined,linesnumbered]{algorithm2e}
\SetAlFnt{\footnotesize\sffamily}

\usepackage{hyperref}

\usepackage[skip=7pt plus3pt, indent=0pt]{parskip}

\usepackage{calc}
\usepackage{graphicx} 
\newlength\myheight
\newlength\mydepth
\settototalheight\myheight{Xygp}
\DeclareRobustCommand*\inlinegraphics[1]{%
  \settototalheight\myheight{Xygp}%
  \settodepth\mydepth{Xygp}%
  \raisebox{-\mydepth}{\includegraphics[height=\myheight]{#1}}%
}

\usepackage{pifont}
\newcommand{\cmark}{\ding{51}}%
\newcommand{\xmark}{\ding{55}}%

\usepackage{xr} 
\usepackage{ifthen}

\newcommand{\tightcolspace}{8pt} 

\newtheorem{theorem}{Theorem}[chapter]

\newtheorem{lemma}{Lemma}[chapter]
\newtheorem{assumption}{Assumption}[chapter]
\newtheorem{proposition}{Proposition}[chapter]

\newtheorem{example}{Example}[chapter]
\newtheorem*{example*}{Example}

\theoremstyle{definition}
\newtheorem{definition}{Definition}[chapter]

\newcommand{\BB}{\mathbb{B}}
\newcommand{\EE}{\mathbb{E}}
\newcommand{\NN}{\mathbb{N}}
\newcommand{\PP}{\mathbb{P}}
\newcommand{\RR}{\mathbb{R}}
\newcommand{\ZZ}{\mathbb{Z}}
\newcommand{\A}{\mathcal{A}}
\newcommand{\B}{\mathcal{B}}
\newcommand{\D}{\mathcal{D}}
\newcommand{\F}{\mathcal{F}}
\newcommand{\G}{\mathcal{G}}
\newcommand{\I}{\mathcal{I}}
\newcommand{\J}{\mathcal{J}}
\renewcommand{\L}{\mathcal{L}}
\newcommand{\M}{\mathcal{M}}
\renewcommand{\P}{\mathcal{P}}
\newcommand{\R}{\mathcal{R}}
\renewcommand{\S}{\mathcal{S}}
\newcommand{\T}{\mathcal{T}}
\newcommand{\X}{\mathcal{X}}
\newcommand{\Y}{\mathcal{Y}}
\newcommand{\Z}{\mathcal{Z}}

\newcommand{\eps}{\varepsilon}
\newcommand{\1}{\mathds{1}}
\DeclareMathOperator*{\argmax}{arg\,max}
\DeclareMathOperator*{\argmin}{arg\,min}
\newcommand{\ol}[1]{\overline{#1}}
\newcommand{\round}[1]{\lfloor #1\rceil}

\newcommand\shield{%
    \tikz [baseline] \draw (0,1.75ex) -- (0,0.75ex) arc [radius=0.75ex, start angle=-180, end angle=0] -- (1.5ex,1.75ex) -- cycle;%
}

\newcommand\myclock{%
    \tikz [baseline] {
        \draw (0,0) circle (0.48ex); 
        \draw (0,0) -- (-0.24ex,0.18ex); 
        \draw (0,0) -- (0.27ex, 0.3ex); 
    }%
}

\newcommand{\Obs}{\mathcal O} 
\newcommand{\Act}{\mathcal A} 
\newcommand{\Env}{\mathscr{E}} 
\newcommand{\Trans}{\mathscr{T}} 
\newcommand{\Pol}{\pi} 
\newcommand{\Ag}{\mathit{Ag}} 
\newcommand{\Sh}{\shield} 
\newcommand{\spec}{\varphi} 
\newcommand{\Supp}{\mathrm{Supp}} 
\newcommand{\Acc}{\mathrm{Acc}} 
\newcommand{\Beh}{\mathtt{Beh}} 
\newcommand{\Inter}{\mathbb{J}} 
\newcommand{\Cyl}{\mathtt{Cyl}} 
\newcommand{\Reach}{\mathtt{Reach}} 
\newcommand{\Avoid}{\mathtt{Avoid}} 
\newcommand{\Delayed}{\mathtt{Delayed}} 
\newcommand{\Mcar}{{\mathcal M}_{\mathrm{car}}} 
\newcommand{\Scar}{{\mathcal S}_{\mathrm{car}}} 
\newcommand{\Pcar}{{\mathcal P}_{\mathrm{car}}} 
\newcommand{\Xcar}{{X}_{\mathrm{car}}} 
\newcommand{\Vcar}{{V}_{\mathrm{car}}} 
\newcommand{\Mped}{{\mathcal M}_{\mathrm{ped}}} 
\newcommand{\Sped}{{\mathcal S}_{\mathrm{ped}}} 
\newcommand{\Pped}{{\mathcal P}_{\mathrm{ped}}} 
\newcommand{\Xped}{{X}_{\mathrm{ped}}} 
\newcommand{\Yped}{{Y}_{\mathrm{ped}}} 
\newcommand{\Pag}{{P}_{\mathrm{ag}}} 
\newcommand{\FinShield}{\textsf{FinHzn}\xspace}
\newcommand{\StaticDP}{\textsf{Static-Fair}\xspace}
\newcommand{\StaticBAR}{\textsf{Static-BW}\xspace}
\newcommand{\Dyn}{\textsf{Dynamic}\xspace}

\newcommand{\sgen}{\theta}

\newcommand{\sshield}{\pi}
\newcommand{\sShield}{\Pi}
\newcommand{\sShieldFeas}{\Pi_{\mathtt{fair}}}
\newcommand{\sShieldFeasPeriodic}{\Pi_{\mathtt{fair\text{-} per}}}
\newcommand{\sShieldFeasBounded}{\Pi_{\mathtt{BW}}}
\newcommand{\sShieldFeasDyn}{\Pi_{\mathtt{fair\text{-} dyn}}}
\newcommand{\costset}{\mathbb{C}}
\newcommand{\Trfeas}{\mathtt{FT}}
\newcommand{\Trbalance}{\mathtt{BT}}

\newcommand{\welfare}[1]{\mathtt{WF}^{#1}}
\newcommand{\cost}{\mathit{cost}}
\newcommand{\expe}{\mathbb{E}}
\newcommand{\indicator}[1]{\1 \left[ #1 \right]}
\newcommand{\set}[1]{\left\{#1\right\}}
\newcommand{\distrset}{\D}
\newcommand{\len}[1]{|#1|}
\newcommand{\numer}{\mathtt{num}}
\newcommand{\den}{\mathtt{den}}
\newcommand{\DP}{\mathtt{DP}}
\newcommand{\EO}{\mathtt{EqOpp}}

\newcommand{\Val}{\mathrm{Val}}
\newcommand{\AP}{\mathrm{AP}}
\newcommand{\tauref}{\tau_{\textit{ref}}}

\newcommand{\dashline}{\cdashline{1-1}[0.5pt/2pt]}

%% file: 05_acknowledgements.tex
This work would not have been possible without the help and support of so many people.

A very special thanks to my parents for their unconditional love and support, my grandmothers Lina and Ana, and my surrogate grandmother Isa, who raised me to be the person I am today.
I also want to thank the rest of my family, Dani, Carmen, Juan, Maria, Toni, and Antonia, for being there and reminding me where I come from.

I thank all the people I had the pleasure and privilege to collaborate all these years. 
A special thanks to Sam, Timos, and Katrine for their patient discussions; Kaushik and Konstantin for their motivation and the long work hours; Haritz for his hard work and kind soul; and Scott, Tom, Martin, Oliver, and Ruzica for the inspiration to do great research.

Thanks to all the friends and colleagues in Graz, whose daily presence and support has been very valuable: Malte, Sonja, Masoud, Vedad, Benedikt, Johannes;  and a very special appreciation to Stefan, with whom I've had the pleasure to share a working space, to work and learn (and sometimes complain about life) together.
For the time I spent in G\"ossendorf, thanks to Ben and Niki for so many joyful moments spent together, to Alex and Irmi for the wine and the kind words, and thanks to Erika for her unending enthusiasm.

Thanks to all my friends who have had the patience to keep and feed our friendship from the distance, specially Ander, Zaira, and Irene for reminding me what is important in life; Maribel, Cristina, and Sa\"id for the random walks; Marc for being simply amazing; Susana for being weird together; Katia and Genís for their enormous heart; Tania for her resilience. A very special thanks to Alberto, without whom I would not have come to Graz, and without whom I may not have managed to finish. Thanks for being sunshine in spring and a lighthouse in my darkest hours.

None of this would have been possible without my teachers, who taught me everything I know and inspired me to learn more. From the high school teachers at Lestonnac and Sant Pere, who inspired my love for mathematics and showed me the kind of person I want to be, to the teachers in FME and CFIS, who showed me the passion, rigour, and beauty in mathematics and brought me to the border between mathematics and computer science, where I've spent some of my best time.


Research, for me, involves coffee, and someone had to pay for all the coffee. In my case,
this work has been partially supported by the European Union's Horizon 2020 research and innovation programme under grant agreement  $\mathrm{N^o}\, 956123$ - \textsc{Foceta} and by the State Government of Styria, Austria - Department Zukunftsfonds Steiermark.

Last but certainly not least, I want to thank Roderick and Bettina for their guidance during these years, for hours and hours of work together, for helping me understand the inner workings of the business, for setting high standards of work, for introducing me to so many great researchers, and for inspiring me to do great work.

%% file: 10_intro.tex
\ifthenelse{\boolean{includequotes}}{
\begin{quotation}
    \foreignlanguage{greek}{\textit{
    \textsc{δ}εικνύναι δὴ δεῖ τοῖς τοιούτοις ὅτι ἔστι πᾶν τὸ πρᾶγμα οἷόν τε καὶ δι᾽ ὅσων πραγμάτων καὶ ὅσον πόνον ἔχει. 
    \textsc{ὁ} γὰρ ἀκούσας, ἐὰν μὲν ὄντως ᾖ φιλόσοφος οἰκεῖός τε καὶ ἄξιος τοῦ πράγματος θεῖος ὤν, ὁδόν τε ἡγεῖται θαυμαστὴν ἀκηκοέναι συντατέον τε εἶναι νῦν καὶ οὐ βιωτὸν ἄλλως ποιοῦντι: μετὰ τοῦτο δὴ συντείνας αὐτός τε καὶ τὸν ἡγούμενον τὴν ὁδόν, οὐκ ἀνίησιν πρὶν ἂν ἢ τέλος ἐπιθῇ πᾶσιν, ἢ λάβῃ δύναμιν ὥστε αὐτὸς αὑτὸν χωρὶς τοῦ δείξοντος δυνατὸς εἶναι ποδηγεῖν.
    }}
    \footnote{``One should show such people what philosophy is in all its extent; the range of studies by which it is approached, and how much labour it involves. 
    For the person who has heard this, if she has the true philosophic spirit and that godlike temperament which makes her a kin to philosophy and worthy of it, thinks that she has been told of a marvellous road lying before her, that she must forthwith press on with all her strength, and that life is not worth living if she does anything else.''
    }
    \\
    \textcolor{white}{a} \hfill --- Plato, seventh letter.
\end{quotation}
}
{}



\section{Motivation}

As AI systems increasingly permeate critical domains such as healthcare, finance, mobility, and human resources; ensuring the responsible and trustworthy behaviour of these autonomous systems becomes imperative. Without proper safeguards, AI models can make decisions that are unsafe, biased, or otherwise misaligned with societal values. 
The need for trust in AI has gathered the attention of different stakeholders,
including academic institutions, private corporations, and regulatory bodies~\cite{EU_AI_Act_2021,US_AI_Bill_of_Rights_2022}.

The concept of trustworthy AI is broad, recent, and aspirational. Given its nascent stage, there is no consensus on what makes an AI system trustworthy or who has the authority to define it. Different stakeholders emphasize certain aspects over others, whether to serve their own interests or to build trust incrementally by addressing specific challenges. Meanwhile, many public and private institutions strive to be pioneers in deploying autonomous systems for critical decision-making, aligning their ambitions with the pursuit of more trustworthy AI.
One of the most influential attempts to shape the meaning and requirements for the broad concept of \emph{trustworthy AI} is the ``Ethics guidelines'' document~\cite{EU_HLEG} produced by the \emph{High-Level Expert Group on AI}, a diverse group of experts from both academia and industry appointed by the European Commission.

In~\cite{EU_HLEG}, trustworthy AI is generally defined to be lawful, ethical, and robust. 
The document outlines seven key requirements to achieve trustworthy AI, which can be seen as seven different fields of study that we need to collectively develop and understand. 
In a nutshell, these requirements are:
\begin{enumerate}
    \item \emph{Human agency and oversight}. AI systems should be designed ensuring oversight through human-in-the-loop, on-the-loop, and in-command mechanisms.
    \item \emph{Technical robustness and safety}. AI systems must be resilient, secure, and reliable, with fallback mechanisms to ensure safety in unexpected situations and protection against potential attacks.
    \item \emph{Privacy and data governance}. AI systems must respect privacy and data protection while ensuring data integrity, and legitimate access through robust governance mechanisms.
    \item \emph{Transparency}. AI systems must be transparent, with traceability mechanisms and clear explanations tailored to stakeholders. Users should be aware of AI interactions and understand its capabilities and limitations.
    \item \emph{Diversity, non-discrimination, and fairness}. AI systems must prevent unfair bias to avoid marginalization and discrimination while fostering diversity and accessibility.
    \item \emph{Societal and environmental well-being}.
    AI systems should benefit all, including future generations, by being sustainable and environmentally friendly.
    \item \emph{Accountability}. AI systems must have accountability mechanisms, including auditability for assessing algorithms, data, and design. Clear redress processes should be in place, especially for critical applications.
\end{enumerate}

While all requirements are important, in this thesis, we present advances in the directions of safety, fairness, transparency, and accountability, so we will only focus on these fields.
In the following Sections~\ref{sec:intro-safety},~\ref{sec:intro-fairness}, and~\ref{sec:intro-account}, we introduce each field of study, presenting first a broad approximation to the main problems and debates, followed by a concrete problem inside of each field that motivates the work presented in this thesis. 
We start with safety, follow with fairness and end with transparency and accountability. 
We bundle transparency and accountability together because our contribution, while mostly motivated by the accountability requirement, is essentially a method to better understand the behaviour of an AI system, and thus fits as well in the category of transparency.
A floating contribution of this thesis is a novel formalization that unifies previously existing concepts. Just as the list of requirements in~\cite{EU_HLEG} can be seen as \emph{what} an AI system needs to be trustworthy, formal methods are a popular answer to \emph{how} to implement these requirements.
We dedicate Section~\ref{sec:intro-formal} in this chapter to present a broad motivation for the use of formal methods and summarize our novel formalization.

\section{Safety}
\label{sec:intro-safety}

\subsection{Background}
Safety in AI broadly refers to ensuring that a system does not produce harmful or unintended consequences. It encompasses a range of issues, from preventing system failures to aligning AI decisions with ethical and legal standards. 
A key concern is technical robustness, which ensures that AI functions correctly under both normal and unexpected conditions. 
Examples of unexpected conditions that have been particularly studied are 
adversarial examples~\cite{GoodfellowSS14} and distributional shifts in the input data~\cite{wiles2022a}.

A fundamental concept in AI safety is verification and validation, where formal methods are used to mathematically prove the correctness of an AI model’s behaviour.
In safety-critical applications like aviation and medical diagnostics, regulatory frameworks often require rigorous validation before deployment. Additionally, fail-safe mechanisms must be in place to handle unexpected situations gracefully, allowing the system to revert to a safe state when necessary~\cite{CrenshawGRSK07,SetoKSC1998}.


Ensuring AI robustness requires designing models that generalize well beyond their training phase. 
Techniques such as adversarial training improve resilience by exposing AI models to perturbed or adversarial examples during training, making them less susceptible to manipulation~\cite{ijcai2021p591}. 
Formal verification methods, such as model checking~\cite{baier2008principles} and theorem proving~\cite{harrison2009handbook}, provide mathematical guarantees on system behaviour, ensuring that certain safety properties always hold.

Another crucial approach is certifiable AI, where models are designed to provide provable guarantees about their predictions~\cite{fisher2021towards}.
This approach has gathered particular attention for neural networks, where verification techniques can analyze how slight variations in input data affect model outputs, helping establish bounds on safe behaviour~\cite{albarghouthi2021introductionneuralnetworkverification}.

AI-driven control systems, particularly those used in robotics, autonomous vehicles, and industrial automation, require additional safety considerations~\cite{Pereira2020}.



\subsection{Safe Reinforcement Learning}
Reinforcement learning (RL)~\cite{sutton2018reinforcement} is one of the most successful approaches to several types of problems where an agent interacts with a probabilistic environment, modelled as a Markov decision process (MDP).
Notable examples beating human performance at complex games~\cite{mnih2015human,silver2016mastering} and discovering higher order structures of proteins~\cite{jumper2021highly}.
RL poses unique safety challenges because it learns optimal behaviour through trial and error, often by exploring unknown states. 
This exploration can lead to catastrophic failures if the system takes unsafe actions while learning. 
Safe RL methods aim to mitigate such risks by integrating safety constraints into the learning process.

One common approach in safe RL is reward shaping~\cite{ng1999policy}, where the reward function is designed to penalize unsafe actions, guiding the agent away from hazardous behaviours. 
Another method is constrained RL, where policies are optimized under predefined safety constraints~\cite{SimaoJS21,wen2018constrained}. 
These methods can be used not only to produce safe results but also to ensure safe exploration~\cite{WienhoftSSDB023,YangSJTS23}.

Constraints can be implemented directly to the MDP~\cite{altman2021constrained,gattami2021reinforcement,wachi2020safe}, or as regularizers to the corresponding loss functions in learning schemes like constrained policy optimization~\cite{achiam2017constrained} and trust region-based approaches~\cite{schulman15trust}, which steer the policy updates away from unsafe behaviours.

Another popular approach to safe RL is the use of restraining bolts~\cite{de2019foundations,de2020restraining}, which steer the learning process towards safe policies by restricting unsafe behaviour.
Another recent approach is to learn a controller together with a certificate that proves the controller to be safe~\cite{chatterjee2023learner}.

Shielded RL~\cite{AlshiekhBEKNT18,GiacobbeHKW21,Carr2022,ijcai2023p637} is another promising approach, where a safety layer acts as a filter that prevents the agent from taking dangerous actions. This can be achieved using formal verification techniques to ensure that the learned policy remains within safe bounds. 
Shields can be placed before the agent, serving as a mask of allowed actions (pre-shields), or after the agent, overwriting unsafe actions by safe ones (post-shields).
We illustrate these two settings in Figure~\ref{fig:intro-shielding-setting}.
While shielding methodologies are popular and collision avoidance in autonomous driving is among the common motivations, the work presented in this thesis is the first implementation of shielding techniques for collision avoidance in realistic driving simulators.


In this thesis, we present two contributions to shielding techniques for ensuring safety, 
that constitute Chapters~\ref{chap:delayed_shields}~and~\ref{chap:foceta}.
In Chapter~\ref{chap:delayed_shields}, we develop the theory of deterministic shields resilient to delayed observations, and present experiments in a gridworld and in the \textsc{Carla} driving simulator.
In Chapter~\ref{chap:foceta}, we report on our experience in using probabilistic shielding for autonomous valet parking. Since probabilistic shielding requires a more complex model of the agent and the environment, Chapter~\ref{chap:foceta} strongly focuses on how to build a realistic model.

\subsection[Deterministic Shielding Resilient to Delayed Observations]{Contribution: Deterministic Shielding Resilient to Delayed Observations (Chapter~\ref{chap:delayed_shields})}

Deterministic shields ensure system safety by constructing a safety game from an environmental model and a formal safety specification~\cite{koenighofer2019}. The maximally-permissive winning strategy allows all actions that won’t cause safety violations over an infinite horizon. Shields allow any action allowed by the maximally-permissive winning strategy, and overwrite potentially unsafe actions by safe ones.

Real-world control systems face delays due to data collection, processing, or transmission. Ignoring delays can cause safety-critical failures. 
We propose delay-resilient pre- and post-shields to guarantee safety under delays in the observations.
To synthesize shields, we extend the safety game to include worst-case delays, introducing imperfect state information.
Both pre- and post- shields require the maximally-permissive winning strategy for the corresponding safety game under delayed observation. 
We compute it following the algorithm proposed in~\cite{chen2018s}.

For post-shields, we need to define which of the available correct actions the shield will use to overwrite each potentially unsafe action.
To do so, we compute shields that maximize a given fitness function of states, and propose two fitness criteria: robustness and controllability. 
The robustness of a state measures how close it is to an unsafe state in the safety game graph. 
The controllability of a state is the maximal amount of delay under which that state can still be considered safe.

We tested deterministic shields resilient to delayed observations in the open source simulator \textsc{Carla}~\cite{dosovitskiy_carla_2017}, for avoiding collisions with pedestrians as well as car-on-car collisions in intersections.
We used the default driver available in \textsc{Carla} as our shielded agent.

\begin{figure}
\centering
\begin{subfigure}[b]{0.48\textwidth}
         \centering
         \includegraphics[width=\textwidth]{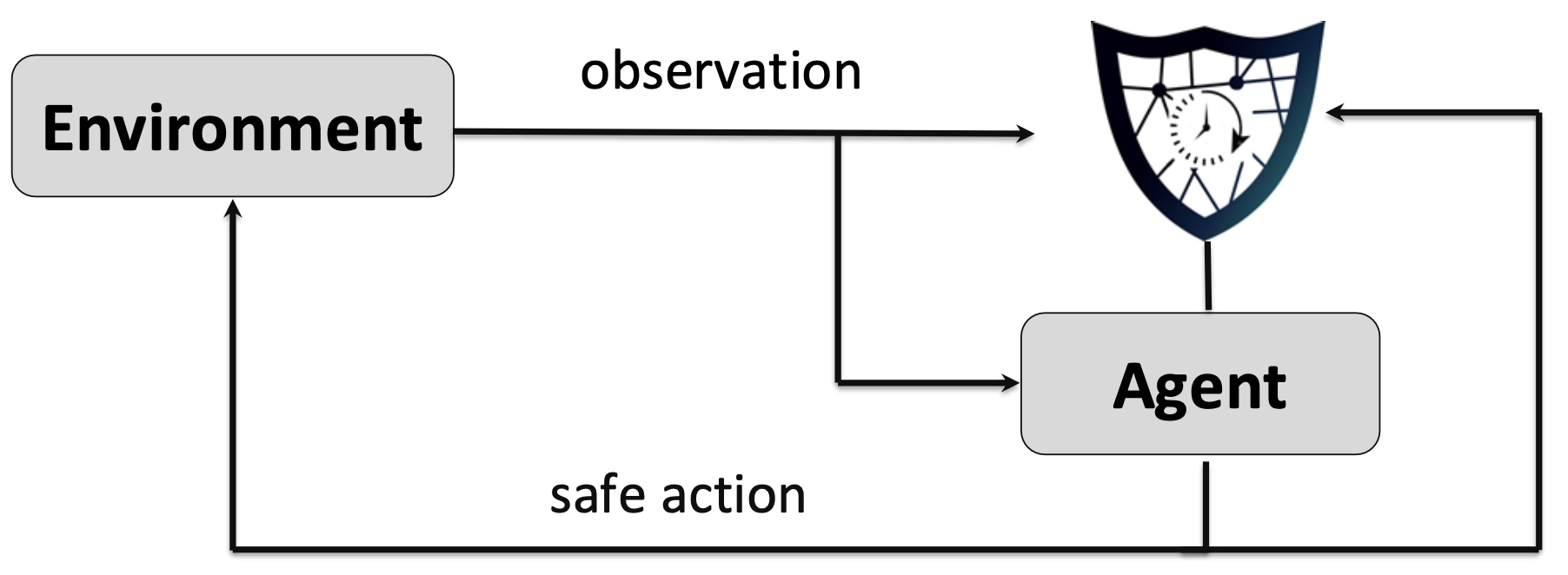}
         \caption{Pre-shield.}
     \end{subfigure}
    \hfill
     \begin{subfigure}[b]{0.48\textwidth}
         \centering
         \includegraphics[width=\textwidth]{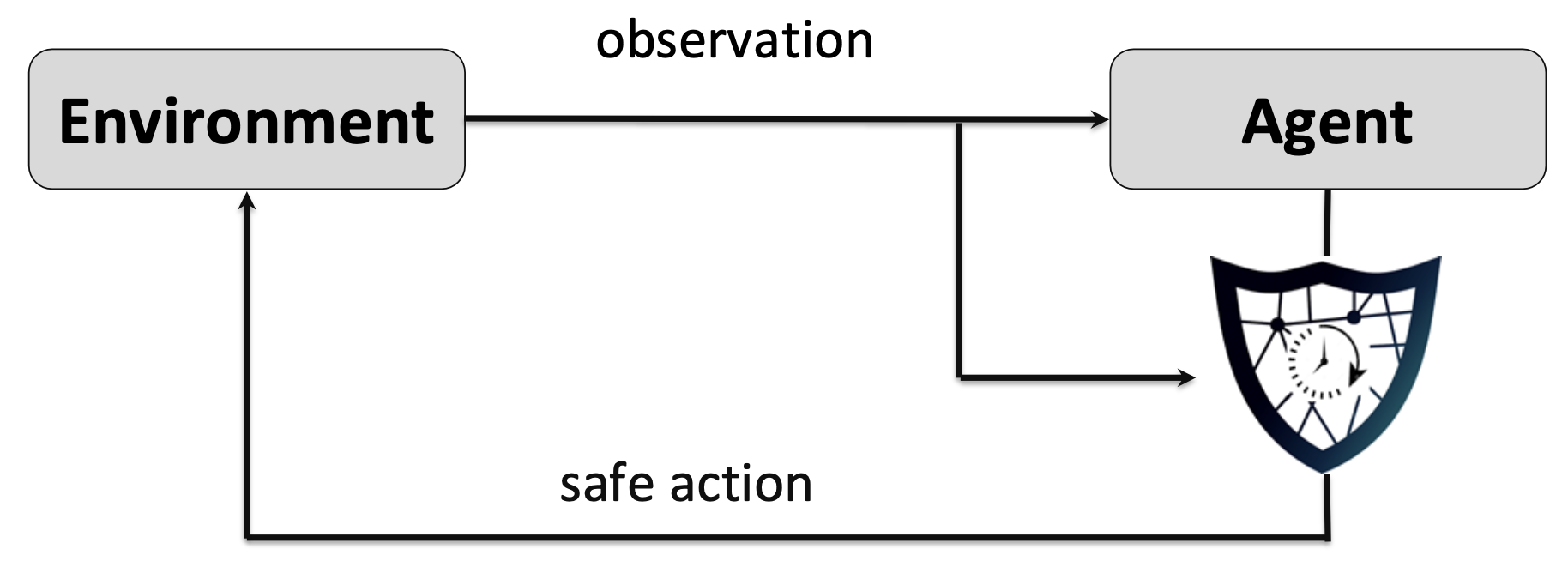}
         \caption{Post-shield.}
     \end{subfigure}
        \caption{Shielding scheme.}
        \label{fig:intro-shielding-setting}
\end{figure}

\subsection[Probabilistic Shielding for Autonomous Valet Parking]{Contribution: Probabilistic Shielding for Autonomous Valet Parking (Chapter~\ref{chap:foceta})}

Unlike deterministic shields, which enforce safety strictly, probabilistic shielding accounts for low-probability events but only intervenes when the risk of a collision exceeds a predefined threshold. This approach reduces unnecessary interventions, and uses the Markov decision process (MDP) as its underlying model, instead of the safety game.

The shield evaluates control commands by mapping sensor data and prior actions to an MDP state, then estimating the probability of avoiding collisions if the command is executed. If this probability falls below the threshold, the shield overrides the command. These probability computations rely on probabilistic model checking, requiring a well-structured MDP representation of the vehicle and its environment.

Building an appropriate MDP model is challenging—it must balance accuracy with computational feasibility. The model consists of the ego car and the pedestrians.
The ego car is represented via an abstraction of the Simrod digital twin~\cite{simrod}, with discretized actions and states to handle uncertainty.
The pedestrians are modeled with movement speeds following a normal distribution, varying across adults, elders, and children.

we tested probabilistic shields as part of a more complex agent developed as a shared effort in the \textsc{Foceta} project~\cite{aisola23foceta}.
In this case, the simulator used was Prescan, a proprietary tool partially developed within the project, and the goal of the shield is to act together with an emergency brake system to avoid collisions with pedestrians.

\section{Fairness}
\label{sec:intro-fairness}

\subsection{Background}

Fairness in AI is essential to prevent discriminatory outcomes, ensuring that automated decisions do not reinforce or exacerbate societal biases. 
AI models, trained on historical data, often inherit biases present in society, leading to discriminatory outcomes that disproportionately affect marginalized groups~\cite{barocas2023fairness}. 
Since AI systems play an increasingly significant role in decision-making processes across domains such as hiring, lending, healthcare, and law enforcement, ensuring fairness and preventing discrimination have become critical concerns and the focus of a burgeoning field of research~\cite{zhou2023fairness,blum2018preserving,chen2020fair,corbett2017algorithmic,dwork2019fairness,elzayn2019fair,ge2021towards,grazzi2022group,segal2023policy,wang2023survey}.
Addressing these biases is essential to developing ethical AI systems that align with societal values of justice and equality.

\paragraph*{Group fairness vs. individual fairness.}
Fairness in AI is typically framed in terms of two broad categories: group fairness and individual fairness.
Group fairness ensures that different demographic groups (e.g., based on race, gender, or age) receive similar outcomes from an AI system. This can be formalized using constraints such as demographic parity (equal selection rates across groups) or equalized odds (equal error rates across groups).
Individual fairness~\cite{gupta2021individual}, on the other hand, requires that similar individuals receive similar treatment, independent of their group membership. This is typically formulated using similarity metrics that measure how closely two individuals resemble each other in relevant attributes.

Balancing these two notions is challenging, as enforcing strict group fairness constraints may sometimes lead to violations of individual fairness and vice versa. Different fairness interventions prioritize one over the other, depending on the context and ethical considerations. In this thesis, we focus on group fairness properties.

\paragraph*{Sources of bias.}
AI systems can exhibit bias due to different models of the world that induce disparities. These biases can be broadly classified into two categories: intrinsic and extrinsic.

Intrinsic bias stems from biased training data that reflects historical inequalities or prejudices. For example, a hiring algorithm trained on past hiring decisions may reinforce gender disparities in hiring practices.
Extrinsic bias arises from the way AI models process and generalize information. Even if the data itself is unbiased, the learning algorithms may still introduce disparities due to optimization choices, feature selection, or model architecture.

Understanding the origin of bias is crucial in determining appropriate mitigation strategies. If the bias is intrinsic, interventions may involve adjusting the data representation, while extrinsic bias may require changes to the model’s learning process.

A key discussion in fairness research is whether lower accuracy on paper equates to a more just and correct model. 
If the training data shows an intrinsic bias against a certain group, 
it stands to reason that maximizing accuracy with respect to the biased dataset does not induce the most accurate model with respect to the real underlying unbiased data.
Therefore, an unbiased model, that will achieve lower accuracy with respect to the training data, is not only more fair, but arguably more accurate. 
However, it is not possible in many cases to determine what is the best-performing compromise.
This compromise and the impossibility to find a solution that satisfies all constraints has been studied in~\cite{suresh2021framework}.

\paragraph*{Types of fairness-inducing methods.}
Fairness interventions can be broadly categorized into three main approaches:

\begin{itemize}
    \item \emph{Pre-processing methods:} These focus on modifying the training data to remove bias before model training. Techniques include re-weighting samples, adjusting labels, and generating fair representations that obfuscate sensitive attributes~\cite{kamiran2012data,zemel2013learning}.
    \item \emph{In-processing methods:} These modify the training procedure to incorporate fairness constraints directly into the learning process. Regularization techniques and adversarial training are commonly used to ensure that the model does not learn biased patterns~\cite{zhang2018mitigating,kamishima2012fairness}.
    \item \emph{Post-processing methods:} These adjust the model’s predictions after training to 
    equalize outcomes across demographic groups without altering the underlying model~\cite{hardt2016equality}. 
\end{itemize}

Each approach has advantages and trade-offs. 
In summary, preprocessing ensures fairness at the data level but may not generalize well, while in-processing methods provide direct fairness guarantees but can be computationally expensive. 
Post-processing methods are easy to implement but tend to compromise individual fairness.

\subsection{Fairness in Sequential Decision-Making Problems.}
In sequential decision-making settings, such as loan approvals or criminal risk assessments, fairness concerns are magnified due to the compounding effects of biased decisions. Biased initial decisions can lead to feedback loops, where disadvantaged groups receive consistently lower opportunities over time, exacerbating inequalities.

Fairness interventions in sequential decision-making often involve tracking disparities over multiple time steps and designing policies that compensate for historical disadvantages.


Group fairness properties are described in terms of the joint probability distribution of the population and the outcomes. 
For example, demographic parity states that the probability of a favourable outcome must be independent of group membership.
In a sequential setting, these probabilities can be estimated using the relative frequencies of each outcome for each group. 

For example, consider a company building a large team, with a population that we can divide into two groups, $A$ and $B$, with respect to which the decisions must be fair.
After $T=1000$ candidates, $n_A$ candidates were from group $A$, out of which $n_A^1$ were offered a job. The rest $n_B$ candidates were from group $B$, and of them $n_B^1$ were offered a job. We can estimate the probability of a candidate from groups $A$ and $B$ of getting an offer as:
\[
\PP(\textit{offer} \mid A) \approx \frac{n_A^1}{n_A},\quad 
\mbox{and} \quad
\PP(\textit{offer} \mid B) \approx \frac{n_B^1}{n_B} = \frac{n_B^1}{T-n_A}.
\]
Demographic parity is formally expressed as $\PP(\textit{offer} \mid A) = \PP(\textit{offer} \mid B)$, and in terms of relative frequencies it would mean that 
\begin{equation}
\label{eq:limit-fairness}
    \lim_{T\to\infty}\left( \frac{n_A^1}{n_A} -  \frac{n_B^1}{n_B} \right) = 0.    
\end{equation}
However, in many cases, if the convergence is too slow, it is not enough to guarantee fairness in the long run. 
Group fairness metrics are emerging properties, which by their own nature cannot be expected when only looking at a few decisions. After seeing $T=6$ candidates, three from each group, and hiring one from group $A$ and two from group $B$, the difference between relative frequencies is $1/3$, far from 0, but the process has just started, so it is not reasonable for demographic parity to emerge yet.
If the company continues the interview process and after $T=1000$ the acceptance ratio of group $A$ is still $2/3$ and the acceptance ratio of group $B$ is only $1/3$, we can argue for an underlying bias.
This concept can be formalized by stating that after a certain finite horizon of $T$ decisions, the relative frequencies may differ by no more than a certain threshold $\kappa\in [0,1]$.

One way of looking at the study of fairness in a bounded horizon is with a monitoring perspective: if a process is biased after $T=1000$ decisions, it is likely to be fundamentally biased, in the sense that the limit in Equation~\ref{eq:limit-fairness} does not converge to 0, so we raise the alarm.
However, even if the relative frequencies would converge in the limit to the same value, it can be the case that the convergence is too slow. In such cases, finding a considerable disparity between relative frequencies should not be interpreted as a proxy for a bias in the limit, but as a tangible bias that is a problem in itself.

On the other hand, once a finite horizon $T$ has been predefined, 
an algorithm may act with fairness until the $T$-th decision, and then act with bias after that, satisfying fairness in the bounded horizon, but failing in the unbounded horizon. To this end, we also study the concept of $T$-periodic fairness, 
where we require the relative frequencies to differ by no more than a certain threshold $\kappa$ after $k\cdot N$ decisions, for all $k\in\{1,2,3,\dots\}$~\cite{pmlr-v235-alamdari24a}.

\subsection[Fairness Shielding]{Contribution: Fairness Shielding (Chapter~\ref{chap:fairness})}
In this thesis, we present \emph{fairness shields} as a post-processing group fairness enforcement solution for bounded and periodic horizons in sequential classification problems.

\begin{figure}[t]
    \centering
    \begin{tikzpicture}
        \draw[fill=black!10!white]   (-1,0)   rectangle   (1.2,1)   node[pos=0.5, align=center]    {Biased\\classifier};
        \draw[fill=yellow!35!white]   (5,0)   rectangle   (6.8,1) node[pos=0.5,align=center]  {Fairness\\ shield};
        \draw[->]   (1.2,0.6)   --  node[align=center,above]    {prelim. decision\\[-0.15cm] {\scriptsize \texttt{accept?}/\texttt{reject?}}}    (5,0.6);

        \node   (a)  at   (2.5,-2)   {\fbox{\Huge\faMale}};
        \node   (b) [left=0.01cm of a]    {\Huge\faFemale};
        \node  (c) [left=0.01cm of b]   {\Huge\faFemale};
        \node[black!75!white] (d) [left=0.01cm of c]    {\Huge\faMale};
        \node[black!50!white] (e) [left=0.01cm of d]    {\Huge\faMale};
        \node[black!25!white] (f) [left=0.01cm of e]    {\Huge\faMale};

        \draw[->]   (a)  -- (2.5,-0.6)   --  (0.1,-0.6) --  (0.1,0) node[align=center]    at    (-0.4,-0.6)  {input\\ features\\[-0.15cm] {\scriptsize \texttt{gender},  \texttt{age}, ...}};
        \draw[->]   (a) --  (2.5,-0.6)  -- node[below,align=center]  {protected feature\\[-0.15cm] {\scriptsize \texttt{gender}}} (5.2,-0.6)  --  (5.2,0);

        \node[align=center] (x)   at  (4.2,-2)  {final decision\\[-0.15cm] {\scriptsize \texttt{accept/reject}}};
        \draw[->]   (5.9,0)   --  (5.9,-2)    --  (x);
    \end{tikzpicture}
    \caption{The operational diagram of fairness shields.}
    \label{fig:fair-shield-schematic-intro}
\end{figure}
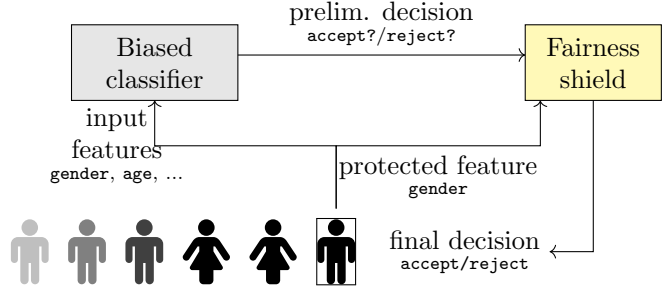

As we illustrate in Figure~\ref{fig:fair-shield-schematic-intro}, the shield monitors the decisions of a potentially biased classifier and has the power to override them. 
Given a predefined fairness criterion and a time horizon or period, the shield observes individuals’ protected attributes, the classifier’s recommendations, and the cost of modifying decisions. It then ensures fairness while minimizing intervention costs.

To guarantee fairness in finite horizons, fairness shields are computed as bounded-horizon optimal control problems with a hard fairness constraint and a soft cost constraint. The fairness constraint ensures that empirical bias remains below a threshold, measured either at the end of the horizon or periodically. The soft cost constraint discourages excessive interventions by minimizing total expected costs.

The problem becomes harder for periodic horizons, as there are infinitely many input sequences that the shield has to potentially deal with. We conjecture that optimal shields for periodic horizons cannot be described with finite resources, and propose three ``best effort'' solutions that modify the computation of bounded horizon shields to obtain periodic shields.
With these solutions, we lose the hard fairness guarantee for all traces. As a remedy, we study conditions on the incoming traces that ensure the shields achieve fair outputs.
These solutions can be classified into the \emph{static} approach, and the \emph{dynamic approach}.

The static approach consists of resetting and reusing the same shield after each time period. If the shield has been computed for a finite horizon $T$, at step $T+1$, the internal counters are reset to zero and the shield enforces fairness in the segment from $T+1$ to $2T$ in the same way as it did for the segment from $1$ to $T$.
A static shield applies the same fairness criterion for each segment of decisions of length $T$, with the hope that the same fairness criterion applies when concatenating all segments of length $T$.
The advantage of this approach is simplicity, both in design and computational complexity. 
The main drawback is that the hard guarantees hold for small subsets of traces.

The dynamic approach consists of recomputing the shield after each period, modifying the fairness condition to account for the accumulated decisions of the trace so far. The advantage of this approach is that the fairness criterion is tracked more accurately, so these shields tend to interfere less often with the classifier while ensuring fairness in a large subset of traces. 
The main drawback is that the synthesis algorithm has to be executed at the end of each period. 

To understand the difference between the static and dynamic approach, recall the example of hiring applicants from groups $A$ and $B$, trying to enforce a threshold on demographic parity no larger than $\kappa=0.2$. 
After the first $T=1000$ decisions, the acceptance rate for group $A$ is 0.5, and the acceptance rate for group $B$ is 0.35. Thus, this segment is biased towards group $A$, but not more than the threshold. 
A dynamic shield would allow the next segment of $T$ decisions to have an acceptance rate for group $A$ of 0.5 and $0.72$ for group $B$. Even if looking at the segment from $T+1$ to $2T$, the difference in acceptance rates is larger than the threshold, the dynamic shields knows that group $B$ can overcompensate for the low rate in the first segment, as long as demographic parity is kept in the threshold for the longer segment of the first $2T$ decisions.
On the other hand, a static shield would not allow $B$ to overcompensate, being more restrictive than the dynamic shield.

Shields rely on a known or learned distribution of future decisions and costs. Even if the distribution is imprecise, fairness guarantees remain intact—only cost-optimality may be affected. Shields are computed via dynamic programming, optimized to run in polynomial time for a wide variety of group fairness metrics by abstracting traces to a relevant set of counters.

\section{Transparency and Accountability}
\label{sec:intro-account}
\subsection{Background}

Beyond ensuring that AI systems behave safely and fairly, they must also be explainable. Trust in AI depends not only on its performance but also on its transparency --- users and stakeholders must understand why a system made a particular decision. 
Moreover, when AI systems cause harm, it is essential to have robust accountability mechanisms in place to determine responsibility and take corrective action. 
Explainability is also key for accountability: if an AI system causes harm or fails in an unexpected way, it must be possible to trace its reasoning to diagnose the issue and assign responsibility. Without explainability, AI remains a ``black box'', making it difficult to audit, improve, or justify its decisions in legal and ethical contexts. By integrating explainability into AI design, we can build systems that foster trust, enable human oversight, and ensure accountability in decision-making.


Explainability and accountability are closely intertwined. 
In human accountability processes, understanding why a person acted in a certain way is essential for assigning responsibility and determining degrees of culpability. Courts, for example, consider intent, circumstances, and explanations when assessing guilt. Similarly, for AI systems, understanding why a particular decision or action was taken is crucial in determining liability when harm occurs.


Accountability in software systems has long been a topic of interest in fields such as cybersecurity, safety-critical systems, and software engineering~\cite{feigenbaum2011towards,jagadeesan2009towards,kusters2010accountability}. 
Traditional software accountability often relies on clear specifications, audit logs, and formal verification techniques to determine responsibility when a system fails or produces an unexpected outcome~\cite{kroll2017accountable,feigenbaum2020accountability}.

However, AI-based systems, particularly those leveraging machine learning, present unique challenges. Unlike traditional rule-based software, many AI models operate effectively as black boxes, making it difficult to trace the logic behind their decisions. This lack of transparency complicates accountability, as it becomes unclear whether failures arise from design flaws, biased training data, unforeseen interactions, or user misuse.


\subsection{The Role of Intention in Accountability}

A crucial aspect of human accountability is the notion of \emph{intention}. 
Understanding whether an action was intentional, accidental, or due to negligence is key in determining degrees of responsibility. Courts, for instance, distinguish between premeditated actions and unintended mistakes, applying different legal consequences accordingly~\cite{knoops2016mens}.

For AI, the concept of intention requires further study. 
As a caveat,  we explicitly avoid the debate on whether AI systems may have consciousness or free will~\cite{buttazzo2001artificial,chella2013artificial}.
In any case, a large part of the theoretical development on the concept of intention can be applied to any rational planning agent that acts with goals and constrained resources~\cite{BratmanIP88,bratman1987intention}.
This general definition also applies to many AI agents, regardless of the working notion of consciousness and free will.
AI systems exhibit functionally intentional behaviour, such as pursuing a specified objective or optimizing for a particular reward. 
Understanding intentional behaviour in AI can help refine accountability frameworks by distinguishing between different sources of harm.

\subsection[Intentional Behaviour in Agents operating on MDPs]{Contribution: Intentional Behaviour in Agents operating on MDPs (Chapter~\ref{chap:intention})}
Interpreting the decision-making processes of modern machine-learning-based agents in probabilistic settings presents significant challenges due to the absence of explicit goals or intentions in their models. Traditionally, intention is connected to planning in both cognitive and computational reasoning~\cite{BratmanIP88}. 
In this thesis, we consider intention as the ``state of the world'' an agent plans toward, serving as a proxy for its internal reasoning. 
Since modern agents, particularly those trained using reinforcement learning, do not have explicitly modeled beliefs or reasoning processes, their intentions can only be inferred probabilistically.

Our proposed framework quantitatively assesses whether an agent's behaviour exhibits evidence of intentionality. 
Instead of making binary assertions about intention, it provides confidence levels and quantified evidence. 

We model autonomous agents as policies within a probabilistic environment as MDPs. 
Key to our framework are the notions of agency and intention quotient. 
The agency measures the agent’s ability to influence outcomes, defined as the probability difference between optimally achieving or avoiding a goal.
The intention quotient is a normalized value between 0 and 1 that measures how close the policy of the agent is to achieving or avoiding a goal.
Intention quotient quantifies the degree of apparent intentional behaviour, with values close to 1 indicating high evidence of intentionality.

These notions can be used to study an agent preemptively by calculating agency and intention quotients towards a particular goal around the states of interest.
With the mindset of serving a potential accountability process, we also propose a retrospective methodology to study concrete traces that end up in a harmful state.
When the trace under study does not offer enough evidence for a confident assessment, we produce counterfactual traces and use them to increase confidence.

Our method can help understand whether an agent shows evidence of acting intentionally towards an end. While this is a necessary step for accountability processes, it is not the only one, and questions such as who is responsible or who has to pay for the harm intentionally produced by the agent are out of the scope of this work.

\section{Formal Methods}
\label{sec:intro-formal}

In the field of computer science, formal methods can be broadly described as rigorous mathematical techniques used to specify, verify, and prove system properties. Rigorous formalizations not only have the advantage of providing provable guarantees, but also a deep understanding of those guarantees.

Formal methods have long been a cornerstone of trusted computing, being used in safety-critical domains like controllers in avionics and medical devices, as well as performance-critical system code like arbiters and process managers~\cite{woodcock2009formal,ter2024formal}.
Because of this history, many researchers and practitioners think that formal methods for AI are destined to play a central role in the future of trustworthy AI~\cite{li2023trustworthy}.
By applying logical reasoning, theorem proving, and model checking, formal methods provide strong guarantees about software and hardware systems. Unlike testing, which only checks for correctness in specific scenarios, formal verification provides mathematical certainty for all possible cases within a given model.
However, the guarantees are only as good as the model, and, as the popular saying goes, all models are wrong --- albeit some of them are useful.
This should serve as a constant reminder throughout the thesis that all results and guarantees hold in the ideal model, and any consequences of those results on reality are mediated by how good the model is as a description of the real world.

\subsection[The Reactive Decision Making Framework]{Contribution: The Reactive Decision Making Framework (Chapter~\ref{chap:reactive_decision_making})}

The frameworks used to formalize the different contributions in this thesis use different models: safety games are used for deterministic safety shielding, MDPs are used for probabilistic safety shielding and intention analysis, and
sequential classification problems are used for fairness shielding.

However, these frameworks always have in common an agent interacting with an environment. 
In this thesis, we introduce the \emph{reactive decision making} framework as a formal generalization of the aforementioned formal models. 
We also formalize the concept of shielding in the reactive decision making framework, generalizing previously differentiated notions~\cite{koenighofer2019} and refining the definitions to account for edge cases that had been previously mistreated.

\section{Outline of the Thesis}

Chapter~\ref{chap:preliminaries} introduces the notation and previously established concepts that are required throughout the thesis. While this chapter, and the thesis as a whole, is self-contained, the exposition may be too succinct for an unfamiliar reader.
For a deeper understanding of the background material, we give pointers to adequate reference materials.
Chapter~\ref{chap:reactive_decision_making} introduces the reactive decision making framework and the notion of shielding. 
We show how the formal frameworks used in this thesis are particular cases of this general formalization.
We also show how previous notions of shielding for safety properties correspond to shielding as described in the reactive decision making framework.

Chapters~\ref{chap:delayed_shields} and~\ref{chap:foceta} explore shielding for safety properties with some source of uncertainty in the context of autonomous driving.
Chapter~\ref{chap:delayed_shields} explores deterministic shielding for safety properties resilient to delayed observations.
We show how shields can be extended to guarantee a safety specification even with imperfect information, and study different methods to choose a corrective action.
In Chapter~\ref{chap:foceta}, we report our experience on using probabilistic shielding 
on an autonomous car operating in a parking lot with the objective of avoiding collisions with pedestrians.
We describe how to build realistic models of a car and pedestrians in its vicinity, and report the results of a comparative test between our shields and an automatic emergency brake system.

We move our focus from safety to fairness in Chapter~\ref{chap:fairness}, where we introduce the notion of fairness shields and describe how to compute different types of fairness shields for finite and periodic time horizons.
We validate the usefulness of fairness shields on an extensive evaluation against standard benchmarks from the literature on algorithmic fairness.

Chapter~\ref{chap:intention} moves away from runtime enforcement towards explainability and accountability. 
In this chapter, we present our framework for studying intentional behaviour on agents operating in MDPs using the notions of agency and intention-quotient. 
We present our retrospective methodology, intended for accountability processes after harm has occurred, and showcase how it would work in a toy example.

Chapter~\ref{chap:conclusion} rounds up the thesis with future work and concluding remarks, followed by an appendix detailing the publications associated with the completion of the PhD program.

\ifthenelse{\boolean{includequotes}}{
While the landscape of AI research is currently very English-centric, its fruits shouldn't be. As a nod to diversity, each chapter is preceded by a famous quote, mostly in languages other than English. 
Do not take them too seriously.
The last pages of this document contain a cheat sheet meant to help the reader go through the notation.
Let the power of well-structured indices prevail when the power of well-written text may fail.
}
{}

%% file: 20_preliminaries.tex
\ifthenelse{\boolean{includequotes}}{
\begin{quotation}
    \foreignlanguage{greek}{
    \textit{ΑΓΕΩΜΕΤΡΗΤΟΣ ΜΗΔΕΙΣ ΕΙΣΙΤΩ.}}
    \footnote{Let no one ignorant of geometry enter here.}
    \\
    \textcolor{white}{a} \hfill
    --- Inscription above the entrance of Plato's Academy.
    \footnote{The existence of this inscription is a disputed fact, since the earliest known documents that mention it date about 700 years after Plato's death. It has become, however, a powerful meme.}
\end{quotation}
}{}

In this chapter, we will briefly cover the basic concepts that will be used throughout the thesis. 
This serves the double purpose of being a lightweight introduction to the topics and fixing the notation used throughout the thesis. 

\section{Basic Notation}
\label{sec:prelim-basic-notation}

\paragraph*{Sets, numbers, and functions.} 
We use $\BB = \{\bot, \top\}$ to denote the Boolean domain, $\NN =\{0,1,\dots\}$ to denote the set of natural numbers, 
$\ZZ$ to denote the set of integer numbers, and
$\RR$ to denote the set of real numbers.
Given $a<b\in\RR$, we use $(a,b)$ to denote the open interval between $a$ and $b$, $[a,b]$ to denote the closed interval, and $(a,b]$ to denote the interval open at one end and closed at the other.
Given a real number $a\in\RR$, we use $\lfloor a\rfloor$ to denote the largest integer that is less or equal than $a$ (i.e., the \emph{floor} of $a$), 
$\lceil a \rceil$ to denote the smallest integer that is greater or equal than $a$ (i.e., the \emph{ceiling} of $a$) and $\round{a}$ to indicate the integer that is closest to $a$ (i.e., the result of \emph{rounding} $a$).
We use the standard convention that $\round{a} = \lceil a \rceil$ when $a$ is at the same distance of $\lfloor a \rfloor$ than $\lceil a \rceil$.
In general, it is true that $a-1 \leq  \lfloor a \rfloor \leq a \leq  \lceil a\rceil \leq a+1$.
Given $a,b\in\RR$, we use the notation $a\ll b$ to indicate that $a$ is \emph{much smaler} than $b$, and $a\gg b$ to indicate that $a$ is \emph{much greater} than $b$, where how much is \emph{much} depends on the context. 

Given a finite set $X$, we denote its cardinality by $|X|$.
Given a function $f\colon \X\to \Y$, and a subset $X\subseteq \X$, the image set is $f(X) = \{ y\in \Y\::\:  \exists x\in X,\, f(x) = y\}$. Similarly, given a subset $Y\subseteq \Y$, the antiimage set is $f^{-1}(Y) = \{ x\::\: f(x)\in Y \}$.
Given an arbitrary domain $\X$ and a function $f\colon X\to \RR^n$, 
the \emph{support} of $f$ is $\Supp(f) = \{x\in \X\::\: f(x) \neq 0\}$.
Let $X$ be a subset $X\subseteq \X$, we define the \emph{indicator function} of $X$ in $\X$ as $\1_X\colon \X\to \{0,1\}$, 
such that $\1_X(x) = 1$ if and only if $x\in X$. 
We will use the equivalent notation $\1[x\in X]$.
In the special case where the set is a single element, $X=\{y\}$, 
we will use the notation $\1[x = y]$. 
Similarly, if $X = \X\setminus y$, we will denote the indicator function as $\1[x\neq y]$.

\nomenclature{$\1_X(x)$}{Indicator function of $x\in X$}
\nomenclature{$\Supp(X)$}{The support of the function $X$}

\paragraph*{Words and languages.}
An alphabet $\Sigma$ is a finite set. A word (or trace) in an alphabet is a sequence of elements in the alphabet.
We use $\Sigma^i$ to denote the set containing words of length $i$ in the alphabet $\Sigma$ 
and $\Sigma^{\leq k}$ to denote the set of words of length at most $k$ in $\Sigma$, i.e., $\Sigma^{\leq k} = \cup_{j=0}^k \Sigma^j$.
We use $\Sigma^*$ for the set of all words of finite length, i.e., $\Sigma^* = \cup_{k=0}^\infty \Sigma^k$.
We use $\Sigma^\omega$ to denote the set of infinite words in $\Sigma$, and $\Sigma^\infty = \Sigma^*\cup \Sigma^\omega$.
We will also use $2^\Sigma$ to denote the power set of $\Sigma$.
Given a word $w\in\Sigma^\infty$, we use $|w|$ to refer to its length. 
Given a word $w = \sigma_0\sigma_1\dots\in\Sigma^\infty$, 
for $n,m$ such that $n\leq m \leq |w|$, 
we will use the notation $w_{[n:m]} = \sigma_n\sigma_{n+1}\dots \sigma_{m}$, 
and in the case $n=0$ we will use the notation $w_{:m}$ to refer to $w_{[0:m]}$.
Given two words $w_1=\sigma_0\ldots\sigma_n, w_2=\sigma_{n+1}\dots\sigma_{n+m}$ 
in $\Sigma$, the concatenation is the word $w_1\cdot w_2 = \sigma_0\dots\sigma_n\sigma_{n+1}\dots\sigma_{n+m}$.
Given a finite word $w = \sigma_1\dots\sigma_n$, 
we use $w^\omega$ to denote the resulting word of concatenating $w$ infinitely many times.
A set of words $L\subseteq \Sigma^\infty$ is called a language.
\nomenclature{$\Sigma$}{Alphabet}
\nomenclature{$\mathcal{L}$}{Language}

%
%


\section{Probability Theory}
\label{sec:prelim-probability-theory}

In this section, we define some basic notions of probability theory. We refer the reader to~\cite[Chap. 1]{durrett2019probability} for a more detailed discussion.

Let $\Omega$ be a non-empty set, a \emph{$\sigma$-algebra} is a set $\F\subseteq 2^\Omega$ that satisfies:
    (i) $\Omega \in \mathcal{F}$,
    (ii) if $A \in \mathcal{F}$, then $\Omega \setminus A \in \mathcal{F}$, and
    (iii) if $\{A_i\}_{i=1}^{\infty}$ is a countable collection of sets in $\mathcal{F}$, then $\bigcup_{i=1}^{\infty} A_i \in \mathcal{F}$.
The tuple $(\Omega,\F)$ is a \emph{measurable space}.
Given a set $\Omega$ and a subset $F\subseteq 2^\Omega$,
the $\sigma$-algebra generated by $F$ is the smallest $\sigma$-algebra $\F$ such that $F\subseteq \F$.
The most commonly used measurable space is $(\RR^n,\B^n)$, where $\B^n$ is the Borel $\sigma$-algebra on $\RR^n$, i.e., the $\sigma$-algebra generated by the open sets of $\RR^n$.
Let $(\Omega,\F)$ and $(\Omega', \F')$ be two measurable spaces, a function $f\colon \Omega\to\Omega'$ is \emph{measurable} if for every $B\in \F'$, $f^{-1}(B)\in \F$.

\nomenclature{$\Omega$}{Sample space (of a probability space)}
\nomenclature{$\F$}{$\sigma$-algebra}
\nomenclature{$\PP$}{Probability measure}
\nomenclature{$\B^n$}{The Borel $\sigma$-algebra of $\RR^n$}
\index{space!measurable space}
\index{space!probability space}
\index{Borel $\sigma$-algebra}

A \emph{probability space} is a tuple 
$(\Omega, \mathcal{F}, \PP)$, 
where $(\Omega, \F)$ is a measurable space and $\PP\colon \F\to [0,1]$ is a function satisfying:
    (i) $\mathbb{P}(\Omega) = 1$, 
    (ii) if $A\in \F$, then $\PP(\Omega\setminus A) = 1- \PP(A)$, and
    (iii)  for any countable collection $\{A_i\}_{i=1}^{\infty}$ of disjoint sets in $\mathcal{F}$, we have $\mathbb{P}\left(\bigcup_{i=1}^{\infty} A_i\right) = \sum_{i=1}^{\infty} \mathbb{P}(A_i)$.
    $\PP$ is called a \emph{probability measure} over $(\Omega,\F)$.
A \emph{random variable} on $\Omega$ is a measurable function $X\colon \Omega\to \RR$ from a probability space $(\Omega, \F, \PP)$ to $(\RR, \B)$. We typically think of a random variable as a symbol that takes a value in $B\in\B$ with probability $\PP(X^{-1}(B))$. 
We also denote $\PP(X^{-1}(B))$ as $\PP[X\in B]$. 
When $B=\{b\}$, we may denote $\PP(X^{-1}(B))$ as $\PP[X=b]$.

A random variable $X$ induces a a probability measure $\mu$ on $(\RR, \B)$, called its \emph{probability distribution} (or simply \emph{distribution}), by setting $\mu\colon\B\to\RR$, $\mu(A) = \PP(X\in A)$.
The distribution of a random variable is usually described using its \emph{distribution function} $F\colon \RR\to [0,1]$, defined as $F(x) = \PP(X\in (-\infty, x))$. 
We say that $X$ \emph{follows} the distribution $F$, and denote it by $X\sim F$.
The set of distributions over $\RR$ is the set of probability measures over $(\RR, \B)$, and denoted as $\D(\RR)$.
For a countable set $\Omega$, a \emph{distribution over $\Omega$}
is any function $d\colon \Omega\to \RR$, such that there exists a probability space $(\Omega, \F, \PP)$ satisfying for all $\omega\in\Omega$ that $\{\omega\}\in\F$ and $\PP(\{\omega\}) = d(\omega)$. 
We denote the set of distributions over $\Omega$ as $\D(\Omega)$.

\nomenclature{$\D(\Omega)$}{Set of probability distributions over $\Omega$}
\nomenclature{$X\sim F$}{The random variable $X$ follows the distribution $F\colon \RR\to [0,1]$}
\nomenclature{$\EE$}{Expected value (of a random variable)}

The \emph{expected value} of a random variable is 
$\EE[X] = \int_{\Omega} X(\omega) d\PP(\omega)$, 
where $\int$ is the standard Lebesgue integral on the measure given by $\PP$.
\index{expected value}

\section{Deterministic Two-Player Games}
\label{sec:prelim-TwoPlayerGames}

\subsection{Games with Perfect Information}

The deterministic two-player game is a widely used formalism to model 
deterministic interactions between an agent and an environment, 
where the environment is considered adversarial.

Formally, a \emph{deterministic two-player game} is a tuple 
$\G = (S, s_0, S_{env}, S_{ag}, \Act, \T, \Acc)$, 
where $S=S_{env}\cup S_{ag}$ is a finite set (the \emph{state space}),
composed of the disjoint sets $S_{env}$ (states of the environment) 
and $S_{ag}$ (states of the agent).
There is a special state,
$s_0\in S_{env}$, the initial state.
$\Act$ is the set of actions, 
$\T \subseteq (S_{env}\times S_{ag})\cup (S_{ag}\times \Act\times S_{env})$
is the transition relation 
and $\Acc \subseteq S^\omega$ is the winning condition for the agent.
For convenience, we usually write
$s \xrightarrow{\sigma} s'$
for $(s, \sigma, s')\in \mathcal T$, 
and $s' \xrightarrow{u} s$ for $(s',s)\in \T$, where $u$ stands for ``undefined''. 
The state space is required to be deadlock-free, i.e., for all $s\in S_{env}$,
there exists $s'\in S_{ag}$ such that $(s,s')\in \T$.
Similarly, for all $s'\in S_{ag}$, there exist $\sigma\in\Act$ and $s\in S_{env}$
such that $(s',\sigma,s)\in \T$.
The transition relation is deterministic for the agent, i.e.,
for any $s\in S_{ag}$ and  $\sigma\in\Act$,
if $s \xrightarrow{\sigma} s'$ and $s \xrightarrow{\sigma} s''$ then $s'=s''$. 
It is useful to think of $\G$ as a bipartite graph, 
where the set of nodes is partitioned into $S_{ag}$ and $S_{env}$,
the edges from $S_{ag}$ to $S_{env}$ are labeled with symbols in $\Act$, 
and the edges from $S_{env}$ to $S_{ag}$ are unlabeled. 

On some occasions, it will be useful to work with a concrete set of actions for the environment. 
An \emph{environment action set} is a set $\Act_{env}$, 
that uniquely labels each transition of the environment,
making the transition relation deterministic for the environment.
That is, for every pair of states $s\in S_{env}, s'\in S_{ag}$
such that $(s,s')\in \T$, 
there exists a unique label $x\in \Act_{env}$ associated with the pair.
We write this relation as $s\xrightarrow{x}s'$.
To make the environment transitions deterministic, 
we require that given $s\in S_{env}, x\in \Act_{env}$, 
there is at most one state $s'\in S_{ag}$ such that $s\xrightarrow{x}s'$.

Note that a set $\Act_{env}$ has at least as many labels as
the maximum out-degree of $\T$ for states in $\S_{env}$, that is
\begin{equation}
\label{eq:safety-games-env-out-degree}
    |\Act_{env}| \geq \max_{s\in S_{env}} \left|\{ s'\in S_{ag} \::\: (s,s')\in \T\}\right|.
\end{equation}
Furthermore, this bound is tight:
since the condition on the transition is defined independently for
every environment state, 
we can reuse labels without any issue for different environment states.

\paragraph*{Plays.} The game is played by two players: the agent and the environment.
In the safety game, at every state $s\in S_{ag}$, 
the agent chooses a transition $s\xrightarrow{\sigma}s'$, 
and at every state of the environment $s\in S_{env}$, 
the environment chooses a transition $s\xrightarrow{u}s'$.
Together, they produce a trace of states 
$\tau = [ s_0, s_1, \dots ]$, where for all $i\in \NN$, 
$s_{2i}\xrightarrow{u}s_{2i+1}$ and there exists $\sigma_i$
such that $s_{2i-1}\xrightarrow{\sigma_i}s_{2i}$.
In the game context, a trace of states is also called a \emph{play}.
Sometimes, it is useful to consider the concrete actions of the agent that give rise to a play. 
For a given play $\tau = [ s_0,s_1,s_2,\dots ]$, 
the corresponding \emph{state-action play} is $\tau=[s_0,s_1,\sigma_1,s_2, s_3,\sigma_2, s_4]$, 
where for all $i\in\NN$, $s_{2i-1}\xrightarrow{\sigma_i}s_{2i}$.
Since the actions are determined by state transitions, it is equivalent to talking about plays as state sequences or state-action sequences.
We will use state-action plays whenever concrete actions take an essential role.
The set of all plays of a game $\G$ is denoted by $\Pi(\G)\subseteq (S_{env}\times S_{ag})^\omega$.

\paragraph*{Strategies.}
A \emph{strategy} for the agent is a function
$\xi\colon (S_{env}\times S_{ag})^*\to~2^{\Act}$,
that given a trace $\tau\in S^*$, 
and produces a set of actions $\xi(s)\subseteq \Act$.
A strategy is \emph{memoryless} if it only depends on the last state of the trace.
In such cases, we will denote it as a function $\xi\colon S_{ag}\to 2^{\Act}$.
A strategy is \emph{deterministic} if $\forall \tau\in S^*$, 
$|\xi(\tau)| = 1$.
When it is clear from context that a strategy $\xi$ is deterministic, 
we will denote it as a function $\xi\colon S_{env}\times S_{ag}^*\to \Act$,
and as a function $\xi\colon S_{ag}\to\Act$ when it is also memoryless.
A play 
$\tau = [ s_0, s_1, \ldots]$  
\emph{is valid under a strategy} 
$\xi$ of the agent if for all $i\in \NN$, 
there exists $\sigma_i\in \xi(s_{2i-1})$
such that $s_{2i-1}\xrightarrow{\sigma_i}s_{2i}$.
The set of behaviours 
$\Beh(\G, \xi)$ consists of all plays following $\xi$, that is:
\begin{equation}
    \label{eq:safety-game-behaviour}
    \Beh (\G, \xi) = \{\tau = [s_0,  s_1,\dots ] \in \Pi(\G) \: : \:
    \forall i\in\NN, \, \exists \sigma_i,\, \mbox{ \textit{s.t.} }
    \sigma_{i}~\in~\xi(s_{2i-1})\}.
\end{equation}
A strategy $\xi$ is \emph{winning} for the agent if 
$\Beh (\G, \xi)\neq\emptyset$  and 
for all $\tau\in\Beh (\G, \xi)$ we have 
$\tau\in \Acc$.
A winning strategy $\xi$ is \emph{maximally permissive} 
if $\Beh(\G, \xi') \subseteq \Beh (\G, \xi)$
for every winning strategy $\xi'$.

An \emph{output-restricted} strategy is a function 
$\xi\colon S_{ag}\times 2^\Act \to 2^{\Act}$, 
satisfying that $\forall s\in S_{ag}, \forall \Sigma\in 2^\Act$,
$\xi(s,\Sigma)\subseteq\Sigma$.
Each output-restricted strategy $\xi$ has an equivalent 
unrestricted strategy $\hat{\xi}$, 
defined as $\hat{\xi}(s) = \xi(s,\Act)$.
An output-restricted strategy is 
\emph{deterministic} if $\forall s\in S_{ag}$, $\forall \Sigma\in 2^\Act$,
$|\xi(s,\Sigma)|=1$.
The previous definitions for winning and maximally permissive strategy 
apply to the unrestricted strategy.

\paragraph*{Synthesis of winning strategies of safety games.} 
A game $\G$ is a \emph{safety game} if its winning condition 
is defined by the game staying in a set of safe states $\F\subseteq S$. That is, if 
$\Acc = \F^\omega$.
It is a classical result~\cite{Thomas1995} that the maximally permissive strategy of a safety game, if it exists, 
is unique, memoryless, and can be computed in $\mathcal O(|S|\cdot|\Act|)$.
This is done by computing the so-called winning region.
The winning region is the set $W$ of states that are part of a trace valid under a winning strategy:
\begin{equation}
\label{eq:prelim-safety-games-winning-strategy}
    W = \left\{
    \begin{matrix}
         s\in S\::\: & 
          \exists \tau = s_0s_1\dots s,\,\mbox{ and }
         \exists\xi_\tau\colon S_{ag}\to 2^\Act \mbox{, such that } \\
         & 
         \xi_\tau \mbox{ is winning and $\tau$ is valid under $\xi_\tau$}
    \end{matrix}
    \right\}
\end{equation}

With the help of the winning region,
we can define the strategy $\xi_{\mathtt{max.perm.}}$ as
\begin{equation}\label{eq:max-perm-strat-safety-game}
    \xi_{\mathtt{max.perm.}}(s) = \{ \sigma\in\Act\::\:
    s\xrightarrow{\sigma}s', \, \mbox{ for } s'\in W
    \}.
\end{equation}
This strategy has the property of being winning and maximally permissive~\cite{Thomas1995}.
In safety games, we often say that a state $s$ is \emph{safe} when $s\in W$.
\nomenclature{$W$}{winning region of a safety game}
\nomenclature{$\F$}{set of safe states in a safety game}
\nomenclature{$\xi_{\mathtt{max.perm.}}$}{maximally permissive winning strategy of a safety game}

\subsection{Games Under Delay}
\label{sec:prelim-games-under-delay}
Delayed games follow the intuition that the system does not have access to the most recent inputs,
forcing it to make decisions having only partial information on the current state. 
In this section, 
we summarize the construction of delayed safety games from~\cite{Chen2020IndecisionAD},
with two differences.
The first one is that we consider delays as full-step and not half-step delays.
This is not a conceptual change but rather a change that lightens notation in some definitions.
The second is that we extend the definitions to include an arbitrary amount of memory.

\paragraph*{Game graph under delay.} 
Introducing delays does not change the game graph itself,
but we add two parameters, one for the amount of delay and one for the amount of memory available.
Formally, a \emph{deterministic two-player game with delay $\delta$ and memory $\mu$}
is a tuple
$\mathcal G_{\delta,\mu} = 
\langle S, s_0, S_{env}, S_{ag}, \Act, \mathcal T, \Acc, \delta, \mu \rangle$,
where 
$\delta\in\NN$, 
represents the delay in the input observation, 
$\mu\in \NN$ represents the length of the register (memory) of own outputs allowed to the agent, 
and $\G = \langle S, s_0, S_{env}, S_{ag}, \Act, \mathcal T, \Acc \rangle$ is a two-player game.

\paragraph*{Game play under delay.}
A delay of $\delta$ causes the agent to be unaware of the last $\delta$ transitions produced by the environment. 
The agent is aware of its own actions and can store them in a register of size $\mu$ used to limit the uncertainty about the current (partially unknown) state of the game.
We give the name of \emph{observed state} of the agent to the state
from which a strategy chooses which action to take. 
We give the name of \emph{current state} of the agent to the state in which the action chosen by the agent is applied.
In perfect information games (i.e., when $\delta = 0$), the current state and the observed state of the agent are the same.
In delayed games, the \emph{current state} of the agent is $\delta$ agent states ahead of the observed state.
Therefore, a strategy under delay $\delta$ 
is aware that an observed state $s$ and output register $\ol{\sigma}$
limits the possible current states, 
and has to choose an action to be applied,
knowing it will be applied in one of the possible current states, 
without explicit knowledge of which one.
\index{state!current}
\index{state!observed}

\paragraph*{Strategies under delay.} 
In a game with delay $\delta$, the agent can only observe the first state after $\delta$ steps in the game, 
and has to play the first moves ``blindly''.
In this initial transient period, the agent needs to ensure it can travel from
the initial state to a state with a defined winning strategy.
Therefore, we also need to define the strategy for an unobserved state. 
Let 
$S_{ag*} = S_{ag}\cup \{\eps\}$, 
where $\eps$ represents the unobserved state.

A \emph{strategy for the agent under delay $\delta$ with memory $\mu$}
is a function
$\xi_{\delta,\mu}\colon S_{ag*}\times \Act^{\leq \mu} \to 2^{\Act}$.
When $\delta$ and $\mu$ are clear from context or irrelevant,
we will denote it simply as $\xi$.
A state-action play $\tau = [s_0, s_1,\sigma_1, \dots]$ 
is valid under strategy $\xi$ if $\forall i\in \NN$, 
\begin{equation}
\label{eq:safety-game-strategy-delay}
    \sigma_{i} \in \xi\big( s_{\scriptscriptstyle{2(i-\delta)-1}}, [\sigma_{\scriptscriptstyle{ i-\mu}},\dots,\sigma_{\scriptscriptstyle{i-1}}]\big),
\end{equation}
with $s_j$, 
$\sigma_j=\eps$ for any $j<0$,
to account for the 
transient period.

Most of the time, the strategy takes as input a state $s\in S_{ag}$ and a sequence of actions $\ol{\sigma}\in\Act^\mu$ representing the last $\mu$ actions executed by the agent.
The state $s$ is the \emph{observed} state, and the strategy has to take care that any action $\sigma'\in\xi(s,\ol{\sigma})$ will be executed in the \emph{current} state, about which the agent has only partial information given by $s$ and $\ol{\sigma}$.
Note that the state-action trace used in Equation~\ref{eq:safety-game-strategy-delay} is 
$\tau = [s_0, s_1,\sigma_1, s_2, s_3, \sigma_2, s_4, \dots]$, 
where $s_i$ is a state of the agent or the environment depending on the parity of $i$.
More concretely, for all $i$, we have $s_{2i-1}\in S_{ag}$, $s_{2i}\in S_{env}$, and 
$s_{2i-1}\xrightarrow{\sigma_i}s_{2i}$.
The strategy has to define the action $\sigma_i$ to take in the (unknown) current state $s_{2i-1}$, using only information of the $\mu$ latest actions and the observed state $s_{2(i-\delta)-1}$ that is $\delta$ observations behind the current state.

At the beginning of the game, before there are $\delta$ environment transitions, 
the agent has not had time to observe any state yet, and thus the observed state is represented by the empty state $\eps$. Similarly, if there have only been $\nu < \mu$ actions played by the agent, the action register is $\ol{\sigma}\in\Act^\nu$. 
This transient phase is the reason why we need to define the strategy also for the domain $\{\eps\}\times\Act^{\leq \mu}$.

As in the undelayed setting,
we define the behaviour $\mathtt{Beh}(\mathcal G_{\delta,\mu}, \xi)$,
as the set of traces that are valid under $\xi$,
and a strategy $\xi$ is winning if all traces valid under $\xi$ are inside winning set $\Acc$.
A winning strategy $\xi$ is \emph{maximally permissive}
if $\mathtt{Beh}(\mathcal G_{\delta,\mu}, \xi') \subseteq \mathtt{Beh} (\mathcal G_{\delta,\mu}, \xi)$
for every winning strategy $\xi'$.
Chen et al. showed that 
a safety game $\G$ with delay $\delta$ and memory $\mu=\delta$
is equivalent to a safety game with no delay $\G'$ with a set of states $S' = \left(S\times \Act^\delta\right)\cup \left(\{s_0'\}\times\Act^{\leq \delta}\right)$~\cite[Lemma 2]{Chen2020IndecisionAD}.
In particular, this result, together with~\cite{Thomas1995}, proves that 
for any game that can be solved with delay $\delta$,
there exists a winning strategy with $\mu=\delta$ memory.
Some games have winning strategies with less or no memory.
We refer to strategies with less 
memory as \emph{memory-restricted} strategies,
with the special case of \emph{memoryless} strategies
when $\mu=0$.

\begin{figure}[b]
    \centering
    \includegraphics[width=0.85\linewidth,page=5]{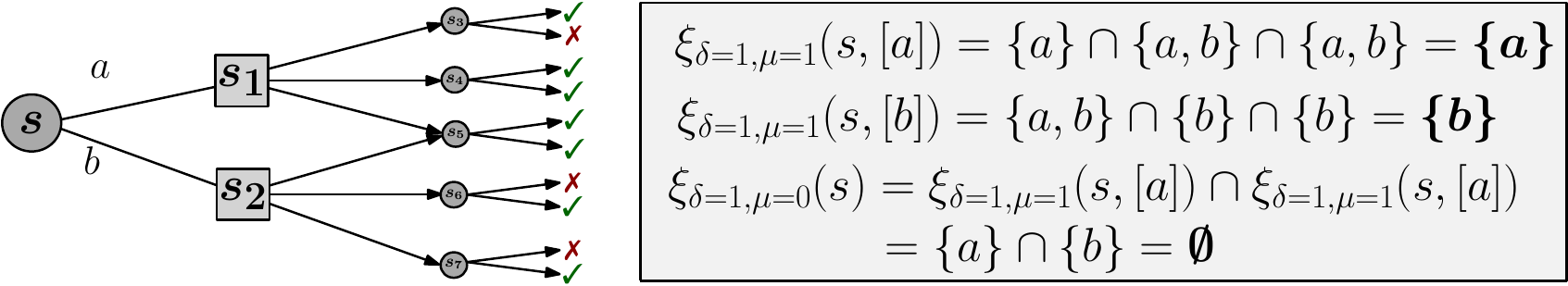}
    \caption[Example of computation of winning strategies under delay]{Computation of winning strategies under delay $\delta = 1$ 
    with memory $\mu = 0,1$. 
    In this example, the observed state is $s$, and the current state is one of $s_3, s_4, s_5, s_6,$ or $s_7$.
    From state $s\in S_{ag}$, if the last action was $a$, 
    the agent knows the environment has made a transition from state $s_1$, making $b$ potentially unsafe. 
    Similarly, if the last action by the agent was $b$,
    the current state is one of $s_5, s_6,$ or $s_7$, 
    making $a$ unsafe.
    If the agent cannot store its last action in memory, there is no winning strategy from $s$.
    }
    \vspace{-1.5em}
    \label{fig:delayedGame}
\end{figure}

\paragraph*{Synthesis of winning strategies under delay.} 
We outline the algorithm to solve delayed safety games
with delay $\delta$ and memory $\mu=\delta$,
as presented in~\cite{Chen2020IndecisionAD}.
A more detailed version of the algorithm is given is
presented as part of Chapter~\ref{chap:delayed_shields}, 
where we extend it to strategies with any amount of memory 
$\mu \leq \delta$.


The algorithm to solve a delayed safety game iteratively constructs and solves the safety game with increasing delays 
$d = 0, 1,\dots, \delta$ and memory size $\mu=d$.
At every iteration in $d$, 
the maximally permissive strategy for the agent is
computed using the strategy for the previous delay $d-1$,
followed by a reduction of the game graph aiming to mitigate the exponential blow-up in the state space
and the computation of the transient phase.
To compute the maximally permissive strategy using previous delays, 
for each state $s\in S_{ag}$ and output register $[y_1,\dots,y_d]$,
we compute the set $\mathcal I_{s,y_d}$, 
containing all states $s''\in S_{ag}$ 
that can be reached by a pair of transitions 
$s\xrightarrow{y_d}s'\xrightarrow{u}s''$.
The intersection of the actions allowed by the maximally permissive strategy for delay $d-1$ in states of $\mathcal I_{s,y_d}$
corresponds then to the actions allowed with delay $d$ in 
state $s$ with output register $[y_1,\dots,y_d]$.
Figure~\ref{fig:delayedGame} shows an example of computing 
the maximally permissive winning strategy under delay $\delta=1$
from the one under delay $\delta=0$, for memory $\mu\in\{0,1\}$.

%
%
%

%


\section{Markov Decision Process}
\label{sec:prelim-MDPs}
A Markov decision process (MDP) is a tuple
$\M = (\S,\A,\P)$,
where 
$\S$ is a countable set of states, 
$\A$ is a countable set of actions, and 
$\P\colon \S \times \A \rightarrow \D(\S)$
is the probabilistic transition function.
Given a state $s\in\S$ and an action $a\in \A$, 
$\P(s,a)$ is a distribution of states.
It is common in the literature to denote the probability of a state $s'\in \S$ under the distribution $\P(s,a)$ as 
$\P(s,a,s')$ instead of $\P(s,a)(s')$.
Sometimes, an MDP is described together with a special state $s_0\in\S$, indicating it is the \emph{initial state}. 
Sometimes, instead of a single initial state, the MDP is accompanied by a distribution of initial states $\iota\in\D(\S)$.
A \emph{policy} \( \pi\colon \S \to \D(\A) \) is a function mapping each state \( s \in \S \) to a probability distribution over the actions in \( \A \).

A Markov chain (MC) is a tuple $\M = (\S, \P)$, where $\S$ is a countable set of states and $\P\colon \S\to\D(\S)$ is a transition function. 
Similarly to the notation used in MDPs, for $s,s'\in\S$, it is usual to denote
the probability of $s'$ in the distribution $\P(s)$ as $\P(s,s')$.
Given an MDP $\M = (\S, \A, \P)$ and a policy $\pi\colon \S\to\D(\A)$, the Markov chain \emph{induced by $\pi$} is $\M_\pi = (\S, \P_{\pi})$, 
where $\P_{\pi}(s, s') = {\sum_{a\in \A}\pi(s)(a)\cdot \P(s,a,s')}$.

An infinite \emph{trace} (also known as \emph{path} in the literature) in a MC \( \M \) is a sequence 
\( \tau = s_0 s_1 s_2 \dots \) where $\P(s_i,s_{i+1}) > 0$ for all $i\geq 0$. We denote the set of all infinite traces by \( \Omega^\M \).

A \emph{distance} in an MDP 
is a function $d\colon \S\times \S\to \RR_{\geq 0}$ such that for all $x,y,z\in \S$,
we have 
\begin{itemize}
    \item Simmetry: $d(x,y) = d(y,x)$.
    \item Triangular inequality: $d(x,y) \leq d(x,z) + d(z,y)$.
    \item Identity: $d(x,y) = 0$ if and only if $x=y$. 
\end{itemize}

A \emph{ball} of radius $r>0$ centered at state $s\in \S$ is the set 
\begin{equation*}
  B_r(s) = \{ s'\in\mathcal S\::\: d(s,s')< r\}.
\end{equation*}
Similary, a \emph{ball} of radius $r>0$ centered at a set of states $S\subseteq \S$ is the set
\begin{equation*}
    B_r(S) = \{ s'\in\mathcal S\::\: \exists s\in S.\:  d(s,s')< r\}.   
\end{equation*}

\nomenclature{$\S$}{Set of states of an MDP}
\nomenclature{$\A$}{Set of actions of an MDP}
\nomenclature{$\P$}{Transition probability function of an MDP. $\P\colon\S\times\A\to \D(\S)$}
\index{Markov decision process (MDP)}
\index{Markov chain (MC)}
\index{induced!MC induced by an MDP and a policy}

\index{Markov decision process (MDP)!product MDP}
\paragraph*{Product MDP.}
Given two MPDs $\M_1 = (\S_1, \A_1, \P_1)$ and $\M_2 = (\S_2, \A_2, \P_2)$,
the \emph{product MDP of $\M_1$ and $\M_2$} is the MDP $\M = (\S_1\times \S_2, \A_1\times\A_2, \P)$,
where the transition function $\P$ is defined as follows. 
For each $(s_1, s_2), (s_1', s_2')\in \S_1\times \S_2$, and for each $(a_1,a_2)\in \A_1\times\A_2$, the transition probability is defined as
\[
\P\Big(\big(s_1, s_2\big), \big(a_1,a_2\big), \big(s_1',s_2'\big)\Big) = \P_1\big(s_1, a_1, s_1'\big)\cdot \P\big(s_2, a_2, s_2'\big).
\]

Similarly, we can define the product between an MDP and a Markov chain. 
Given an MDP $\M_D = (\S_D, \A, \P_D)$ and a Markov chain $\M_C = (\S_C, \P_C)$, 
the \emph{product MDP of $\M_D$ and $\M_C$} is the MDP $\M = (\S_M\times \S_C, \A, \P)$, 
where the transition function is defined as follows. 
For each $(s_D, s_C), (s_D', s_C') \in \S_D\times\S_C$ and each $a\in\A$, the transition probability is defined as
\[
\P\Big(\big(s_D,s_C\big), a , \big(s_D',s_C'\big)\Big) = 
\P_D\big(s_D,a,s_D'\big)\cdot \P_C\big(s_C, s_C'\big).
\]

\index{Cylinder set construction}

\subsection{Cylinder Set Construction}
\label{sec:prelim-cylinder-set-construction}
To define a probability measure over traces,
we use the \emph{cylinder set construction}. 
This is a standard construction in the literature; details can be found in~\cite[Chap. 10]{baier2008principles}.
Let $\M = (\S,\P)$ be a Markov chain.
For a finite trace prefix \( \omega = s_0 s_1 \dots s_n \),
the \emph{cylinder set} generated by \( \omega \), 
denoted \( \Cyl(\omega) \), is the set of all infinite traces starting with \( \omega \). Formally:
$
\Cyl(\omega) = {\{ \omega' \in \Omega \::\: \omega' \text{ begins with } \omega \}}.
$
The probability of the cylinder set \( \Cyl(\omega) \) is defined as
$\PP^\M(\Cyl(\omega)) = \prod_{i=0}^{n-1} \P(s_i, s_{i+1})$.
The $\sigma$-algebra associated with $\M$, denoted by $\F^\M$,
is the $\sigma$-algebra generated by all $\Cyl(\omega)$, where $\omega$ is a finite trace prefix. 
With this construction, $(\Omega^\M, \F^\M, \PP^\M)$ is a probability space 
that lets us measure the probabilities of finite trace prefixes (as the probability of its corresponding cylinder set), and in the limit lets us measure the probability of infinite traces.

Let $\M = (\S,\A, \P)$ and $\pi\colon \S\to\D(\A)$ be a policy.
We cannot define a probability space on $\M$. 
To talk about probabilities in an MDP, 
we need to make a cylinder set construction on the Markov chain induced by an MDP and policy.
The construction for MDPs is slightly different, as we include
actions as part of the trace. 

An infinite \emph{state-action} trace is a sequence $\omega = s_0a_0s_1a_1\dots$, where $\pi(s_i)(a_i) > 0$ and $\P(s_i,a_i,s_{i+1}) > 0$ for all $i\geq 0$.
We denote the set of all infinite state-action traces as $\Omega^\M_\pi$ and make the same cylinder set construction to define the sigma algebra on $\Omega^\M_\pi$ generated by all cylinder sets of finite state-action trace prefixes,
and the corresponding probability measure, 
where for a given trace prefix $\omega = s_0a_0s_1a_1\dots s_n$, 
the probability of the cylinder set associated with it is
$\PP^\M_\pi(\Cyl(\omega)) = \prod_{i=0}^{n-1} \pi(s_i)(a_i) \P(s_i, s_{i+1})$.
We will denote the generated probability space as $(\Omega^\M_\pi, \F^\M_\pi, \PP^\M_\pi)$.

\subsection{Reachability Properties}
\label{sec:prelim-reachability-properties}
A \emph{reachability property} is defined as the probability of
reaching a given set of target states \( T \subseteq S \) 
from an initial state \( s\in\S \) under a policy \( \pi \) after at most $k\in\NN\cup \{\infty\}$ transitions. 
Formally, we define the reachability probability of $T$ in $\M$ from $s_0$ using the policy $\pi$ in less than $k$ transitions as:
\begin{equation}
\label{eq:def-reachability-property}
    \PP^\M_{\pi}(\Reach_{\leq k} (s, T)) = \PP^{\M}_{\pi}\left( \{ \omega \in \Omega^\M_{\pi} \::\: s_0 = s \mbox{ and } \exists n \leq k, \, s_n \in T \} \right).    
\end{equation}

Similarly, we may be interested in computing a reachability probability after a particular action has been fixed. 
The probability of reaching $T$ in $\M$ from $s$ after performing action $a\in\Act$ is defined as:
\begin{equation}
\label{eq:def-act-reachability-property}
    \PP^\M_{\pi}(\Reach_{\leq k} (s, a, T)) = \sum_{s'\in \S} \P(s,a,s')\cdot  \PP^{\M}_{\pi}\left( \{ \omega \in \Omega^\M_{\pi} \::\: s_0 = s' \mbox{ and } \exists n \leq k, \, s_n \in T \} \right).
\end{equation}

Sometimes, it is of interest to know the maximum and minimum values of the reachability probability when considering the space of all policies. We define these probabilities as:
\[
\PP^\M_{\min}(\Reach_{\leq k} (s, T)) = 
\min_{\pi\colon\S\to\D(\A)} \PP^\M_{\pi}(\Reach_{\leq k} (s, T)),
\]
\[
\PP^\M_{\max}(\Reach_{\leq k} (s, T)) = 
\max_{\pi\colon\S\to\D(\A)} \PP^\M_{\pi}(\Reach_{\leq k} (s, T)),
\]
\[
\PP^\M_{\min}(\Reach_{\leq k} (s, a, T)) = 
\min_{\pi\colon\S\to\D(\A)} \PP^\M_{\pi}(\Reach_{\leq k} (s, a, T)),
\]
\[
\PP^\M_{\max}(\Reach_{\leq k} (s, a, T)) = 
\max_{\pi\colon\S\to\D(\A)} \PP^\M_{\pi}(\Reach_{\leq k} (s, a, T)).
\]

It may also be interesting to consider the maximum and minimum values when restricting to a certain subset of available policies $\Pi\subseteq \{\pi\colon\S\to\D(\A)\}$. 
In such cases, the definitions are analogous and we denote them as
\[
\PP^\M_{\min\mid \Pi}(\Reach_{\leq k} (s, T)) = 
\min_{\pi\in\Pi} \PP^\M_{\pi}(\Reach_{\leq k} (s, T)),
\]
\[
\PP^\M_{\max\mid \Pi}(\Reach_{\leq k} (s, T)) = 
\max_{\pi\in\Pi} \PP^\M_{\pi}(\Reach_{\leq k} (s, T)),
\]
\[
\PP^\M_{\min\mid \Pi}(\Reach_{\leq k} (s, a, T)) = 
\min_{\pi\in\Pi} \PP^\M_{\pi}(\Reach_{\leq k} (s, a, T)),
\]
\[
\PP^\M_{\max\mid \Pi}(\Reach_{\leq k} (s, a, T)) = 
\max_{\pi\in\Pi} \PP^\M_{\pi}(\Reach_{\leq k} (s, a, T)).
\]

\paragraph*{Avoidance properties.} 
With the same spirit, we define the \emph{avoidance} probability as the complement of the reach probability. We use the notation \newline 
$\Avoid_{\leq k}(s,T)~=~\lnot~\Reach_{\leq k}(s,T)$, and define it as
\begin{equation}
    \label{eq:avoidance_property}
    \PP^\M_\pi(\Avoid_{\leq k}(s,T)) = 1-\PP^\M_\pi(\Reach_{\leq k}(s,T)).    
\end{equation}
We also consider the minimum and maximum probabilities as described before for avoidance probabilities.
Since the avoidance property is the complement of the corresponding reachability property, the policy that maximizes one minimizes the other, and viceversa. 
That is:
\begin{align*}
    \PP^M_{\min}(\Avoid_{\leq k}(s,T)) &= 1 - \PP^\M_{\max}(\Reach_{\leq k}(s,T)), \qquad \mbox{and}\\
    \PP^M_{\max}(\Avoid_{\leq k}(s,T)) &= 1 - \PP^\M_{\min}(\Reach_{\leq k}(s,T)).
\end{align*}
The same reasoning applies in the case that an action has already been decided ---  $\PP_{\min/\max}(\Avoid_{\leq k}(s,a,T))$ --- and the case when there is a restriction on the set of policies --- $\PP_{\min/\max\mid \Pi}(\Avoid_{\leq k}(s,T))$.

\paragraph*{Bounded and unbounded properties.}
When $k\in\NN$, we say that these are \emph{bounded reachability/avoidance} properties. 
When $k = \infty$, we say that these are \emph{unbounded reachability/avoidance} properties.
In such cases, we may drop the explicit reference to $k$ in our notation, 
writing $\PP^\M_{\pi}(\Reach(s, T))$ instead of $\PP^\M_{\pi}(\Reach_{\leq \infty} (s, T))$ for unbounded reachability.

The probabilities for bounded and unbounded reachability can be computed using probabilistic model checking algorithms~\cite{katoen2016probabilistic}.
\index{probabilistic model checking}
%


\subsection{Reinforcement Learning}
\label{sec:prelim-RL}
\index{Reinforcement learning (RL)}

Reinforcement learning (RL)~\cite{sutton2018reinforcement} is a category of machine learning where an agent learns to select actions from observations through trial and error, with the goal of maximising the long-term returns defined by a reward function.
A reinforcement learning problem is formalized with an MDP $\M = (\S,\A,\P)$.

In RL problems, the MDP is accompanied by a reward function 
$\R \colon \S\times\A\times\S \to \RR$.
An RL agent executes a policy $\pi\colon \S\to \D(\A)$ in the MDP.
A \emph{state-action trace} is a sequence of states, actions and rewards $\tau = s_0a_0r_1,s_1a_1r_2\dots$, 
where $s_0s_1\dots$ is a trace in the MDP induced by $\pi$ 
and $r_{i+1} = \R(s_i, a_i, s_{i+1})$.
\index{trace!state-action trace}

The interaction between the environment and the agent generates 
state-action-reward traces as follows.
At each step, the agent observes the current state $s_i\in\S$, 
selects an action $a_i\in\A$, 
the environment transitions to a next state $s_{i+1}$, 
sampled from the probability distribution $\P(s_i, a_i)$, 
and the agent receives a reward $\R(s_i, a_i, s_{i+1})$. 
Note that adding the rewards is just a formalism, 
since a state-action trace already determines the corresponding state-action-reward trace.
The \emph{discounted return} for a state-action-reward trace $\tau$ is 
$G(\tau) = \sum_{k=0}^{\infty}\gamma^t r_{t}$, 
where $\gamma\in[0,1)$ is the \emph{discount factor}.
Note that if $\R$ is a bounded function, $\gamma<1$ guarantees that $G(\tau)$ is finite for any $\tau$.
Since $\tau$ can be seen as an element of $\Omega^\M_\pi$, 
$G$ is a random variable on $\Omega^\M_\pi$, so we can consider its expectation $\EE(G)$. 
When it becomes important to state which policy $\pi$ is being used to induce an MC and generate a probability space, we will write $\EE_{\tau\sim\pi}(G(\tau))$.
\index{discounted return}
\index{discount factor}
\index{trace!state-action-reward trace}

The goal of the agent is to find a policy that maximises the expected discounted return.
Formally, the goal of the agent is to find $\pi^*\colon \S\to\D(\A)$ such that
\[\
\pi^*\in \argmax_{\pi\colon \S\to\D(\A)}\EE_{\tau\sim\pi}[G(\tau)].
\]
There are many algorithms to approximate the optimal policy from available traces; see~\cite {shakya2023reinforcement} for a recent survey. 
We will use in this thesis the Q-learning algorithm~\cite{watkins1992q}, which is one of the most classic approaches to the problem.
\index{Q-learning}

\section{Classification Problems and Fairness}
\label{sec:prelim-classification-problems}
Classification problems are a standard setting in supervised machine learning where there is an input space $\X\subseteq \RR^n$, 
a discrete set of labels $\Y\subseteq \NN$ 
and a ground truth distribution $ \sgen \in \D(\X\times \Y)$.
An element $x = (x_1,\dots, x_n)\in\X$ is an \emph{instance},
and each of the $x_i$'s are the different \emph{features}. 

The \emph{classification problem} consists on finding $f\colon \X\to\Y$ that minimizes the expected loss, defined as 
$\L(f) = \EE_{(x,y)\sim \sgen}\big[ \1[y \neq f(x)] \big]$, 
when given a set of samples $(x_0,y_0),\dots,(x_N,y_N)$ sampled from $\sgen$.


In some problems, there are concrete features of the instances that are considered \emph{protected} or \emph{sensitive} features, 
and it is of utmost importance to protect against bias with respect to those features. For example, in the problem of screening applications for a job, one of the protected features may be the applicant's gender.
There are many metrics used to determine whether a classifier is biased with respect to sensitive features.
When talking about fairness with respect to a given feature, it is useful to partition the input space into $\X = \F\times \G$, 
where $\G$ represents the sensitive feature, and $\F$ represents the rest of the features.
It is also useful to think of $X$ as a random variable on $\X$ that follows the input part of the ground truth distribution $\sgen$,
and $f(X)$ as a random variable on the label space $\Y$.
Similarly, $X = (F, G)$, where $F$ is a random variable on $\F$
and $G$ is a random variable on $\G$.
In this thesis, we will use fairness metrics based on demographic parity and equal opportunity that assume the sensitive feature can only take a finite set of values, i.e., $\G = \{g_1,\dots, g_k\}$. 
A classifier $f\colon \F\times\G\to\Y$ satisfies \emph{demographic parity} (DP) if for all $i\in \{1,\dots,k\}$, we have
\(
\EE[f(X) \mid G = g_i] = \EE[f(X)]
\).
When $\Y = \{0,1\}$, 
a classifier $f\colon \F\times \G\to\Y$ satisfies \emph{equal opportunity} (EqOpp) if for all $i\in \{1,\dots,k\}$, we have
\(
\EE[f(X) \mid G = g_i,\, y=1] = \EE[f(X) \mid y=1]
\).
\index{demographic parity}
\index{equal opportunity}

The literature on enforcing fairness properties in classification
problems is vast and rich (see~\cite{barocas2023fairness} for a recent account of the state of the art).

%% file: 25_reactive_decision_making.tex
\ifthenelse{\boolean{includequotes}}{
\begin{quotation}
    \textit{If I seem to wander, if I seem to stray, remember that true stories seldom take the straightest way.}
    \hfill
    --- Patrick Rothfuss, The Name of the Wind.
\end{quotation}
}{}

\section{Motivation and Outline}
This thesis presents different lines of work with the common motivation of advancing trust in autonomous systems, 
using different formal models. 
While each chapter can be seen as a standalone contribution towards this goal, 
the formal models and methods used in each chapter can be encompassed as part of a general framework.

In this chapter, we introduce the reactive decision-making framework, which generalizes the many models used throughout the thesis. Following the general definition, we justify how safety games, MDPs, and classification problems can be expressed as particular cases in this framework.

We also introduce a generalized definition of shielding, 
a method that is used in most of the chapters in this thesis. 
We present a unified definition and justify how this adapts to 
different notions of shielding in the literature.

\paragraph*{Outline.} 
In Section~\ref{sec:reactive-decision-making}, we introduce the reactive decision-making framework and show how other formalizations used in the paper can be viewed as particular cases of it.
In Section~\ref{sec:rdm-shielding} we introduce a generalized definition of shielding and in Section~\ref{sec:rdm-classical-shielding} we show 
how classical notions of shielding can be expressed in this generalized framework.

\paragraph*{Declaration of sources.}
This chapter is the original work of the author of this thesis and, at the time of writing, remains unpublished.

\section{Reactive Decision-Making}
\label{sec:reactive-decision-making}
The \emph{reactive decision-making} \index{Reactive decision-making} framework is an abstract framework that
models the different problems presented in this thesis. 
In a reactive decision-making framework, an agent interacts with an environment, 
depicted in Figure~\ref{fig:reactive-decision-making-framework}. 
There is a set of observations $\Obs$, controlled by the environment, 
and a set of actions $\Act$, controlled by the agent. 
An \emph{environment} is a tuple $\Env = (\Obs, \Act, \Trans)$, 
where $\Trans\colon (\Obs\times\Act)^* \to \D(\Obs)$.
We call $\Trans$ the \emph{environment transition function} \index{environment transition function}.
An \emph{agent} is a tuple $\Ag = (\Obs, \Act, \Pol)$, 
where $\Pol\colon (\Obs\times\Act)^*\times \Obs \to \D(\Act)$.
We call $\Pol$ the \emph{agent policy function}.
Given $\Obs$ and $\Act$, the set of all policies is $\mathtt{Pol}(\Obs,\Act)$.
An environment and an agent are \emph{compatible} if they share
the same set of observations and actions.
\index{agent policy function}

An environment $\Env = (\Obs, \Act, \Trans)$ is \emph{deterministic} if for all input $\tau\in (\Obs\times\Act)^*$, 
the support of the environment transition function,
$\Supp(\Trans(\tau))$, has only a single element. 
In such cases, we would write the environment transition function as  
$\Trans\colon (\Obs\times\Act)^*\to\Obs$.
An agent $\Ag = (\Obs, \Act, \Pol)$ is \emph{deterministic} if for all input
$(\tau, o)\in (\Obs\times\Act)^*\times\Obs$, 
$\Supp(\Pol(\tau,o))$ has only a single element. 
In such cases, we would write the policy function as 
$\Pol\colon (\Obs\times\Act)^*\times\Obs\to \Act$.
Given a policy $\Pol\colon (\Obs\times\Act)^*\times\Obs\to \D(\Act)$, 
a \emph{determinization of $\Pol$} is any deterministic policy 
$\Pol_{det}\colon (\Obs\times\Act)^*\times\Obs\to \Act$, 
such that for all $(\tau,o)\in (\Obs\times\Act)^*\times\Obs$, 
we have that $\Pol_{det}(\tau, o) \in\Supp\left( \Pol(\tau, o)  \right)$.

An environment $\Env = (\Obs, \Act, \Trans)$ is \emph{memoryless} if 
the transition function depends only on the last pair of action and observation, 
that is, if for all $(o, a)\in \Obs\times\Act$ and 
for all $\tau, \tau' \in (\Obs\times\Act)^*$, we have
$\Trans(\tau\cdot(o,a)) = \Trans(\tau'\cdot(o,a))$.
Similarly, an 
agent $\Ag = (\Obs, \Act, \Pol)$ is \emph{memoryless} 
if the policy function depends only on the last observation. 
Formally, if for all $o\in \Obs$ and
for all $\tau, \tau' \in (\Obs\times\Act)^*$, 
$\Pol(\tau, o) = \Pol(\tau',o)$.

Note that the definitions we give of agents and environments are strictly functional, 
so we do not consider how the elements are internally designed. 
For example, an agent may be designed as an automaton, with internal states and internal transition functions. 
For the general theory presented in this chapter, we do not model such details of the inner structure; we just study agents and environments as abstract functions that produce an output when given an input.

An \emph{observation-action trace} is a (finite or infinite) sequence of observations and actions 
$\tau = (o_0a_0),(o_1a_1),\dots \in (\Obs\times\Act)^\infty$. 
Given an environment $\Env = (\Obs, \Act,\Trans)$ and 
an agent $\Ag = (\Obs, \Act, \Pol)$, 
an observation-action trace is \emph{valid} if 
for all $k <  |\tau|$, we have that
$o_{k+1}\in\Supp(\Trans(\tau_{:k}))$ and 
$a_{k+1}\in\Supp(\Pol(\tau_{:k}, o_{k+1}))$.
In other words, a trace is valid if it could have been produced by the pair agent-environment.

An \emph{observation trace} is a (finite or infinite) sequence of observations
$\tau_O = o_0,o_1,\ldots\in \Obs^\infty$.
An \emph{action trace} is a (finite or infinite) sequence of actions 
$\tau_A = a_0a_1,\ldots\Act^\infty$. 
Given an observation trace $\tau_O$ and an action trace $\tau_A$, 
we can produce an observation-action trace $\tau=(o_0a_0),(o_1a_1),\dots \in (\Obs\times\Act)^\infty$ by interlacing them. 
We will denote this by $\tau_O||\tau_A$.
\index{trace!observation trace}
\index{trace!action trace}

Given an environment $\Env$, 
an observation trace $\tau_O$ is \emph{valid} if there exists an agent $\Ag$ and an action trace $\tau_A$ such that $\tau_O||\tau_A$ is valid for $\Env$ and $\Ag$.
Given an agent $\Ag$, 
an action trace $\tau_A$ is \emph{valid} if there exists an environment $\Env$ and an observation trace such that $\tau_O||\tau_A$ is valid for $\Env$ and $\Ag$. \index{trace!valid trace}

\begin{figure}[t]
    \centering
    \includegraphics[width=0.5\linewidth,page=2]{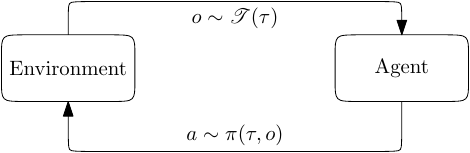}
    \caption{Reactive decision-making framework.}
    \label{fig:reactive-decision-making-framework}
\end{figure}

\subsubsection*{Probability of traces}

Given an environment $\Env = (\Obs,\Act,\Trans)$ and an
agent $\Ag = (\Obs,\Act,\Pol)$, 
we can make a similar cylinder set construction as in Section~\ref{sec:prelim-cylinder-set-construction} to define a probability measure on the space of
finite and infinite traces.

Let $\Omega^{\Env,\Ag}$ be the set of all observation-action traces. 
Let $\omega= o_1a_1,\dots, o_na_n$ be a finite prefix, the cylinder set generated by $\omega$ is 
\[
\Cyl(\omega) = \{ \omega'\in \Omega^{\Env, \Ag}\::\: \omega' \mbox{ begins with } \omega \}.
\]
The probability of the cylinder set $\Cyl(\omega)$ is 
\[
\PP^{\Env, \Ag}(\Cyl(\omega)) =  \prod_{i=0}^n \Trans\left(\omega_{[:i-1]}\right)(o_{i})\cdot 
\Pol\left(\omega_{[:i-1]}, o_{i}\right)(a_{i}).
\]

Let $\F^{\Env, \Ag}$ be the $\sigma$-algebra generated by the cylinder sets of all finite traces $\omega$. The space $(\Omega^{\Env, \Ag}, \F^{\Env, \Ag}, \PP^{\Env, \Ag})$ is a probability space. 

Sometimes it is useful to talk about traces of a given length and assign probabilities to them, instead of thinking of infinite traces. 
Let $\Omega^{\Env,\Ag}_{k}$ be the set of traces of length equal to $k$.
Since $\Omega^{\Env,\Ag}_{k}$ is countable, 
we can use $\F^{\Env, \Ag}_k = 2^{\Omega^{\Env,\Ag}_{k}}$ as our $\sigma$-algebra.
Let $\omega\in \Omega^{\Env,\Ag}_{k}$, its probability is defined as
$\PP^{\Env, \Ag}_k(\omega) = \PP^{\Env, \Ag}(\Cyl(\omega))$.
Then $(\Omega^{\Env, \Ag}_{k}, \F^{\Env, \Ag}_k,  \PP^{\Env, \Ag}_k)$ is a probability space.

Note that in $\Omega^{\Env,\Ag}_k$, a trace $\omega$ is valid if and only if 
$\PP^{\Env,\Ag}_k(\omega) \neq 0$. This is not true for infinite traces.

\subsection{Deterministic Two-player Games}
\label{sec:rdm-two-player-games}

In this section, we show how a deterministic two-player game corresponds to a memoryless environment in the reactive decision-making framework,
where the sets of observations and actions are both finite, and all probability distributions over observations and actions are uniform over their support.

A deterministic two-player game is formalized as a tuple
$\G = (S, s_0, S_{env}, S_{ag}, \Act, \T, \Acc)$ (Section~\ref{sec:prelim-TwoPlayerGames}).
The set $S_{ag}$ corresponds to the set of observations $\Obs$ in the reactive decision-making framework, 
and the set $S_{env}$ corresponds to $\Obs\times\Act$. 
The set of actions $\Act$ is the same for both formalisms.
Given $o\in S_{ag}$, and $\sigma\in \Act$, the transition is trivially $o\xrightarrow{\sigma}(o,\sigma)$.
Given $(o,\sigma)\in S_{env}$, the allowed transitions 
$(o,\sigma)\xrightarrow{u}o'$ are those for which $o'\in\Trans(o,\sigma)$.

Note that a play in the safety game, 
$\tau = [s_0, s_1, s_2, \dots]$,
always has the form
$\tau = [ (o_0,\sigma_{0}), o_1, (o_1,\sigma_1), o_2, (o_2,\sigma_2),\dots ]$, 
where $s_{2i} = (o_i, \sigma_i)$ and $s_{2i+1} = o_{i+1}$, 
so it naturally corresponds to an observation-action trace.

In a deterministic two-player game, 
the concrete value of the probability of a given observation or action is 
unimportant, it is only relevant whether the probability is non-zero.
Therefore, the transition functions can be expressed as 
$\Trans\colon \Obs\times\Act\to 2^\Obs$ and 
$\Pol\colon(\Obs\times\Act)^*\times\Obs\to 2^\Act$,
instead of $\Trans\colon \Obs\times\Act\to \D(2^\Obs)$
and $\Pol\colon(\Obs\times\Act)^*\times\Obs\to \D(2^\Act)$.
Note that this is an \emph{interpretation} of the model. 
All allowed transitions are considered to have the same probability because 
the only thing that matters is whether the probability is non-zero. 
In the game setting, the environment is considered adversarial, 
so any action with a non-zero probability (no matter how low), 
needs to be considered when computing strategies for the agent.

\subsection{Markov Decision Processes}
In the reactive decision-making framework,
a Markov decision process corresponds to a memoryless environment. 
In typical MDP notation, as introduced in Section~\ref{sec:prelim-MDPs}, observations are called \emph{states},
denoted as $\S$ (instead of $\Obs$).
Since the environment transition function is memoryless, 
it is written with 
$(\S\times\A)$ as its domain -- instead of $(\S\times\A)^*$ --,
denoted by $\P\colon \S\times\A\to\D(\S)$, and called the \emph{probabilistic transition function}.


\subsection{Classification Problems}
\label{sec:rdm-classification-problems}


In classification problems, the observation space is the input space of the problem,
i.e., following the notation of Section~\ref{sec:prelim-classification-problems}, $\Obs = \X$.
It is standard to assume that there is a single distribution from which problem instances are sampled. 
This would correspond with an environment transition function that does not depend in any way on the current trace. 
This is a stronger condition as being memoryless since we are imposing that for all $\tau, \tau' \in (\Obs\times\Act)^*$, $\Trans(\tau) = \Trans(\tau')$.
Therefore, we can simply represent the environment as a distribution $\Theta_{\X}\in \D(\Obs)$.
There is also literature studying classification problems where
the data distribution changes according to the actions (accepts or rejects) 
given by a classifier. This phenomenon has been studied as it relates to fairness in~\cite{d2020fairness}. 
Our framework is well adapted to such cases, as it would mean just keeping a full environment transition function. 

Another characteristic of classification problems is that the set of actions represents the labels available for classification, i.e. $\Act = \Y$.

\begin{figure}[b]
    \centering
    \includegraphics[width=0.8\linewidth,page=1]{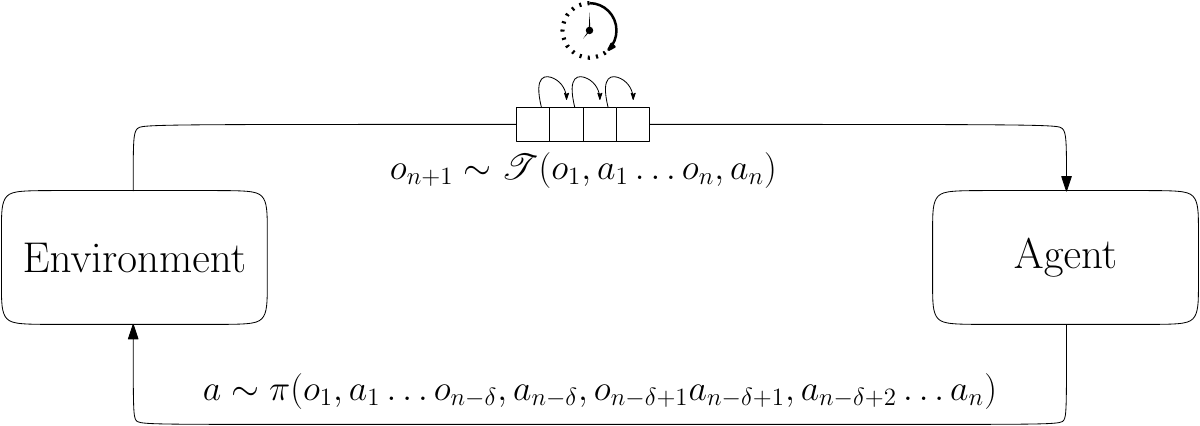}
    \caption{Reactive decision-making framework with delayed observations.}
    \label{fig:reactive-decision-making-delayed}
\end{figure}

\subsection{Delayed Observations}
\label{sec:rdm-delayed-obs}
In some use cases, it is interesting to consider agents that work with uncertain observations. 
We are particularly interested in the case of agents that receive information about observations produced by the environment delayed by a certain number of steps,
in the same spirit as safety games with delayed input presented in Section~\ref{sec:prelim-games-under-delay}.
These delays are a common challenge when dealing with asynchronous control architectures.
In this section, we provide two formalizations of the delayed framework and show that they are equivalent. 

The first formalization is intended to be an intuitive one, 
while the second formalization is an operational one, 
that we will use to develop the theory of shielding resilient to delayed observations.
We will show that both formulations are, in fact, equivalent in Section~\ref{sec:rmd-delayed-equivalence}.

\subsubsection{Delayed Observations through Modified Agents}

We can formalize delays as part of the reactive decision-making framework by considering a variation of the actions available to the agent, 
as depicted in Figure~\ref{fig:reactive-decision-making-delayed}.
The environment is the same, and samples are the next observation from the full observation-action trace.
On the other side, the agent does not have access to the last $\delta$ observations produced by the environment.
In the initial transient phase, the agent only has access to its own action record. 
After the transient phase, after $\delta$ steps, 
the agent starts receiving the first observations,
entering the steady observation phase.
\index{transient phase}
\index{steady observation phase}

Given an environment $\Env = (\Obs,\Act,\Trans)$, 
an agent under delay $\delta$ has a policy function $\Pol\colon \Delayed_\delta(\Obs,\Act) \to \D(\Act)$, 
where
\begin{equation}
\label{eq:reactive-delayed-obs}
    \Delayed_\delta(\Obs,\Act) = \Act^{\leq \delta} \cup  \left\{ 
    \begin{matrix}
    (o_1,a_1,\dots, o_n, a_n, a_{n+1},\dots, a_{n+\delta-1})\: :\: \\
     \hfill  o_i\in\Obs, \, a_i\in \Act,\, n \geq 0
    \end{matrix}
    \: \right\} .    
\end{equation}

In other words, the agent has either access to a trace longer than $\delta$ with the $\delta$ last observations removed (steady observation phase), or it has access to a trace shorter than $\delta$ composed of only actions (transient phase).
This is similar to how we defined strategies in games under delay in Equation~\eqref{eq:safety-game-strategy-delay}.


\subsubsection{Delayed Observations through Restriction to Agnostic Agents}

While Figure~\ref{fig:reactive-decision-making-delayed} is a more intuitive formulation, 
we will use an equivalent formulation that is more operational, in terms of considering only agents restricted to a certain domain. 

\begin{definition}[Delayed-observation agent]
\label{def:delayed-obs-agent}
    Let $\Env = (\Obs,\Act, \Trans)$ be an environment and $\delta\geq 0$. 
    An agent $\Ag = (\Obs, \Act, \Pol)$ 
    \emph{works with observations delayed by $\delta$} 
    if it is agnostic to the last $\delta$ observations. 
    That is, if for all
    $\tau\in (\Obs\times\Act)^*$, 
    all $(a_1,\dots, a_\delta)\in \Act^\delta$, 
    and all 
    $(o_1,\dots, o_\delta), (o_1',\dots, o_\delta')\in\Obs^\delta$
    \begin{equation}
    \label{eq:def-delayed-obs-agent}
        \Pol\Big(\big(\tau\cdot (o_1,a_1),\dots, (o_{\delta-1}, a_{\delta-1})\big), o_\delta\Big) =     
        \Pol\Big(\big(\tau\cdot (o_1',a_1),\dots, (o_{\delta-1}', a_{\delta-1})\big), o_\delta'\Big)    
    \end{equation}
    We denote the set of agents working with observations delayed by $\delta$ as $\Pi^{\myclock}_\delta$. 
\end{definition}

\subsubsection{Equivalence}
\label{sec:rmd-delayed-equivalence}

The equivalence between Definition~\ref{def:delayed-obs-agent}
and an agent with domain $\Delayed_\delta(\Obs,\Act)$ as in Equation~\eqref{eq:reactive-delayed-obs} stems from the fact 
that if an agent $\Pol$ with domain $(\Obs\times\Act)^*\times\Obs$ is agnostic to the last $\delta$ observations, 
it is fully determined by the agent $\Pol'\colon \Delayed_\delta(\Obs,\Act)\to\D(\Act)$, and vice versa.

For every $\tau\in(\Obs\times\Act)^*\times\Obs$, 
we can factor it as $\tau = \big( (o_1, a_1), \dots, (o_n, a_n), \dots, (o_{n+\delta-1}, a_{n+\delta-1}), o_\delta \big)$, 
and then define $\Pol(\tau)$ as
\[
\Pol(\tau) = \Pol'\big( (o_1, a_1), \dots, (o_n, a_n), a_{n+1},\dots, a_{n+\delta-1} \big).
\]

For the backwards direction, 
given $\tau \in \Delayed_\delta(\Obs,\Act)$, 
we can determine $\Pol'(\tau)$ as follows.
The element $\tau\in \Delayed_\delta(\Obs,\Act)$ is of the form
$\tau = (o_1,a_1,\dots, o_n, a_n, a_{n+1},\dots, a_{n+\delta-1})$,
for some $n\in \NN$, $o_i\in \Obs$ and $a_i\in\Act$.
We choose $\delta$ arbitrary observations 
$(o_{n+1},\dots, o_{n+\delta})\in \Obs^{\delta}$ and define $\Pol'(\tau)$ as 
\[
\Pol'(\tau) = \Pol\Big( \big( (o_1,a_1),\dots, (o_{n+\delta-1}, a_{n+\delta-1}) \big), o_{n+\delta}  \Big).
\]
The distribution corresponding to $\Pol'(\tau)$ is well defined, i.e., does not depend on the choice of observations $o_{n+1},\dots, o_{n+\delta}$, by virtue of Equation~\eqref{eq:def-delayed-obs-agent}.

\section{Shielding}
\label{sec:rdm-shielding}

\subsection{Definitions}

\begin{figure}[b]
     \centering
     \begin{subfigure}[b]{0.45\textwidth}
         \centering
         \includegraphics[width=\textwidth,page=3]{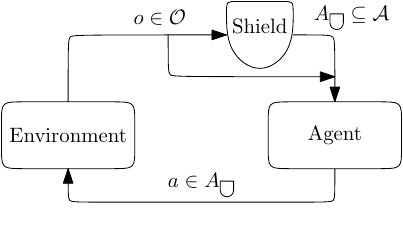}
         \caption{Pre-shield.}
         \label{fig:pre-shield}
     \end{subfigure}
     \hfill
     \begin{subfigure}[b]{0.45\textwidth}
         \centering
         \includegraphics[width=\textwidth,page=4]{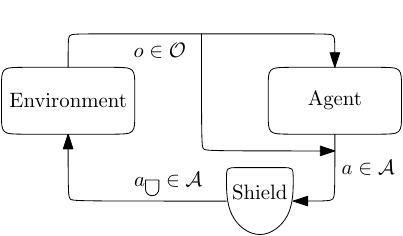}
         \caption{Post-shield.}
         \label{fig:post-shield}
     \end{subfigure}
        \caption{Shielded reactive decision-making framework.}
\end{figure}

A \emph{shield} is an element that modifies the behaviour of an agent, 
filtering actions either before (pre-shield) or after (post-shield) 
the agent decides on them. 
Figure~\ref{fig:pre-shield} illustrates the introduction of a pre-shield in a reactive decision-making pair, 
while Figure~\ref{fig:post-shield} illustrates it if for a post-shield.
We use the symbol $\shield$, intended to be read as ``shield''.

\index{shield}

\begin{definition}[pre-shield-ready agent]
    A \emph{pre-shield-ready agent} is a function $\Ag_{pre}\colon (\Obs\times\Act)^*\times\Obs\times 2^\Act \to \D(\Act)$, 
    such that for any input $(\tau, o, A)\in (\Obs\times\Act)^*\times\Obs\times 2^\Act$, 
    the agent only proposes actions in $A$, i.e., $\Supp(\Ag_{pre}(\tau, o, A)) \subseteq A$.
\end{definition}

\begin{definition}[Pre-shield]\label{def:preshield-abstract}
    A \emph{pre-shield} is a function $\Sh\colon (\Obs\times\Act)^*\times \Obs\to 2^\Act$, 
    that, given a trace and an observation, produces a set of allowed actions.
    A pre-shield together with a pre-shield-ready agent form a new agent  
    $\Ag_{\shield}\colon (\Obs\times\Act)^*\times\Obs \to \D(\Act)$, 
    defined as
    $\Ag_{\shield}(\tau,o) = \Ag_{pre}(\tau, o, \Sh(\tau, o))$.  \index{shield!pre-shield}
\end{definition}

Given a pre-shield-ready agent 
$\Ag\colon (\Obs\times\Act)^*\times\Obs\times 2^\Act \to \D(\Act)$, it induces a regular agent 
$\Ag_{reg} \colon (\Obs\times\Act)^*\times\Obs \to \D(\Act)$, by considering only unrestricted actions.
That is, for any input $(\tau,o)\in (\Obs\times\Act)^*\times\Obs$, 
$\Ag_{reg}$ is defined as $\Ag_{reg}(\tau, o) = \Ag(\tau, o, \Act)$.
We call this the \emph{regular agent induced by $\Ag$}.
\index{induced!regular agent induced by pre-shield-ready agent}

\begin{definition}[Post-shield]\label{def:postshield-abstract}
    A \emph{post-shield} is a function $\Sh\colon (\Obs\times\Act)^*\times \Obs\times\Act \to \Act$ that given a trace, an observation and an action, 
    produces a new distribution of allowed actions. 
    A shield together with an agent $\Ag_{pos}$ form a new agent 
    $\Ag_{\shield}\colon (\Obs\times\Act)^*\times \Obs\to\D(\Act)$, defined as
    $\Ag_{\shield}(\tau, o) = \Sh(\tau, o, \Ag_{pos}(\tau, o))$.
    \index{shield!post-shield}
\end{definition}

Note that the probabilistic nature of a post-shielded agent $\Ag_{\shield}$ does not come from the shield itself, 
but from the fact that given $(\tau, o)$, $\Ag_{pos}(\tau,o)$ is a distribution of actions, which induces a distribution of shielded 
actions $\shield(\tau, o, \Ag_{pos}(\tau,o))$.

As mentioned for the environment-agent construction, our definitions of shields are strictly functional, so we develop the general theory of shielding without discussing how these functions are to be computed or implemented. 
The discussions on implementation are specific for each type of shielding and are a substantial part of the content of Section~\ref{sec:rdm-classical-shielding} and Chapter~\ref{chap:fairness}.

\subsection{Shielding Induced by Agents}
\index{induced!shield induced by agent}

A trivial way of building shields is through the \emph{induced shield construction}.

\begin{definition}[Induced pre-shield]
\label{def:induced-pre-shield}
     Let $\Ag\colon (\Obs\times\Act)^*\times\Obs\to\D(\Act)$ be an agent.
     The \emph{pre-shield induced by $\Ag$} is the pre-shield $\Sh_{\Ag}^{pre}$, 
    defined for any $(\tau,o)\in (\Obs\times\Act)^*\times \Obs$ as 
    \[
    \Sh_{\Ag}^{pre}(\tau, o) = \Supp\left(\Ag(\tau, o)\right). 
    \]
\end{definition}

To build the induced post-shield, we also need a determinization of the agent.

\begin{definition}[Induced post-shield]
\label{def:induced-post-shield}
    Let $\Ag\colon (\Obs\times\Act)^*\times\Obs\to\D(\Act)$ be an agent, and let $\Ag_{det}$, a determinization of $\Ag$.
    The \emph{post-shield induced by $\Ag$ and $\Ag_{det}$}
    is the post-shield $\Sh_{\Ag}^{pos}$, 
    defined for any $(\tau,o,a)\in (\Obs\times\Act)^*\times \Obs\times \Act$ as 
    \[
    \Sh_{\Ag,\Ag_{det}}^{pos}(\tau, o, a) = \begin{cases}
        a & \mbox{ if } a\in \Supp\left(\Ag(\tau, o)  \right)  \\ 
        \Ag_{det}(\tau, o) & \mbox{ otherwise. }
    \end{cases}
    \]
\end{definition}

Any agent shielded with these shields will produce only traces that are valid for $\Ag$. 

With a similar spirit, we can also build agents from given shields. 
However, a shield may be compatible with many agents. 
We formalize this concept with the following definitions.

\begin{definition}[Agents associated with a pre-shield]
    Let $\Sh\colon (\Obs\times\Act)^*\times \Obs\to 2^\Act$ be a pre-shield. 
    An agent $\Ag\colon(\Obs\times\Act)^*\times\Obs\to\D(\Act)$ is
    \emph{associated with $\Sh$} if $\Sh^{pre}_{\Ag} = \Sh$.
\end{definition}

\begin{definition}[Agents associated with a post-shield]
    Let $\Sh\colon (\Obs\times\Act)^*\times \Obs\times\Act\to\Act$. 
    An agent $\Ag\colon(\Obs\times\Act)^*\times\Obs\to\D(\Act)$ is 
    \emph{associated with $\Sh$} if there exists a determinization of $\Ag$, named $\Ag_{det}$, such that 
    $\Sh^{pos}_{\Ag, \Ag_{det}} = \Sh$.    
\end{definition}

For both pre- and post-shields, the set of agents associated with a shield $\Sh$ is denoted by $\Pi_{\Sh}$.

Given a set of agents $\Pi$, 
the set of shields associated with $\Pi$ is $\Sigma_\Pi$ defined as
\begin{equation}
\label{eq:SigmaPi}
    \Sigma_\Pi = \left\{ \Sh\::\: \Pi_{\Sh} \subseteq \Pi \right\}.    
\end{equation}

We can characterize certain aspects of a shield by its set of associated agents.
For example, we have the following technical result, 
that will be used to characterize shields in a delayed observation setting as shields whose set of associated agents is inside $\Pi_\delta^{\myclock}$ for some $\delta\in\NN$.

\begin{lemma}
\label{lem:shields-associated-delayed}
    Let $\Env = (\Obs,\Act, \Trans)$ be an environment, 
    $\delta\in\NN$,
    and let $\Sh$ be a shield such that $\Pi_{\Sh}\subseteq \Pi^{\myclock}_\delta$.
    Then $\Sh$ is also agnostic to the last $\delta$ observations. 
    That is, for all
    $\tau\in (\Obs\times\Act)^*$, 
    all $(a_1,\dots, a_{\delta-1})\in \Act^\delta$, 
    and all 
    $(o_1,\dots, o_\delta), (o_1',\dots, o_\delta')\in\Obs^\delta$ 
    we have:
    \begin{enumerate}
        \item If $\Sh$ is a pre-shield, 
        i.e., $\Sh\colon (\Obs\times\Act)^*\times\Obs\to 2^\Act$, 
        then 
        \begin{equation}
        \label{eq:lemma-induced-pre-shield}
            \Sh\Big( \big(\tau\cdot (o_1,a_1),\dots, (o_{\delta-1}, a_{\delta-1}), \big), o_\delta\Big) = 
            \Sh\Big( \big(\tau\cdot (o_1',a_1),\dots, (o_{\delta-1}', a_{\delta-1}), \big), o_\delta'\Big).
        \end{equation}
        \item If $\Sh$ is a post-shield, i.e., $\Sh\colon (\Obs\times\Act)^*\times \Obs\times\Act\to\Act$, 
        then for all action $a_{\delta}\in \Act$ 
        \begin{equation}
        \label{eq:lemma-induced-post-shield}
            \Sh\Big( \big(\tau\cdot (o_1,a_1),\dots, (o_{\delta-1}, a_{\delta-1}), \big), o_\delta, a_\delta\Big) = 
            \Sh\Big( \big(\tau\cdot (o_1',a_1),\dots, (o_{\delta-1}', a_{\delta-1}), \big), o_\delta', a_\delta \Big).
        \end{equation}
    \end{enumerate}
\end{lemma}

\begin{proof}
    This results follows directly from the definitions. 
    Suppose $\Sh$ is a pre-shield and let $\Ag\in \Pi_{\Sh}$.
    Since $\Ag$ is associated with $\Sh$, 
    we have that for any trace $\tau\in(\Obs\times\Act)^*$ and any observation $o\in\Obs$, $\Sh(\tau, o) = \Supp(\Ag(\tau, o))$ (by Definition~\ref{def:induced-pre-shield}).
    Using the hypothesis that $\Sigma_{\Sh}\subseteq \Pi_\delta^{\myclock}$ we have that $\Ag\in\Pi_\delta^{\myclock}$, 
    so the condition in Equation~\eqref{eq:lemma-induced-pre-shield} is satisfied by virtue of Definition~\ref{def:delayed-obs-agent}.
    The proof is analogous for the case of $\Sh$ being a post-shield.
\end{proof}

\subsection{Correctness}
\index{shield!correct shield}

The idea behind shielding is that a shielded agent should satisfy a certain desirable property, while being as close as possible to the original agent.
In its most general form, a property can be defined as
a language of correct traces $\L\subseteq (\Obs\times\Act)^\infty$.

\begin{definition}[Correctness]
    \label{def:correctness-of-a-shield}
    Let $\L \subseteq (\Obs\times\Act)^\infty$.
    An agent is \emph{correct with respect to $\L$} 
    if any trace $\tau$ that is valid under $\Ag$ is in $\L$.
    A pre-shield-ready policy $\Ag\colon(\Obs\times\Act)^*\times \Obs\times 2^\Act\to\D(\Act)$ is correct if the regular policy induced by $\Ag$ is correct.
    A (pre- or post-) shield is \emph{correct} with respect to  $\L$ if any shielded policy $\Ag_{\shield}$ is correct with respect to $\L$.
\end{definition}

A correct shield always exists as long as a correct agent exists, 
since we can build a shield that strictly follows a correct agent,
using the \emph{induced shield} construction from the previous section.
Therefore, if we pick $\Ag$ to be correct with respect to a specification $\L \subseteq (\Obs\times\Act)^\infty$, 
then any agent shielded by $\shield^{pre}_{\Ag}$ or
$\shield^{pos}_{\Ag,\Ag_{det}}$ will also be correct with respect to $\L$.

While it is good to know that a correct shield always exists,
it is not very useful to build
a shield that is trivially correct by not caring
about the underlying agent.

\subsection{Interference}

Another desirable property in shields is that they minimize their interference with the agent being shielded. 
Intuitively, a post-shield interferes with its agent every time that 
it overwrites the agent's action.
For pre-shields the intuition is a bit different. 
A pre-shield interferes with an agent every time that 
the action the agent would take if it had no restrictions
is not allowed by the shield.
We formalize these ideas in the following definition.

\begin{definition}[Interference set of a shield]
    \label{def:interferenceSet}
    Let $\Sh\colon (\Obs\times\Act)^*\times \Obs\to 2^\Act$ be a pre-shield, and $\Ag_{pre}\colon (\Obs\times\Act)^*\times\Obs\times 2^\Act \to \D(\Act)$ 
    be a pre-shield-ready agent.
    The \emph{interference set} of $\Sh$ applied to $Ag_{pre}$ is
    \[
    \Inter_{\Sh}(\Ag_{pre}) = \left\{ (\tau, o)\in (\Obs\times\Act)^*\times \Obs\::\: 
    \exists a\in \Supp\left(\Ag_{pre}(\tau, o, \Act)\right),\,
    a\notin  \Sh(\tau, o)
    \right\}.
    \]
    Let $\Sh\colon (\Obs\times\Act)^*\times \Obs\times\Act \to \Act$ 
    be a post-shield, 
    and $\Ag_{pos}\colon (\Obs\times\Act)^*\times \Obs\to\D(\Act)$ be an agent.
    The \emph{interference set} of $\Sh$ applied to $Ag_{pos}$ is
    \[
    \Inter_{\Sh}(\Ag_{pos}) = \left\{ (\tau, o)\in (\Obs\times\Act)^*\times \Obs\::\: 
    \exists a\in\Supp\left(\Ag_{pos}(\tau, o)\right),\,
    a \neq \Sh(\tau, o, a)
    \right\}.
    \]
\end{definition}

A non-interfering shield always exists, independent of whether there exist correct agents or not, 
we call it the \emph{transparent shield}.
It is built as follows.
The transparent post-shield is $\Sh^{pos}_{trans}$ defined for any
$(\tau,o,a)\in (\Obs\times\Act)^*\times \Obs\times\Act$ as 
$$\Sh_{trans}^{pos}(\tau, o, a) = a.$$
The transparent pre-shield is $\Sh^{pre}_{trans}$ defined for any
$(\tau,o)\in (\Obs\times\Act)^*\times \Obs$ as 
$$\Sh_{trans}^{pre}(\tau, o) = \Act.$$

As in the case of the trivially correct shield, this transparent shield is an interesting theoretical construct, but it is not of much utility,
because it is generally not correct.

\begin{definition}[Equivalence modulo interference]
    Let $\Sh$, and $\Sh'$ be two post-shields. 
    We say that $\Sh$ and $\Sh'$ are \emph{equivalent modulo interferences}, 
    denoted it by $\Sh\equiv_i \Sh'$, 
    if for all $(\tau,o, a)\in (\Obs\times\Act)^*\times\Obs\times \Act$,
    we have that $\Sh(\tau,o,a)= a$ if and only if 
    $\Sh'(\tau,o,a)= a$.
\end{definition}

Interference sets are useful because a shield is determined by its interference sets, 
and the interference set of a shield induced by an agent $\Ag$ on that same agent $\Ag$ is always empty.
We formalize these properties in the following result.

\begin{lemma}\label{lem:shield-induced-empty}
Let $\Env = (\Obs,\Act,\Trans)$ be an environment
and $\Pi$ be a set of agents.
The following are true.
\begin{enumerate}
    \item Let $\Sh, \Sh'$ be two pre-shields,
    $\Sh, \Sh'\subseteq \Sigma_\Pi$, and
    such that for any pre-shield-ready agent $\Ag\in\Pi$,
    we have $\Inter_{\Sh}(\Ag) = \Inter_{\Sh'}(\Ag)$. 
    Then $\Sh = \Sh'$.
    \item Let $\Sh, \Sh'$ be two post-shields,
    $\Sh, \Sh'\subseteq \Sigma_\Pi$, and
    such that for any agent $\Ag\in\Pi$,    
    we have $\Inter_{\Sh}(\Ag) = \Inter_{\Sh'}(\Ag)$. 
    Then $\Sh \equiv_i \Sh'$.
    \item Let $\Ag\colon(\Obs\times\Act)^*\times\Obs\times 2^\Act \to\D(\Act)$
    be a pre-shield-ready agent,
    $\Ag_{reg}$ be the regular agent induced by $\Ag$ and
    $\Sh^{pre}_\Ag$ be the shield induced by $\Ag_{reg}$.
    Then $\Inter_{\Sh^{pre}_\Ag}(\Ag) = \emptyset$.
    \item Let $\Ag\colon(\Obs\times\Act)^*\times\Obs\to\D(\Act)$
    be an agent and $\Sh^{pos}_{\Ag,\Ag_{det}}$ be any shield induced by $\Ag$.
    Then $\Inter_{\Sh^{pos}_{\Ag,\Ag_{det}}}(\Ag) = \emptyset$.
\end{enumerate}
\end{lemma}

\begin{proof}
    We argue all cases by contradiction.
    \begin{enumerate}
        \item Suppose $\Sh\neq \Sh'$, and let $(\tau, o)$ be an input such that
        $\Sh(\tau, o) \neq \Sh'(\tau,o)$. 
        Without loss of generality, we may assume there exists $a\in\Act$ such
        that $a\in \Sh(\tau, o)$ and $a\notin \Sh'(\tau,o)$. 
        Consider a pre-shield-ready agent $\Ag$ such that $\Supp\left(\Ag(\tau,o)\right)=\{a\}$.
        Since $a\in \Supp\left(\Ag(\tau,o)\right)=\{a\}$, but $a\notin \Sh'(\tau,o)$, we have that $(\tau, o)\in \Inter_{\Sh'}(\Ag)$.
        However, since $\Supp\left(\Ag(\tau,o)\right) \subseteq \Sh(\tau,o)$, 
        we have that $(\tau,o)\notin\Inter_{\Sh}(\Ag)$, 
        contradicting $\Inter_{\Sh}(\Ag) = \Inter_{\Sh'}(\Ag)$.
        \item Suppose $\Sh\not\equiv_i\Sh'$, and let $(\tau, o, a)$ be an input making them so.
        Without loss of generality, we may assume that $\Sh'(\tau,o,a)\neq a$,
        and $\Sh(\tau,o,a) = a$.
        Consider an agent $\Ag$ such that $\Ag(\tau, o) = a$.
        As in the previous point, this implies that $(\tau, o)\in \Inter_{\Sh'}(\Ag)$, but on the other hand $(\tau,o)\notin \Inter_{\Sh}(\Ag)$, 
        contradicting $\Inter_{\Sh}(\Ag) = \Inter_{\Sh'}(\Ag)$.
        \item  Suppose $(\tau, o)\in \Inter_{\Sh^{pre}_\Ag}(\Ag)$.
        Then there would exist $a\in \Supp\left( \Ag(\tau,o,\Act)  \right)$, such that $a\notin \Sh^{pre}_{\Ag}$.
        However, by definition, $\Sh^{pre}_{\Ag}(\tau, o) = \Supp\left(\Ag_{reg}(\tau,o)\right)$, and $\Ag_{reg}(\tau, o) = \Ag(\tau,o,\Act)$. Therefore, such $a$ cannot exist,
        contradicting $(\tau, o)\in \Inter_{\Sh^{pre}_\Ag}(\Ag)$.
        \item Suppose $(\tau, o)\in \Inter_{\Sh^{pos}_{\Ag,\Ag_{det}}}(\Ag)$.
        Then there would exist $a\in\Supp\left( \Ag(\tau, o)  \right)$,
        such that $a \neq \Sh^{pos}_\Ag(\tau, o, a)$.
        However, by definition, $\Sh^{pos}_{\Ag,\Ag_{det}}(\tau, o, a) = a$ for all $a\in \Supp\left( \Ag(\tau, o)  \right)$.
        Therefore, such $a$ cannot exist,
        finishing the proof.
    \end{enumerate}
\end{proof}

\subsection{Minimal Correctness}

The goals of generating correct traces and minimizing interference often conflict. As we have seen, one can construct a correct shield from a correct policy by simply overwriting any action proposed by the agent with the corresponding action from the correct policy. However, this approach results in significant interference for many agents. At the other extreme, a fully transparent shield imposes no interference at all, leaving interference sets empty.
An ideal shield balances these objectives, ensuring correctness while keeping interference sets as small as possible. Formally, we define this as follows.

\begin{definition}[Minimal correctness]
\label{def:minimal_correctness}
    Let $\Env=(\Obs,\Act,\Trans)$ be an environment
    and $\Pi$ be a set of agents.
    A shield $\Sh\in\Sigma_\Pi$ is \emph{minimally correct} 
    restricted to $\Pi$
    if it is correct and
    for all agent $\Ag\in\Pi$ (pre-shield-ready in case of pre-shield, regular in case of post-shield), and for all correct shield $\Sh^\prime\in\Sigma_\Pi$, 
    we have $\Inter_{\shield}(\Ag) \subseteq \Inter_{\shield^\prime}(\Ag)$.
\end{definition}

\begin{theorem}\label{thm:minimallycorrectshield}
    Let $\Env=(\Obs,\Act,\Trans)$ be an environment, 
    $\L\subseteq (\Obs\times\Act)^\infty$ be a specification,
    $\Pi$ be a set of agents,
    and $\Sh$ be a minimally correct shield.
    Then:
    \begin{enumerate}
        \item If $\Sh$ is a pre-shield,  $\Sh$ is unique.
        \item If $\Sh$ is a post-shield, $\Sh$ is unique modulo interferences.
        \item For any correct agent $\Ag\in\Pi$ (pre-shield-ready in case of pre-shield, regular in case of post-shield), we have 
    $\Inter_{\shield}(\Ag) = \emptyset$.
    \end{enumerate}
\end{theorem}

\begin{proof}
    Uniqueness is a consequence of Lemma~\ref{lem:shield-induced-empty}.
    Suppose $\Sh$ and $\Sh'$ are minimally correct. 
    Since $\Sh$ is minimally correct and $\Sh'$ is correct, we have for any agent $\Ag$ that 
    $\Inter_{\Sh}(\Ag) \subseteq \Inter_{\Sh'}(\Ag)$.
    And using that $\Sh'$ is minimally correct and $\Sh$ is correct, 
    we have that $\Inter_{\Sh'}(\Ag) \subseteq \Inter_{\Sh}(\Ag)$, 
    leading to $\Inter_{\Sh}(\Ag) = \Inter_{\Sh'}(\Ag)$.
    By Lemma~\ref{lem:shield-induced-empty}, this implies $\Sh = \Sh'$ for pre-shields and $\Sh\equiv_i\Sh'$ for post-shields.

    Since $\shield$ is minimally interfering, for any correct shield $\shield'$, 
    we have that $\Inter_{\shield}(\Ag) \subseteq \Inter_{\shield'}(\Ag)$.
    In particular, since $\Ag$ is correct, we can choose $\shield'$ to be the shield induced by $\Ag$. 
    By Lemma~\ref{lem:shield-induced-empty}, $\Inter_{\shield'}(\Ag) = \emptyset$.
    Therefore, $\Inter_{\shield'}(\Ag)$ is a subset of the empty set,
    so it can only be that $\Inter_{\shield'}(\Ag) = \emptyset$.
\end{proof}



While this will be our operational definition of a useful shield
for most problems,
we want to note that it is not without drawbacks. 
We explore now its two main drawbacks: non-existence and cost-independence. 

\paragraph*{Existence of minimally correct shields.}
The first drawback of this definition is that there are environments and specifications for which a minimally correct shield does not exist, 
even if correct agents do exist.
Consider an environment with a single possible observation $\Obs = \{o\}$, 
two actions $\Act = \{a,b\}$, and a specification 
$\L = \{ w\::\: w \mbox{ contains at least one } a\}$.
Given this specification, the transparent shield cannot be correct, 
since there exists the agent that only outputs $b$, 
which is not correct. 
Suppose $\shield$ is a correct post-shield. 
Then there exists $(\tau, o)\in (\Obs\times\Act)^*\times\Obs$ such that
$\shield(\tau,o, b) = \{a\}$.
Let $k = |\tau|$. Consider the agent $\Ag$ such that $\Ag(\tau',o) = b$ if $|\tau'|\leq k$ and $\Ag(\tau',o) = a$ if $|\tau'| > k$. 
Clearly $\Ag$ is a correct agent, however, $\Inter_{\shield}(\Ag) \neq \emptyset$.
By Theorem~\ref{thm:minimallycorrectshield}, $\shield$ cannot be minimally correct.
An analogous construction can be made for pre-shields.
In Section~\ref{sec:rdm-classical-shielding}, we explore some concrete cases where the existence of minimally correct shields is guaranteed.

\paragraph*{Uniform cost.}
The second drawback is that all interferences are given the same importance. 
Depending on the problem, it may be useful to assign a numerical cost to each intervention and ask which is the shield that guarantees certain correctness properties while having minimal cost. 
We explore this in the context of fairness shields in Chapter~\ref{chap:fairness}.

\section[Classical Shielding]{Classical Shielding through the Lens of the Reactive Decision-Making Framework}
\label{sec:rdm-classical-shielding}

Shielding has been successfully used for specifications of the safety type
in deterministic two-player games and in MDPs. 
In this section, we explain how these methods can be seen in our framework.

\subsection{Shielding in Safety Games with Perfect Information}
\label{sec:shielding-in-safety-games}
\index{shield!in safety games}

Shielding was first introduced in safety games~\cite{DBLP:conf/tacas/BloemKKW15}.
Recall (Section~\ref{sec:rdm-two-player-games}) how safety games can be regarded in the reactive decision-making framework as memoryless environments with a uniform transition function.
Also, when considering the perfect information regime, i.e., no delayed observations, memoryless strategies are enough to define any winning strategy.

Following the general definition (Definition~\ref{def:preshield-abstract}), 
and adapting to the fact that transitions in a safety game are memoryless, 
a pre-shield is a function $\Sh\colon S_{ag}\to 2^\Act$. 
Given a strategy of the safety game $\xi\colon S_{ag}\to 2^\Act$, 
we call $\Sh^{pre}_\xi = \xi$ the pre-shield induced by the strategy $\xi$.

Given a strategy $\xi\colon S_{ag}\to 2^\Act$ and a deterministic strategy
$\chi\colon S_{ag}\to \Act$, we call $\Sh^{pos}_{\xi,\chi}$ the post-shield induced by the pair of strategies $(\xi,\chi)$, defined as:
\[
\Sh^{pos}_{\xi,\chi}(s, a) = \begin{cases}
    a & \mbox{ if } a\in \xi(s) \\
    \chi(s) & \mbox{ otherwise}.
\end{cases}
\]

A pre-shield can be added before an output-restricted strategy $\xi\colon S_{ag}\times 2^\Act \to 2^\Act$ to generate a regular strategy $\xi_{\shield}\colon S_{ag}\to 2^\Act$.
%
Similarly, a post-shield is a function $\Sh\colon S_{ag}\times \Act\to \Act$.
A post-shield can be added after a regular strategy $\xi\colon S_{ag}\to 2^{\Act}$
to generate a new regular strategy $\xi_{\shield} \colon S_{ag}\to 2^\Act$.

Note that, by definition, if $\xi$ is a winning strategy of $\xi$, 
then $\Sh^{pre}_{\xi}$ is a correct pre-shield. 
Similarly, if $\xi$ is a winning strategy and $\chi$ is a determinization of $\xi$, 
then $\Sh_{\xi, \chi}^{pos}$ is a correct post-shield.

\begin{theorem}\label{thm:shield-safety-game}
    Let $\G = (S, s_0, S_{env}, S_{ag}, \Act, \T, \F)$ be a safety game with a winning strategy.
    Let $\xi$ be the maximally permissive winning strategy of $\G$.
    Then:
    \begin{enumerate}
        \item The minimally correct pre-shield exists and is $\Sh_\xi^{pre}$.
        \item For any deterministic winning strategy $\chi$, 
        the shield $\Sh_{\xi, \chi}^{pos}$ is minimally correct.
    \end{enumerate}
\end{theorem}

\begin{proof}
    \textbf{(1.)}
    We argue the first point by contradiction. 
    Since $\xi$ is a winning strategy, $\Sh_\xi^{pre}$ is correct by construction, so we only have to argue the minimality property.
    Suppose that $\Sh_\xi^{pre}$ is not minimally correct.
    Then there exists a pre-shield-ready agent $\Ag\colon S_{ag}\times 2^\Act\to 2^\Act$, 
    a correct pre-shield $\Sh\colon S_{ag}\to 2^\Act$, and 
    $s\in S_{ag}$ such that $s\in \Inter_{\Sh_\xi^{pre}}(\Ag)$, 
    but $s\notin\Inter_{\Sh}(\Ag)$.
    Since $s\notin \Inter_{\Sh}(\Ag)$, 
    then $\Ag(s, \Act)\subseteq \Sh(s)$. 
    Since $s\in \Inter_{\shield^{pre}_\xi}(\Ag)$, 
    then there exists $a\in \A$ such that $a\in \Ag(s,\Act)\subseteq \Sh(s)$, 
    but $a\notin \Sh_\xi^{pre}(s)$.
    Since $\Sh_\xi^{pre}$ is implemented with the maximally permissive winning strategy, by definition (recall Equation~\eqref{eq:max-perm-strat-safety-game}) we have $s\xrightarrow{a}s'$, with $s'\notin W$, 
    where $W$ is the winning region of $\G$.
    
    On the other hand, consider an pre-shield-ready agent $\Ag'$ such that $\Ag'(s, \Sh(s)) = \{a\}$,
    which exists since $a\in \Sh(s)$.
    Let $\Ag'_{\Sh}$ be the agent resulting from applying $\Sh$ on $\Ag'$.
    This agent is correct because $\Sh$ is correct, 
    and by construction, $\Ag'_{\Sh}(s) = a$.
    But then $s'$ would be part of a valid trace under a winning strategy, 
    so $s'\in W$. This is a contradiction, as we had previously established that $s'\notin W$.

    \textbf{(2.)} As for the second point, the shield is correct since $\xi$ is winning and $\chi$ is a determinization of $\xi$, 
    and we will use a similar argument to prove minimality by contradiction.
    Suppose $\Sh_{\xi,\chi}^{pos}$ is not minimally correct.
    Then there exists an agent $\Ag\colon S_{ag}\to 2^\Act$, 
    a correct post-shield $\Sh\colon S_{ag}\times\Act\to\Act$
    and $s\in S_{ag}$ such that 
    $s\in \Inter_{\Sh_{\xi,\chi}^{pos}}(\Ag)$, 
    but $s\notin\Inter_{\Sh}(\Ag)$.

    Since $s\notin \Inter_{\Sh}(\Ag)$, 
    then for all $a\in \Ag(s)$, we have $\Sh(s,a) = a$.
    Since $s\in \Inter_{\shield^{pre}_\xi}(\Ag)$, 
    then there exists $a\in \Ag(s)$ such that $\Sh^{pos}_{\xi,\chi}(s,a)\neq a$.
    Since $\Sh_{\xi,\chi}^{pos}$ is implemented with the maximally permissive winning strategy $\xi$,
    and a determinization of it $\chi$,
    it means that $a\notin\xi(s)$,
    and thus $s'$ defined by $s\xrightarrow{a}s'$ is not in the winning region $W$.

    On the other hand, consider an agent $\Ag'\colon S_{ag}\to 2^\A$ satisfying 
    $\Ag(s) = a$, and consider $\Ag'_{\Sh}$, that same agent shielded with $\Sh$.
    We know that $\Ag'_{\Sh}$ is correct, because $\Sh$ is correct.
    Since $\Sh(s,a) = a$, we have that $\Ag'_{\Sh}(s) = a$.
    The proof finishes by noting that in such case $s'\in W$,
    contradicting $s'\notin W$. 
\end{proof}

While the definition of minimally correct shields with interference sets is original to this work, 
the construction of shields using the maximally permissive strategy of the 
underlying safety game is the same as has been described in 
the original shielding literature~\cite{DBLP:conf/tacas/BloemKKW15,konighofer2017shield,AlshiekhBEKNT18,koenighofer2019}.

The definition presented in this work captures corner cases that previous definitions did not, as we show in the following example.

\begin{figure}
    \centering
    \includegraphics[width=0.7\linewidth,page=5]{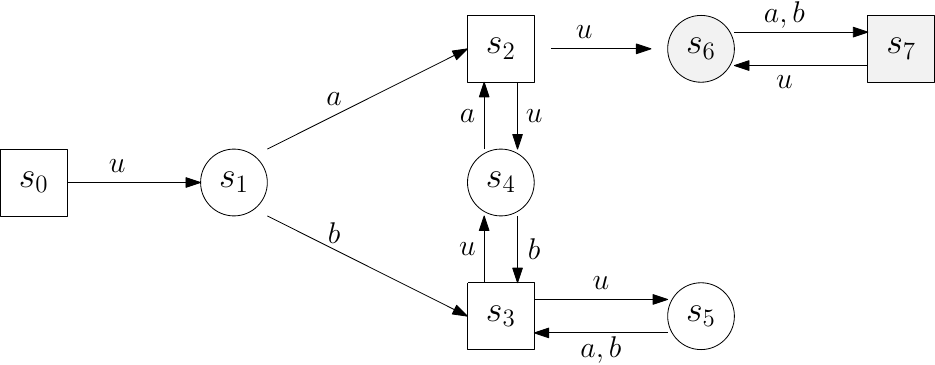}
    \caption{Safety game illustrating Example~\ref{ex:safety-game-counterexample}.}
    \label{fig:safety-game-counterexample}
\end{figure}

\begin{example}\label{ex:safety-game-counterexample}
    Consider the safety game $\G = (S, s_0, S_{env}, S_{ag}, \Act, \T, \F)$ illustrated in Figure~\ref{fig:safety-game-counterexample}, 
    where the $S_{env} = \{ s_0, s_2, s_3, s_7 \}$, 
    $S_{ag} = \{ s_1, s_4, s_5, s_6 \}$, 
    $\A = \{a,b\}$, $\F = S\setminus \{s_6, s_7\}$, and $\T$ is as described in the figure.
    The winning region of this safety game is $W = \{ s_0, s_1, s_3, s_4, s_5 \}$, 
    and therefore no trace containing $s_2$ would be allowed by a shielded agent. 
    However, according to the definition of shield as minimally interfering in previous work~\cite[Def. 1]{konighofer2017shield}, 
    a trace $\tau = s_0 s_1 s_2 s_4 (s_3 s_5)^\omega$
    should be allowed by a shielded agent since it does leave $\F$ at any time.
\end{example}

\subsection{Shielding in Safety Games with Delayed Observations}
\label{sec:shielding-in-safety-games-delayed}
\index{shield!in safety games with delayed observation}

When considering games under delay,
we need to be aware that memoryless strategies are not enough, 
as discussed in Section~\ref{sec:prelim-games-under-delay}.

From a reactive decision-making framework point of view, 
the correspondence between a game $\G_{\delta,\mu} = 
\langle S, s_0, S_{env}, S_{ag}, \Act, \mathcal T, \Acc, \delta, \mu \rangle$, and an environment $\Env = (\Obs,\Act, \Trans)$ is the same as the correspondence described in the previous section for games with $\delta=\mu=0$. 

The only relevant difference is that we consider only agents restricted to the set of agents agnostic to the last $\delta$ observations, and with a restricted memory $\mu$.
Therefore, the shields we consider are going to have the same restrictions. 

In Section~\ref{sec:rdm-delayed-obs} we have defined, 
for a given environment $\Env = (\Obs,\Act,\Trans)$, 
the set of agents $\Pi^{\myclock}_\delta$ as those agnostic to the last $\delta$ observations, and we have explained how restricting to this set of agents is equivalent to considering agents that work with delayed observations.

A characteristic of safety games is that the transition relation $\T$ depends only on the state -- and not the trace leading to that state.
Therefore, the agents relevant for solving safety games under delay $\delta$ are only not agnostic to the $\delta+1$-th observation, counting from the tail. 
Formally, mirroring Definition~\ref{def:delayed-obs-agent},
an agent $\Ag=(\Obs,\Act,\Pol)$ works only with the last $\delta+1$-th observation if 
for all $\tau,\tau'\in (\Obs\times\Act)^*$, all $o\in\Obs$, 
all $(a_0,a_1,\dots, a_\delta)\in\Act^\delta$ and all $(o_1,\dots,o_\delta), (o'_1,\dots, o'_\delta)\in\Obs^\delta$, we have 
\begin{equation}
    \label{eq:def-delayed-obs-agent-full-mem-safety-game}
        \Pol\Big(\big(\tau \cdot (o, a_0), (o_1,a_1),\dots, (o_{\delta-1}, a_{\delta-1})\big), o_\delta\Big) =     
        \Pol\Big(\big(\tau'\cdot (o,a_0),(o_1',a_1),\dots, (o_{\delta-1}', a_{\delta-1})\big), o_\delta'\Big).   
\end{equation}

Following the same equivalence as in Section~\ref{sec:rmd-delayed-equivalence}, 
the agents satisfying Equation~\eqref{eq:def-delayed-obs-agent-full-mem-safety-game} are equivalent to agents working on the domain $\A^{\leq\delta}\cup (\Obs\times\A^\delta)$.
Introducing a restriction on memory, as described for safety games in Section~\ref{sec:prelim-games-under-delay} is also modelled as a restriction to agents agnostic to certain parts of the input. 
Concretely, the part of the input that corresponds to actions that overflow the memory. 

Putting these two concepts together, we have the following definition.

\begin{definition}
    Let $\G_{\delta,\mu} = 
    \langle S, s_0, S_{env}, S_{ag}, \Act, \mathcal T, \Acc\rangle$ be a two player game and let $\Env = (\Obs, \Act, \Trans)$ be its corresponding environment.
    Let $\delta\in\NN$ and $\mu\leq \delta$ be two integers representing delay and memory.
    An agent $\Ag = (\Obs,\Act, \Pol)$ \emph{works in the safety game with memory $\mu$ and observations delayed by $\delta$} if
    for all $\tau,\tau'\in (\Obs\times\Act)^*$, all $(a_{\delta-\mu}, \dots, a_{\delta-1})\in\Act^\mu$, all $(a_0,\dots, a_{\delta-\mu-1}), (a'_0,\dots, a'_{\delta-\mu-1})\in \Act^{\delta-\mu}$, all $o\in \Obs$, and all $(o_1,\dots, o_\delta), (o'_1,\dots, o'_\delta)\in\Obs^\delta$, we have:
    \begin{align}
        \Pol\Big(\big(\tau \cdot (o, a_0), (o_1,a_1),\dots,
        (o_{\delta-\mu-1}, a_{\delta-\mu-1}), 
        (o_{\delta-\mu}, a_{\delta-\mu}), \dots,
        (o_{\delta-1}, a_{\delta-1})\big), o_\delta\Big) = \nonumber \\
        \Pol\Big(\big(\tau' \cdot (o, a'_0), (o'_1,a'_1),\dots,
        (o'_{\delta-\mu-1}, a'_{\delta-\mu-1}), 
        (o'_{\delta-\mu}, a_{\delta-\mu}), \dots,
        (o'_{\delta-1}, a_{\delta-1})\big), o'_\delta\Big) 
    \end{align}
    The set of agents in this regime is denoted as $\Pi^{\myclock}_{\delta,\mu}$.
\end{definition}

Again, following the same equivalence as in Section~\ref{sec:rmd-delayed-equivalence}, 
the agents in $\Pi^{\myclock}_{\delta,\mu}$ can be characterized as agents of the form $\Pol\colon \Act^\mu\cup (\Obs\times\Act^\mu)\to 2^{\Act}$, 
which is the same form as that of strategies in games with delay $\delta$ and memory $\mu$ (Section~\ref{sec:prelim-games-under-delay}).

This serves as the basis for the result analogous to Theorem~\ref{thm:shield-safety-game} for games with delayed observations.

\begin{theorem}\label{thm:shield-delayed-safety-game}
    Let $\G = (S, s_0, S_{env}, S_{ag}, \Act, \T, \F,\delta,\mu)$ be a safety game under delayed observation $\delta$ with a winning strategy with memory $\mu$.
    Let $\xi\colon S_{ag*}\times \Act^{\leq \mu}\to 2^\Act$ be the maximally permissive winning strategy of $\G$.
    Then, restricting to the set of agents $\Pi^{\myclock}_{\delta,\mu}$, we have that:
    \begin{enumerate}
        \item The minimally correct pre-shield exists and is $\Sh_\xi^{pre}\in\Sigma_{\Pi^{\myclock}_{\delta,\mu}}$.
        \item For any deterministic winning strategy $\chi$, 
        the shield $\Sh_{\xi, \chi}^{pos}\in \Sigma_{\Pi^{\myclock}_{\delta,\mu}}$ is minimally correct.
    \end{enumerate}
\end{theorem}


The proof of Theorem~\ref{thm:shield-delayed-safety-game} follows the same argument as the proof of the analogous theorem for regular safety games (Theorem~\ref{thm:shield-safety-game}), 
since games under delay are equivalent to regular games with exponentially many states~\cite[Lemma 2]{Chen2020IndecisionAD}.
In any case, we include it here for the sake of completeness.


\begin{proof}
    \textbf{(1.)}
    We argue the first point by contradiction. 
    Since $\xi$ is a winning strategy, $\Sh_\xi^{pre}$ is correct by construction, so we only have to argue the minimality property.
    Suppose that $\Sh_\xi^{pre}$ is not minimally correct.
    Then there exists a pre-shield-ready agent 
    $\Ag\colon S_{ag*}\times \Act^{\leq \mu}\times 2^\Act\to 2^\Act$, 
    a correct pre-shield $\Sh\colon S_{ag*}\times\Act^{\leq \mu}\to 2^\Act$, and 
    $(s,\ol{\sigma})\in S_{ag*}\times\Act^{\leq\mu}$ such that $(s,\ol{\sigma})\in \Inter_{\Sh_\xi^{pre}}(\Ag)$, 
    but $(s,\ol{\sigma})\notin\Inter_{\Sh}(\Ag)$.
    Assuming
    \footnote{We show at the end of the proof how to treat the case $s=\eps$.}
    $s\neq \eps$, we have 
    $\ol{\sigma}\in\Act^{\mu}$, thus it is of the form $\ol{\sigma} = (\sigma_{\delta-\mu+1},\dots, \sigma_\delta)$.
    Since $(s,\ol{\sigma})\notin \Inter_{\Sh}(\Ag)$, 
    then $\Ag(s, \ol{\sigma},  \Act)\subseteq \Sh(s)$. 
    Since $(s,\ol{\sigma})\in \Inter_{\shield^{pre}_\xi}(\Ag)$, 
    then there exists $a\in \A$ such that $a\in \Ag(s,\ol{\sigma},\Act)\subseteq \Sh(s)$, 
    but $a\notin \Sh_\xi^{pre}(s,\ol{\sigma})$.
    Since $\Sh_\xi^{pre}$ is implemented with the maximally permissive winning strategy, 
    this means that there exists $s'_1,\dots, s'_{2\delta+1}\in S$,
    and $(\sigma_{1},\dots,\sigma_{\delta-\mu})\in\Act^{\delta-\mu}$, 
    such that
    \begin{equation}
    \label{eq:proof-delayed-shield-games1}
        s\xrightarrow{\sigma_1} s'_1\xrightarrow{u} s'_2 \xrightarrow{\sigma_2}\dots\xrightarrow{\sigma_\delta} s'_{2\delta-1} \xrightarrow{u} s'_{2\delta} \xrightarrow{a} s'_{2\delta+1},
    \end{equation}
    and such that $(s'_2,\ol{\sigma}')$ is not part of any winning strategy, 
    where $\ol{\sigma}' = (\sigma_{\delta-\mu+2},\dots, \sigma_{\delta}, a)\in\Act^\mu$.    
    
    
    On the other hand, consider a pre-shield-ready agent 
    $\Ag'$ such that $\Ag'(s, \ol{\sigma}, \Sh(s,\ol{\sigma})) = \{a\}$,
    which exists since $a\in \Sh(s,\ol{\sigma})$.
    Let $\Ag'_{\Sh}$ be the agent resulting from applying $\Sh$ on $\Ag'$.
    This agent is correct because $\Sh$ is correct, 
    and by construction, $\Ag'_{\Sh}(s,\ol{\sigma}) = a$.
    But then $s'_2,\ol{\sigma}'$ would be part of a valid trace under a winning strategy. This is a contradiction, as we had previously established that $(s'_2,\ol{\sigma}')$ cannot be part of any winning strategy.

    \textbf{(2.)} The shield is correct since $\xi$ is winning and $\chi$ is a determinization of $\xi$, 
    and we use a similar argument to prove minimality by contradiction.
    Suppose $\Sh_{\xi,\chi}^{pos}$ is not minimally correct.
    Then there exists an agent $\Ag\colon S_{ag*}\times\Act^{\leq\mu}\to 2^\Act$, 
    a correct post-shield $\Sh\colon S_{ag*}\times\Act^{\leq\mu} \times\Act\to\Act$
    and ($s,\ol{\sigma})\in S_{ag*}\times\Act^{\leq\mu}$ such that 
    $(s,\ol{\sigma})\in \Inter_{\Sh_{\xi,\chi}^{pos}}(\Ag)$, 
    but $(s,\ol{\sigma})\notin\Inter_{\Sh}(\Ag)$.

    Since $(s, \ol{\sigma})\notin \Inter_{\Sh}(\Ag)$, 
    then for all $a\in \Ag(s,\ol{\sigma})$, we have $\Sh(s,\ol{\sigma},a) = a$.
    Since $(s,\ol{\sigma})\in \Inter_{\shield^{pre}_\xi}(\Ag)$, 
    then there exists $a\in \Ag(s,\ol{\sigma})$ such that $\Sh^{pos}_{\xi,\chi}(s,\ol{\sigma}, a)\neq a$.
    Since $\Sh_{\xi,\chi}^{pos}$ is implemented with the maximally permissive winning strategy $\xi$,
    and a determinization of it $\chi$,
    it means that $a\notin\xi(s, \ol{\sigma})$,
    and thus $(s'_2,\ol{\sigma}')$ defined by 
    the same procedure as Equation~\eqref{eq:proof-delayed-shield-games1}
    cannot be part of any winning strategy.
    To build $(s'_2,\ol{\sigma})'$ we need to assume again that $s \neq \eps$.

    On the other hand, consider an agent $\Ag'\colon S_{ag*}\times\Act^{\leq\mu}\to 2^\A$ satisfying 
    $\Ag(s,\ol{\sigma}) = a$, and consider $\Ag'_{\Sh}$, that same agent shielded with $\Sh$.
    We know that $\Ag'_{\Sh}$ is correct, because $\Sh$ is correct.
    Since $\Sh(s,\ol{\sigma}, a) = a$, we have that $\Ag'_{\Sh}(s,\ol{\sigma}) = a$.
    The proof finishes by noting that in such case
    $(s'_2,\ol{\sigma}')$ as obtained in Equation~\eqref{eq:proof-delayed-shield-games1} is part of a valid trace under a winning strategy, 
    contradicting the previously established point. 

    \textbf{Initial phase.}
    In both proofs we have assumed that $s\neq\eps$, and therefore $\sigma\in\Act^\mu$ in order to obtain the construction in Equation~\eqref{eq:proof-delayed-shield-games1}. 
    The case for $s=\eps$ follows the same argument, only considering that $\sigma = (\sigma_1,\dots, \sigma_\nu)$ for some $\nu\leq \mu$ and thus fewer transitions in Equation~\eqref{eq:proof-delayed-shield-games1}.
\end{proof}

\subsection{Probabilistic Shielding in Markov Decision Processes}
\index{shield!probabilistic shield}
\label{sec:rdm-probabilistic-shielding-mdps}
Probabilistic safety shielding in MDPs was introduced in~\cite{0001KJSB20}.
We present probabilistic safety shielding adapting the definitions in~\cite{0001KJSB20} to our framework.
The main difference of shielding in MDPs is that the safety specification has 
a probabilistic nature. A probabilistic shield blocks an action when the probability of the action causing harm is larger than some threshold $\lambda$.

\subsubsection*{Property specification.}

Let $\M = (\S,\A, \P)$ be an MDP, 
$T\subseteq \S$ be a set of ``unsafe'' states to avoid,
$k\in\NN\cup\{\infty\}$ a step horizon,
and $\lambda \in [0,1]$.
The language specifying correct traces will be denoted $\L_{T,\lambda,k}$.
A trace $\tau = (s_0,a_0, s_1, a_1, \dots)$ is in $\L_{T,\lambda,k}$
if for every $i\geq 0$, we have 
\begin{equation}\label{eq:property-of-prob-shields}
 \PP^\M_{\max}\left(\Avoid_k(s_i, a_i, T)\right) \geq \lambda\cdot
 \PP^\M_{\max}\left(\Avoid_k(s_i, T)\right).  
\end{equation}

This means that for an action $a_i$ to be safe at state $s_i$,
the probability of not reaching a bad state if the agent behaves ``optimally''
has to be at least $\lambda$ times the probability of an agent that is optimal in avoiding bad states.

For example, if $\lambda = 1/2$, 
and the policy that best avoids $T$ reaches it with a $10\%$ probability, 
any action from which $T$ can be avoided with a $20\%$ probability is considered to be ``safe enough''.
In general, the larger the value of $\lambda$, 
the more restrictive or cautious the shield is. 
In the extremes, when $\lambda = 0$, any action is allowed, 
and when $\lambda =1$, only the safest actions are allowed, i.e., the actions for which $\PP^\M_{\max}\left(\Avoid_{\leq k}(s_i, a_i, T)\right) =
\PP^\M_{\max}\left(\Avoid_{\leq k}(s_i, T)\right)$.

An agent $\Pol\colon \S\to\D(\A)$ is correct with respect to a specificaiton $\L_{T,\lambda, k}$ if any valid trace of $\Pol$ is in $\L_{T,\lambda,k}$.
Following the same notation convention as in Section~\ref{sec:prelim-reachability-properties}, 
when we are considering unbounded properties, i.e., when $k=\infty$,
we may drop the $k$ from our notation, writing the specification as $\L_{T,\lambda}$
instead of $\L_{T,\lambda, \infty}$.

\subsubsection*{Shield synthesis.}

Since the environment and the safety specification are defined in a memoryless manner, we can also consider shields as memoryless.
We define the shields induced by a specification $\L_{T,\lambda, k}$ as one would expect.

A pre-shield in an MDP is a function $\Sh\colon \S \to 2^\Act$
A pre-shield-ready agent in an MDP is a function $\Pol\colon \S\times 2^\Act\to \D(\Act)$.
A post-shield in an MDP is a function $\Sh\colon \S\times\Act\to\Act$.

Let $T\subseteq\S$, $k\in\NN\cup\{\infty\}$, and $\lambda\in[0,1]$ defining a probabilistic safety specification $\L_{T,\lambda, k}$. 
The pre-shield induced by $T$, $k$, and $\lambda$ is 
$\Sh_{T,\lambda}^{pre}$ defined as:
\[
\Sh_{T,\lambda,k}^{pre}(s) = \left\{ a\in \Act\::\:  \PP^\M_{\max}\left(\Avoid_{\leq k}(s_i, a_i, T)\right) \geq \lambda\cdot
 \PP^\M_{\max}\left(\Avoid_{\leq k}(s_i, T)\right) \right\}.
\]
Similarly, let $\Pol\colon \S\to \D(\Act)$ be a correct agent. 
The post-shield induced by $T$, $\lambda$, and $\Pol$ is $\Sh_{T,\lambda, \Pol}^{pos}$ defined as
\[
\Sh_{T,\lambda, k, \Pol}^{pos}(s,a) = \begin{cases}
    a & \mbox{ if } \quad \PP^\M_{\max}\left(\Avoid_{\leq k}(s_i, a_i, T)\right) \geq \lambda\cdot
 \PP^\M_{\max}\left(\Avoid_{\leq k}(s_i, T)\right) \\
    \Pol(s) & \mbox{ otherwise. }
\end{cases}
\]

\begin{theorem}
    Let $\M = (\S, \A, \P)$ be an MDP. 
    Let $T\subseteq \S$, $k\in\NN\cup\{\infty\}$, and $\lambda\in[0,1]$ forming a probabilistic safety specification $\L_{T,\lambda, k}$.
    Then:
    \begin{enumerate}
        \item The minimally correct pre-shield exists and is $\Sh^{pre}_{T,\lambda,k}$.
        \item For any correct agent $\Pol$, 
        the shield $\Sh_{T,\lambda,k,\Pol}^{pos}$ is minimally correct.
    \end{enumerate}
\end{theorem}
\begin{proof}
Correctness should be clear by construction in both cases, so we only need to argue for minimality.
The arguments to prove minimality are similar to the arguments used to prove Theorem~\ref{thm:shield-safety-game}.

\textbf{(1.)}
    We argue the first point by contradiction. 
    Assume that $\Sh^{pre}_{T,\lambda,k}$ is not minimally correct. 
    Then there exists a pre-shield-ready agent $\Ag\colon \S\times 2^\Act\to\D(\Act)$ and a correct pre-shield $\Sh$ such that 
    $\Inter_{\Sh_{T,\lambda,k}^{pre}}(\Ag) \not\subseteq \Inter_{\Sh}(\Ag)$. 
    Therefore, there exists $s\in \S$ such that $s\in \Inter_{\Sh_{T,\lambda,k}^{pre}}(\Ag)$ but 
    $s\notin \Inter_{\Sh}(\Ag)$. 
    Since $s\in \Inter_{\Sh_{T,\lambda}^{pre}}$, 
    there exists $a\in \Supp(\Ag(s,\Act))$ with $a\notin \Sh_{T,\lambda,k}^{pre}(s)$. 
    By the definition of $\Sh_{T,\lambda,k}^{pre}$,
    the action $a$ is not in the allowed actions of the shield only if
    \begin{equation}\label{eq:aux1}
      \PP^\M_{\max}\left(\Avoid_{\leq k}(s_i, a_i, T)\right) < \lambda\cdot
 \PP^\M_{\max}\left(\Avoid_{\leq k}(s_i, T)\right).     
    \end{equation}
    
    On the other hand, since $s\notin \Inter_{\Sh}(\Ag)$, 
    it means that $\Supp(\Ag(s,\Act))\subseteq \Sh(s)$.
    In particular, $a\in \Sh(s)$.
    Consider a pre-shield-ready agent $\Ag'$ 
    such that $\Ag'(s,\shield(s)) = \{a\}$.
    Let $\Ag'_{\Sh}$ be the shielded agent resulting from applying $\Sh$ to $\Ag'$.
    This is a correct agent such that $\Ag'(s) = a$. 
    But then the fragment $(s,a)$ would be part of a correct trace, 
    meaning that 
    \[
   \PP^\M_{\max}\left(\Avoid_{\leq k}(s_i, a_i, T)\right) \geq \lambda\cdot
 \PP^\M_{\max}\left(\Avoid_{\leq k}(s_i, T)\right),
    \]
    which contradicts Equation~\eqref{eq:aux1}.

    \textbf{(2.)} We also argue the second point by contradiction. 
    Suppose there exists $\Pol\colon \S\to \D(\Act)$ correct such that 
    $\Sh^{pos}_{T,\lambda, k, \Pol}$ is not minimally correct.
    Then there exists an agent $\Ag\colon \S\to \D(\Act)$
    and a correct post-shield $\Sh$ such that 
    $\Inter_{\Sh^{pos}_{T,\lambda,k, \Pol}}(\Ag)\not\subseteq \Inter_{\Sh}(\Ag)$.
    Therefore, there exists $s\in\S$ such that
    $s\in \Inter_{\Sh_{T,\lambda,k,\Pol}^{pos}}(\Ag)$ but 
    $s\notin \Inter_{\Sh}(\Ag)$. 

    Since $s\notin \Inter_{\Sh}(\Ag)$, 
    it means that for all $a\in \Supp(\Ag(s))$, we have $\Sh(s,a) = a$.
    Since $s\in \Inter_{\Sh_{T,\lambda,k,\Pol}^{pos}}$, 
    there exists $a\in \Supp(\Ag(s))$ with $a\neq \Sh_{T,\lambda,k,\Pol}^{pos}(s,a)$. In particular, since $a\in\Supp(\Ag(s))$, we have $\Sh(s,a) = a$.
    This means that 
        \begin{equation}\label{eq:aux2}
        \PP^\M_{\max}\left(\Avoid_{\leq k}(s_i, a_i, T)\right) < \lambda\cdot
 \PP^\M_{\max}\left(\Avoid_{\leq k}(s_i, T)\right).     
    \end{equation}
    
    On the other hand, 
    consider an agent $\Ag'\colon\S\to\D(\A)$ such that 
    $\Ag'(s)(a) > \eps$ for some $\eps > 0$, i.e., an agent that from $s$ outputs $a$ with a positive probability.
    The shielded version of $\Ag'$, denoted by $\Ag'_{\Sh}$ is a correct agent, 
    and $\Ag'_{\Sh}(s) = a$ with probability $\eps$. 
    Therefore, $(s,a)$ is a fragment contained in valid traces of a correct agent, and thus satisfies
    \[
   \PP^\M_{\max}\left(\Avoid_{\leq k}(s_i, a_i, T)\right) \geq \lambda\cdot
 \PP^\M_{\max}\left(\Avoid_{\leq k}(s_i, T)\right),
    \]
    which contradicts Equation~\eqref{eq:aux2}.
\end{proof}

While this is our operational definition of a probabilistic shield and the one we will be using in Chapter~\ref{chap:foceta}, 
there are certain variations that have been proposed in the literature

\subsubsection*{Absolute threshold}
The value of $\lambda$ in Equation~\eqref{eq:property-of-prob-shields} is considered a \emph{relative threshold},
as it states what the minimum probability of reaching an unsafe state can be relative to the best choice of action.
An alternative that has been studied in the literature~\cite{tempest} is to use $\lambda$ as an absolute minimum threshold on the probability of not reaching a bad state. 
In such shields, the condition to let an action pass is
\begin{equation}\label{eq:property-of-prob-shields-absolute}
\PP^\M_{\max}\left(\Avoid_{\leq k}(s_i, a_i, T)\right) \geq \lambda.
\end{equation}

Note that if an action satisfies Equation~\eqref{eq:property-of-prob-shields-absolute},
then it also satisfies Equation~\ref{eq:property-of-prob-shields}, making shields with absolute threshold more restrictive than shields with relative threshold.

This choice has the advantage of being properly restrictive in the more ``critical'' states. 
For example, consider a threshold $\lambda = 0.6$.
In a state $s\in\S$ with 
$\PP^\M_{\max}\left(\Avoid_{\leq k}(s, a, T)\right)=0.75$,
the shield with an absolute threshold would only let actions pass with an optimal probability of avoiding an unsafe state between $0.6$ and $0.75$. 
On the other hand, the shield with $\lambda$ as a relative threshold would allow actions with the optimal probability of avoiding an unsafe state as low as $0.45$.

The main downside of the ``absolute'' approach is that in some states there may not be any ``allowed'' action. In such states, one should just resort to the optimal actions (that would always be allowed with a relative threshold approach).

\subsubsection*{Global guarantees}

The property defined in Equation~\eqref{eq:property-of-prob-shields} is local for every decision.
A desirable property would be that a shield $\Sh$ synthesized for a specification $\L_{T,\lambda,k}$ 
would satisfy for every agent $\Ag$ and any initial state of the MDP $s$, that we have
$\PP^\M_{\Ag_{\Sh}}(\Avoid_{\leq k}(s, T)) \geq \lambda$.

This is not the case. In fact, the following example shows that we can make $\PP^\M_{\Ag_{\Sh}}(\Avoid_{\leq k}(s, T))$ arbitrarily small while keeping $\lambda$ arbitrarily large.

\begin{figure}[t]
    \centering
    \includegraphics[width=0.5\linewidth,page=1]{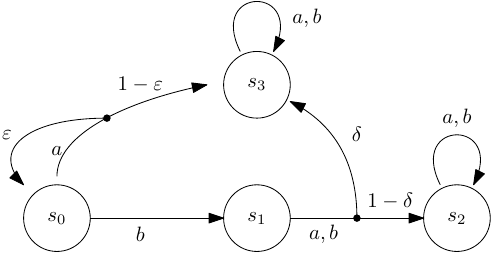}
    \caption{MDP described in Example~\ref{ex:MDP-shielding-counter}.}
    \label{fig:prob-shield-counterexample}
\end{figure}

\begin{example}\label{ex:MDP-shielding-counter}
   Let $\eps, \delta \in (0,1)$.
   Consider the MDP $\M = (\S, \A,\P)$ illustrated in Figure~\ref{fig:prob-shield-counterexample}, 
   with $\S = \{s_0, s_1, s_2, s_3\}$, $\A = \{a,b\}$ and $\P$ as described in the figure, 
   where any transition that is drawn and has no number on it has probability 1.
   Consider $T = \{s_3\}$ and $k=\infty$.
   The only state where the agent's decision matters is $s_0$, so an agent can be described as $P_a$, the probability of taking action $a$ in $s_0$.
   If an agent chooses $a$, it reaches $s_3$ with probability $1-\eps$ and goes back to $s_0$ with probability $\eps$. 
   If an agent chooses $b$, it reaches $s_3$ with probability $\delta$ and avoids $s_3$ altogether with probability $1-\delta$.
   The optimal strategy to avoid $s_3$ is clearly $p_a = 0$, 
   and in such case $\PP_{\min}(\Reach(s_0, \{s_3\})) = \delta$.
   For a general agent $p_a$, we have
   \begin{align}
   \PP_{p_a}\left(\Reach(s_0, a, \{s_3\})\right) & = 1-\eps + \eps \cdot \PP_{p_a}\left(\Reach(s_0, \{s_3\})\right),
   \nonumber \\
   \PP_{p_a}\left(\Reach(s_0, b, \{s_3\})\right) & = \delta,
   \nonumber \\
   \PP_{p_a}\left(\Reach(s_0, \{s_3\})\right) & =
   p_a\cdot \left[  1-\eps + \eps \cdot \PP_{p_a}\left(\Reach(s_0, \{s_3\})\right)   \right] + (1-p_a)\delta.
   \label{eq:aux3}
   \end{align}
    Therefore, with a threshold $\lambda$, 
    action $a$ is allowed by the shield if 
    \[
    1 - (1-\eps+\eps\delta) \geq \lambda (1-\delta),
    \]
    which is equivalent to $\eps \geq \lambda$.
    On the other hand, isolating from Equation~\ref{eq:aux3} we get:
    \begin{align}
        \PP_{p_a}\left(\Reach(s_0, \{s_3\})\right) & = 
        \frac{p_a(1-\eps) + (1-p_a)\delta}{1-p_a\eps} = 1 - \frac{(1-p_a)(1-\delta)}{1-p_a\eps}.
        \label{eq:aux4}
    \end{align}
    Once $\eps$ is fixed, we can make the value in Equation~\ref{eq:aux4} arbitrarily close to 1 by modifying $\delta$ and $p_a$.
\end{example}

A potential solution for such cases would be to define the set of correct agents as those agents that satisfy 
\begin{equation}
\label{eq:aux5}
\PP^\M_{\Ag}(\Avoid(s,T))\geq \lambda.
\end{equation}
This would require us to change the concepts used to define shields, as the property described in Equation~\ref{eq:aux5} cannot be described in terms of traces.

%% file: 30_delayed_safety_shields.tex
\ifthenelse{\boolean{includequotes}}{
\begin{quotation}
    \textit{No oblideu mai que si ens llevem ben d'hora, per{\`o} ben d'hora ben d'hora, i no hi ha retrets ni hi ha excuses, i ens posem a pencar, som un pa{\'i}s imparable.}
    \footnote{Never forget that if we get up very early, but very early, very early, and there are no reproaches or excuses, and we get down to business, we are unstoppable.}
    \hfill
    --- Josep Guardiola i Sala.
\end{quotation}
}{}

\section{Motivation and Outline}

Incorporating delays into safety computations is essential for nearly all real-world control problems. These delays, arising from data collection, processing, or transmission, are ubiquitous in systems operating within complex environments~\cite{DBLP:conf/icalp/HoschL72,Balemi92, nilsson1998real,Tripakis04, BerwangerCDHR08,Chen16,HFM16}.
When not properly addressed, such delays can become the root cause of critical safety issues.

\begin{example} 
    Consider a scenario where a car detects a pedestrian at position $(x, y)$ and accounts for a known time delay $\delta$ between sensing and acting. The vehicle must plan its actions to ensure safety for any possible position of the pedestrian within the interval $(x \pm \varepsilon, y \pm \varepsilon)$, where $\varepsilon$ is determined based on assumptions about the pedestrian's velocity and the delay $\delta$. 
\end{example}

Safety shielding is often used to guarantee 
safe execution of an agent in an environment that a safety game can model.
However, traditional safety shields assume no delay between sensing and acting, which limits their applicability in real-world scenarios.

In this chapter, we introduce synthesis algorithms for \emph{delay-resilient} safety shields, i.e., shields specifically designed to maintain safety even when input delays are present. 
These algorithms account for the uncertainties introduced by delays, enabling robust performance in dynamic environments. Figure~\ref{fig:shielding_setting} illustrates the shielding setup under delayed conditions.

\begin{figure}
\centering
\begin{subfigure}[b]{0.48\textwidth}
         \centering
         \includegraphics[width=\textwidth]{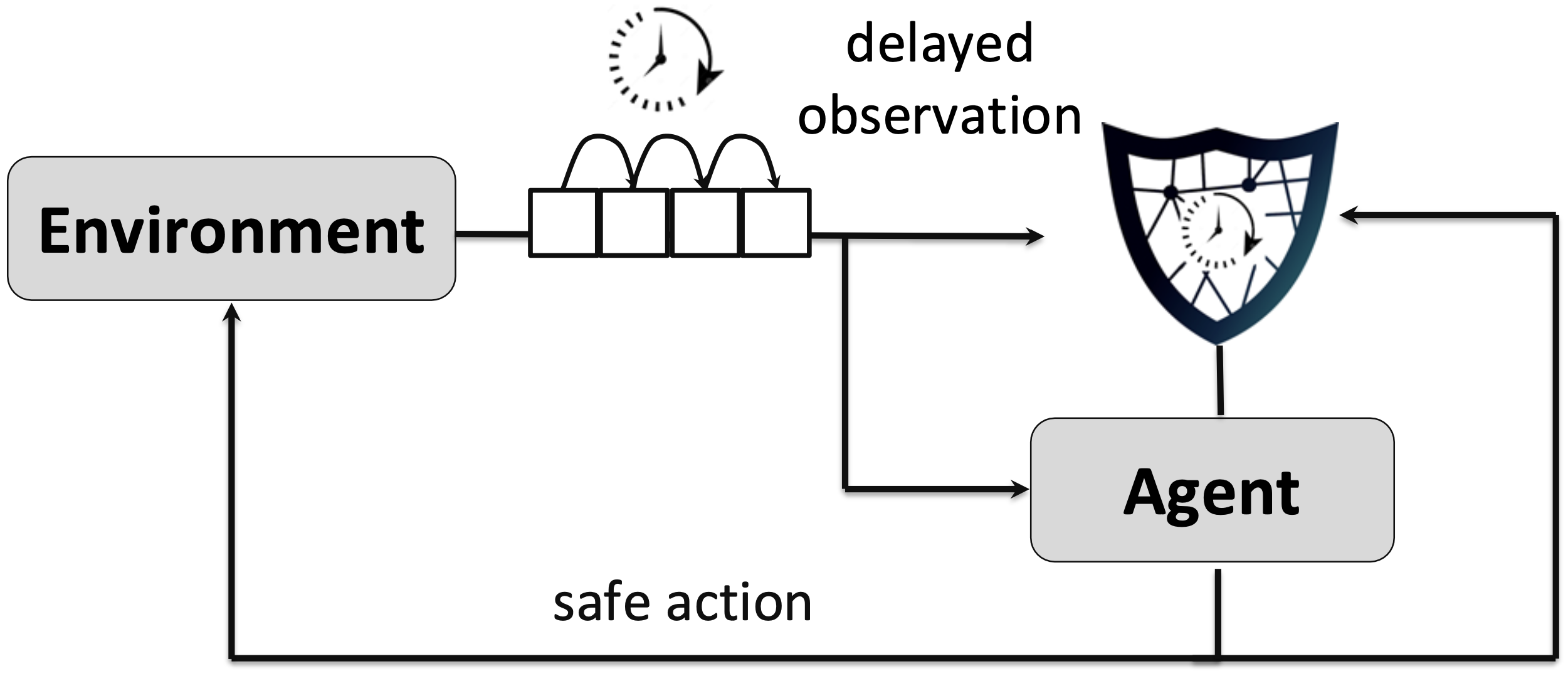}
         \caption{Pre-shield.}
     \end{subfigure}
    \hfill
     \begin{subfigure}[b]{0.48\textwidth}
         \centering
         \includegraphics[width=\textwidth]{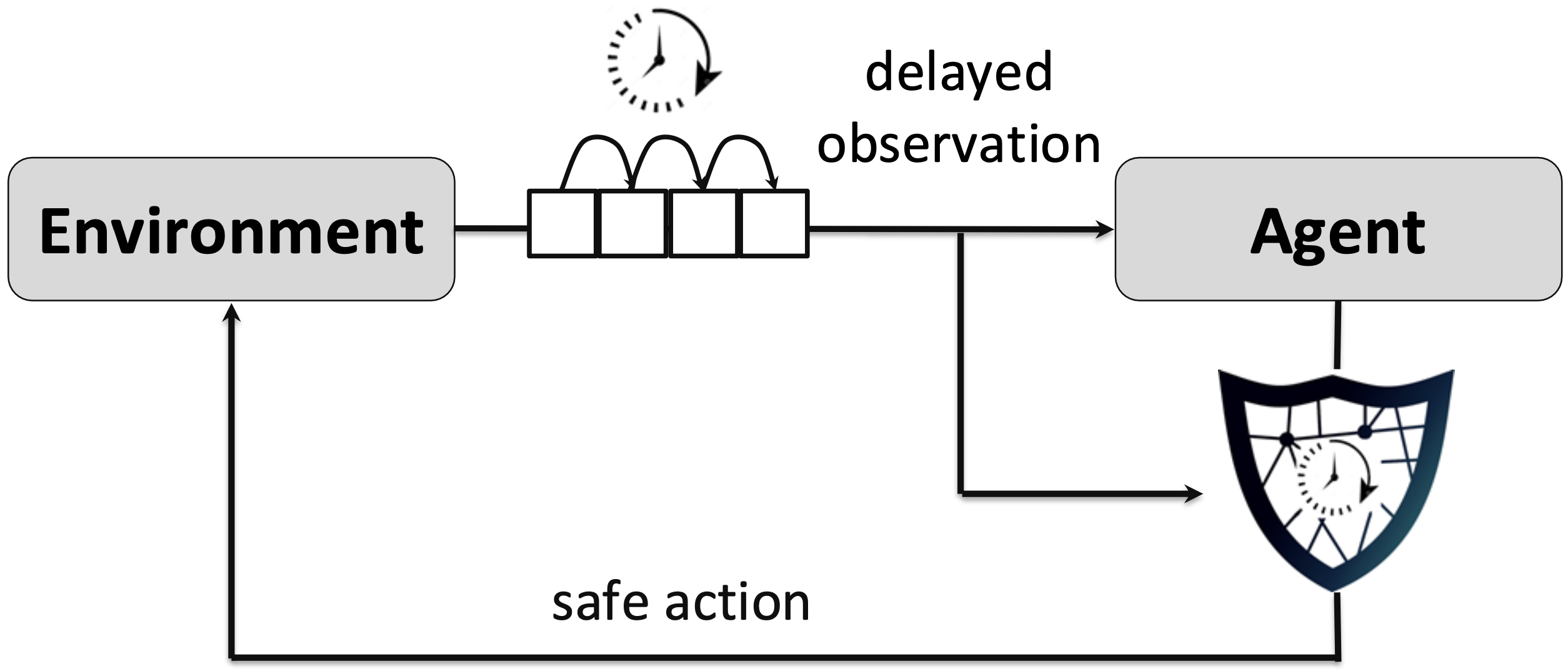}
         \caption{Post-shield.}
     \end{subfigure}
        \caption{Delay-resilient shielding scheme.}
        \label{fig:shielding_setting}
\end{figure}

To synthesise delay-resilient shields, we incorporate a worst-case delay in the safety game, which induces imperfect state information,
and use the algorithm proposed in~\cite{chen2018s,Chen2020IndecisionAD}
to compute the maximally permissive winning strategy. 
The delay-resilient pre-shields are then computed from the maximally permissive winning strategy in the delayed safety game.
For post-shields, in addition to the maximally permissive winning strategy, 
we need a deterministic winning strategy 
that will be used to obtain a fixed replacement action for any unsafe action.
To do so, we can define a property over the state space
and set the action maximising such property as the one fixed by the shield.
We study two such properties: controllability and robustness.
The \emph{controllability value} assigns to any state $s$ the \emph{maximal delay} on the input under which $s$ stays safe. 
The \emph{robustness value} of a state $s$ is the length of the minimal path from $s$ to any unsafe state.
We discuss how to maximise a state property under the uncertainty introduced by the delayed input.

Finally, we evaluate delay-resilient shields in two case studies.
The first one is a gridworld in which we implemented all proposed types of delay-resilient 
shields and compare computation cost and interference rates.
We show that delay-resilient post-shields that choose corrective actions
maximising either robustness or controllability tend to stabilise the execution, requiring fewer interferences by the shield.
In the second case study, we integrate shielding under delay in the driving simulator 
\textsc{Carla}~\cite{dosovitskiy_carla_2017}
to enforce collision avoidance for autonomous driving agents at intersections with cars and pedestrians under delayed observations. 
Our results show the effects of delays on the safety analysis and that our method is scalable enough to be applied in complex application domains.
The source code and scripts to reproduce the experiments, 
along with \emph{videos} from our experiments in \textsc{Carla},
are available on the accompaning repository\footnote{\url{https://github.com/filipcano/safety-shields-delayed}}.

\paragraph*{Contribution.}
The work presented in this chapter can be summarized in the following contributions.
\begin{itemize}
    \item We formalize the concept of pre-shield and post-shield resilient to delayed observation,
    showing how to compute them with the maximally permissive strategy of the corresponding safety game.
    \item We describe in detail the algorithm to compute the maximally permissive strategy of a safety game under delay with restricted memory, extending the algorithm presented in~\cite{Chen2020IndecisionAD}.
    \item We introduce the concepts of \emph{robustness} and \emph{controllability} and how to build post-shields maximising each one.
    \item We provide theoretical insight about the differences and similarities of the \emph{controllability} and the \emph{robustness} criterion when choosing a corrective action in post-shielding. 
    \item We validate our approach in two use cases: a gridworld and a realistic driving scenario. 
    As far as we know, we present the first integration of shields in a realistic driving simulator. 
\end{itemize}

\paragraph*{Outline.} 
We use in this chapter the formalism of two-player safety games with delayed inputs as defined in Section~\ref{sec:prelim-TwoPlayerGames}.
%
In Section~\ref{sec:delayed-shields-delayed-shields} we present the concept of shields resilient to delays and explain the algorithm required to computed them by finding the maximally permissive winning strategy of the underlying safety game undel delayed information.
In Section~\ref{sec:controllability-and-robustness} we present the two 
properties proposed to guide the synthesis of post-shields and 
how to synthesise shields, maximising them.
In Section~\ref{sec:relations}, we explore the relation between robustness and controllability, 
proving that they can be arbitrarily different.
Finally, in Section~\ref{sec:shields-experiments} we present the
results of our experimental evaluation on two use cases
and in Section~\ref{sec:delayed-shields-discussion} we discuss limitations and related work.

\paragraph*{Declaration of sources.}
This chapter is partially based and reuses material from the following source
previously published by the author of this thesis:

\cite{Cano2023shielding} \fullcite{Cano2023shielding}.

\section[Shields as Safety Games]
{Delay Resilient Shields as Strategies in Safety Games}
\label{sec:delayed-shields-delayed-shields}

As we have described in 
Sections~\ref{sec:rdm-two-player-games}~and~\ref{sec:shielding-in-safety-games}, 
a safety game can be seen as a particular case of the 
reactive decision-making framework, 
where minimally correct shields (Definition~\ref{def:minimal_correctness}) correspond to maximally permissive winning strategies of the corresponding safety game.

Following the construction outlined in 
Section~\ref{sec:prelim-games-under-delay}, 
given a safety game, \newline
$\mathcal G = \langle S, s_0, S_{env}, S_{ag}, \Act, \mathcal T, \F \rangle$, 
we consider the corresponding game played with delay $\delta$ and memory $\mu$, 
for given values of $\delta, \mu\in \NN$, $\mu\leq\delta$.

As described in Section~\ref{sec:shielding-in-safety-games-delayed}, specifically stated in Theorem~\ref{thm:shield-delayed-safety-game},
computing minimally correct shields in safety games under delay $\delta$ and memory $\mu$ corresponds to computing the maximally-permissive strategy $\xi$ of the corresponding game.
Then,
using the notation in Section~\ref{sec:shielding-in-safety-games-delayed}, 
given the maximally permissive winning strategy $\xi$, 
the minimally correct pre-shield is $\Sh_{\xi}^{pre}$, 
and minimally correct post-shields are computed as $\Sh_{\xi, \chi}^{pos}$, 
where $\chi$ is a determinization of $\xi$. 

The following section describes the algorithm used to compute such strategies. 
As mentioned before, this is a natural extension of 
the algorithm presented in~\cite{Chen2020IndecisionAD}
for games where the amount of memory and the delay are the same.

\subsection[Maximally Permissive Winning Strategies]
{Computation of Maximally Permissive Winning Strategies in Safety Games under Delay}


The algorithm is given in pseudocode in Algorithm~\ref{alg:delayed_strategies_extended}.
The method to solve a delayed safety game consists of iteratively constructing and solving the safety game with increasing delays 
$d = 0, 1,\dots, \delta$ and memory size 
$m=\min(d, \mu)$,
starting with $d=m=0$,
which corresponds to the case without delays, 
as presented in Section~\ref{sec:prelim-TwoPlayerGames}.
At every iteration in $d$, 
the maximally permissive strategy for the agent is computed using the strategy for the previous delay $d-1$
(lines 4,5 or 9,10 depending on the value of $m$),
followed by a reduction of the game graph aiming to mitigate the exponential blow-up in the state space (line 11)
and the computation of the transient phase (line 12).
Note that, following the convention in Equation~\ref{eq:safety-game-strategy-delay}, 
the action register $[y_1,\dots,y_m]$ is in reversed order, i.e., the last action performed by the agent is $y_m$.

The method to compute the maximally permissive strategy using the previous delays is slightly different for the case of full memory ($m=d$) and the case of restricted memory $m < d$.

\begin{algorithm}[h!]
    \caption{Maximally Permissive Strategy under Delay, memory $\mu\leq\delta$
    (adapted and extended from~\cite[Algorithm 1]{chen2018s}).
    }\label{alg:delayed_strategies_extended}
    \SetKwInOut{Input}{input}\SetKwInOut{Output}{output}\SetNoFillComment

    \Input{Safety Game $\G$, maximum delay $\delta$, memory $\mu$}
    $\xi_{0} \gets \mathtt{StrategyPerfectInfo}(\G)$\;
    \For{$d = 1,\dots, \delta$}{
        $m  \gets \min(d, \mu)$\;
        \For {$s\in S_{ag}$, $[y_1,\dots,y_m]\in \Act^m$}{
            \If{$m = d$}{
                $\mathcal I_{s,y_m} \gets \{
                s''\in S_1 \::\: s\xrightarrow{y_{m}}s'\xrightarrow{u}s''\}
                $\;
                $\xi_{d, m}(s, [y_1,\dots, y_m]) \gets
                \bigcap_{s''\in \mathcal I_{s,y_m} } \xi_{d-1, m-1}(s'', [y_1, \dots, y_{m-1}])
                $\;
            }
            \Else{
                $\mathcal I_{s} \gets \{
                s''\in S_1 \::\: s\xrightarrow{y}s'\xrightarrow{u}s'', \:
                y\in \Act \}
                $\;
                $\xi_{d, m}(s, [y_1,\dots, y_m]) \gets
                \bigcap_{s''\in \mathcal I_{s} } \xi_{d-1,m}(s'', [y_1, \dots, y_{m}])
                $\;
            }
        }
        $\mathtt{Shrink}(\xi_{d, m})$\;
        $\mathtt{InitialMoves}(\xi_{d, m})$\;
    }
    
    \Return $\xi_{\delta,\mu}$
\end{algorithm}

\begin{itemize}
    \item Case $m= d$.
    To compute the maximally permissive strategy using previous delays, 
    we compute $I_{s,y_m}$ (line 6), 
    corresponding to the set of states that the agent can get as the next observation when the current observation is state $s$
    and the chosen action is $y_m$.
    From states where the agent has already decided upon an output, 
    it is equivalent to playing with delay $d-1$. 
    Therefore, the strategy allows the actions that would be safe for delay $d-1$
    on all possible next observations, 
    eliminating the last executed action, $y_m$, from the action register (line 7).
    \item Case $m<d$.
    To compute the maximally permissive strategy using previous delays, 
    we compute $I_{s}$ (line 9), 
    corresponding to the set of states that the agent can get as the next observation when the current observation is state $s$
    and chosen action is any $y\in\Act$.
    In this case, $y$ is undetermined because of the restricted memory:
    the output that is provided just next to the observed state 
    has already been forgotten by the system.
    From states where the agent has already decided upon an output, 
    it is equivalent to playing with delay $d-1$. 
    Therefore, the strategy allows the actions that would be safe for delay $d-1$
    in all possible next observations, maintaining, in this
    case, the same memory $[y_1,\dots,y_m]$ (line 10).
\end{itemize}

The method $\mathtt{StrategyPerfectInfo}$ (line 1) computes the maximally permissive strategy for the
game with perfect information~\cite{Thomas1995}.
The method $\mathtt{Shrink}$ (line 11) 
ensures that in case the intersection in lines 5 or 10 is empty, 
the maximally permissive strategy in a state 
$s''\in \mathcal I_{s,y_m}$ or $s''\in\mathcal I_{s}$
does not contain the output $y_m$~\cite[Algorithm 3]{Chen2020IndecisionAD}.
The method $\mathtt{InitialMoves}$ (line 12) computes the strategy for the transient period,
before the agent can get any observed state, see Algorithm~\ref{alg:initialmoves}.

\begin{algorithm}
\caption{$\mathtt{InitialMoves}$: Strategy for the transient period (adapted from~\cite[Algorithm 1]{chen2018s}).}
\label{alg:initialmoves}
\SetKwInOut{Input}{input}\SetKwInOut{Output}{output}\SetNoFillComment

    \Input{Safety game $\mathcal G$, ongoing maximally permissive strategy $\xi_{d,m}$}
    $\J = \{ s\::\: s_0\xrightarrow{u}s \}$\;
    \For{$[y_1,\dots, y_{m-1}]\in\Act^{m-1}$}{
    
        $\xi_{d, m}(\eps, [y_1,\dots, y_{m-1}]) \gets 
        \left\{
        y_d \::\: \bigcup_{s\in \J}
        \xi_{d}(s, [y,y_1,\dots, y_{m-1}]) \neq \emptyset 
        \right\}$\;
    }
    \For{$k=m-2,\dots, 0$}{
        \For{$[y_1,\dots, y_{k}]\in\Act^{k}$}{
            $\xi_{d, m}(\eps,[y_1,\dots, y_{k}]) \gets
            \left\{
            y_0\::\: \xi_{d, m}(\eps,[y_0,y_1,\dots, y_k]) \neq \emptyset
            \right\}
            $\;
        }
    }
    \Return $\xi_{d, m}$
\end{algorithm}

\paragraph*{Complexity analysis.}
Algorithm~\ref{alg:delayed_strategies_extended} 
computes each strategy $\xi_{d, m}$
for increasing values of $d =1,\dots, \delta$ (main loop, lines 2-12).
For each strategy, 
the algorithm goes over all states in $S_{ag}$ and registers in $\Act^m$ (loop in lines 4-10), 
and at each iteration, 
computes an intersection of $\Act_{env}$ elements. 
The cost of $\mathtt{Shrink(\xi_{d,m})}$ and 
$\mathtt{InitialMoves(\xi_{d,m})}$
is negligible in comparison.

Therefore, the cost of computing the strategy $\xi_{d,m}$ is 
$\mathcal O(S_{ag}\cdot |\Act|^m\cdot |\Act_{env}|)$,
and the total cost of computing 
$\xi_{\delta, \mu}$ is 
$\mathcal O\left(
|S_{ag}|\cdot|\Act_{env}|\cdot
\left(
\mu \cdot|\Act|^\mu + (\delta-\mu)\Act^\mu
\right)
\right)$, 
simplified to 
\begin{equation}
\label{eq:delayed-shields-winning-strategy-complexity}
    \mathcal O\left(
    \delta\cdot 
    |S_{ag}|\cdot|\Act_{env}|\cdot\Act^\mu
    \right).    
\end{equation} 

Recall that $\Act_{env}$ represents a set of actions for the environment.
As discussed in Equation~\eqref{eq:safety-games-env-out-degree}, without loss of generality, we can assume $|\Act_{env}|$ is the maximum out-degree of the environment transitions.

\section[Determinization of Strategies]{Determinization of Strategies Resilient to Delays}
\label{sec:controllability-and-robustness}
The synthesis procedure for a delay-resilient post-shield relies on the construction of a deterministic winning strategy, denoted by $\chi_{\delta, \mu}$. This section outlines the process for deriving such a strategy.

In Section~\ref{sec:delayed-shields-maximize-fitness},
we describe the method for computing the deterministic strategy that maximises a given fitness value, considering both memory and delay.

We present two examples of fitness functions in Sections~\ref{subsec:delayresilience} and~\ref{subsec:robustsafety}. These are specifically designed to minimise the number of instances where the post-shield must interfere due to delays in the input.

\subsection[Determinization Maximizing a Fitness Function]
{Determinisation of Delayed Strategies Maximising a Fitness Function}
\label{sec:delayed-shields-maximize-fitness}

When deciding which action to use as a corrective action for post-shields, we want to decide on an action that maximizes a certain criterion.
In this section, we assume the existence of a fitness function 
$\varphi\colon S \to \mathbb{R}$ 
that assigns a fitness value to each state. 
We will show how to find the actions that maximize an abstract fitness value, and in the following sections, we will apply this method to concrete fitness functions designed to prevent unsafe transitions due to delayed inputs.

Our goal is to choose at each state the action that maximises this fitness function among all actions allowed by $\xi_{\delta, \mu}$.
However, because of the uncertainty of the transitions of the environment, 
it is not clear what it means to maximise the fitness function,
since the agent has no complete control over the state of the safety game after each of the agent's actions.

For this computation, we will take the implicit assumption that from a given state $s\in S_{env}$,
and a given number of transitions $n$, 
all traces from $s$ with $n$ transitions are equally probable.
In this sense, we say that the strategy we compute maximises the \emph{expected fitness value} -- with the implicit understanding that the expectation is taken under the assumption of a uniform probability environment.
This assumption can be refined to include more accurate representations of the probabilistic nature of the environment whenever such models are available. 
We leave this extension for future work.

For our computation, we need to define the 
\emph{k-forward multiset of states} $F_k(s,\overline{\sigma})$, 
which captures the states reachable from $s$
within $k$ steps respecting a given memory  $\overline{\sigma}\in \Act^\mu$, 
i.e. the last $\mu$ actions of the agent.


\begin{definition}[$k$-Forward Multiset of States]
\label{def:forward-multiset-states}
    Let
    $\overline{\sigma} = [z_1\dots z_\mu]\in\Act^\mu$ be a register of actions. 
    For a state $s\in S_{ag}$, 
    the
    \emph{$k$-forward multiset} is 
    \[
    F_k(s,\overline{\sigma}) = \left\{
    \begin{matrix}
    s_{2k}&\::\: \exists s_1,\dots,s_{2k-1} \in S,\, \mbox{and} \, \exists  
    y_1,\dots, y_k \in \Act\, \mbox{such that}
    \\ &
    (1)\, \forall i = 1\dots \mu,\, y_{k-\mu+i} = z_i,\, 
    \mbox{and}
    \\ &
    (2)\, s\xrightarrow{y_1} s_1 \xrightarrow{u} s_2 
    \xrightarrow{y_2} s_3
    \xrightarrow{u} \dots \xrightarrow{u} s_{2k-2} 
    \xrightarrow{y_k} s_{2k-1} \xrightarrow{u}s_{2k}
    \end{matrix}
    \right\},
    \]
    where each state $s_{2k}$ is counted as many times as 
    there are distinct sequences 
    $s_1,\dots,s_{2k-1}\in S$ and 
    $y_1,\dots, y_k\in\Act$ satisfying conditions (1) and (2).
\end{definition}

The expected fitness value is computed over the $k$-forwarded multiset of states,
thus, each state adds to the value as many times as it appears in the multiset.
    
\begin{definition}[Expected Fitness Value] 
Let $s\in S$ be a state, 
$\overline{\sigma}\in \Act^\mu$ a register of actions and 
$\varphi:S\to \mathbb R$ a fitness function.
    For a given delay $\delta$,
    the \emph{expected fitness value} $\mathbb{E}_{\varphi}(s,\overline{\sigma})$ 
    is defined as the \emph{average} of the fitness values of all states 
    $s'$ in $F_\delta(s, \sigma)$,
    \[
    \mathbb{E}_{\varphi}(s,\overline{\sigma}) = 
    \frac{1}{|F_{\delta}(s,\overline{\sigma})|}
    \sum_{s'\in F_{\delta}(s, \overline{\sigma})} \varphi(s').
    \]
\end{definition}

The strategy 
$\chi_{\delta, \mu} \colon S_{ag} \times \Act^{\mu}  \to \Act$ 
that maximises the expected value of $\varphi$ is:
\[
\chi_{\delta,\mu}\big(s,[z_1\dots z_\mu]\big) = \argmax_{y\in\xi_{\delta,\mu}\left(s,[z_1\dots z_\mu]\right)} 
\mathbb{E}_{\varphi}\big(s, [y, z_1\dots z_\mu]\big).
\]

\paragraph*{Complexity of strategy determinisation.} 
The deterministic strategy is computed for $|S_{ag}|\cdot |\Act|^\mu$ states.
Each forward multiset contains at most 
\({|\Act_{env}|^{\delta}\cdot|\Act|^{\delta-\mu}}\) states,
where $\Act_{env}$ is a set of actions for the environment. 
For each of these states, 
the fitness value $\varphi$ is computed.
Assuming $c(\varphi)$ is the computational cost of computing $\varphi$ for one state, and $\varphi$ is stored in a lookup table, the total complexity adds up to 
\begin{equation}
 \label{eq:delayed-shields-determinization-complexity}
 \mathcal O\left(|S_{ag}|\cdot \left(|\Act_{env}\times \Act|^{\delta} + c(\varphi)\right)
\right).
\end{equation}

\subsection{Post-Shields that Maximise Controllability}
\label{subsec:delayresilience}
In this section, we define and compute a fitness function
called the controllability value.
that assigns to each state the maximum delay for which a 
safe output exists.
For any state $s\in S_{ag}$,
the controllability value 
$\varphi_{c} : S_{ag} \rightarrow \RR$
is the largest delay for which a register of actions exists that makes $s$ safe.
To formally define the controllability value, 
we use the notion of controllable states.

\index{controllability}
\index{state!controllable}
\begin{definition}[Controllable State]
\label{def:controllable-state}
 A state $s\in S_{ag}$ is \emph{controllable} under delay $\delta$ and memory $\mu$
 if there exists 
 $\overline{\sigma}\in\Act^\mu$ such that
 $\xi_{\delta,\mu}(s,\overline{\sigma})\neq \emptyset$, 
 and \emph{uncontrollable} otherwise.
 A state $s\in S_{env}$ is \emph{controllable} under delay $\delta$ and memory $\mu$ if all states $s'\in S_{ag}$
 such that
 $s\xrightarrow{u}s'$ are controllable under delay $\delta$ and memory $\mu$.
\end{definition}

\begin{definition}[Controllability Value]
 The \emph{controllability value} with memory $\mu$ 
 of a state $s\in S$ is the maximum delay $\delta$ for which $s$ is controllable with delay $\delta$ and memory $\mu$.
 We denote it as $\varphi_c(s)$.
\end{definition}

To unpack this definition, for an agent state $s\in S_{ag}$, we say that $\varphi_c(s) = \delta$ if 
there exists $\ol{\sigma}\in\Act^{\mu}$ such that
$\xi_{\delta,\mu}(s,\ol{\sigma}) \neq \emptyset$
and for all $\ol{\sigma}'\in\Act^{\mu}$,
we have 
$\xi_{\delta+1,\mu}(s,\ol{\sigma}') = \emptyset$.

In Definition~\ref{def:controllable-state}, guaranteeing only the existence of $\ol{\sigma}\in\Act^\mu$ such that $\xi_{\delta,\mu}(s,\overline{\sigma})\neq \emptyset$ might seem too weak because one does not know in advance what the action register may be when observing a state. 
Note that, however, Algorithm~\ref{alg:delayed_strategies_extended} guarantees that an agent following $\xi_{\delta, \mu}$ can only go through pairs $(s,\ol{\sigma})\in S_{ag*}\times\Act^\mu$ such that $\xi_{\delta,\mu}(s,\ol{\sigma})$ is non-empty.
This is one of the main consequences of the $\mathtt{Shrink}$ method in Algorithm~\ref{alg:delayed_strategies_extended}, and is extensively discussed in~\cite[Algorithm 3]{Chen2020IndecisionAD}.

If a state $s$ is inside the winning region $W$ of the safety game without delay, it has a controllability value greater or equal to 0. 
Furthermore, note that as a consequence of the iterative computation of winning strategies for safety games under delay, if a state $s$ is controllable for delay $\delta>0$, it is also controllable for delay $\delta-1$.
By convention, if $s\notin W$, i.e., there is no delay $\delta$ that makes it controllable, 
we say that $\varphi_c(s) = -1$.

A shield that maximizes controllability will tend to steer the agent towards states that can be safe even with large delays. In a setting with variable delay, the shield always operates with the worst-case delay in mind, but the agent may make a more refined use of the variable delay, so steering the agent towards high-controllability states translates to more freedom for which actions to choose in the future, as there are fewer paths leading to uncontrollable states.

Since the maximal delay possible
can be very large and 
the state space of the corresponding
safety game grows exponentially with the delay, 
we introduce a cutoff value 
$\delta_{\max}$
and compute the maximally-permissive winning strategy until $\delta_{\max}$.

\begin{figure}[t!]
     \centering
     \begin{subfigure}[b]{0.252\textwidth}
         \centering
         \includegraphics[width=\textwidth, page=2]{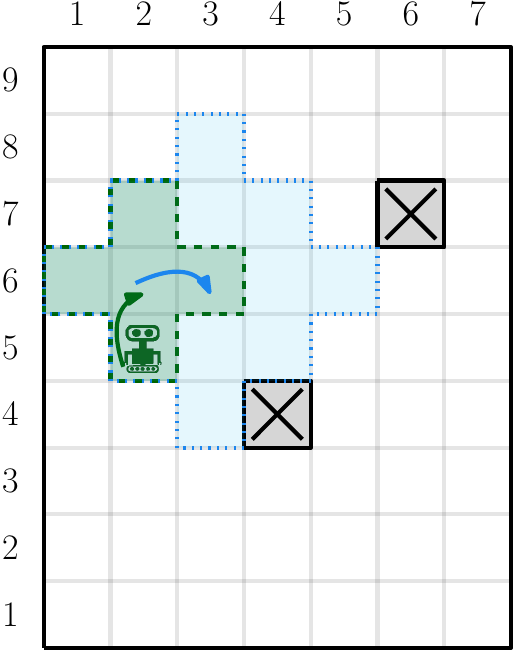}
         \caption{}
         \label{fig:example1}
     \end{subfigure}
     \hfill
     \begin{subfigure}[b]{0.233\textwidth}
         \centering
         \includegraphics[width=\textwidth, page=3]{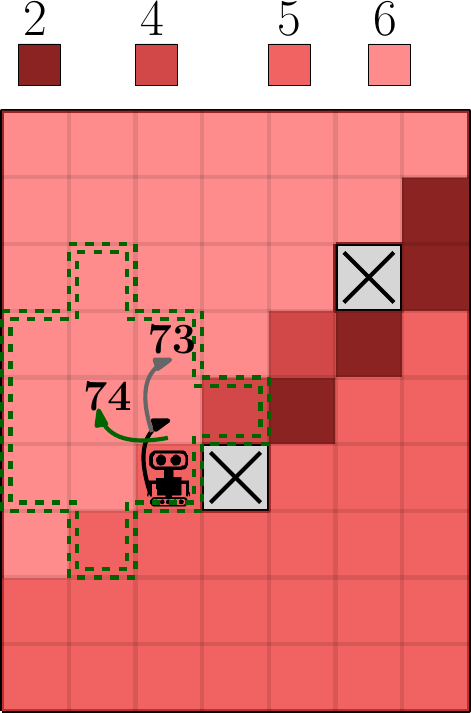}
         \caption{}
         \label{fig:example2}
     \end{subfigure}
     \hfill
     \begin{subfigure}[b]{0.233\textwidth}
         \centering
         \includegraphics[width=\textwidth, page=4]{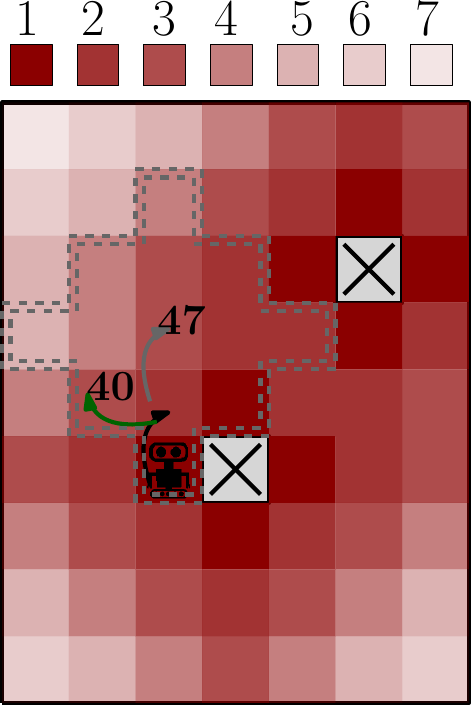}
         \caption{}
         \label{fig:example3}
     \end{subfigure}
     \hfill
     \begin{subfigure}[b]{0.233\textwidth}
         \centering
         \includegraphics[width=\textwidth, page=5]{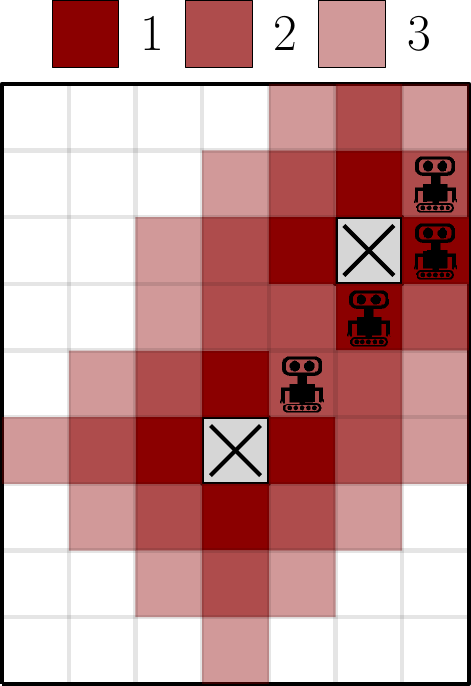}
         \caption{}
         \label{fig:example4}
     \end{subfigure}
        \caption[Gridworld example]
        {\textbf{(a)} Effects of delay on state observation.
    \textbf{(b)} Gridworld depicting the least delay-resilient states.
    \textbf{(c)} Gridworld with $\varphi_{dr}$ values for all states.
    \textbf{(d)} Gridworld with $\varphi_{rs}$  values for all states.}
        \label{fig:3dgridDR}
    \vspace{-1em}
\end{figure}

\begin{example}
\label{ex:delayeds-shields-example}
We showcase the computation of  $\varphi_{c}$  on a simple  $7\times 9$ gridworld, depicted in Fig.~\ref{fig:3dgridDR}(a).
Initially, a robot is placed at $(1,9)$. 
The environment and the agent can move the robot by one field in alternating turns. 
The safety specification requires that the robot never
visits the fields  $(4,4)$ nor $(6,7)$.
Encoding the model and the specification leads to the safety game
$\mathcal G=\langle S, S_{ag}, S_{env}, \Act,
\mathcal T, \F  \rangle$:
\begin{itemize}
    \item $S = X\times Y\times \mathbb B$, 
    where $X = \{1,\dots, 7\}, Y = \{1,\dots,9\}$ 
    represent the robot's position 
    and 
    $\mathbb B = \{\top, \bot\}$ indicates whether it is the turn of the agent ($\top$) or the environment $(\bot)$ to move the robot.
    The states of the environment are 
    $S_{env} = X\times Y\times \{\bot\}$, 
    and of the agent are 
    $S_{ag} = X\times Y\times \{\top\}$.
    \item The unsafe states are $S\setminus \mathcal F = \{(4,4), (6,7)\}\times\mathbb B$.
    \item  The agent's actions are 
    $\Act = \{ \mathtt{U},\mathtt{D},\mathtt{R},\mathtt{L},\mathtt{N}\}$ to move the robot
    one field up, down, right, left, or to hold.
    Formally: $(x,y,\top) \xrightarrow{\mbox{\tiny{$\mathtt{U}$}}} (x,y+1,\bot)$,
    $(x,y,\top) \xrightarrow{\mbox{\tiny{$\mathtt{D}$}}} (x,y-1,\bot)$,
    $(x,y,\top) \xrightarrow{\mbox{\tiny{$\mathtt{R}$}}} (x+1,y,\bot)$, 
    $(x,y,\top) \xrightarrow{\mbox{\tiny{$\mathtt{L}$}}} (x-1,y,\bot)$,
    $(x,y,\top)\xrightarrow{\mbox{\tiny{$\mathtt{N}$}}} (x,y,\bot)$.

    \item The actions of the environment player to move the robot are 
    $\Act_{env} = \{\mathtt{U'},\mathtt{D'},\mathtt{R'},\mathtt{L'},\mathtt{N'} \}$,
    with a meaning analogue to those of the agent's actions.
    
\end{itemize}
Moves that would lead the robot outside of the game's boundary are replaced by 
$\mathtt{N}$. 
For this game $\mathcal{G}$, we now compute the
controllability values $\varphi_{c}(s)$ for
all states $s\in S_{ag}$.

First, we illustrate in Fig.~\ref{fig:3dgridDR}(a)  the effects of delays on the state information of the play. 
In the example, we have a delay $\delta=1$ and memory $\mu=1$, and the observed state of the game is
$s=(2,5,\top)$ (green robot) with memory  $\overline{\sigma}=[\mathtt{U}]$.
The set of possible current states is 
$F_1(s,\overline{\sigma})$ (marked green).
To check whether a next action $y=\mathtt{R}$ is safe,
we compute $F_2(s,[\mathtt{R},\mathtt{U}])$ (marked blue or green).
Since $F_2(s,[\mathtt{R},\mathtt{U}])\subseteq \mathcal F$, 
$\mathtt R$ is a safe action from
$(s,[\mathtt U])$. 

Next, we exemplify the computation of the controllability values for 
the states $(5,5), (6,6), (7,7)$ and $(7,8)$. 
In Fig.~\ref{fig:3dgridDR} (b), each field of the grid is coloured according to 
the smallest distance to one of the unsafe states for distances 1, 2 and 3.
With this colour coding, 
the state $(x,y,\top)$ is controllable with delay $\delta$ 
if there exists sequence of actions $\overline{\sigma}$ of size $\delta$, 
that takes the robot outside of the region coloured with $\delta$.
The reader can see that all states are controllable for delay $\delta = 1, 2$, but for $\delta=3$, states 
$(5,5), (6,6), (7,7)$ and $(7,8)$ are uncontrollable (marked with a black robot).

Fig.~\ref{fig:3dgridDR} (c) illustrates the controllability value $\varphi_{c}(s)$ for all states. Each field of the grid is coloured according to its controllability value. 
Next, we exemplify how to compute a deterministic
strategy $\chi_{\delta=1,\mu=1}(s,\overline{\sigma})$ that maximizes the average of $\varphi_{c}$ over all possible current states.
Consider a state $s = (3,4,\top)$ (black robot) with memory $\overline{\sigma} = [\mathtt{U}]$.
The only two outputs allowed by $\xi_{1,1}(s,\overline{\sigma_1})$ are $\mathtt{U}$ and $\mathtt{L}$ 
since any other output would lead the robot to a state at a distance two or less from an unsafe state.
The forward multiset $F_2(s,[\mathtt L, \mathtt U])$ is marked with a dashed green line and results in the expected value $\mathbb E_{\varphi_{c}}(s,[\mathtt L, \mathtt U])= 74/26$. 
The expected value $\mathbb E_{\varphi_{c}}(s,[\mathtt U, \mathtt U])= 73/26$ is computed analogously. 
A delay-resilient post-shield that maximises the controllability value corrects the outputs $\mathtt R$, $\mathtt D$, and $\mathtt N$ in state $((3,4,\top), [\mathtt U])$ to the output $\mathtt L$.

\end{example}

\paragraph*{Complexity of computing controllability values.} 
Computing the controllability value as a fitness function 
only requires computing the maximally permissive winning strategy for the delay 
$\delta_{\max}$ chosen as the cutoff value --- see Equation~\eqref{eq:delayed-shields-winning-strategy-complexity}.


\subsection{Post-Shields that Maximise Robustness}
\label{subsec:robustsafety}
In this section, we define an alternative fitness function.
The \emph{robustness value} 
$\varphi_{r} : S \rightarrow \mathbb R$
assigns to every state
the shortest distance to any unsafe state
in the game graph.
Intuitively, a large robustness value suggests 
that the system is in a state that ``easily'' satisfies the specification, 
while values near zero suggest that the system is close to violating it.
A shield that maximises robustness potentially requires fewer corrections in the near future. 

%


\index{robustness value}
\begin{definition}[Robustness Value]
 Let 
 $\mathcal G = \langle S, s_0, S_{env}, S_{ag}, \Act, \mathcal T, \F \rangle$
 be a safety game with winning region $W$
 -- as defined in Equation~\eqref{eq:prelim-safety-games-winning-strategy}.
 For any state $s\in S_{ag}$,
 the robustness value $\varphi_{r}(s)$
 is defined as the 
 smallest $k$ such that there exists
 $\overline{\sigma}\in\Act^k$
 such that $F_k(s,\overline{\sigma})\not\subseteq W$.
\end{definition}


Note that in our definition, we are counting distance only in states of the agent. 
That is, when a state has a robustness value of $k$, 
there exists a trace with $2k$ states -- half of them of the agent, half of them of the environment -- that leads to a state outside of the winning region.

\begin{example}[Continuation of Example~\ref{ex:delayeds-shields-example}]
We exemplify the computation of $\varphi_{r}(s)$
on the gridworld of Fig.~\ref{fig:3dgridDR} (d).
Each field of the gridworld is colored with $\varphi_{r}(s)$
of its corresponding agent state $s$.
For any $s$, the fitness function 
$\varphi_{r}(s)$ is computed as the distance to the closest unsafe state. 
From $s=(3,4,\top)$ with $\overline{\sigma} = [\mathtt{U}]$ 
at delay $\delta =1$ and memory $\mu=1$, $\xi_{1,1}(s,\overline{\sigma})$
allows the actions $\mathtt{U}$ and $\mathtt{L}$. 
Since the expected robustness value 
$\mathbb E_{r}(s,[\mathtt{U}, \mathtt{U}])$
is greater than $\mathbb E_{r}(s,[\mathtt{R}, \mathtt{U}])$,
the deterministic strategy $\chi_{1,1}(s,[\mathtt{U}])$ that maximizes $\varphi_{r}$ would choose $\mathtt{U}$ as corrective output.
\end{example}

A shield that maximizes robustness will steer the agent towards states that are as far away as possible in the game graph from unsafe states. While this is a useful heuristic, note that distance in the safety game may not translate to real safety under delayed observations, as there may be states that are far away from the unsafe region, but with well-defined paths that the environment can force the agent to take toward unsafe states.
We explore some of these examples in the following section.

\noindent\textbf{Complexity of computing robustness values.}
The fitness function
$\varphi_{r}$ can be computed as a breadth-first search on the states, 
so all robustness values can be computed in $\mathcal O(|S|)$ time and memory.


\section{Relation between Robustness and Controllability}  
\label{sec:relations}

Once we have shields that maximize robustness and controllability, we would like to find results that guarantee certain safety properties when using these shields. 
By construction, the post-shields defined in the previous sections maximize their corresponding fitness function. 

By increasing the value of $\delta$, we can make post-shields that guarantee a certain controllability value. Similarly, we can add a buffer zone to the winning region $W$, to force that only states at a distance at least $d$ from unsafe states are ever visited. 

A more interesting guarantee would be some result that guarantees robustness values in terms of controllability, and vice-versa. 
In this section, we study the relationship between robustness and controllability values.

Although the intuition behind robustness and controllability is very similar, and in our experiments, we found them to be equal most of the time, we show that only very basic relations hold in general.
We show that there are example games where robustness is arbitrarily higher than controllability (Theorem~\ref{thm2}), and vice-versa (Theorem~\ref{thm3}).
These examples would break any result of guaranteeing controllability when maximizing robustness, or robustness when maximizing controllability.


\subsection[Memory-Restricted Strategies]
{Relation between Robustness and Controllability for Memory-Restricted Strategies.}

For strategies with a restricted memory. i.e., with $\mu \leq \delta$, we show a single result
and explore its consequences for the edge cases.

\begin{theorem}
\label{thm1}
    Let 
    $\mathcal G$ be a safety game with delay $\delta$ and memory size $\mu$.
    For any controllable state $s\in S_{ag}$ 
    it holds that
    \begin{equation}
    \label{eq:delayed-shields-thm1}
        \varphi_{r}(s) \geq \delta - \mu + 1.    
    \end{equation}
\end{theorem}

\begin{proof}
    We prove the result by contradiction.
    Let $s\in S_{ag}$ be controllable, 
    i.e. with 
    $\xi_{\delta,\mu}(s,\ol{\sigma})\neq \emptyset$,
    for some $\ol{\sigma}\in\Act^\mu$.
    Assume that $\varphi_{r}(s) < \delta - \mu + 1$, 
    or equivalently, 
    $\varphi_{r}(s) \leq \delta - \mu$.
    
    Then, there exists a trace of length $\delta - \mu$ leading outside of the winning region $W$. 
    Let $\tau = s, s_1, \dots, s_{2(\delta-\mu)}$ be such trace, 
    where $s_{2(\delta-\mu)}\notin W$
    and 
    $s\xrightarrow{y_1} s_1 \xrightarrow{u} s_2 
    \xrightarrow{y_2} s_3
    \xrightarrow{u} \dots \xrightarrow{u} s_{2(\delta-\mu-1)} 
    \xrightarrow{y_{\delta-\mu}} s_{2(\delta-\mu)-1} \xrightarrow{u}s_{2(\delta-\mu)}$,
    for some actions $y_1,\dots, y_{\delta-\mu}\in\Act$.

    Since the memory of the agent is limited to $\mu$, 
    when the observed state is $s$, 
    the agent only knows that the current state is $s'$ at a distance $\delta$ from $s$, 
    with a trace where the last $\mu$ actions are known.
    However, this would already be too late:
    whatever the last $\mu$ actions are, 
    any trace starting with $\tau$ -- of which the agent has no control -- will pass through $s_{2(\delta-\mu)}\notin W$.
    Therefore, $s$ cannot be controllable with delay $\delta$ and memory $\mu$.
    This proves the result.
\end{proof}

A corollary of Theorem~\ref{thm1} is that 
$\varphi_{r}(s) \geq \varphi_{c}(s) - \mu + 1$,
since the controllability value is a valid delay
for which a state $s$ is controllable.

This result gives us information about the minimum amount of memory required for a winning strategy.
The argument is as follows.

A safety game $\mathcal G$ admits a winning strategy under delay $\delta$ if
any state $s\in S_{ag}$ with 
$s_0\xrightarrow{u}s$
is controllable under delay $\delta$.
Therefore, the minimum amount of memory required for a winning strategy
is 
\begin{equation}
\label{eq:delayed-shields-min-required-memory}
    \mu \geq  \delta-\varphi_{r}(s_0) + 1.    
\end{equation}
This implies that for a fixed game $\mathcal G$ and increasing delay $\delta$,
the amount of memory needed to have a winning strategy increases 
after a certain threshold.
At some point, the delay is so high that no memoryless strategies exist:
if we set $\mu=0$ in Equation~\eqref{eq:delayed-shields-min-required-memory}, we get 
$\delta~\leq~\varphi_{r}(s_0)~-~1 $.

On the other extreme case, if we set $\mu=\delta$ in Equation~\eqref{eq:delayed-shields-min-required-memory}, 
we get $\varphi_{r}(s_0) \geq \delta - \delta + 1=1$,
which just means $s_0\in W$.
More generally, if we set $\mu=\delta$ in 
Equation~\eqref{eq:delayed-shields-thm1},
we get that any state $s$ that is controllable
satisfies $\varphi_{r}(s) \geq 1$,
which just means that $s$ is in the winning region.
Therefore, Theorem~\ref{thm1} provides no bound for $\varphi_{c}$ 
in terms of $\varphi_{r}$.

\subsection[Strategies with Full Memory]
{Relation between Robustness and Controllability for Strategies with Full Memory.} 

In this section, we will prove two results that
give counterexamples to any possible bound of $\varphi_c$ in terms of $\varphi_r$ and vice-versa, for strategies with full memory, i.e., with $\mu = \delta$.

\begin{theorem}
\label{thm2}
    For all delay $\delta > 0$, 
    and all $k \geq 0$, 
    there exists a safety game
    $\mathcal G_{\delta,\mu=\delta}^k=\langle S, s_0, S_{ag}, S_{env}, \Act, \mathcal T, \mathcal F\rangle$ with one state $s\in S$ satisfying
    \begin{equation*}
      \varphi_{c}(s) < \delta \quad \mbox{ and } \quad  \varphi_{r}(s) \geq \delta + k + 1.  
    \end{equation*}
\end{theorem}

\begin{proof}

We will do the proof by induction on $k$ for any delay $\delta$.

\emph{Base Case.} For $k = 0$, we need to construct a safety game $\mathcal G_\delta^k$ containing a state $s$, 
that is uncontrollable for delay $\delta$, 
but at least $\delta+1$ steps are needed to 
get to an unsafe state. 
A game with a section as depicted in Figure~\ref{fig:prop2basecase} serves as an example,
with action set $\Act = \{x,y\}$.

For any given delay $\delta$, 
    the dotted pattern in the middle of the figure repeats
    $\delta-3$ times. 
    In this case, we will prove that state $s$ is not controllable for delay $\delta$.
    The environment has a choice in $s_e$ for the next state to be $\hat s$ or $\tilde s$.
    When the observed state is $s$, 
    the current state is either $\hat{s}'$ or $\tilde{s}'$.
    Any action register will consist of a sequence of $ x$'s and $ y$'s, and
    the only information relevant in the register is the parity of $ y$'s:
    an even number of $y$ actions takes the state of the game 
    from $\hat s$ to $\hat s'$ or from $\tilde s$ to $\tilde s'$, 
    while an odd number of $y$ actions takes the state of the game
    from $\hat s$ to $\tilde s'$ or from $\tilde s$ to $\hat s'$.

    Therefore,
    in this game graph, without knowing the first choice of the environment -- state $\hat s$ or $\tilde s$ --,
    the agent cannot know whether the current state is 
    $\hat{s}'$ or $\tilde{s}'$. 
    Since a safe action in $\hat{s}'$ leads to the unsafe state $s_{\times}$ when taken from $\tilde{s}'$ and vice versa, 
    the state $s$ is uncontrollable for delay $\delta$.

\begin{figure}
    \centering
    \includegraphics[width=0.95\linewidth,page=1]{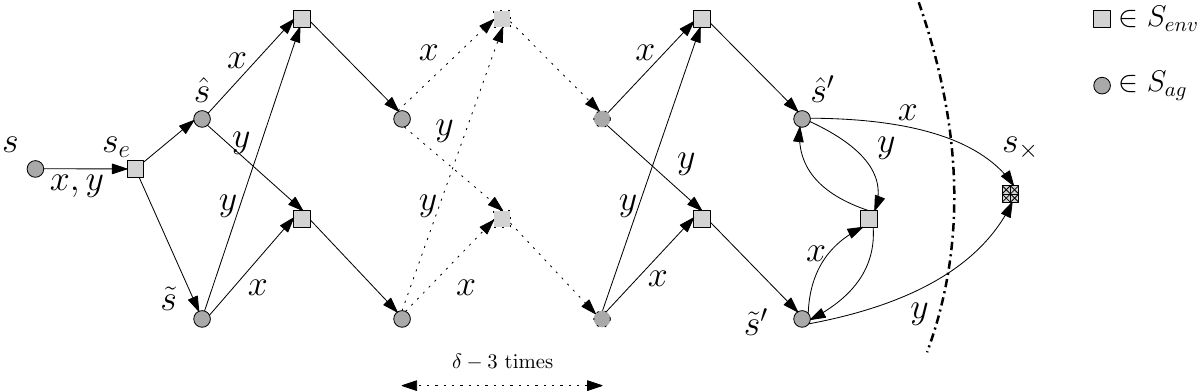}
    \caption[Construction for base case of Theorem~\ref{thm2}.]
    {Construction for base case of Theorem~\ref{thm2}.
    Square nodes represent states of the environment, circle nodes represent states of the agent.
    The dash-dotted line represents the divide between the winning region and the rest of the game. 
    Except for states $s_{e}$ and $s_{\times}$, 
    only states and transitions of the agent are labelled.
    }
    \label{fig:prop2basecase}
\end{figure}

\emph{Induction step.}  For a general $k$, we use the property for $k-1$. 
Consider the game graph $\G_{\delta}^{k-1}$ satisfying the hypothesis.

This graph contains a state $s$ uncontrollable for delay $\delta$ and with a robustness value of at least 
$\delta + (k-1) + 1 = \delta + k$.
Without loss of generality, we assume that the robustness value is exactly $\varphi_r(s) = \delta + k$.
If $\varphi_r(s)$ was larger,
$\G_{\delta}^{k-1}$ at state $s$ would serve already as $\G_{\delta}^k$ and the induction step would be finished.

State $s$ being uncontrollable means that for any action register
$\overline{\sigma} = (y_1\dots y_{\delta})$, 
the forward multiset $F_{\delta}(s,\overline{\sigma})$
contains at least another new uncontrollable state $s'$ 
--- which may be inside or outside $W$.
The same argument can be applied to the newly introduced uncontrollable states.
Therefore, for any action register $\ol{\sigma}$,
following repeatedly transitions from uncontrollable state to uncontrollable state,
eventually leads to a 
state outside of the winning region $W$.


We define 
$\Pi_U(s)$ as the set of all paths 
starting from $s$
that end outside of the winning region in exactly 
$\delta + k$ transitions.
That is, paths of the form 
$\tau = s, s_1,\dots, s_{2(\delta+k)}$, 
with $s_{2(\delta+k)}\notin W$.
This set is non-empty because $\varphi_{r}(s) = \delta + k$.
We enumerate paths in $\Pi_U(s) = \{\tau_i\::\: i\in I \}$ with some appropriate 
index set $I$.

While these paths are longer than $\delta$,
each of them contains at least one state that is uncontrollable because it leads directly outside the winning region $W$. 
For each $i\in I$, 
the path $\tau_i$ contains a state $s^i$ that is uncontrollable for delay $\delta$ because it leads directly outside the winning region (and not to a state uncontrollable but inside the winning region).

Since $s^i$ leads directly outside the winning region and is uncontrollable,
we have the following construction,
illustrated in Figure~\ref{fig:prop2border}.

From state $s^i$, 
there is at least a path of length $\delta$ that ends 
in a state $s^{b}\in S_{ag}$,
for which one transition with label $x$ leads to a state outside of the winning region $s_{d}\notin W$. 
Since $s_b\in W$, 
there is another transition $y$ leading to a safe state $s^{c}$.
Since $s^i$ is uncontrollable for delay $\delta$, 
there is at least another state $s'_{b}$ with a transition to an unsafe state $s'_{d}$
labeled by $y$, 
and a transition to a safe state $s'_{c}$ 
labelled by another action $z$.
This builds a tuple of states and actions 
$(s_{b}, s_{d}, x)$, 
as illustrated in Figure~\ref{fig:prop2border}.
We do not keep track of the rest of the states and actions defined but keep in mind that they exist.

Note that $s_{b} \neq s'_{b}$,
$s\neq s^i$ and 
$x \neq y \neq z$.
All the other states and actions could be the same. 
In particular, Fig.~\ref{fig:prop2k33} is drawn assuming 
$s_{d} = s'_{d}$.

There may be other tuples of states and outputs
bordering with the unsafe region. 
We enumerate them as
$(s_{b}, s_{d}, x)^{i,j}$, 
where $i$ is the index of $s^i$
and $j\in\{1,\dots, n_i\}$, 
where $n_i$ is the number of different tuples 
when fixed $i\in I$.
When it is convenient to distinguish concrete states and actions for each index $(i,j)$,  
we use the equivalent notation 
$(s_{b}^{i,j}, s_{d}^{i,j}, x^{i,j})$
instead of 
$(s_{b}, s_{d}, x)^{i,j}$.

\begin{figure}[t]
    \centering
    \includegraphics[width=0.85\linewidth,page=2]{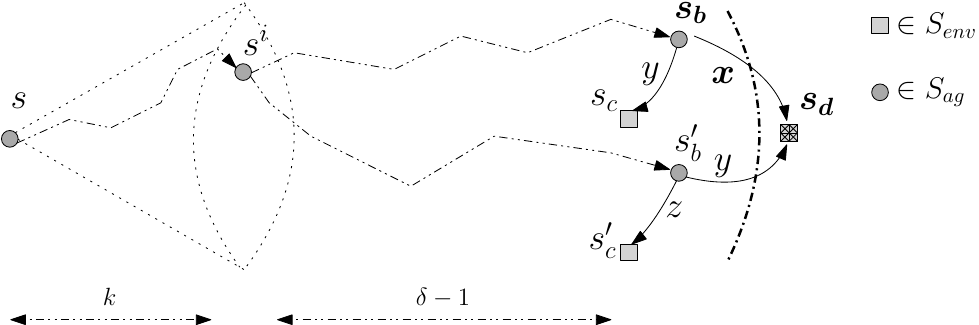}
    \caption[Tuple in the border of the winning region.]
    {Construction of a tuple 
    $(s_{b}, s_{d}, x)$ 
    on the border of the winning region as described in the text.
    Most $i$ superscripts are omitted to make the image cleaner. 
    Marked in bold are the elements that are part of the tuple.
    }
    \label{fig:prop2border}
\end{figure}

For each $i$, we make the following construction.
Consider all the tuples 
$(s_{b}^{i,j}, s_{d}^{i,j}, x^{i,j})$ as previously described 
for $j \in \{1, \dots, n_i\}$. 
We construct a new game $\mathcal G_{\delta}^{k-1}$ from $\mathcal G_{\delta}^k$
where we add a state of $s_{m}^i\in S_{env}$ 
such that 
$s_{b}^{i,j} \xrightarrow{x^{i,j}} s_m^i$ for all $i,j$. 
Then we add new states $s_{a}^{i,j}\in S_{ag}$,
one for each tuple.
We connect these states as follows, for all $i\in I$ and all $j\in \{1,\dots, n_i\}$:
\[
s_{m}^i \xrightarrow{u} s_{a}^{i,j}, \qquad \mbox{and} \qquad
s_{a}^{i,j} \xrightarrow{x^{i,j}} s_{d}^{i,j}.
\]

We also add another family of states, indexed by $k$, denoted $t_k^i\in S_{env}$. 
We add as many of them to make it such that for all $i,j$ and all action $x \neq x^{i,j}$, there exists $k$ such that \[
s_{a}^{i,j}\xrightarrow{x} t_k^i.
\]
We also connect all $t_k$ with all the newly added agent states, that is
\[
t_{k}^i \xrightarrow{u} s_{a}^{i,j}, \qquad
\mbox{for all } i,j,k.
\]
The unsafe states in $\G^k_\delta$ are inherited from $\G^{k-1}_\delta$. 
In particular, recall that the states $s_d^{i,j}$ are unsafe by construction for all $i$ and $j$.

In Figure~\ref{fig:prop2k33}, we illustrate this construction for the case of three tuples ($n_i=3$)
on a single index $i$ and an action set comprised of 
three actions ($\Act = \{x^1, x^2, x^3\}$).
For each value of $i$, 
there would correspond a similar separate construction. 
For larger values of $n_i$, 
the corresponding complete bipartite graph would become larger.

\begin{figure}[t]
    \centering
    \includegraphics[width=0.85\linewidth,page=3]{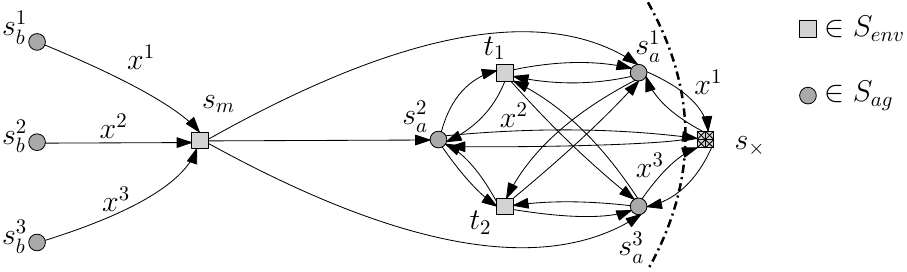}
    \caption[Bipartite graph for the induction case of Theorem~\ref{thm2}]
    {Construction of the bipartite complete graph described in the proof, 
    for a single index $i$ and $n_i=3$. 
    For the sake of simplicity
    all all three unsafe states $s_d^{i,j}$
    for $j\in \{1, 2, 3\}$ are collapsed into a single unsafe state $s_{\times}\notin W$.
    }
    \label{fig:prop2k33}
\end{figure}

The first observation is that the states $s_a^{i,j}$
are only controllable for delay $\delta=0$.
This is because for any register of action $\ol{\sigma} = [x]$, $s^{i,j}_a\xrightarrow{x} t_k$ for some $k$, 
and then $t_k\xrightarrow{u} s_{a}^{i',j'}$ for all $i', j'$. 
So when the observed state is $s_a^{i,j}$,
the current state (with delay $\delta=1)$, 
can be any of the states $s_a^{i',j'}$. 
By construction, any possible action $x$ will be $x=x^{i',j'}$ for some $i',j'$, 
and would lead to the state $s_d^{i',j'}\notin W$.

With this construction, 
the state $s_i$ is still uncontrollable for delay $\delta$, 
because any strategy that made it uncontrollable before
leads now to $s_{m}^i$, 
which is uncontrollable for any delay larger or equal to one, as we have explained in the previous paragraph.

The existence of at least one of the $t_k^i$ for each $i\in I$
is enough to ensure that states $s_{a}^{i,j}$ are safe, 
i.e., $s_a^{i,j}\in W$ for all $i,j$.
By adding enough states $t_k^i$ we ensure that each state newly added to $S_{ag}$
has a defined transition for each action in $\Act$.

Since the states $s_{a}^{i,j}$ are safe, 
the path $\tau_i$ needs to be extended by length 2 to arrive at an unsafe state, 
which would be one of the $s_{d}^{i,j}$.


Repeating this construction for each path $\tau_i$ 
of length $\delta +k$, 
we make all previous paths of length $\delta+k$ go through a construction as illustrated in Figure~\ref{fig:prop2k33} before reaching any unsafe state, making it take at least one more action to reach any unsafe state. 
Thus, the robustness value of $s$ is increased to $\delta+k+1$ in the new game graph $\G^k_\delta$, 
while the controllability value of $s$ stays the same as it was in $\G^{k-1}_\delta$.

\end{proof}


\begin{theorem}
\label{thm3}
    For all delay $\delta > 0$, 
    and all $k > 0$, 
    there exists a safety game
    $\mathcal G_{\delta,\mu=\delta}^k=\langle S, s_0, S_{ag}, S_{env}, \Act, \mathcal T, \mathcal F\rangle$ with one state $s\in S$ satisfying
    \begin{equation*}
      \varphi_{r}(s) \leq \delta \quad \mbox{ and } \quad  \varphi_{c}(s) \geq \delta + k.  
    \end{equation*}
\end{theorem}

\begin{proof}
    Consider a safety game 
    where the environment has only one choice in each state. 
    In these kind of games, 
    a player with memory can know exactly where it
    is making the next move, so each safe state is controllable.
    With this idea in mind, 
    we construct a family of safety games $\mathcal{G}^k$
    for which the initial state $s_0$ satisfies $\varphi_{r}(s_0) = k$ and is 
    controllable with memory for any delay.
    Figure~\ref{fig:prop3} illustrates this family of games.

\begin{figure}
    \centering
    \includegraphics[width=0.55\linewidth, page=4]{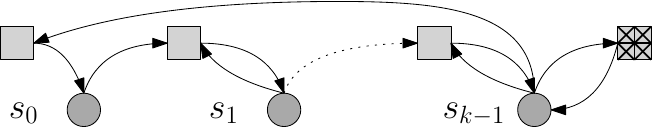}
    \caption{Game graph where $\varphi_{c}(s_0)$ is arbitrarily large, 
    and $\varphi_{r}(s_0) = k$.}
    \label{fig:prop3}
\end{figure}

\end{proof}

\section{Experimental Evaluation}
\label{sec:shields-experiments}

For our experimental evaluation, we evaluate different types of shields
resilient to delays with full memory on two use cases: a simple gridworld and a more complex scenario 
based on a realistic driving simulation.

\subsection{Shielding in a Gridworld}

\begin{figure}
 \centering
 \includegraphics[height=4.2cm,page=2]{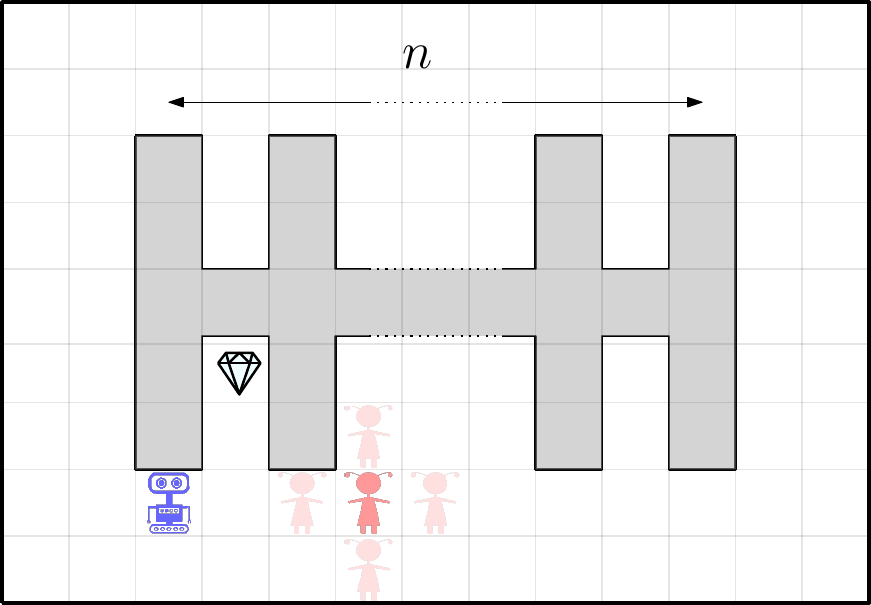}
 \caption{Gridworld with possible states after delay $\delta=1$.}
 \label{fig:delayed-shields-gridworld}
\end{figure}

\paragraph*{Setting.}
Our first case study is an extension of the one from~\cite{Chen2020IndecisionAD}.
Figure~\ref{fig:delayed-shields-gridworld} illustrates a grid world of size $3n+4\times 9$,
where the width is parameterised
by the number of pairs of dead-ends $ n$.
There are two actors that operate in the grid world: 
a robot (controlled by the agent), and 
a kid (controlled by the environment).
The safety specification requires the robot to avoid any collision with the kid.

\paragraph*{Game graph.}
The game graph 
encoding the relevant safety dynamics for the grid world is 
$\G =\langle S, s_0, S_{ag}, S_{env}, \Act, \mathcal T,\mathcal F\rangle$, defined as follows.
\begin{itemize}
    \item $S = X_{env}\times Y_{env} \times X_{ag} \times Y_{ag}\times \mathbb B \setminus P\times P \times \mathbb B$, 
    where $X_{ag/env}= \{1,\dots, 2n+5\}$ and 
    $Y_{ag/env}=\{1,\dots,9\}$ represent the
    $(x,y)$ position of the robot (agent) and the kid (environment), respectively.
    $\mathbb B$ indicates whether it is the turn of the robot or the kid, 
    and $P$ represents the illegal positions, marked in grey in Figure~\ref{fig:delayed-shields-gridworld}.
    Formally, 
    $P = \{((x,5), (2k+1,y)\::\: x\in \{3,\dots, 2n+3\}, \: y\in \{3,\dots, 7\}, \: k \in \{1,\dots, n+1\}\}$.
    The initial state is $s_0 = (0,0,2n+5, 9, \bot)$, 
    indicating that the robot is in the lower left corner, 
    the kid is in the upper right corner
    and it is the kid's turn to move.
    \item The unsafe states are 
    \begin{equation*}
      S\setminus\mathcal F = \{(x_{env}, y_{env}, x_{ag}, y_{ag}, b)\::\:
    (x_{env} = x_{ag}) \:\land\: (y_{env} = y_{ag})\}.
    \end{equation*}
    \item The moves of the kid are defined by an action set $\Act_{env} = \{\mathtt U', \mathtt D', \mathtt R', \mathtt L' \}$, 
    with the usual meanings of up, down, right, left.
    We define a richer action set for the robot
    to compensate for the existence of delays in the input. The action set is
    $\Act = \{\mathtt N, \mathtt U, \mathtt D, \mathtt R, \mathtt L,
    \mathtt{UU}, \mathtt{DD}, \mathtt{RR}, \mathtt{LL},
    \mathtt{UR}, \mathtt{RU}, 
    \mathtt{UL}, \mathtt{LU}, 
    \mathtt{DR}, \mathtt{RD},
    \mathtt{DL}, \mathtt{LD}, \newline
    \mathtt{UUR}, \mathtt{UUL},
    \mathtt{DDR}, \mathtt{DDL},
    \mathtt{RRU}, \mathtt{RRD},
    \mathtt{LLU}, \mathtt{LLD}
    \}$.
    In summary, the robot can move zero, one or two steps in each direction, and can also perform three-step L-shaped moves.
    \item The transitions work as expected. 
    Environment transitions modify the position of the kid $(x_{env},y_{env})$, 
    while the agent's actions modify the position of the robot $(x_{ag}, y_{ag})$.
    Any illegal transition (those that would go out of boundaries or clash with the grey region depicted in Figure~\ref{fig:delayed-shields-gridworld}) is changed to
    $\mathtt{N}$ (``no move'').
\end{itemize}

\begin{table}
\centering
{
\renewcommand{\arraystretch}{1.2} 
\begin{tabular}{llllll}
\multicolumn{2}{c}{Delay (steps)}                                & 0    & 1   & 2     & 3     \\ \hline 
\multirow{3}{*}{Score}                & Pre-shield    & 50.1    & 36.0  & 34.6   & 30.2   \\ \cline{2-6} 
                                      & Robustness    & 42.5    & 34.3  & 31.5   & 26.8   \\ \cline{2-6} 
                                      & Controllability   & 41.3  & 33.9 & 31.8      & 27.5    \\ \hline
\multirow{3}{*}{Interventions}        & Pre-shield    & 117.4   & 150.4  & 160.1 & 182.2  \\ \cline{2-6} 
                                      & Robustness    & 90.9   & 107.5  & 114.1 & 122.0  \\ \cline{2-6} 
                                      & Controllability & 85.0 & 95.9  & 106.9    & 122.7
\end{tabular}
}
\caption{Performance of different shielding strategies.}
\label{tab:exp3}
\end{table}


\paragraph*{Results: interference rates.} To evaluate the interference of the shields during runtime, we implemented a robot with the goal of collecting treasures that are placed at random positions in a grid world with $4$ dead ends. 
At any time step, there is one treasure placed in the grid world. 
As soon as this treasure is collected, the next treasure spawns at a random location.
Collecting a treasure rewards the agent with $+1$ score points.
The kid is implemented such that it chases the robot in a stochastic way.

The interference results are presented in Table~\ref{tab:exp3}.
In the table, the first three rows show the score obtained by the robot,
and the last three rows show the number of times the shield intervenes.
Both score and number of interventions correspond to the amount accumulated over a game of 2000 steps. 
We compare pre-shields with post-shields that maximise either robustness or controllability.
Since both the robot and the kid are implemented with stochastic behaviour,
each data point in the table is the average of 100 plays.

The results show that the agent's score decreases with the delay, as expected.
Since the shield has more uncertainty about the current position of the kid,
it enforces a larger distance between the current position of the robot and the last observed position of the kid. 
For the same reason, the shields need to interfere more frequently 
with increasing delays.
In general, pre-shields compare to post-shields show a better performance in terms of score
and a worse performance in terms of number of interventions, as it was expected.
Additionally, we compared the corrective actions 
chosen by post-shields that maximise controllability with the actions chosen by shields that maximise robustness. 
We noticed that in most states, both shields pick the same corrective action,
which is reflected in the similar results obtained.

\begin{figure}
\centering
\begin{subfigure}[b]{0.48\textwidth}
         \centering
         \includegraphics[width=\textwidth]{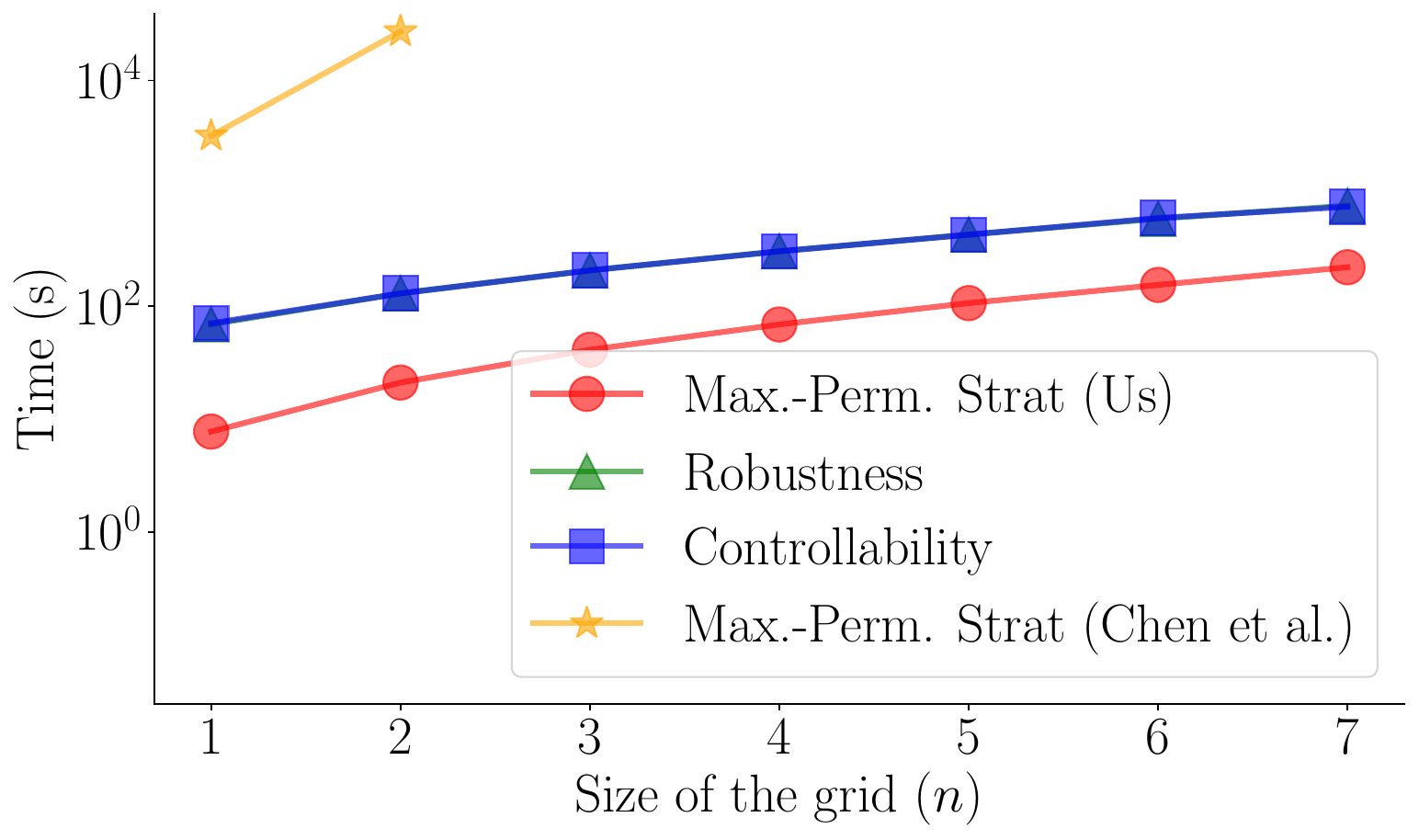}
         \caption{Fixed delay $\delta = 2$, increasing size of the grid.}
         \label{fig:delayed-shields-comptime-fixed-delay}
     \end{subfigure}
    \hfill
     \begin{subfigure}[b]{0.48\textwidth}
         \centering
         \includegraphics[width=\textwidth]{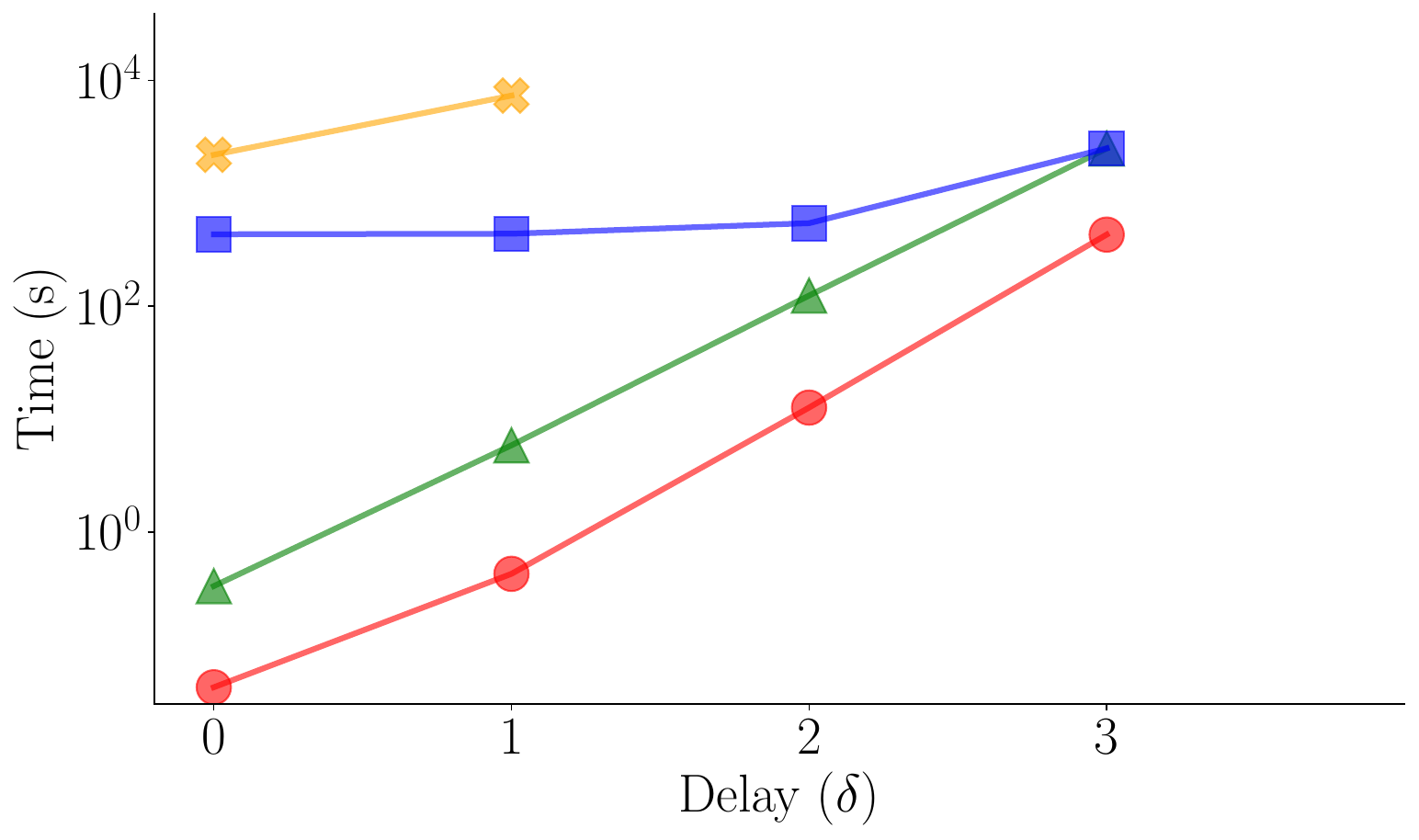}
         \caption{Fixed size of the grid $n = 2$, increasing delay.}
         \label{fig:delayed-shields-comptime-fixed-size}
     \end{subfigure}
        \caption{Shield synthesis times for the grid world experiments.}
        \label{fig:computation_time}
\end{figure}

\paragraph*{Results: synthesis times.}

We compute all types of presented shields. 
The synthesis times are presented in Figure~\ref{fig:computation_time}, 
where Figure~\ref{fig:delayed-shields-comptime-fixed-delay} corresponds to a fixed delay of $\delta=2$, 
and Figure~\ref{fig:delayed-shields-comptime-fixed-size} corresponds to a fixed-size grid with four dead-ends, i.e., $n=2$.

In the figure
we compare the synthesis times for the synthesis of shields, maximising 
robustness (\inlinegraphics{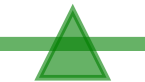}) and controllability (\inlinegraphics{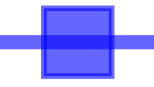}).
We also include the cost of computing the maximally permissive winning strategy, 
which is required for all shields and is the only cost associated with synthesising pre-shields. 
To compare with a baseline,
we show the cost of computing the
maximally-permissive strategy in the delayed safety game 
for our implementation (\inlinegraphics{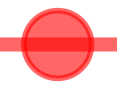})
and the implementation of~\cite{Chen2020IndecisionAD}
(\inlinegraphics{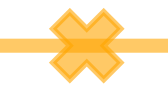}).

The improvement of our method compared to the baseline 
results from a faster implementation in
{$\mathtt{C}$\nolinebreak[4]\hspace{-.05em}\raisebox{.3ex}{\scriptsize\bf ++}},
with only minor algorithmic reasons.
The cutoff value for controllability is set to $\delta_{\max} = 3$.
Since the cost for computing shields grows exponentially with $\delta$,
the synthesis times for shields maximising robustness grow exponentially. 
This effect does not show for shields maximising controllability, 
as they always compute the maximally permissive strategy until delay $\delta_{\max}$ irrespective of the particular delay $\delta$.

\subsection{Shielded Driving in \textsc{Carla}}
We implemented our delayed shields in the driving simulator \textsc{Carla}~\cite{dosovitskiy_carla_2017}. 
In all scenarios, the default autonomous driver agent in \textsc{Carla} is used 
with adequate modifications to make it a more reckless driver.
To capture the continuous 
dynamics of \textsc{Carla}
using discrete models, 
we designed the safety game with 
overly conservative transitions, i.e., accelerations are overestimated, and braking power is underestimated.
In both scenarios, we use delay-resilient shields, maximising robustness.

\subsubsection{Shielding against Collisions with Cars}

We consider a scenario in which two cars (one of them controlled by the driver agent) approach an uncontrolled intersection. The shield has to guarantee collision avoidance for any braking and acceleration behaviour of the uncontrolled car while the observation of the uncontrolled car is delayed. 
A screenshot of the \textsc{Carla} simulation is given in Figure~\ref{fig:carscreenshot}.

\paragraph*{Game graph.}
To compute delay-resilient shields, the 
scenario is encoded as a safety game 
$\G =\langle S, s_0, S_{ag}, S_{env}, \Act, \mathcal T,\mathcal F\rangle$, 
defined as follows.

    The set of states is defined as $S =P_{\mathrm{ag}} \times P_{\mathrm{env}} \times V_{\mathrm{ag}} \times V_{\mathrm{env}}$, 
    where $P_{\mathrm{ag}}$ and $P_{\mathrm{env}}$ represent, 
    respectively,
    the distances of the agent's car and the environment's car to the crossing,
    and 
    $V_{\mathrm{ag}}$ and $V_{\mathrm{env}}$ represent the velocity of the agent's car and the environment's car, respectively.
    The range of modeled distances is $P_\mathrm{ag} = P_\mathrm{env} = \{0, 2, 4, \dots, 100 \}~\unit{\metre}$.
    The range of modelled velocities is
    $V_\mathrm{agent} = V_\mathrm{env} = \{0, 1, 2, \dots, 20 \}~\unit[per-mode=symbol]{\metre\per\second}$.

Each time step in the game corresponds to $\Delta t= 0.5~\unit{\second}$ in the simulation.
Each car can perform three actions: 
$\mathtt a$ (accelerate), $\mathtt b$ (brake) or $\mathtt c$ (coast, touch no pedal). 
Therefore, the set of environment actions is
$\A_{env} = \{\mathtt a_\mathrm{env}, 
\mathtt b_\mathrm{env}, \mathtt c_\mathrm{env}\}$ and the set of actions of the agent is
$\A=\{\mathtt a_\mathrm{ag}, 
\mathtt b_\mathrm{ag}, \mathtt c_\mathrm{ag}\}$.
In our model, braking and throttling have the effect of applying a constant acceleration of $a=\pm 2~\unit[per-mode=symbol]{\metre\per\second^2}$.
Therefore, the position $p_t$ and the velocity $v_t$ at time step $t$ is updated
as
\begin{equation}
    \label{eq:delayed-shields-carla-carcrossing-updates}
    p_{t+\Delta t} = p_t - v_t\Delta t - \tfrac{1}{2}a\Delta t^2, \qquad
    v_{t+\Delta t} = v_t + a\Delta t.
\end{equation}
Unsafe states represent collisions, therefore
$\mathcal{S_\mathrm{unsafe}}=\{(p_\mathrm{agent}, v_\mathrm{agent},
p_\mathrm{env}, v_\mathrm{env}) : p_\mathrm{agent} = p_\mathrm{env}\}$.
From the safety game, we compute delay-resilient shields that maximise the expected robustness.
Note that in the transitions in 
Equation~\eqref{eq:delayed-shields-carla-carcrossing-updates}
the velocity is applied as negative because the car
gets closer to the intersection at every step.

In this use case, we implemented post-shields that always correct to the 
most conservative safe action, with the understanding that $\mathtt c$ (coast) is more conservative than $\mathtt{a}$ (accelerate) and that $\mathtt{b}$ (brake) is more conservative than both $\mathtt{a}$ and $\mathtt{c}$.
 
\begin{figure}
    \centering
     \begin{subfigure}[b]{0.49\textwidth}
         \centering
         \includegraphics[width=\textwidth]{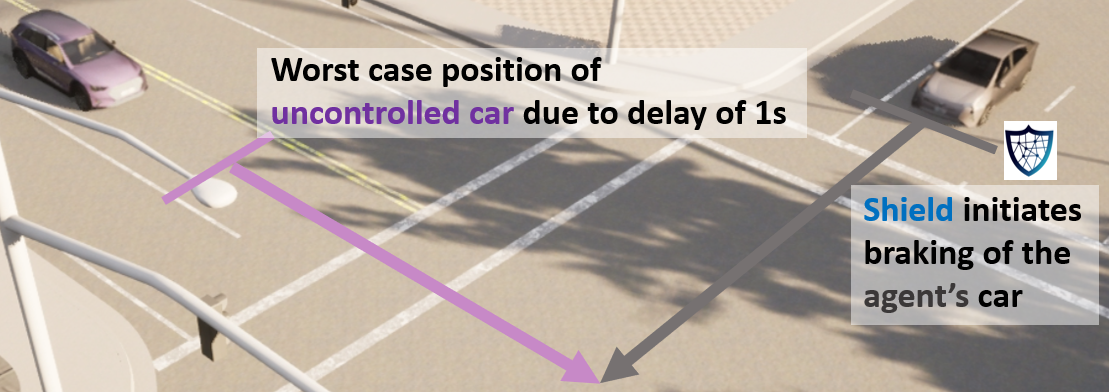}
         \caption{Car intersection.}
         \label{fig:carscreenshot}
     \end{subfigure}
     \hfill
     \begin{subfigure}[b]{0.49\textwidth}
         \centering
         \includegraphics[width=\textwidth]{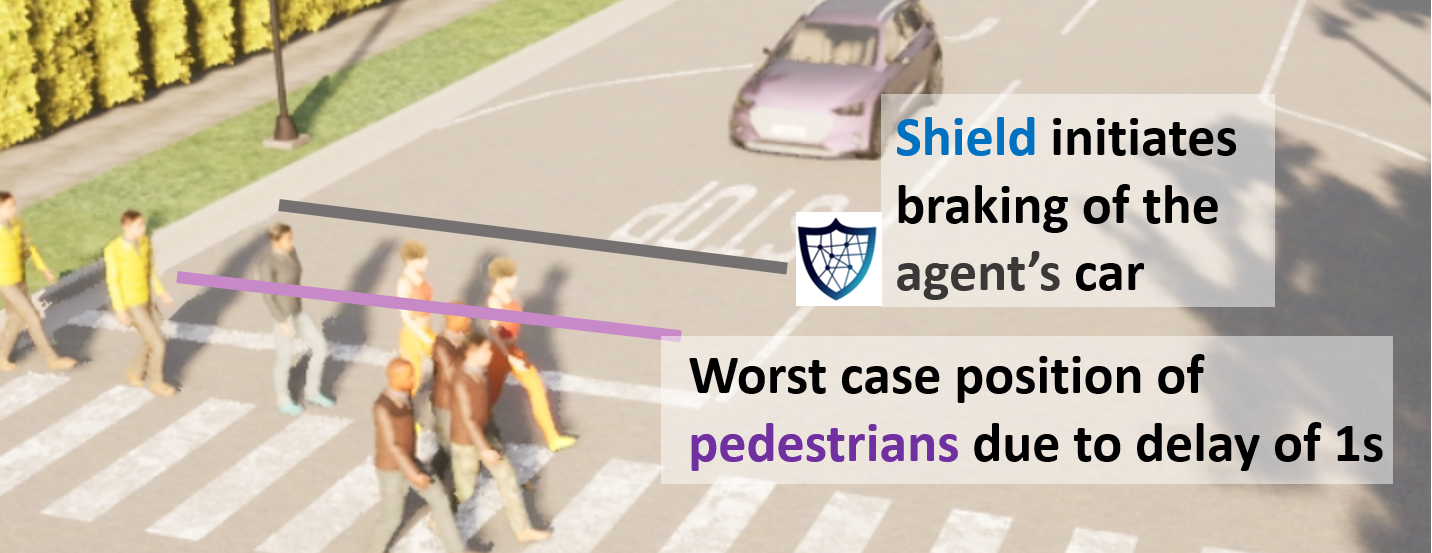}
         \caption{Pedestrians at a crosswalk.}
         \label{fig:carlascreenshot}
     \end{subfigure}
    \caption{Screenshots of the \textsc{Carla} simulator.}
    \label{fig:delayed-shields-carla-screenshot}
\end{figure}

\begin{table}[b]
\centering
{
\renewcommand{\arraystretch}{1.2} 
\begin{tabular}{llllll}
\multicolumn{2}{l}{Delay (in steps)}                                & 0    & 1   & 2     & 3     \\ \hline
\multirow{2}{*}{Synthesis times (in s)}                & Car example   & 
1.5    & 13   & 48   & 167   \\ \cline{2-6} 
                                      & Pedestrian  example 
& 0.8  & 9 & 34     & 119 
\end{tabular}
}
\caption{Shield synthesis times (in seconds).}
\label{tab:computingtimes}
\end{table}

\paragraph*{Results.}
In Figure~\ref{fig:car_results}, we present the speed of the agent's car over time, alongside the occurrences of shield interventions, represented as coloured bars, for various delays measured in increments of $\Delta t = 0.5~\unit{\second}$. As anticipated, the duration of shield interference increases with larger delays.

For a delay of 0, the agent's car brakes continuously until it exits the danger zone. However, as the delay increases, the shield intervenes earlier, ensuring the car accounts for the worst-case behaviour of the other vehicle. The shield always assumes the most adverse environmental conditions, even when these conditions fail to materialise. This conservative approach explains the frequent switching between active and inactive shield states within the same execution, particularly for larger delays.

We evaluated the shields across multiple safety-critical scenarios by varying initial positions and velocities. In all cases, the shields successfully prevented collisions, demonstrating their robustness. Table~\ref{tab:computingtimes} provides the synthesis times required to compute the shields. Each delay step listed in Table~\ref{tab:computingtimes} corresponds to an increment of $\Delta t = 0.5~\unit{\second}$.

\begin{figure}
     \centering
     \begin{subfigure}[b]{0.52\textwidth}
         \centering
         \includegraphics[width=\textwidth]{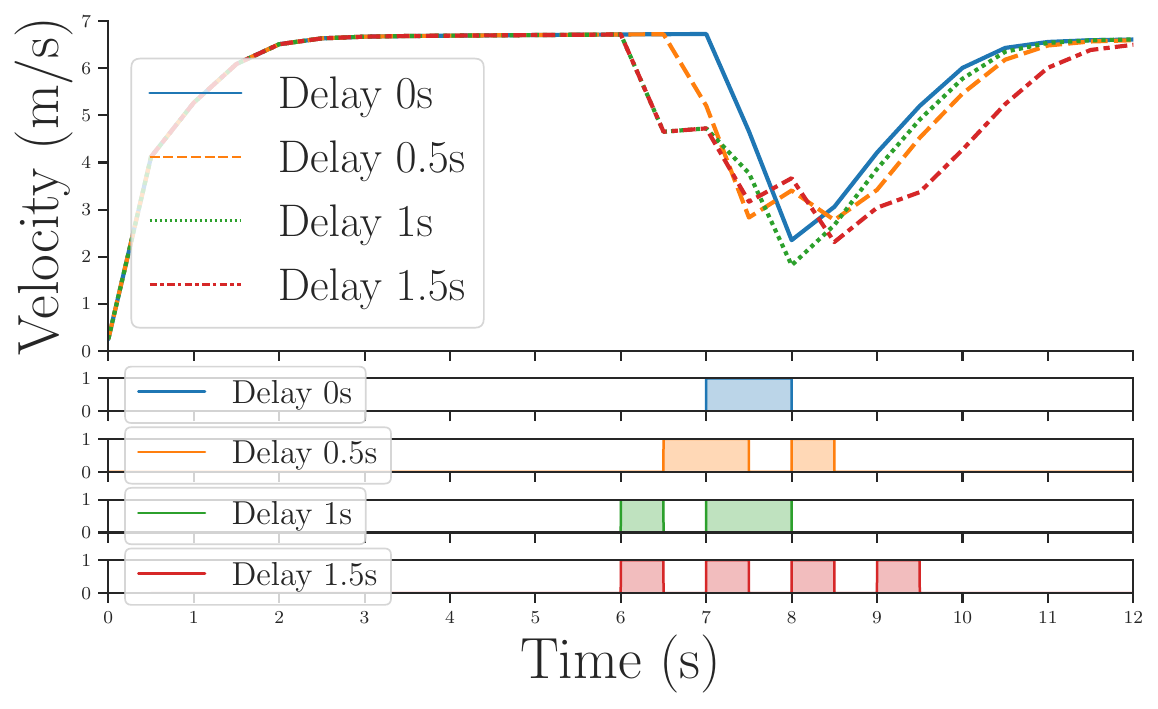}
         \caption{Activation times results.}
         \label{fig:car_results}
     \end{subfigure}
     \hfill
     \begin{subfigure}[b]{0.47\textwidth}
         \centering
         \includegraphics[width=\textwidth]{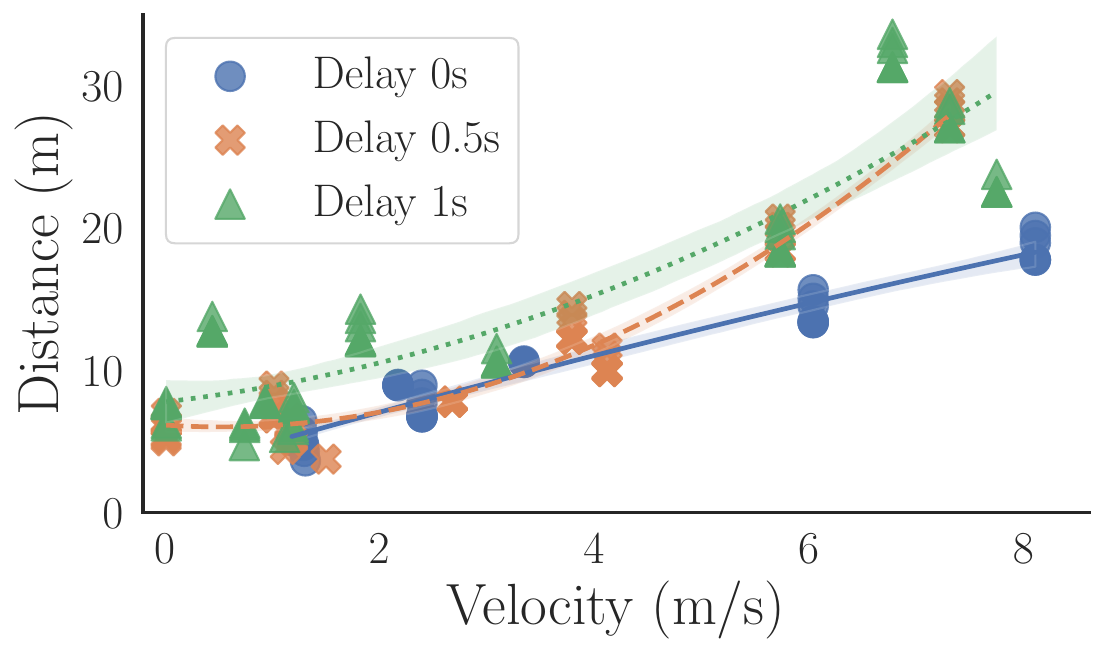}
         \caption{Activation speeds and distances.}
         \label{fig:carlacross}
     \end{subfigure}
    \caption[Experimental results on the driving simulator]
    {Experimental results on the driving simulator. 
    In these experiments, we measure when and how often the shields get activated in each scenario.}
    \label{fig:delayed-shields-carla-experiment-results}
\end{figure}

\subsubsection{Shielding against Collisions with Pedestrians}
In the second experiment, we compute shields for collision avoidance with pedestrians. Similar to before, the shields guarantee safety under delay, even under the worst possible behaviour of the pedestrians. 
A screenshot of the \textsc{Carla} simulation is given in Figure~\ref{fig:carlascreenshot}.

\paragraph*{Shield computation.}
The car, which is controlled by the driver agent, is modelled in the same manner as before. Pedestrians are controlled by the environment and only have their position as state variables. 
In our model, we assume that a pedestrian can move $1~\unit{\metre}$ in any direction within one timestep of $\Delta t=0.5~\unit{\second}$.
We consider a state to be unsafe whenever the ego car moves fast while being close to a pedestrian and the pedestrian is closer to the crosswalk than the car.
Formally
\begin{equation*}
    \mathcal{S_\mathrm{unsafe}}=\big\{(p_\mathrm{ag}, v_\mathrm{ag},
    p_\mathrm{ped}) : 
    (v_\mathrm{ag} > 2~\unit[per-mode=symbol]{\metre\per\second}
    \land |p_\mathrm{ag}-p_\mathrm{ped}|<5~\unit{\metre}
    \land p_\mathrm{ped} < p_\mathrm{ag})\big\}   
\end{equation*}

\paragraph*{Results.}
In Figure~\ref{fig:carlacross}, we illustrate the shield's interference points by plotting the distance to the pedestrian and the car's speed at the moment of each shield intervention. Because pedestrians are modelled to potentially move toward the car, the shield must account for closer pedestrian positions than those directly observed, as delays in sensing introduce uncertainty.

With larger delays, this uncertainty grows, requiring the shield to initiate braking earlier to ensure safety. This conservative approach ensures that the shield compensates for any positional ambiguity introduced by the delay. The synthesis times required for the shield computation are provided in Table~\ref{tab:computingtimes}.

In our experiments, we occasionally observed the system entering states with no available strategy due to discretisation errors. However, despite these occurrences, the safety specification was never violated. Our findings suggest that by using a sufficiently fine-grained model, these discretisation errors can be minimised to the point of being negligible, ensuring the system operates reliably under all tested conditions.

\section{Discussion}
\label{sec:delayed-shields-discussion}
\subsection{Limitations}

We have demonstrated with our experiments
that shielding can be a useful tool to ensure
safety specifications in an application so complex as autonomous driving.
However, this is still a methodology that is not ready to be implemented in today's technology. 
In this section, we discuss the main limitations that hinder the applicability of our method. 
Further research and development is needed to address these challenges.

\paragraph*{Requiring a deterministic model.} 
One of the foundational assumptions of our method is the availability of a deterministic model of the system and its environment. 
While we tackle one of the sources of uncertainty
in this chapter by proposing shields resilient to delayed information, 
this assumption may still be unrealistic.
Many real-world systems operate in inherently stochastic environments, where uncertainties arise due to sensor noise, unpredictable human behaviour, or dynamic external factors. 
Attempting to model such systems deterministically may lead to oversimplifications, resulting in shields that fail to capture the full complexity of the environment. 
In our driving simulator experiments, we circumvent this limitation by finding a rather conservative model that reflects the reality most of the time, 
and observed that this model was good enough in our experiments to enforce the safety specification.
This is an imperfect solution and requires fine-tuning the model for each application, 
making the implementation of shielding more labour-intensive.

\paragraph*{Overly conservative strategies}
Safety shields are designed to handle worst-case scenarios, ensuring that safety is maintained regardless of how adverse conditions may become.
While this is the only way to get the strong safety guarantees that shielding provides,
it can lead to overly conservative strategies that limit the agent's performance and utility excessively. 
Furthermore, the handling of delayed observations only accentuates the shield's conservativism.
Over-conservatism can also erode user trust, as the system may appear unnecessarily cautious or suboptimal in typical operational conditions. 
This trade-off between safety and performance poses a significant challenge and calls into question whether deterministic shielding can be a viable solution.

\paragraph*{Handling divergences from the model}
A fundamental limitation of safety shielding lies in its reliance on a predefined safety game model.
In practice, the real world may diverge from these models due to inaccuracies in modeling or changes in the environment over time. 
When such divergences occur, the shield is forced to react to situations that were not considered reachable in the original safety game. 
In our current approach, the shield's behaviour is undefined in these situations.

Overcoming these challenges will be essential for deploying safety shielding techniques in increasingly complex and dynamic real-world systems.

\subsection{Related Work}
Runtime enforcement is a technique in which a monitor modifies the execution of a system to comply with a specified property~\cite{FalconeP19,RenardFRJM19}. 
Shields for discrete systems were introduced in \cite{DBLP:conf/tacas/BloemKKW15}
and several extensions and applications have already been published~\cite{AlshiekhBEKNT18,DBLP:conf/atal/Elsayed-AlyBAET21,DBLP:conf/amcc/PrangerKTD0B21,0001KJSB20,ijcai2023p637}.

Chen et al.~\cite{chen2018s,Chen2020IndecisionAD}
first investigated the synthesis problem for time-delay discrete systems by the reduction to solving 
two-player safety games. 
We base our shields 
on the proposed algorithm for solving delayed safety game. 
Note that the delayed games discussed by Zimmermann et al.~\cite{KleinZimmermann15,KleinZimmermann15b,Zimmermann17a,DBLP:journals/iandc/WinterZ20}
follow a concept different from the delayed safety games considered in this paper.
In their setting, \emph{a delay is a lookahead} that grants an advantage to the delayed player: 
the delayed agent player P1 lags behind
the environment player P0 in that P1 has to produce the $i$-th action when
$i+j$ environment actions are available. 
In contrast 
to Zimmermann et al., 
we do not grant a lookahead into future inputs but consider delays in the input data. 
It was shown in~\cite{chen2018s} that the different concepts of delays could not be exchanged for each other by a swap of roles
i.e., by exchanging the players and then giving a lookahead of $j$ to the input player in order to simulate a delay of $j$ for the output player.

The notion of delay employed in this paper is different from that in timed games~\cite{DBLP:conf/cav/BehrmannCDFLL07}.
In timed games, delay refers to the possibility of deliberately
delay the next single action. 
However, both players have full and up-to-date information in timed games. 
In~\cite{DBLP:conf/pts/DavidLMNR13} 
a general framework on using UPPAAL-TIGA with partial observability was presented.
Combining both approaches to synthesise delay-resilient shields
from timed automata specifications is potential future work. 

For runtime enforcement in continuous dynamical and hybrid systems, control barrier functions~\cite{DBLP:conf/eucc/AmesCENST19} are used to verify and to enforce safety properties.
Prajna and Jadbabaie extended the notion of barrier certificates to time-delay systems~\cite{Prajna05}.
Bai et al.~\cite{BaiGJX0Z21} introduced a new model of hybrid
systems, called delay hybrid automata, to capture the continuous dynamics of dynamical
systems with delays. 
However, this work does not address the fact that 
state observation in embedded systems is \textit{de facto} in discrete time 
and that a continuous-time shielding mechanism, therefore would require adequate interpolation between sampling points, 
which could be an interesting future endeavour.

%% file: 40_probabilistic_shielding_autonomous_driving.tex
\ifthenelse{\boolean{includequotes}}{
\begin{quotation}
    \textit{Qui no s'arrisca no pisca.}
    \footnote{The one who does not risk, does not gain.}
    \hfill
    --- Catalan popular saying.
\end{quotation}
}{}

\section{Motivation and Outline}

In this chapter, we present work done in the framework of the \textsc{Foceta}~\cite{aisola23foceta} project.
One of the two use cases of the project consists of building an autonomous driving car
capable of operating safely and effectively in a parking lot. 
To achieve this goal, various partners incorporate different components for perception, planning, and movement execution. 
The work presented in this chapter is the theoretical conception and experimental evaluation of a safety element in the form of a probabilistic shield designed to prevent car collisions with pedestrians in the parking lot. 

We have already demonstrated the usability of deterministic shields in autonomous driving use cases in Chapter~\ref{chap:delayed_shields}.
In this chapter, we present an alternative approach using probabilistic shielding. 
By assuming a probabilistic model of the environment, 
the shield can consider low probability events as possible but only 
react to them when the probability of a harmful event goes over a particular threshold value. 

A \emph{probabilistic shield} (see Section~\ref{sec:rdm-probabilistic-shielding-mdps}) is an enforcer that 
overwrites control commands when the probability of violating a safety specification is 
larger than some state-dependent thresholds. 
While probabilistic shielding does not enforce safety with full reliability, 
a probabilistic safety guarantee makes the use of costly measures like emergency braking less intrusive during execution.
In our use case, the agent being shielded is an RL-based controller trained to follow 
a pre-computed trajectory along the parking lot.
Probabilistic shielding is especially well-suited for RL-based controllers, 
since, in both cases, the underlying model of the system is a Markov Decision Process (MDP).

When given a control command $a$, the shield maps the information available 
(from sensors, previous actions, etc.) 
to a state in the MDP and checks that the maximum probability
of avoiding a collision with a pedestrian after executing $a$ in the MPD is large enough. 
In case it is not, the shield overwrites the control command
appropriately.
This probability is computed using probabilistic model checking techniques~\cite{katoen2016probabilistic}.
To do so, we need an explicit description of the MDP. 

Constructing an appropriate MDP model is a challenging task. 
The MDP must faithfully represent the agent and its environment while remaining compact enough to enable feasible model-checking computations. 
However, our model only needs to capture the safety-relevant dynamics of the system.
Our approach factors the model into two components: the ego car, controlled by our agent, and the pedestrians. 
For the ego car, we build an abstraction based on the digital twin model of the Simrod vehicle~\cite{simrod}. 
To derive a tractable MDP from the digital twin~\cite{jones2020characterising,singh2021digital}, we discretise actions and states, incorporating uncertainty into transitions to account for discretisation errors.
Pedestrian behaviour is modelled with movement speeds following a normal distribution, varying the parameters for three different types of pedestrians: adults, elders, and children.

Finally, we evaluate our shielding strategy in several scenarios of the car interacting with moving pedestrians. 
These scenarios are implemented using the proprietary driving simulator \texttt{Prescan}\footnote{\texttt{https://plm.sw.siemens.com/en-US/simcenter/autonomous-vehicle-solutions/prescan}}.
The goal of our experiments is to show that shielding provides a more gentle and efficient safety layer than the coarser approach of an automatic emergency brake (AEB) based only on the expected time to collision.

\paragraph*{Contribution.}
The work presented in this chapter constitutes the first instance of implementing a 
probabilistic shielding approach in a realistic driving simulation environment.
To make it possible, we had to:
\begin{itemize}
    \item Design a suitable MDP structure for the car and the pedestrian, 
    and populate it by mimicking a digital twin model of the car and suitable 
    behavioural models of the pedestrians.
    \item Implement our shielding pipeline and integrate it into an existing agent controlling an autonomous vehicle in a simulated environment.
    \item Validate experimentally the fitness of the model as well as the effectiveness of shielding as a safety measure.
\end{itemize}

\paragraph*{Outline.} In Section~\ref{sec:prob-shields-methodology}
we explain how we build the models required for shielding,
as well as the integration of shielding into the existing controller-simulation framework.
In Section~\ref{sec:prob-shields-evaluation}, we show the results of our experimental evaluation.
Finally, in Section~\ref{sec:prob-shields-discussion}, we discuss limitations and related work.

\paragraph*{Declaration of sources.} 
This chapter is based on work performed by the author of this thesis in the 
framework of the \textsc{Foceta} project~\cite{aisola23foceta}, 
and it reuses material from currently unpublished deliverables of said project.

\section{Methodology}
\label{sec:prob-shields-methodology}

In this section, we discuss the methodology used to adapt the theoretical concept of shielding to our realistic use case. 
We discuss the shielding setting, the constructions of MDP models for the ego car
and the pedestrians, the computation of shields for the models and the integration
of the shielding module on the whole autonomous driving controller.

\subsection{Modeling Scenarios as Markov Decision Processes}

\begin{figure}
    \centering
    \includegraphics[width=0.75\linewidth]{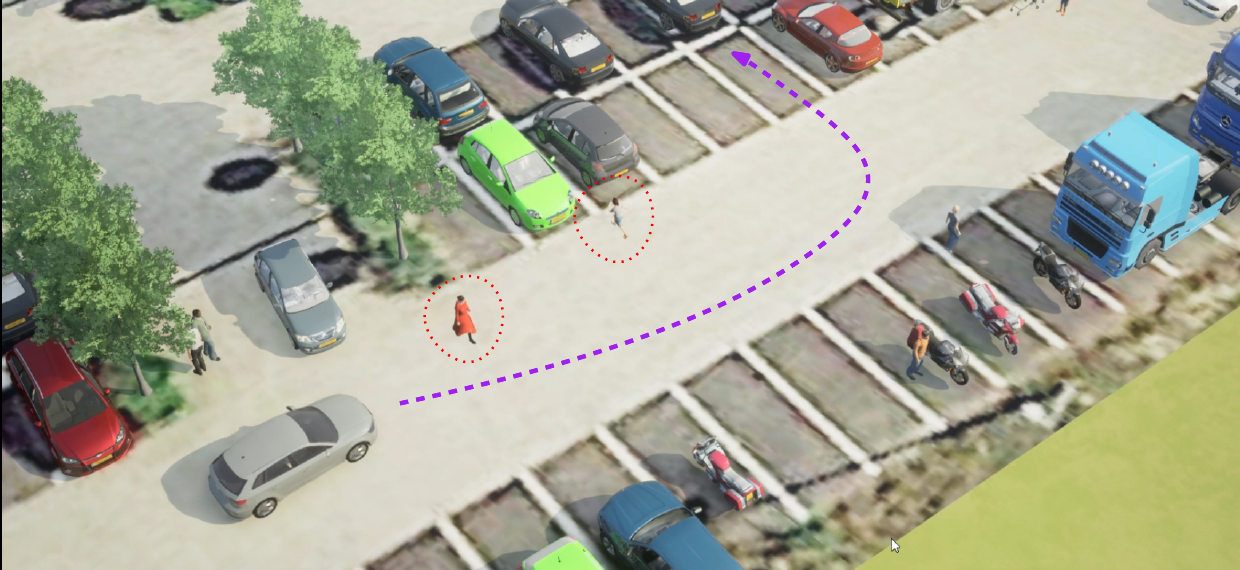}
    \caption[Screenshot of the parking lot simulation]
    {Screenshot of the \texttt{Prescan} simulation. 
    The path of the ego vehicle is marked with a dashed purple line.
    The two pedestrians that are currently being shielded, the ones that are close enough, are marked with red dotted circles.}
    \label{fig:prescan-screenshot}
\end{figure}

In Figure~\ref{fig:prescan-screenshot}, we illustrate the use case with a car being controlled by 
our shielded agent and several pedestrians that move around the parking lot.
The global car controller is broadly composed of a path-planning module and an RL agent that controls the pedals and steering wheel. 
Given an initial position and orientation of the ego car, a target position and orientation, and a map with the static elements of the parking lot, the path-planning module computes the path to follow, and the RL agent controls the pedal and steering commands to best follow said path.
While the goal of the RL agent is to follow the predefined path, 
the goal of the shield is to ensure that the car does not collide with any of the pedestrians while following the path.

Instead of having a unique shield that models interactions with all other pedestrians, we develop shields that model the interaction between the ego car and a single pedestrian. 
For each pedestrian detected in the vicinity of the ego car, we instantiate a shield that ensures collision avoidance against that one pedestrian. 
All instantiated shields work then cooperatively: each shield computes the set of actions that are safe according to its own safety specification. 
Finally, the action proposed by the agent is checked against the intersection of all safe actions and overwritten if needed.

This approach lets us use models that are less complex, with fewer states and transitions, by not modelling pedestrian-to-pedestrian interactions.
The reduced complexity of the models permits using accurate models with low  
computational cost, as well as producing models that are easier to develop, test, and understand.

The drawback is that we have no theoretical guarantee that the intersection of all safe actions is non-empty. It is theoretically possible to have a scenario with two pedestrians in which the only way to avoid colliding with the first pedestrian is accelerating, while the only way to avoid colliding with the second pedestrian is braking. In such cases, the behaviour of the shield is undefined. 
While this is theoretically possible, we have not encountered such cases in our experiments, since braking at maximum strength is typically a safe action at relatively low speeds, regardless of the positions of the pedestrians.

\subsection{MDP Structure and State Discretisation}

We need to model the relevant dynamics of the ego car and the behaviour of pedestrians with an MDP with finite sets of states and actions.
The states are built from sensor readings,
with relevant magnitudes such as positions and velocities being continuous. 
We need a discrete representation of those readings. 
Similarly, the set actions proposed by the agent are continuous, 
so we will need to find a suitable discretisation for the action space.

\index{Markov decision process (MDP)!product MDP}

As stated before, we instantiate an individual shield for each pedestrian. 
Under the assumption that the behaviour of the pedestrian and the ego vehicle are independent, we can build the model for a shield as a product MDP
$\M = \Mcar\times\Mped$,
where $\Mcar = (\Scar, \A, \Pcar)$
is an MDP that encodes the dynamics of the ego vehicle and
$\Mped = (\Sped, \Pped)$
is a Markov chain that encodes the behaviour of the pedestrian.
For both $\Mcar$ and $\Mped$,
the transition-probability function models transitions of a fixed timestep $\Delta t$.

In the following sections, we describe $\Mcar$ 
that models the ego vehicle's dynamics, 
and $\Mped$ that models the behaviour of a pedestrian.
But before that, we need to introduce the three coordinate systems we will be using 
for building the shields.

\begin{figure}
	\centering
	\includegraphics[width=0.55\linewidth,page=1]{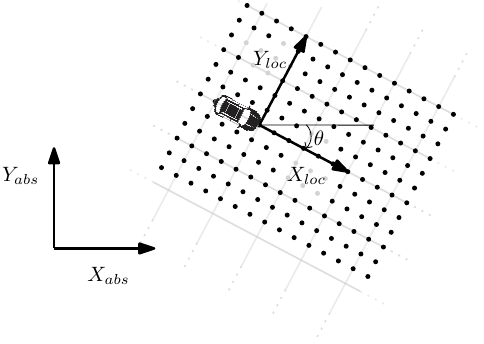}
	\caption[Representation of absolute and local coordinate systems]
    {Representation of absolute and local coordinate systems.
    The dotted grid represents the local-discrete system.}
	\label{fig:coordinates}
\end{figure}

\subsubsection{Coordinate Systems.}
Both sets of states $\Scar$ and $\Sped$ are defined in terms of positions and velocities of the ego vehicle and the pedestrian, respectively. 
To define positions and velocities, we work with three coordinate systems,
depicted in Figure~\ref{fig:coordinates}.
We call them \emph{absolute}, \emph{local-continuous} and \emph{local-discrete}.
\begin{itemize}
    \item \textbf{Absolute}. The absolute coordinate system has its origin at an arbitrary reference point $O$,
    and two orthogonal axes $X$ and $Y$. 
    This coordinate system does not move during an execution, 
    and the input from sensors is assumed to be given in absolute coordinates.
    Coordinates can be any real numbers, and values are in units of 
    metre ($\unit{\metre}$) and metre per second ($\unit{\metre\per\second}$).
    \item \textbf{Local-continuous}. 
    The local-continuous coordinate system is a reference frame that moves with the ego vehicle and is oriented in such a way that the $X$ axis is always parallel to the ego vehicle's velocity.
    In local-continuous coordinates, the position of the ego vehicle is always $(0,0)$, and the velocity is $\left(\sqrt{v_x^2 + v_y^2},0\right)$, where $v = (v_x, v_y)$ is the velocity of the ego car in absolute coordinates.
    Coordinates can be any real numbers, and values are in units of $\unit{\metre}$ and $\unit{\metre\per\second}$.
    The difference between the absolute and local-continuous systems is a translation and a rotation. 
    Therefore, magnitudes stay constant, and the formula to change from local to absolute coordinates is
    \[
    \begin{bmatrix} x\\y\end{bmatrix}_{abs} = \begin{bmatrix} \textrm{car}_x\\\textrm{car}_y \end{bmatrix}_{abs} + 
    \begin{bmatrix}
        \cos\theta & -\sin\theta \\ \sin\theta & \cos\theta   
    \end{bmatrix}\cdot \begin{bmatrix} x'\\y' \end{bmatrix}_{loc},
    \]
    where $\theta$ is the angle between the absolute $X$-axis and the velocity of the ego car, 
    as depicted in Figure~\ref{fig:coordinates}.
    The rotation matrix 
    can be inverted to change from absolute to local coordinates.
    \item \textbf{Local-discrete}. The local-discrete coordinate system is a discretization of the local-continuous system.
    This is the coordinate system used internally in the shield. 
    We define two multipliers $\mu_{\mathrm{pos}}, \mu_{\mathrm{vel}} \in \mathrm{R}$ for the magnitudes of distance and velocity, respectively.
    Magnitudes in local-discrete coordinates are therefore given in units of $\frac{1}{\mu_{\mathrm{pos}}}~\unit{\metre}$ for positions and 
    $\frac{1}{\mu_{\mathrm{vel}}}~\unit{\metre\per\second}$ for velocities. 
    Given a distance magnitude $x$ in local-continuous coordinates, 
    it is transformed to local-discrete as $X = \round{\mu\cdot x}$, 
    where $\mu$ is the corresponding multiplier for the type of magnitude $x$ and $\round{z}$ is the result of rounding $z$ to its closest integer.
\end{itemize}

Since both local-continuous and local-discrete coordinate systems share the same 
$X$-axis and $Y$-axis, 
we will call them in the following the local $X$-axis and local $Y$-axis,
without specifying discrete or continuous.

Finding adequate multipliers $\mu_{\mathrm{pos}}$ and $\mu_{\mathrm{vel}}$
is a matter of finding a suitable compromise.
Larger values provide coarser discretisations, so the models are smaller and easier to work with. However, coarse discretisation incurs larger modelling errors. 
After some trial and error, we found $\mu_{\mathrm{pos}} = \mu_{\mathrm{vel}} = 1/2$ to be a good value for our experiments.

The absolute reference is the reference frame used by the simulator. 
Global variables monitored in the experiments, such as vehicle and pedestrian positions and velocities, are logged in this reference system. 
The local continuous is the system used internally by the car sensors.
The information that arrives to the shield, such as the relative positions of pedestrians to the ego car is obtained in this reference frame.
Finally, the local-discrete is the internal system that the shield uses to represent states in the MDP and to decide on the safety of actions depending on the results given by the model checker. 
During execution, the conversion from local-continuous to local-discrete and back is necessary for every action taken, as the information needs to arrive to the shield and be given back to the car controller. 
In contrast, conversion from absolute to local-continuous is only required to interpret the results of the experiments.

\subsection{Model of the Car}
\label{sec:prob-shielding-model-of-car}

The model of the car is based on the digital twin model of the Simrod vehicle~\cite{simrod}.
The digital twin model was provided to us as a functional mock-up unit (FMU) compatible with both C++ and Simulink.
We are only allowed to interact with this unit as a black box: we can initialise certain variables such as positions and velocities, set an action profile, and let the digital twin run for a given span of time $\Delta t$. 
At the end, we read the new values of the variables of interest.
The digital twin only models the dynamics of the car on an empty road, so any interactions of the ego car with other road users have to be modelled on top of it.

\subsubsection{State Space}

We consider the position and velocity along the $X$-axis of the local-discrete coordinate system. 
Since we only shield for throttle and brake, 
we can assume that the direction of the ego vehicle stays locally unchanged.

Moreover, the MDP model used by the shield is constrained to positions no greater than a predefined limit, $x_{\max}$, and velocities not exceeding a threshold, $v_{\max}$. The value of $v_{\max}$ is set slightly above the maximum speed permitted on the road, ensuring realistic constraints. In contrast, $x_{\max}$ is a more flexible parameter, chosen to be sufficiently large to encompass all possible consequences of the ego vehicle's decisions made at the origin of the local-continuous coordinate system.

Formally, we consider $\Scar = \Xcar \times \Vcar$, where $\Xcar = [0,\dots,x_{\max}]$ represents the position of the ego car along the $X$-axis and $\Vcar = [0,\dots,v_{\max}]$ represents the velocity of the ego car.
The values in $\Xcar$ are integers and represent positions in units of $\frac{1}{\mu_{\mathrm{pos}}}~\unit{\metre}$,
while values in $\Vcar$ are integers representing velocities in units of 
$\frac{1}{\mu_{\mathrm{vel}}}~\unit{\metre\per\second}$.

\subsubsection{Action Space}
Following the convention found in the Simrod model, 
we consider actions in the range $[-1,1]$. 
For an action $a\in [-1,1]$, the sign indicates which pedal to press (brake for $a<0$, throttle for $a>0$), and $|a|$ indicates how much 
the pedal is pressed.
Since the MDP needs to work with a discrete set of actions, 
we define a set of representative actions 
$\mathcal A = \{\alpha_1,\dots,\alpha_n\}$, 
satisfying
$-1 \leq \alpha_1 < \cdots < \alpha_n \leq 1$. 

For both space and action spaces, the input received by sensors and controller is rounded to the closest value in the MDP discretisation.

For the sake of simplicity, the shield can only overwrite the throttle and brake commands and leaves the steering command untouched.

\subsubsection{Transition Probabilities}
We obtain the transition probabilities for our model by probing the Simrod digital twin model at concrete values, as we will describe in this section. 
These experiments are best thought of within the local-continuous coordinate system.

Given an initial position $(x, y)$ and initial velocity $(v_x, v_y)$ 
for the ego car, a driving command $a\in [-1, 1]$ 
and a timestep $\Delta t$, 
the Simrod digital twin provides us with a new position 
$(x+\Delta x, y+\Delta y)$ and 
velocity $(v_x +\Delta v_x, v_y + \Delta v_y)$, 
as the result of applying the driving command $c$ 
for a duration of $\Delta t$.
Figure~\ref{fig:fmu_experiment} illustrates such an experiment.
\begin{figure}
	\centering
	\includegraphics[width=0.55\linewidth, page=2]{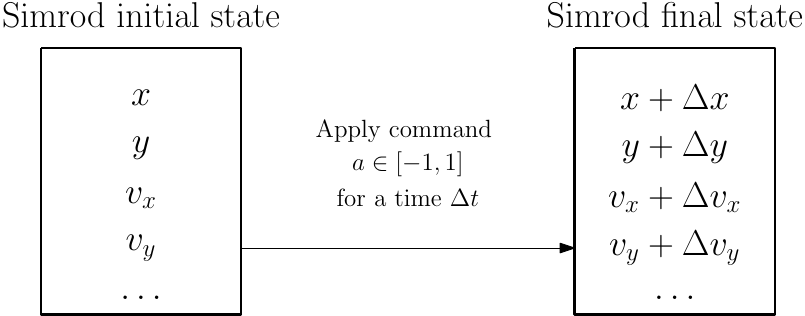}
	\caption{Overview of an experiment on the Simrod model.}
	\label{fig:fmu_experiment}
\end{figure}
In an ideal scenario, we would execute the experiment in Figure~\ref{fig:fmu_experiment} 
for each combination of initial positions, velocities and driving commands. 
However, this is unfeasible because of the complexity of the digital twin model, 
and the need to account for a discretised MDP model.
Therefore, we build the transition probabilities by adjusting them 
to experimental data obtained by the Simrod model. 
We design the following experiment:
\begin{itemize}
    \item We define $m$ reference velocities, 
    $0 < v_1 < \dots < v_{m} < v_{\max}$.
    For each reference velocity $v_j$, 
    the interval
    $rv_j = [(v_{j-1}+v_j)/2, (v_j+v_{j+1})/2)$
    is the range of velocities that have $v_j$ as its reference 
    velocity, i.e. for all $v\in rv_j$, $v_j$ is the closest 
    velocity to $v$ among the set of reference velocities.
    For the first and last ones, we include the minimum and maximum
    velocities, respectively, i.e.
    $rv_0 = [0, (v_1 + v_2)/2)$ and
    $rv_{m} = [(v_{m -1}+v_{m})/2, v_{\max}]$.
    \item Similarly, for each value in the action space $\alpha_i$,
    we define the action range $r\alpha_i = [(\alpha_{i-1}+\alpha_i)/2, (\alpha_i+\alpha_{i+1})/2)$, 
    with the special cases of 
    $r\alpha_0 = [0, (\alpha_1 + \alpha_2)/2)$ and
    $r\alpha_{n} = [(\alpha_{n -1}+\alpha_{n})/2, \alpha_{\max}]$.
    In Figure~\ref{fig:action_ranges}, we provide a visual representation of these ranges for actions. 
    The ranges for velocities are analogous.
    \item We choose a value $N$ of the number of samples used for each reference action and velocity.
    For each action range $r\alpha_i$ and velocity range $rv_j$, 
    we sample $N$ pairs $(a_i^k, u_j^k)$, for $k\in\{1,\dots, N\}$ 
    uniformly at random from $r\alpha_i\times rv_j$. 
    \item For each pair $(a_i^k, u_j^k)$, 
    we perform the experiment described in Figure~\ref{fig:fmu_experiment}, 
    setting the initial velocity in the $X$-direction to $v_x = u_j^k$,
    the remaining initial parameters to zero (i.e. $x=y=v_y=0$)
    and the driving command to $a^k_i$. 
    We label the position and velocity increases as 
    $\Delta x = \Delta x_{i,j}^k$ and $\Delta v_x = \Delta v_{i,j}^k$.
\end{itemize}

\begin{figure}
	\centering
	\includegraphics[width=0.7\linewidth,page=3]{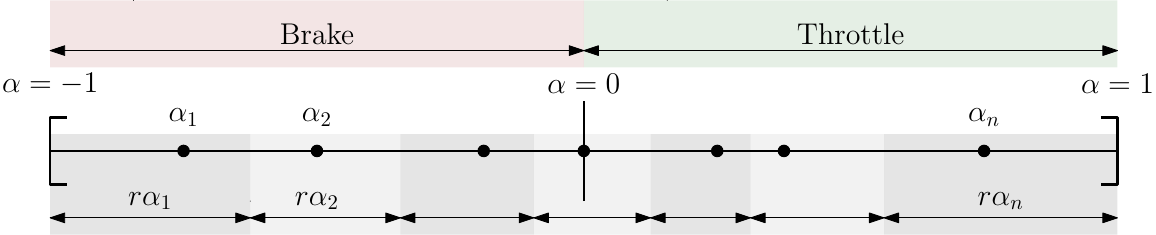}
	\caption{Scheme of reference actions and action ranges.}
	\label{fig:action_ranges}
\end{figure}

To build the transition probabilities of $\Pcar$ from the available data, we adopt the following two assumptions:
\begin{itemize}
    \item \emph{Assumption 1.} 
    The increase in velocity, $\Delta v_{i,j}^k$, primarily depends on the reference action ($\alpha_i$) and the velocity at which it is applied ($u_{i,j}^k$), and does not depend significantly on other variables such as position, previous accelerations, or the concrete applied action ($a_i^k$).
    \item \emph{Assumption 2.} 
    The increase in position, $\Delta x_{i,j}^k$ is proportional to the initial velocity ($u_{i,j}^k$), and the proportionality factor primarily depends on the applied action ($a_i^k$).
\end{itemize}

Assumption 1 implies that transition probabilities should be derived as statistical measures based on the experimental data described earlier,
with a single probability distribution for each reference action $\alpha_i$.
Assumption 2 suggests that for each reference velocity $v_j$ and each reference action $\alpha_i$, there exists a value $\gamma_{i,j}$ such that
$
\Delta x_{i,j}^k \approx \gamma_{i,j}^k u_{i,j}^k
$
provides a good approximation. 
To determine $\gamma_{i,j}$, we minimise the error using the following optimisation criterion. 
For each $i\in [1,\dots, n]$ and $j\in[1,\dots, m]$:
\begin{equation}
\label{eq:thegammaeq}
    \gamma_{i,j} = \underset{\gamma\in\RR}{\arg\min} \sum_{k=1}^N 
    \left(\Delta x_{i,j} - \gamma u_{i,j}^k\right)^2.    
\end{equation}

With all this data, we compute the transition probabilities as follows.
For each $s=(X,V)\in \mathcal \Scar$ and action $\alpha_i\in \Act$, 
let $v_j$ be the corresponding reference velocity to $V$, i.e.,
the only $j$ for which $V/\mu_{\mathrm{vel}}\in rv_j$. 
Then for each $s' = (X+\Delta X, V+ \Delta V)$ such that $\Delta X = \round{\gamma_{i,j}V}$:
\begin{equation}\label{eq:transition_probability}
    \Pcar\big(s, \alpha_i, s'\big) = 
    \frac{1}{N} \#\left\{k\in [1,\dots, m] \::\: \round{\Delta v_{i,j}^k\cdot \mu_{\mathrm{vel}}} = \Delta V\right\}.
\end{equation}
For any $s'$ such that $\Delta X \neq \round{\gamma_{i,j}V}$, 
the transition probability is $\Pcar(s,\alpha_i,s') = 0$.
In Equation~\eqref{eq:transition_probability}, 
the transition probability for a given velocity increase $\Delta V$ is 
defined as the relative frequency of the increase in velocity that 
gets discretized into $\Delta V$.

\subsection{Model of the Pedestrian}

The behaviour of the pedestrian is modelled by a Markov chain
$\Mped = (\Sped, \Pped)$, 
in which the states $\Sped = \Xped \times \Yped$
represent the position of the pedestrian relative
to the same coordinate origin as the car in the local-discrete system.
The probability transition function $\Pped$ determines how the pedestrian behaves in a stochastic manner. 

At each timestep, the pedestrian moves in the $X$-axis and the $Y$-axis following two independent Gaussian distributions. 
The distributions are centred at $0~\unit[per-mode=symbol]{\metre\per\second}$ so that the pedestrian is equally likely to move in any direction. 
The standard deviation $\sigma_{\mathrm{ped}}$ indicates how erratic the movement of the pedestrian tends to be. 
We set $\sigma_{\mathrm{ped}} = 2~\unit[per-mode=symbol]{\metre\per\second}$ for adults, $\sigma_{\mathrm{ped}} = 1~\unit[per-mode=symbol]{\metre\per\second}$ for elders and $\sigma_{\mathrm{ped}} = 3~\unit[per-mode=symbol]{\metre\per\second}$ for children. 
With this model, the average speed of a pedestrian is $\sqrt{2} \sigma_{\mathrm{ped}}$.

A limitation of this model is that pedestrians, as we model them, are non-inertial, i.e., their velocity at one step does not influence their velocity at the next step.
While inertial pedestrians would certainly be more realistic, the issue is mitigated by the fact that the pedestrian's velocities are small.

\subsection{Shield Computation}
A safety specification is given in the form of a set of states 
$S_{\mathrm{crash}}\subseteq \Scar\times\Sped$ 
representing collisions, 
a safety threshold $\lambda\in(0,1)$, and a bounded horizon $k\in \NN$.
Given the models $\Mcar$ and $\Mped$ and a safety specification,
we use \textsc{Tempest}~\cite{tempest} to compute a shield that enforces 
the safety specification, as described in Equation~\eqref{eq:property-of-prob-shields}.
For each tuple $(s_{\mathrm{car}}, s_{\mathrm{ped}}, a) \in \Scar \times \Sped \times \A$,
\textsc{Tempest} determines whether executing action $a$ at
$(s_{\mathrm{car}}, s_{\mathrm{ped}})$ is safe according to the specification.
An action $a$ is considered safe from state $s$ if
\begin{equation*}
 \PP^\M_{\max}\left(\Avoid_{\leq k}(s, a, S_{\mathrm{crash}})\right) \geq \lambda\cdot
 \PP^\M_{\max}\left(\Avoid_{\leq k}(s, S_{\mathrm{crash}})\right).  
\end{equation*}
\textsc{Tempest} produces a lookup table that specifies, for each state,
a safe alternative action to replace any unsafe action.
This lookup table constitutes the shield.


\begin{figure}
\centering
\begin{subfigure}[b]{0.48\textwidth}
         \centering
         \includegraphics[width=\textwidth]{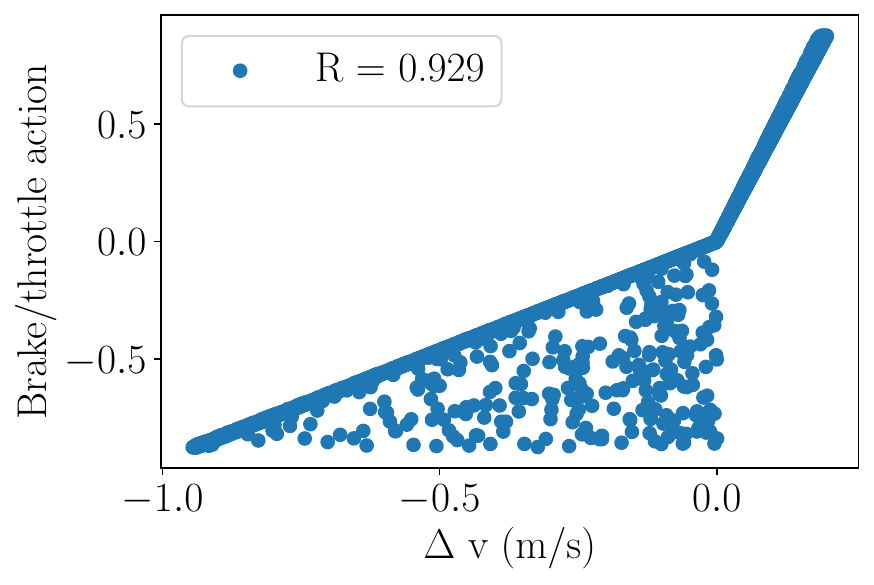}
         \caption{Scatter plot of action $(\alpha)$ vs. $\Delta v$.}
         \label{fig:alpha-deltav}
     \end{subfigure}
    \hfill
     \begin{subfigure}[b]{0.48\textwidth}
         \centering
         \includegraphics[width=\textwidth]{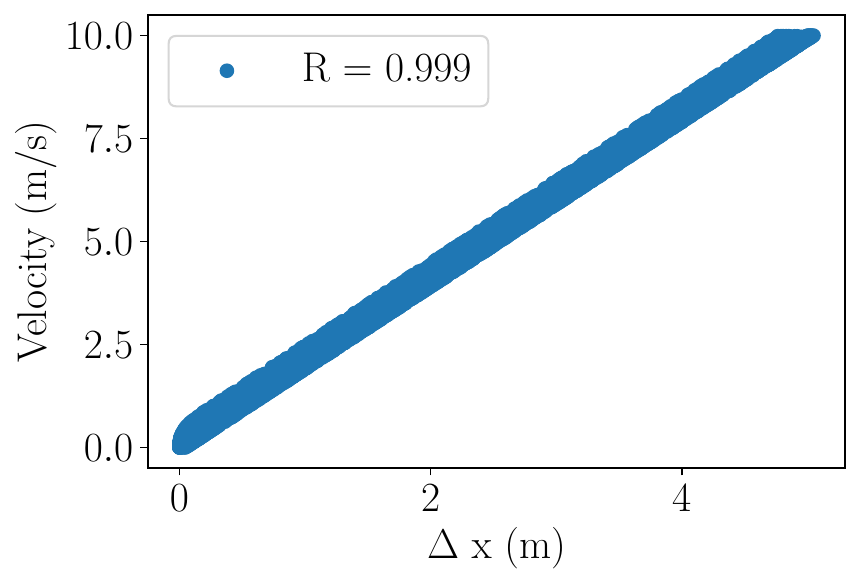}
         \caption{Scatter plot of velocity $(v_x)$ vs. $\Delta x$.}
         \label{fig:deltax-deltav}
     \end{subfigure}
        \caption[Scatter plots to validate the MDP model of the car]
        {Scatter plots to validate $\Mcar$.
        The high correlation factor ($R$) in both cases validates the assumptions on probabilities for the ego car model.
        In particular, the high correlation between the chosen action and the increase in velocity (a) validates Assumption 1, 
        while the high correlation between velocity and $\Delta x$ (b) validates Assumption 2.
        }
        \label{fig:delta-vs-delta}
\end{figure}

\section{Experimental Evaluation}
\label{sec:prob-shields-evaluation}

Our experimental evaluation pursues two primary objectives. First, we empirically validate the assumptions underlying the construction of the car MDP derived from the digital twin FMU model. Next, we assess whether the integrated shields can effectively prevent collisions with pedestrians and determine if their approach outperforms the automated emergency braking system.

\subsection{Validation of the Car Model}

In Section~\ref{sec:prob-shielding-model-of-car} we describe 
a method to obtain experimental data from the Simrod model in a structured way, 
that we then use to build our MDP model of the car. 
The transition probabilities are built from the data, taking two assumptions 
on the dependency of the increases in velocity and position 
($\Delta v$ and $\Delta x$) on the actions applied and the current velocity.

We can validate experimentally the assumptions made to build this model 
by checking the correlations in the data obtained from the experiment described.
In Figure~\ref{fig:delta-vs-delta} we show that both key correlations between $\Delta v$ and $\alpha$ (Assumption 1), 
and between $\Delta x$ and $v_x$ (Assumption 2) are very high.
The data comes from performing the experiment for reference velocities 
$v = [0,1,2,\dots,10]$ m/s, action set $\mathcal A = [-0.75, -0.5, -0.25, 0, 0.25, 0.5, 0.75]$,
and a number of samples $N=100$.
 The high correlation factor ($R$) in both cases validates the assumptions on probabilities for the ego car model.

In Figure~\ref{fig:alpha-deltav} 
there are some data points that look a bit odd for brake actions $(\alpha <  0)$, 
where the decrease in velocity $-\Delta v$ is smaller than expected.
These data points correspond to individual experiments where the initial velocity is already very small, 
so that the brake action applied is more than enough to fully stop the car, 
even with a deceleration smaller than typical for such action.
We also observe a steeper curve in the throttle range than in the brake range.
This indicates that, for this car, 
the deceleration produced by the brake pedal is more potent than the acceleration
produced by the throttle pedal at the same level. 
This is a standard safety feature in automobiles.

\begin{figure}
    \centering
    \includegraphics[width=\linewidth]{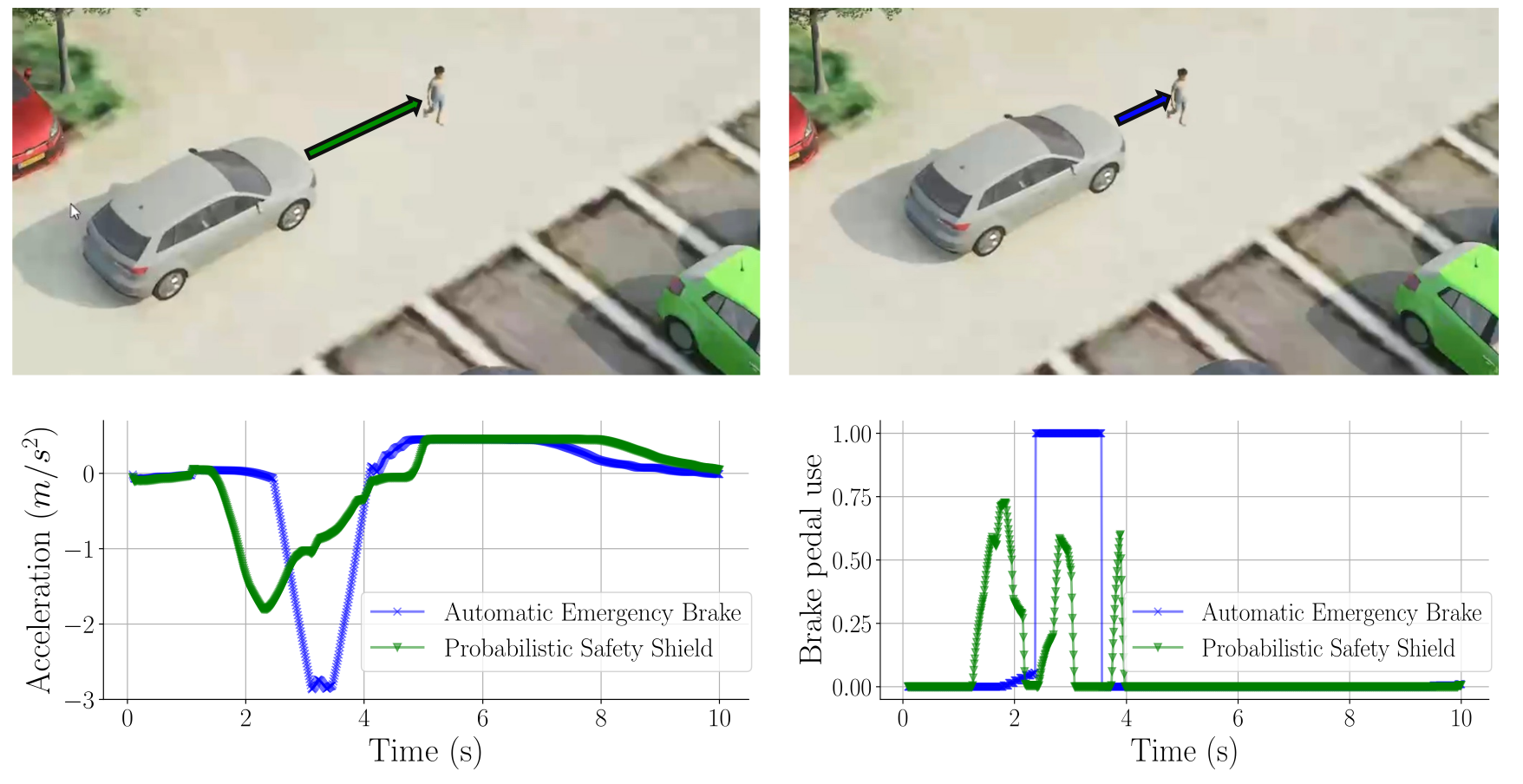}
    \caption[Probabilistic shielding in autonomous valet parking example]
    {Example of the advantage of shielding with respect to an automated emergency brake, avoiding the collision in a smoother and more efficient way.}
    \label{fig:prob-shields-experiment-figure}
\end{figure}

\subsection[Safety Shielding vs. AEB]
{Safety Shielding vs. Automatic Emergency Brake}
To evaluate the performance of the shield, 
we integrate it as part of the agent controlling the ego vehicle developed by several partners in the \textsc{Foceta} project.

In evaluating the performance of the safety shield, we focus on how effective the shield is in enforcing the safety specification. We also assess whether the shield is an improvement with respect to the AEB in terms of efficiency, perceived safety by other road users, and comfort of the passengers. 
Figure~\ref{fig:prob-shields-experiment-figure}
illustrates a case where the shield mitigates collision risks more efficiently and earlier than the standard emergency brake, resulting in lower usage of the brake pedal and more gentle deceleration.
To evaluate the performance of the safety shield in a systematic way, we created a scene in the AVP scenario with fixed initial and goal positions and added several pedestrians. We produced 20 configurations of the pedestrians’ initial positions and moving patterns by random sampling the pedestrians' initial positions and velocities. We executed the scene for a fixed timespan of 20 seconds, having (a) only the shield as a safety enforcer, (b) only the emergency brake as a safety enforcer and (c) both the shield and the emergency brake together. In the latter case, whenever both systems propose an enforcement action, the emergency brake takes priority over the shield. We tested the following metrics.

\begin{itemize}
    \item \textbf{Effectiveness in avoiding collisions.}
    \begin{itemize}
        \item \emph{Distance to pedestrian.}
        The average distance from the front of the car to the pedestrian. 
        Larger values indicate increased safety.
        \item \emph{AEB activation.} The percentage of time in the 20-second experiment where the automated emergency brake is active. 
    \end{itemize}
    \item \textbf{Efficiency in driving.}
    \begin{itemize} 
        \item \emph{Brake pedal.} The average value of the brake pedal signal. Recall that both pedals have a signal from 0 -- not activated -- to 1 -- fully activated.
        \item \emph{Throttle pedal.} Use of the throttle pedal, analogous to the brake pedal. 
        A higher value for either pedal metric suggests a driving style that may accelerate the wear and tear on the vehicle.
        \item \emph{Distance to goal.} Average distance to the goal position during the experiment. Lower values indicate increased efficiency in reaching the goal.
    \end{itemize}
    \item \textbf{Comfort.}
    \begin{itemize}
        \item \emph{Acceleration.} Average value of the acceleration of the ego car. The average is taken in absolute value. 
        Lower values indicate increased efficiency and comfort.
        \item \emph{Jerk.} Average value of the jerk felt by the ego car, that is, 
        the variation in acceleration. 
        This is a standard measure of comfort.
        \item \emph{Time to collision.} The hypothetical time that it would take to collide with the closest pedestrian if the car would maintain its current speed and trajectory. 
        It is a measure of perceived safety by other road users.
    \end{itemize}
\end{itemize}

\begin{table}
\centering
{
\renewcommand{\arraystretch}{1.2} 
\begin{tabular}{l|ccc}
    Test & Only AEB & Both & Only shield \\ 
    \midrule
    Distance to pedestrian ($m$) & 
    $\bm{4.20} \pm 1.00$ & $3.95\pm 0.86$ & $4.07\pm 0.96$ \\
    AEB activation ($\%$) & 
    $19 \pm 11$ & $\bm{8} \pm 8$ & --- \\
    Brake pedal (avg. use, $0$ to $1$) &
    $0.39 \pm 0.22$ & $0.49 \pm 0.26$ & $\bm{0.29} \pm 0.27$ \\
    Throttle pedal (avg. use, $0$ to $1$) &
    $0.51 \pm 0.21$ & $0.46 \pm 0.24$ & $\bm{0.40} \pm 0.18$ \\ 
    Distance to goal ($m$) &
    $13.2 \pm 4.9$ & $14.2 \pm 4.4$ & $\bm{12.8} \pm 5.2$ \\
    Acceleration ($m/s^2$) & 
    $0.49 \pm 0.21$ & $0.44 \pm 0.18$ & $\bm{0.36} \pm 0.20$ \\
    Jerk ($m/s^3$) & 
    $1.00 \pm 0.80$ & $0.93 \pm 0.86$ & $\bm{0.65} \pm 0.65$ \\
    Time to collision ($s$) &
    $4.2 \pm 1.8$ & $\bm{4.4} \pm 1.6$ & $4.0 \pm 2.0$ \\    
\end{tabular}
}
\caption[Quantitative analysis of probabilistic shielding]
{Quantitative analysis of probabilistic shielding. 
Marked in boldface the best result for each metric on average.}
\label{tab:prob-shielding-results}
\end{table}

In Table~\ref{tab:prob-shielding-results}, we present the results we obtained. For each metric being measured, we provide mean and standard deviation across all our experiments.
We do not include the number of collisions as the scenarios are designed in such a way that there would be a collision, but the safety mechanism (be it the shield, the AEB, or both) has to act to prevent it.

In terms of effectiveness, we can see that the three methods show a similar performance in terms of maintaining a safe distance with respect to the closest pedestrian, and we see that when the shield is active, the use of the emergency brake is down by half. 
In terms of driving efficiency, we see that the shield tends to produce less use of both the brake and throttle pedals while maintaining a low distance to the goal. 
In our results, however, we do observe that having both enforcing systems together produces higher use of the brake pedal and higher distance to the goal, suggesting that the resulting controller may be overly conservative.
In terms of comfort, we observe mainly a notable difference in jerk, where shielding significantly reduces the discomfort to the passengers due to high jerk, associated with a harsh use of the brake and throttle pedals. 
The data from acceleration supports this claim as well, albeit in a lower magnitude. Time to collision proves to be very similar across the board, with the shielded controller showing a slight reduction.

\section{Discussion}
\label{sec:prob-shields-discussion}

\subsection{Limitations}
\paragraph*{Probabilistic safety guarantees.}
The approach towards safety using probabilistic shielding mitigates two of the main concerns discussed in 
the previous chapter with regard to deterministic shields,
namely the requirement of a deterministic model of what are sometimes inherently stochastic phenomena, and the worst-case scenario guarantees generating overly conservative controllers.
This step up is made available at the price of relaxing the safety guarantees.
This relaxed specification is also somewhat unintuitive, 
as we have seen in Example~\ref{ex:MDP-shielding-counter},
which can work towards eroding the trust of the user in the runtime enforcement method.

\paragraph*{Model size and control variables.}
Discretising both the observation and action spaces inevitably introduces errors in the model. However, this trade-off is necessary to keep the model size manageable, enabling the use of probabilistic model-checking methods with reasonable resource consumption.
This affects, in our case, both the car and the pedestrian models.
For the pedestrian, a richer model of their behaviour,
distinguishing behaviour modes depending on their context
or having a more fine-grained account of their velocities would make the model 
significantly more useful.
For the car, allowing a larger model would allow us to introduce steering as a control variable to be shielded. 
While steering is not required to enforce safety in our use case,
some pedestrians can be more efficiently avoided by a gentle steer than the
use of the brake pedal.

\subsection{Related Work}
Probabilistic shielding in MDPs was first introduced in~\cite{0001KJSB20},
and has been extended to partially observable MDPs~\cite{Carr2022}.
To the best of our knowledge, probabilistic shielding on MDPs 
has not been previously used for realistic self-driving use cases.

\paragraph*{Probabilistic model checking tools.}
Several tools implement probabilistic model checking to verify the safety of RL agents. 
COOL-MC~\cite{coolMC} takes a Gymnasium~\cite{towers2024gymnasium} environment and an MDP model as inputs, querying the agent's decisions across all MDP states and using the resulting Markov chain to verify a user-defined property. 
MoGym~\cite{gros2022mogym} converts a user-provided MDP into a Gymnasium-compatible environment, enabling RL policy training and statistical model checking through policy queries.
Unlike these verification-based approaches, shielding does not verify the agent's policy but ensures its correct execution at runtime. 
The most widely used tools for probabilistic model checking include \textsc{Storm}~\cite{hensel2022probabilistic}, \textsc{Prism}~\cite{kwiatkowska2011prism}, 
\textsc{Modest}~\cite{HartmannsH14},
and PET~\cite{meggendorfer2024}. Our work utilises \textsc{Tempest}~\cite{tempest}, a fork of \textsc{Storm} specifically designed to synthesise shields for Markov decision processes and stochastic multiplayer games. 

\paragraph*{Shielding methods for autonomous driving.}
RL has been one of the main methods to develop autonomous driving agents~\cite{kiran2021deep,Pan2017VirtualDriving}.

There is further work on probabilistic safe RL methods that fit in the shielding framework, 
even though they do not use the same formalism of model checking on MDPs. 
Most of this work focuses on collision avoidance and safe driving~\cite{Bouton2019ReinforcementDriving, Krasowski2020SafePrediction, Lin2024Safety-awareDriving, He2023Fear-Neuro-InspiredDriving, Saxena2019DrivingLearning}. 
Shielding methods for RL have been used to optimise the navigation path of a self-driving car~\cite{haritz2024enhancing, vu2024validation} or a platoon of self-driving cars~\cite{brorholt2024compositional}.
Moreover, they have also been used for vehicle trajectory tracking control tasks~\cite{xuan2022sem}, as well as to optimise self-driving car fuel consumption during traffic congestion~\cite{Cheng2019End-to-EndTasks}.
Most of the work that integrates any type of shielding for self-driving car applications has been tested in simulation \cite{Cheng2019End-to-EndTasks, brorholt2024compositional, haritz2024enhancing, xuan2022sem, vu2024validation}.
There is also recent work implementing autonomous driving capabilities in car scale prototypes such as the F1Tenth~\cite{Okelly2020F1TENTH:Learning} to validate the proposed solutions~\cite{Kochdumper2023ProvablyZonotopes}.

%% file: 50_fairness_enforcement.tex
\ifthenelse{\boolean{includequotes}}{
\begin{quotation}
    \textit{Jeder nach seinen Fähigkeiten, jedem nach seinen Bedürfnissen.}\footnote{From each according to their ability, to each according to their needs.}
    \textcolor{white}{a} \hfill
    --- Popular socialist slogan.
    \footnote{While it was Karl Marx who most popularized this saying, it was a common slogan within the socialist movement of the XIX century, and its origin is still a disputed fact.}
\end{quotation}
}{}

\section{Motivation and Outline}
\label{sec:fairness-motivation-outline}

With the rise of machine learning (ML) in human-centric decisions, such as banking and college admissions, concerns about bias based on protected attributes like gender and race have grown~\cite{dressel2018accuracy,obermeyer2019dissecting,scheuerman2019computers,liu2018delayed,berk2021fairness}. 
Mitigating such biases is a crucial and active research area in AI.

Most bias prevention methods rely on \emph{design-time} interventions, such as pre-processing training data~\cite{kamiran2012data,calders2013unbiased}, modifying loss functions~\cite{agarwal2018reductions,berk2017convex} --- known as \emph{in-processing methods} ---, or post-processing decisions with calibrated output functions~\cite{hardt2016equality,caton2020fairness}. 
We introduce \emph{fairness shielding}, the first \emph{run-time} intervention method to safeguard fairness in deployed decision-makers.

Fairness shields address fairness in \emph{sequential} decision-making, where observations come one after the other, and decisions have to be made without knowledge of the following observations.
While fairness has traditionally been studied in a history-independent way, the sequential setting better models real-world decisions~\cite{zhang2021fairness}. 
Prior work mostly focuses on fairness over the long run in unbounded horizons~\cite{hu2022achieving}, but \emph{finite-horizon} and \emph{periodic} fairness --- evaluating fairness over fixed timeframes --- better align with real-world regulatory assessments~\cite{pmlr-v235-alamdari24a}. Our fairness shields enforce these fairness guarantees by monitoring decisions and intervening only when necessary. 

Figure~\ref{fig:shield-schematic} illustrates fairness shielding. Given a predefined fairness criterion and time horizon, the shield observes the protected attribute, classifier (agent) recommendation, and the cost of altering that recommendation. The final decision ensures fairness while minimizing intervention costs, with costs either specified by the decision-maker or assumed constant.

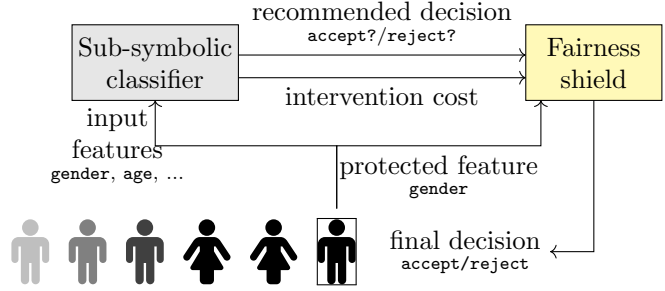
\begin{figure}[t]
    \centering
    \begin{tikzpicture}
        \draw[fill=black!10!white]   (-1,0)   rectangle   (1.2,1)   node[pos=0.5, align=center]    {Sub-symbolic\\classifier};
        \draw[fill=yellow!35!white]   (5,0)   rectangle   (6.8,1) node[pos=0.5,align=center]  {Fairness\\ shield};

        \draw[->]   (1.2,0.3)   --  node[below,align=center]    {intervention cost}  (5,0.3);
        \draw[->]   (1.2,0.6)   --  node[align=center,above]    {recommended decision\\[-0.15cm] {\scriptsize \texttt{accept?}/\texttt{reject?}}}    (5,0.6);

        \node   (a)  at   (2.5,-2)   {\fbox{\Huge\faMale}};
        \node   (b) [left=0.01cm of a]    {\Huge\faFemale};
        \node  (c) [left=0.01cm of b]   {\Huge\faFemale};
        \node[black!75!white] (d) [left=0.01cm of c]    {\Huge\faMale};
        \node[black!50!white] (e) [left=0.01cm of d]    {\Huge\faMale};
        \node[black!25!white] (f) [left=0.01cm of e]    {\Huge\faMale};

        \draw[->]   (a)  -- (2.5,-0.6)   --  (0.1,-0.6) --  (0.1,0) node[align=center]    at    (-0.4,-0.6)  {input\\ features\\[-0.15cm] {\scriptsize \texttt{gender},  \texttt{age}, ...}};
        \draw[->]   (a) --  (2.5,-0.6)  -- node[below,align=center]  {protected feature\\[-0.15cm] {\scriptsize \texttt{gender}}} (5.2,-0.6)  --  (5.2,0);

        \node[align=center] (x)   at  (4.2,-2)  {final decision\\[-0.15cm] {\scriptsize \texttt{accept/reject}}};
        \draw[->]   (5.9,0)   --  (5.9,-2)    --  (x);
    \end{tikzpicture}
    \caption{The operational diagram of fairness shields.}
    \label{fig:shield-schematic}
\end{figure}

\begin{example}[Running example - Bilingual team]
\label{ex:running example}
    Consider the process of assembling a customer service team for a company in a bilingual country, where language $A$ and $B$ hold both official status.
    Because of the nature of the task, it is essential to maintain a balanced representation of native speakers of both languages.
    To achieve this, the company enforces a policy requiring that the difference between the number of employees proficient in each language must not exceed $20\%$ of the total team size. 
    The hiring process operates within a bounded time horizon, with a fixed number of $T$ candidates to be screened.
    Candidates apply sequentially, and decisions about each applicant must be made before considering future candidates. 
    Suppose the company uses an ML model to screen candidates, which is designed without considering linguistic balance, as it is irrelevant in other regions where the company operates, and is biased towards language $A$ candidates.
    Relying solely on this ML model's recommendations could lead to an unbalanced team composition.
    The recruitment team could follow the ML models' recommendation until achieving a balanced team becomes impossible, and then hire some language $B$ candidates.
    This would create a situation in which 
    many qualified language A candidates might have to be rejected. 
    Furthermore, it could also prolong the process of finding suitable language B candidates. 
    Even if the ML model aims for long-term workforce balance, an influx of strong candidates from one language group could still skew the team's composition.
    
    A fairness shield can be deployed, which will monitor and intervene in the decisions at runtime to guarantee that the final team is linguistically balanced as required while keeping the deviations from the decision-maker's at a minimum.
\end{example}

\paragraph*{Computation of fairness shields.} 
Fairness shields are computed by solving bounded-horizon optimal control problems, which incorporate a \emph{hard fairness constraint} and a \emph{soft cost constraint} designed to discourage interventions. 
For the hard fairness constraint, we consider the empirical variants of standard group fairness properties, like demographic parity and equal opportunity. We require that the \emph{empirical bias remains below a given threshold} with the bias being measured either at the end of the horizon or periodically. This hard constraint corresponds to the shield being ``correct'' with respect to a given specification (Definition~\ref{def:correctness-of-a-shield}).

For the soft cost constraint, we assume that the shield receives a separate cost penalty for each decision modification. 
The shield is then required to minimize the total expected future cost, either over the entire horizon or within each period.
The definition of cost may vary by application.
In general, the cost should be associated to the confidence in the classification, 
with high-confidence recommendations requiring high costs.

For shield computation, we assume that the distribution over future decisions (of the agent) and costs are known, either from the knowledge of the model or \emph{learned} from queries.
Fairness shields are computed through dynamic programming. 
While the straightforward approach would require exponential time and memory,
we present an efficient abstraction for the dynamic programming algorithm that reduces the complexity.

\paragraph*{Types of fairness shields.}
We propose four types of shields: (i)~\FinShield, (ii)~\StaticDP, (iii)~\StaticBAR, and (iv)~\Dyn shields. 
\FinShield is specific to the bounded-horizon problem, ensuring fairness in every run while being cost-effective. 
The other three are suited for the periodic setting, 
guaranteeing fairness under diverse assumptions on how often individuals from each group will appear in a period.
\StaticDP and \StaticBAR reuse a statically computed \FinShield shield for each period, while \Dyn shields require online re-computation of shields at the start of each period.


\paragraph*{Experiments.} 
We empirically demonstrate the effectiveness of fairness shielding on various ML classifiers trained on well-known datasets.
While unshielded classifiers often show biases, their shielded counterparts are fair in \emph{every} run in the bounded-horizon setting and in most runs in the periodic setting.
In most cases, the shielded classifiers exhibit a slightly lower classification accuracy as their unshielded counterparts. This discrepancy is more pronounced under stricter fairness conditions and less pronounced, if the classifier was already trained to be fair.

\paragraph*{Contributions.} 
The contributions presented in this work can be summarized as follows.
\begin{itemize}
    \item We formalize the concept of fairness shields, the first runtime intervention procedure for safeguarding the fairness of already deployed decision-makers.
    \item We propose an efficient algorithm for synthesizing fairness shields for finite horizons and explain how it can be extended to fairness shields in a periodic setting.
    \item We study the problem of safeguarding for periodic fairness and propose three solutions formalized in three types of shields: static-fair, static with bounded welfare, and dynamic.
    For each of the proposed solutions, we study under which assumptions they guarantee periodic fairness.
    \item We evaluate our shields with extensive experiments on several benchmark datasets, shielding ML agents trained with state of the art in-processing fairness learning methods.
    In our experiments, we show the effectiveness of our shields, validating our theoretical results and evidencing the gap between theoretical and practical guarantees.
\end{itemize}

\paragraph*{Outline.}
In Section~\ref{sec:fairness-shielding-setting} we present the formal setting and how it fits within the general reactive decision-making framework presented in Chapter~\ref{chap:reactive_decision_making}. 
In Section~\ref{sec:synthesis} we present our main algorithm to synthesize fairness shields for the finite horizon, which is later re-used for the periodic fairness setting. 
We present a general algorithm and a more efficient version for typical fairness properties.
In Section~\ref{sec:unbounded} we present diverse approaches to extend finite horizon shields to an unbounded horizon in a periodic manner. In the periodic setting we loose the strong fairness guarantees of the finite horizon setting, so we focus most of the section on results studying under which conditions fairness can be guaranteed. 
We present our experimental evaluation in Section~\ref{sec:fairness-experimental-evaluation}, and finish the chapter in Section~\ref{sec:fairness-discussion} discussing edge cases, limitations, and related work.

\paragraph*{Declaration of sources.}

This chapter is partially based and reuses material from the following source
previously published by the author of this thesis:

\cite{aaai25} \fullcite{aaai25},

\cite{arxivVersion} \fullcite{arxivVersion}.

\section{Fairness Shielding Setting}
\label{sec:fairness-shielding-setting}

In this section, we present the setting and notation elements that will be used throughout this chapter.
As illustrated in Figure~\ref{fig:shield-schematic}, the problem of fair classification in this chapter can also be interpreted in the general reactive decision-making framework. 
As we have done in previous chapters, we will use a slightly adapted notation, focusing on the relevant elements of the work presented in this chapter. In particular, in this chapter we use for the first time shields that are not minimally correct, but are rather synthesized minimizing a certain cost function.
We continue using the bilingual team-building problem as our running example to illustrate the notation elements being introduced.
We reserve Section~\ref{sec:fairness-within-rdm} to connect the formalization of this chapter with the general framework presented in Chapter~\ref{chap:reactive_decision_making}.

\subsection{Environment and Shielding Setting}
\label{sec:fairness-env-and-shielding}

\paragraph*{Data-driven classifier.}
We are given a population of individuals, described by features.
Among them, we consider one binary feature to be \emph{protected} or \emph{sensitive}. 
Typical protected features are race, gender, language, etc.
Without loss of generality, the protected feature takes values in the set $\G = \{a,b\}$, and
the population can be therefore partitioned into groups $a$ and $b$, according to the value of the protected feature.
We consider a data-driven classifier that at each step samples one individual from the population, and outputs a \emph{recommended decision} from the set $\BB=\set{1,0}$ along with an intervention cost from the finite set $\costset\subset \RR_{\geq 0}$. 
As convention, decisions ``1'' and ``0'' will correspond to ``accept'' and ``reject,'' respectively.
We assume that the sampling and classification process gives rise to a given \emph{input distribution} $\sgen\in \distrset(\X)$, where the set $\X\coloneqq (\G\times\BB\times\costset)$ is called the \emph{input space}.
The non-protected features of individuals are hidden from the input space because they are irrelevant for shielding.
We will assume that $\sgen$ is given, i.e., the shields are computed using knowledge about $\sgen$.
When doing experiments, we estimate an approximation of $\sgen$ from the available data, as we detail in Section~\ref{sec:experimental-setup}.

\begin{example}[Continuation of Example~\ref{ex:running example}]
\label{ex:running_example2}
    In the bilingual team example, an individual is represented by a tuple $(g, z) \in \G \times \mathcal{Z}$, 
    where $\G = \{a, b\}$ denotes the language in which the candidate is proficient,
    and $\mathcal{Z}$ encompasses all non-protected features relevant to evaluating a candidate's suitability for the job,
    such as years of experience, relevant education, and so on.
    For simplicity, we assume that a candidate is proficient in only one of the two languages.
    
    The company uses a classifier 
    $f \colon \G \times \mathcal{Z} \to \BB \times \costset$, 
    which outputs a preliminary decision for each candidate (accept or reject) along with a cost associated with altering that decision. 
    The cost reflects the classifier’s confidence: candidates who are clearly good or bad incur a high cost for decision changes, 
    while borderline candidates can have their decisions reversed at a lower cost.
\end{example}

\paragraph*{Shields.}
A \emph{shield} is a symbolic decision-maker that selects the \emph{final decision} from the \emph{output space} $\Y \coloneqq \BB$ after observing a given input from $\X$, and possibly accounting for past inputs and outputs.

Formally, a shield is a function $\sshield\colon \left(\X\times\Y\right)^*\times \X \to \Y$, and its bounded-horizon variants are functions of the form $\left(\X\times\Y\right)^{\leq t}\times \X \to \Y$, for a given $t$.
Following the notions introduced in Chapter~\ref{chap:reactive_decision_making}, a fairness shield is a particular case of a \emph{post-shield} (Definition~\ref{def:postshield-abstract}). In particular, its output is a concrete \texttt{accept/reject} decision, and not a set of allowed decisions.
Note that the input of a shield is $(\tau, x)$, where $\tau\in(\X\times\Y)^*$ a sequence of previous inputs and outputs, 
and $x \in \X$ is a tuple $x = (g,b,c)$ representing the last individual, for which a final decision has not yet been made, where $g\in \G$ is the group membership, $b$ is the \texttt{accept/reject} recommendation of the classifier and $c\in\costset$ is the cost of overwritting the classifier's recommendation.
We will write $\sShield$ and $\sShield^t$ to respectively denote
the set of all shields and the set of bounded-horizon shields with horizon $t$
\footnote{We define bounded shields and the set $\sShield^t$ to emphasize that our synthesis algorithm only defines the behaviour for traces $\tau$ with a length at most $t$.}.
The \emph{concatenation} of a sequence of shields $\sshield_1,\sshield_2,\ldots\in \sShield^t$ is a shield $\sshield$, such that for every trace $\tau$, if $\tau$ can be decomposed as $ \tau\tau'$ with $|\tau|=jt$ for some $j$ and $\tau'<t$, then $\sshield(\tau,x) \coloneqq \sshield_{j+1}(\tau',x)$.


\paragraph*{Sequential decision making setting.}
We consider the sequential setting where inputs are sampled from $\sgen$ one at a time, and the shield $\sshield$ needs to produce an output without seeing the inputs from the future.
Formally, at every time $i=1,2,\ldots$, we sample an input 
$x_i=(g_i, r_i, c_i)$ from $\sgen$. 
The probability of getting input $x_i$ is $\sgen(x_i)>0$. 
The shield's output at time $i$ is  $y_i=\sshield([(x_1,y_1),\ldots,(x_{i-1},y_{i-1})],x_i)$.
After applying this process of sampling input and getting the corresponding shield output for $t$ time-steps, 
the resulting finite sequence $\tau = (x_1,y_1),\ldots,(x_t,y_t)$ is called a \emph{trace} induced by $\sgen$ and $\sshield$, and the integer $t$ is called the \emph{length} of the trace, denoted as $\len{\tau}$. 
We use $\Trfeas^t_{\sgen,\sshield}$ to denote the set of every such trace. 
For every $t$, the probability distribution $\sgen$ and the shield $\sshield$ induce a probability distribution $\PP(\cdot;\sgen,\sshield)$ over the set $(\X\times\Y)^t$ as follows.
For every trace $\tau\in (\X\times\Y)^t$, 
\begin{equation}
    \PP(\tau;\sgen,\sshield) \coloneqq 
    \begin{cases}
        \prod_{i=1}^t \sgen(x_i) & \mbox{ if } \tau\in\Trfeas^t_{\sgen,\sshield}. \\
        0 & \mbox{ otherwise. }
    \end{cases}
\end{equation}
The notation 
$\PP(\tau;\sgen,\sshield)$ is to be read as ``the probability of obtaining the trace $\tau$ when sampling inputs from the input distribution $\sgen$, and applying shield outputs from $\sshield$''.
Note that $\sgen$ and $\sshield$ are \emph{parameters} of the distribution, i.e., $\PP(\cdot; \sgen, \sshield)$ is a probability distribution, while $\tau$ is the element in the sample space (denoted $\Omega$ in Section~\ref{sec:prelim-probability-theory}) that has a certain probability to be sampled.
%
%
Given a prefix $\tau$, the probability of observing the trace $\tau\cdot \tau'$, for some $\tau'\in (\X\times\Y)^*$, is $\PP(\tau'\mid \tau;\sgen,\sshield) = \PP(\tau\cdot \tau';\sgen,\sshield)/\PP(\tau;\sgen,\sshield)$.
Note that the statistical dependence of $\tau'$ on $\tau$ is due to $\sshield$'s history-dependence.

\smallskip
\paragraph*{Cost.}
Let $\tau = (x_1,y_1),\ldots,(x_t,y_t)$ be a trace of length $t$, where $x_i = (g_i,r_i,c_i)$.
At time $i$, the shield pays the cost $c_i$ if its output $y_i$ is different from the recommended decision $r_i$.
The \emph{total} (intervention) \emph{cost incurred by the shield} on $\tau$ up to a given time $s\leq t$ is 
\begin{equation}
\label{eq:cost-defining-eq}
 \cost(\tau;s) \coloneqq \sum_{i=1}^s c_i\cdot\indicator{r_i\neq y_i}.
\end{equation}
The cost incurred up to time $t$ (the length of $\tau$) is simply written as $\cost(\tau)$, instead of $\cost(\tau;t)$.
For a given time horizon $t$, we define the expected value of cost after time $t$ as 
\begin{equation}
    \label{eq:expected-cost-time-t}
  \expe[\cost;\sgen,\sshield,t] \coloneqq \sum_{\tau\in (\X\times\Y)^t} \cost(\tau)\cdot \PP(\tau;\sgen,\sshield),
\end{equation}
and if additionally a prefix $\tau$ is given, the conditional expected cost after time $t$ (from the end of $\tau$) is 
\begin{equation}
    \label{eq:expected-cost-after-tau-time-t}
  \expe[\cost\mid \tau;\sgen,\sshield,t] \coloneqq \sum_{\tau'\in (\X\times\Y)^t} \cost(\tau')\cdot \PP(\tau'\mid\tau;\sgen,\sshield).  
\end{equation}
Note that the difference between Equations~\eqref{eq:expected-cost-time-t}~and~\eqref{eq:expected-cost-after-tau-time-t} is that in the second one, the prefix $\tau$ is given as part of the trace to the shield.
One can see Eq.~\eqref{eq:expected-cost-time-t} as a particular case of Eq.~\eqref{eq:expected-cost-after-tau-time-t} when the prefix is the empty trace $\tau = \eps$.
If $\tau$ is ``very fair'', i.e., $\spec(\tau) \ll \kappa$, the cost of enforcing fairness in the next $t$ steps will be generally lower than the case where $\tau$ is on the limit of being fair, i.e., $\spec(\tau)\approx \kappa$. 
This effect is amplified with longer prefixes.

The shield $\sshield$ is an element external to the classifier. 
It takes the protected feature of the candidate and the classifier's recommendation as inputs and has the authority to issue a final \texttt{accept/reject} decision. 
If the shield's decision differs from the classifier's, the incurred cost is as specified by the classifier.
The shield's inputs are the features of candidates, the classifier's decisions, and the costs, and the input distribution is assumed to be known in advance.

Note that, from the shield's perspective,
the distribution of non-protected features is unimportant,
as these features are already processed by the data-driven classifier and summarized into a single cost value. 
By sampling individuals from the candidate pool and processing them through both $f$ and $\sshield$, 
we obtain a trace $\tau$ that records the individuals and their decisions. 
In the case of our running example,
this trace encapsulates the results of the hiring process, including the linguistic distribution of hired candidates and the total cost incurred by the shield.

\subsection{Fairness Enforcement with Minimal Cost}
\label{sec:fairness-enf-minimal-cost}

\paragraph*{Fairness.}
We model (group) \emph{fairness properties} as functions that map every finite trace to a real-valued \emph{bias} level through intermediate statistics. 
A \emph{statistic} $\mu$ maps each finite trace $\tau$ to the values of a finite set of counters, represented as a vector in $\NN^p$, where $p$ is the number of counters.
The \emph{welfare} for group $g\in\set{a,b}$ is a function $\welfare{g}\colon \NN^p\to \RR$.
When $\mu$ is irrelevant or clear from the context, we will write $\welfare{g}(\tau)$ instead of $\welfare{g}(\mu(\tau))$.
A fairness property $\spec$ is an aggregation function mapping $(\welfare{a}(\tau),\welfare{b}(\tau))$ to a real-valued \emph{bias}.
Table~\ref{tab:fairness properties} summarize how existing fairness properties, namely demographic parity (DP)~\cite{dwork12fairness}, disparate impact (DI)~\cite{feldman2015certifying}, and equal opportunity (EqOpp)~\cite{hardt2016equality} can be cast into this form.

Estimating EqOpp requires the ground truth labels of the individuals be revealed after the shield has made its decisions on them.
To accommodate ground truth, we 
introduce the set $\Z = \set{0,1}$, such that traces are of the form $\tau = (x_1,y_1,z_1),\ldots,(x_t,y_t,z_t)\in (\X\times\Y\times\Z)^*$, where each $z_i$ is the ground truth label of the $i$-th individual.
The shield is adapted to $(\X\times\Y\times\Z)^*\times\X\to\Y$, where the set $\Z$ is treated as another input space and the probability distribution $\PP(\Z=z_i\mid \X=x_i)$ is assumed to be available.

\begingroup
\renewcommand{\arraystretch}{1.1} 
\begin{table}
    \centering
    \begin{tabular}{lccc}
    \toprule
        {Name} & {Counters} & $\welfare{g}$ & $\spec$ \\
        \midrule
         Demographic parity (DP) & $n_a,n_{a1},n_b,n_{b1}$ & $n_{g1}/n_g$ & $|\welfare{a}(\tau)-\welfare{b}(\tau)|$ 
         \\
         Disparate impact (DI) & $n_a,n_{a1},n_b,n_{b1}$ & $n_{g1}/n_g$  & $\left| 
     \welfare{a}(\tau)\div\welfare{b}(\tau)  \right|$ 
     \\
         Equal opportunity (EqOpp) & $n_a',n_{a1}',n_b',n_{b1}'$ & $n_{g1}'/n_g'$ & $|\welfare{a}(\tau)-\welfare{b}(\tau)|$ 
         \\
         \bottomrule
    \end{tabular}
    \caption[Empirical variants of fairness properties]
    {Empirical variants of fairness properties: For $g\in\set{a,b}$, the counters $n_g$ and $n_{g1}$ represent the total numbers of individuals from group $g$ who appeared and were accepted, respectively. Counters $n_g'$ and $n_{g1}'$ denote the total numbers of appeared and accepted individuals whose ground truth labels are ``$1$.'' 
    If a welfare value is undefined due to a null denominator, we set $\spec = 0$.
    }
    \label{tab:fairness properties}
\end{table}

\endgroup

\begin{example}[Continuation of Example~\ref{ex:running_example2}]
\label{ex:running_example4}
    In the bilingual team example, the welfare of a linguistic group $g$ is defined as the fraction of the team proficient in language $g$. 
    A more nuanced interpretation considers the welfare of group $g$ as the fraction of accepted candidates among those proficient in language $g$, which is the empirical variant of demographic parity (DP). 
    This measure accounts for the possibility that the linguistic distribution of the population may not be evenly split. 
    If one language is more prevalent in the target population,
    the hired team should proportionally include more members proficient in that language.
    To obtain an empirical variant of equal opportunity (EqOpp), 
    we would need to assume the existence of a ground truth on whether a candidate is actually a good employee for this job.
    This can typically only be assessed after the candidate actually works for some time with the team, so it is not easy to estimate a priori. 
    In such case, the welfare of a group would be the fractions of hired candidates with respect to the \emph{actually good} candidates for each linguistic group.
\end{example}

\paragraph*{Bounded-horizon fairness shields.}
From now on, 
we use the convention that $\sgen$ is the input distribution,
$\spec$ is the fairness property, 
and $\kappa$ is the \emph{bias threshold}.
Let $T$ be a given time horizon.
The set $\sShieldFeas^{\sgen,T}$  of \emph{fairness shields over time $T$} is the set of every shield that fulfills $\spec(\cdot)\leq \kappa$ after time $T$, i.e., 
\begin{equation}
    \label{eq:bounded-horizon-def}
  \sShieldFeas^{\sgen,T} \coloneqq \set{ \sshield\in\sShield^T \mid \forall \tau\in \Trfeas^{T}_{\sgen,\sshield}\;.\; \spec(\tau) \leq \kappa}.   
\end{equation}
%
%
%
%
We now define optimal bounded-horizon fairness shields as below.

\index{shield!fairness shield in finite horizon}
\begin{definition}[Finite Horizon Shields]
\label{def:finshield}
    Let $T>0$ be the time horizon.
    A \emph{finite horizon shield} (usually abbreviated to \FinShield) is the one that solves:
    \begin{align}\label{eq:finite horizon shield}
        \sshield^* \coloneqq \argmin_{\sshield\in\sShieldFeas^{\sgen,T}} \expe[\cost;\sgen,\sshield,T].
    \end{align}
\end{definition}

Note that even if the input distribution $\sgen$ is learned and imprecise, as long as it shares the same support as the true distribution, the fairness guarantees provided by the shield remain unaffected; only the cost-optimality may be compromised.


\paragraph*{Periodic fairness shields.}
\FinShield shields stipulate that fairness be satisfied at the end of the given horizon.
However, in many situations, it may be desirable to ensure fairness not only at the end of the horizon but also at intermediate points occurring at regular intervals.
For instance, a human resources department that is required to maintain a fair distribution of employees over the course of a quarter
might also need to ensure a similar property for their yearly revision, after four quarters.
This type of fairness is referred to as \emph{periodic fairness} in the literature~\cite{pmlr-v235-alamdari24a}.
For this class of fairness properties, we define the set of $T$-periodic fairness shields as 
\begin{equation}
\label{eq:sShieldFeasPeriodic_definition}
   \sShieldFeasPeriodic\coloneqq 
\set{ \sshield\in\sShield \mid \forall m\in \NN \;.\; \forall \tau\in \Trfeas^{mT}_{\sgen,\sshield}\;.\; \spec(\tau) \leq \kappa}.
\end{equation}

Note that $\sShieldFeasPeriodic$ does not force every subtrace of length $T$, or $mT$ for some $m\in\NN$, to satisfy a certain fairness constraint.
The reader may think of the multiples of the period $T, 2T, 3T, \dots$ as ``examination dates'':
the trace will be inspected at time $mT$, and by then it has to be correct.
Therefore, subtraces of any length that do not end in an examination date may have bias values slightly over the threshold.


\index{shield!fairness shield in periodic horizon}
\begin{definition}[Optimal $T$-periodic fairness shield]
    \label{def:periodically-fair}
    Let $T>0$ be the time period.
    An \emph{optimal $T$-periodic fairness shield} is given by:
    \begin{equation}\label{eq:periodic def.}
        \sshield^* \coloneqq \argmin_{\sshield\in\sShieldFeasPeriodic}\sup_{\substack{m\in\NN\\ \tau\in \Trfeas_{\sgen,\sshield}^{mT}}} \expe[\cost\mid \tau;\sgen,\sshield,T].
    \end{equation}
\end{definition}
Equation~\eqref{eq:periodic def.} requires fairness at each $mT$-th time (measured from the beginning), and minimizes the maximum expected cost over each period. 
The existence of this minimum remains an open question. 
In Section~\ref{sec:unbounded}, we propose three "best-effort" approaches to compute periodically fair shields (under mild assumptions) that are as cost-optimal as possible.

\subsection[Relation to the Reactive Decision Making Framework]
{Fairness Shielding within the Reactive Decision Making Framework}
\label{sec:fairness-within-rdm}
The sequential input can be modelled by an environment $\Env=(\Obs,\Act,\Trans)$, with $\Obs = \X$ (what we called the input space, $\X = \G\times\BB\times\costset$), $\Act = \Y$, and a transition function $\Trans$ characterized by a single input distribution. That is, for all trace $\tau\in(\Obs\times\Act)^*$, there is a unique distribution $\sgen\in\D(\X)$ such that $\Trans(\tau) = \sgen$.

The main difference between this shielding setting and those presented in Chapters~\ref{chap:delayed_shields}~and~\ref{chap:foceta} is that correctness is not established by avoiding concrete observations states, but rather correctness depends on the whole trace via a series of counters. 
In fact, every observation has a probability that is independent from the behaviour of the agent or the shield.
Another notable difference in this case is that we are not interested in the notion of minimal interference as expressed in Definitions~\ref{def:interferenceSet}~and~\ref{def:minimal_correctness}. 
In contrast, we assign a specific weight, or cost, to each interference, and build the shield that minimizes interferences in expectation.

\section{Algorithm for Finite Horizon Shield Synthesis}
\label{sec:synthesis}

We present our algorithm for synthesizing \FinShield shields as defined in Definition~\ref{def:finshield}.
A \FinShield shield $\sshield^*$ computes an output $y=\sshield^*(\tau,x)$ for every trace $\tau\in(\X\times\Y)^{\leq T}$ and every input $x\in\X$. 
Our synthesis algorithm builds $\sshield^*$ recursively for traces of increasing length, 
using an auxiliary \emph{value function} $v(\tau)$ that represents the minimal expected cost conditioned on traces with prefix $\tau$.
To define $v(\tau)$, we generalize fairness shields with the condition that a certain trace has already occurred. 
Given a time horizon $t$ and a trace $\tau$, whose length can differ from $t$, the set of \emph{fairness shields over time $t$ after $\tau$} is defined as
\[
\sShieldFeas^{\sgen,t \mid 
\tau} \coloneqq \set{ \sshield\in\sShield^t \mid \forall \tau'\in (\X\times\Y)^{t}\;.\; \tau\tau'\in \Trfeas^{\len{\tau}+t}_{\sgen,\sshield}
\implies \spec(\tau\tau') \leq \kappa}.
\] 
Then $v(\tau)$ is given by:
\begin{equation}
\label{eq:v-def-eq}
    v(\tau) \coloneqq \min\limits_{\sshield\in \sShieldFeas^{\sgen,(T-\len{\tau}) \mid\tau}} \expe[\cost\mid \tau;\sgen,\sshield,T-|\tau|].
\end{equation}
For every trace $\tau$ and every input $x\in \X$, 
the optimal value of the shield is 
$\sshield^*(\tau,x) = \argmin_{y\in\Y} v(\tau, (x,y))$.

In Section~\ref{sec:fair-naive-recursive-solution}, we present a recursive dynamic programming algorithm for computing $v(\tau)$, whose complexity grows exponentially with the length of $\tau$.
In Section~\ref{sec:fair-synthesis-efficient}, we present show how the algorithm proposed in Section~\ref{sec:fair-naive-recursive-solution} can actually be adapted to use only the $p$ counters defining the fairness property, thus solving the synthesis problem more efficiently,
with polynomial complexity for the vast majority of fairness properties.


\subsection{Recursive Computation of the Value Function}
\label{sec:fair-naive-recursive-solution}

We compute the value function recursively, defining a trivial value for traces of length $|\tau| = T$, 
and showing how the value function for traces of lenght $|\tau| < T$ can be computed by simulating a single optimization step by the shield and using the value function for traces of length $|\tau|+1$.

\paragraph*{Base case.} Let $T$ be the time horizon and $\tau$ be a trace of length $T$.
Since the horizon has already been reached, if $\spec(\tau)\leq \kappa$, then the expected cost is zero because fairness is already satisfied and no more cost needs to be incurre.
On the other hand, if $\spec(\tau)>\kappa$, the expected cost is infinite,
because, no matter what cost is paid, fairness can no longer be achieved. Formally,
\begin{equation}\label{eq:v-basecase}
 v(\tau) = 
        \begin{cases}
            0   & \spec(\tau)\leq\kappa,\\
            \infty  &   \text{otherwise}.
        \end{cases}
\end{equation}

\paragraph*{Recursive case.}
Let $\tau$ be a trace of length smaller than $T$.
The probability of the next input being $x=(g,r,c)$ is $\sgen(x)$, and the shield decides to output $y$ that either agrees with the recommendation $r$ (the case $y=r$) or differs from it (the case $y\neq r$)---whichever minimizes the expected cost. 
When $y=r$, the trace becomes $(\tau\cdot (x,y=r))$. 
Therefore, no cost is incurred and the total cost is the same as $v(\tau\cdot(x,y=r))$. 
When $y\neq r$, the trace becomes $(\tau\cdot (x,y\neq r))$. 
Thus, the incurred cost is $c$ and the new total cost becomes $c+v(\tau\cdot(x,y=r))$. Therefore
\begin{equation}\label{eq:v-recursion}
    \! v(\tau) = \sum_{x = (g,r, c)\in\X} \sgen(x)\cdot \min\left\lbrace\begin{matrix}
        v(\tau\cdot(x,y=r)),\\
        v(\tau\cdot(x,y\neq r)) + c
    \end{matrix}\right\rbrace.
\end{equation}
Equations~\eqref{eq:v-basecase}~and~\eqref{eq:v-recursion} can be used to recursively compute $v(\tau)$ for every $\tau$ of length up to $T$, and the time and space complexity of this procedure is $\mathcal{O}(|\X\times\Y|^T)$.
%
The correctness of Equation~\eqref{eq:v-recursion} is formally proven in Lemma~\ref{lem:v-recursion}.
Before formalizing the argument, let us see an example of its inner workings.

\begin{example}[Continuation of Example~\ref{ex:running_example4}]
\label{ex:running_example5}
    Consider the task of hiring a linguistically balanced team with a horizon of $T = 50$ candidates and a target demographic parity property $\spec$ with a threshold $\kappa = 0.2$, i.e., $20\%$.
    By the end of the process, a trace $\tau$ of $|\tau| = 50$ candidates must satisfy $\spec(\tau) < \kappa$.

    Consider the following situation.
    Suppose $\tau'$ be the trace obtained after observing the first $48$ candidates, i.e., $|\tau'|=48$, and just two more candidates are going to be observed before the horizon ends.
    In $\tau'$, $24$ candidates have been observed for each language proficiency group among $A$ and $B$, and among them $12$ from group $A$ and $17$ from group $B$ have been accepted, resulting in $\spec(\tau') = |12/24 - 17/24| = 0.208 > \kappa$. This temporary violation of DP is allowed since the process is ongoing.
    Suppose a new candidate $x = (g, r, c)$ appears, with $g = B$. The classifier tentatively accepts $x$ and informs the shield that reversing this decision would incur a cost $c$.
    If the shield accepts $x$,
    the shield will be forced to reject the next candidate proficient in $B$ or accept the next candidate proficient in $A$, regardless of the cost. 
    Conversely, if the shield rejects $x$, it incurs an immediate cost of $c$ but balances the languages to a point where intervention will not be required for the next decision.
    The shield must therefore weigh its options: either incur a known cost $c$ now by rejecting $x$ or risk an unknown future cost $c'$ by accepting $x$. 
    If the candidate is exceptionally qualified, i.e., the classifier recommends acceptance with a high cost of modifying its decision, the shield might choose to accept $x$, accepting the potential risk of rejecting another well-qualified candidate proficient in $B$ in the next round.
    On the other hand, when the shield is considering a borderline candidate, it may be better to pay a small price with the current candidate and ensuring that the agent's decision will be respected for the next candidate, whatever the decision is.
\end{example}

\begin{lemma}\label{lem:v-recursion}
    Let $\sgen\in \distrset(\X)$ be a given joint distribution of sampling individuals and the output of the agent, let $\kappa>0$ be a given threshold for a fairness property $\spec$, and let $T>0$ be a time horizon.
For a trace $\tau\in (\X\times\Y)^{\leq T}$, 
let $v(\tau)$ be  the minimum expected cost after $\tau$,
formally defined as
\begin{align*}
    v(\tau) \coloneqq \min\limits_{\sshield\in \sShieldFeas^{\sgen,(T-\len{\tau}) \mid\tau}} \expe[\cost\mid \tau;\sgen,\sshield,T-|\tau|].
\end{align*}
Then for $\tau$ with length $|\tau| = T$
\begin{equation}\label{eq:v-basecase-bis}
 v(\tau) = 
        \begin{cases}
            0   & \spec(\tau)\leq\kappa,\\
            \infty  &   \text{otherwise},
        \end{cases}
\end{equation}
for $\tau$ with $|\tau| < T$
\begin{equation}\label{eq:v-recursion-bis}
    \! v(\tau) = \sum_{x = (g,r, c)\in\X} \sgen(x)\cdot \min\left\lbrace\begin{matrix}
        v(\tau\cdot (x,y=r)),\\
        v(\tau\cdot (x,y\neq r)) + c
    \end{matrix}\right\rbrace,
\end{equation}
and the shield defined as $\sshield^*(\tau, x) \coloneqq \argmin_{y\in \Y} v(\tau, (x,y))$ is an optimal fairness bounded horizon fairness shield, i.e., 
\begin{equation}
    \label{eq:fairness-bounded-horizon-shield-rec-proof}
    \sshield^* = \argmin_{\sshield\in\sShieldFeas^{\sgen,T} } \expe[\cost;\sgen,\sshield,T].
\end{equation}
\end{lemma}

\begin{proof}
    Consider the term to be minimized:
    \begin{align}\label{eq:local5}
        \expe[\cost\mid \tau;\sgen,\sshield,T-|\tau|] = 
        \sum_{\tau'\in (\X\times\Y)^{T-|\tau|}}
        \cost(\tau')\cdot
        \PP(\tau'\mid\tau;\sgen,\sshield).
    \end{align}
    The sum over traces $\tau'\in (\X\times\Y)^{T-|\tau|}$
    can be partitioned into a sum over inputs $x\in\X$
    and traces $\tau''\in (\X\times\Y)^{T-|\tau|-1}$,
    by taking $\tau' = x\sshield(\tau,x)\tau''$.
    The cost term is then
    \[
    \cost(\tau') = 
    \cost(x\sshield(\tau,x)\tau'') = \cost(\tau'') + \cost(x\sshield(\tau,x)).
    \]
    If $x = (g,r,c)$, then
    \[
    \cost(x\sshield(\tau,x)) = \begin{cases}
        0 \quad \mbox{ if } r = \sshield(\tau,x) \\ 
        c \quad \mbox{ otherwise. }
    \end{cases}
    \]
    The probability term is then:
    \[
    \PP(\tau'\mid\tau;\sgen,\sshield) = 
    \PP(x\sshield(\tau,x) \mid \tau;\sgen,\sshield)\cdot
    \PP(\tau''\mid \tau x\sshield(\tau,x);\sgen,\sshield)
    \]
    The value in Equation~\eqref{eq:local5} can be written as 
    \begin{equation}\label{eq:local6}
    \sum_{\tau'\in (\X\times\Y)^{T-|\tau|}}
        \cost(\tau')\cdot
        \PP(\tau'\mid\tau;\sgen,\sshield) = 
        A + B,
    \end{equation}
    where 
    \begin{equation}
    \label{eq:local7}
        A = \sum_{x\in\X
        }\sum_{\tau''\in (\X\times\Y)^{T-|\tau|-1}}
        \cost(\tau'')
        \cdot \PP(x\sshield(\tau,x) \mid \tau;\sgen,\sshield)\cdot
        \PP(\tau''\mid \tau x\sshield(\tau,x);\sgen,\sshield),
    \end{equation}
    and 
    \begin{equation}
    \label{eq:local8}
    B =\sum_{x\in\X
        }\sum_{\tau''\in (\X\times\Y)^{T-|\tau|-1}}
    \cost(x\sshield(\tau,x))\cdot  
    \PP(x\sshield(\tau,x) \mid \tau;\sgen,\sshield)\cdot
    \PP(\tau''\mid \tau x\sshield(\tau,x);\sgen,\sshield).
    \end{equation}
     Note that the term $\PP(x\sshield(\tau,x) \mid \tau;\sgen,\sshield)$ appears several times. 
    This is the probability of getting a trace $x\sshield(\tau,x)$ after having seen a trace $\tau$. 
    This is, by definition $\PP(x\sshield(\tau,x) \mid \tau;\sgen,\sshield) = \sgen(x)$.
    
    Since $\sgen(x)$
    does not depend on $\tau''$, 
    the sum in $A$ can be rearranged as 
    \begin{equation}
         A = \sum_{x\in\X
        }
        \sgen(x)\cdot
        \sum_{\tau''\in (\X\times\Y)^{T-|\tau|-1}}
        \cost(\tau'')
        \cdot
        \PP(\tau''\mid \tau x\sshield(\tau,x);\sgen,\sshield),
    \end{equation}
    and therefore
    \begin{equation}
         A = \sum_{x\in\X
        }\sgen(x)
        \cdot
        \expe[\cost\mid \tau x\sshield(\tau,x);\sgen,\sshield,T-|\tau|-1].
    \end{equation}
    The term $B$ can be similarly rearranged, taking into consideration that in this case $\cost(x\sshield(\tau,x))$ is also independent of $\tau''$:
    \begin{equation}
    B =\sum_{x\in\X
        }
    \cost(x\sshield(\tau,x))\cdot
    \sgen(x) \cdot
    \sum_{\tau''\in (\X\times\Y)^{T-|\tau|-1}}
    \PP\left(\tau''\mid \tau x\sshield(\tau,x);\sgen,\sshield\right).
    \end{equation}
    The hanging term is the sum of probabilities, so by definition adds up to 1:
    \[
     \sum_{\tau''\in (\X\times\Y)^{T-|\tau|-1}}
    \PP\left(\tau''\mid \tau x\sshield(\tau,x);\sgen,\sshield\right) = 1.
    \]
    Therefore 
    \begin{equation}
        B =\sum_{x\in\X
        }\sgen(x)\cdot
    \cost(x\sshield(\tau,x))
    \end{equation}
    Putting $A$ and $B$ together we get:
    \begin{multline}
        \expe[\cost\mid \tau;\sgen,\sshield,T-|\tau|] 
        = \\ =
        \sum_{x\in\X} \sgen(x)\cdot\Big(
        \cost(x\sshield(\tau,x)) + 
        \expe[\cost\mid \tau x\sshield(\tau,x);\sgen,\sshield,T-|\tau|-1]
        \Big).
    \end{multline}
    This partitions the value of $\expe[\cost\mid \tau;\sgen,\sshield,T-|\tau|]$ into a sum of cost of current decision ($\cost(x\sshield(\tau,x))$)
    and expected cost in the rest of the trace. 
    For every $x$, the optimal value of the shield $\sshield(\tau,x)$ is the one that minimizes 
    \begin{equation*}
        \cost(x\sshield(\tau,x)) + \\
        \expe[\cost\mid \tau x\sshield(\tau,x);\sgen,\sshield,T-|\tau|-1].
    \end{equation*}
    This is precisely, the recursive property that we want to prove.

    Finally, to prove Equation~\eqref{eq:fairness-bounded-horizon-shield-rec-proof}, just note that for the empty trace $\tau = \eps$, we have $\sShieldFeas^{\sgen,(T-|\tau|)\mid\tau} = \sShieldFeas^{\sgen,T}$, 
    which is precisely the set of shields set as minimization domain in Equation~\eqref{eq:fairness-bounded-horizon-shield-rec-proof}.
\end{proof}

\subsection[Efficient Value Function Computation]{Efficient Value Function Computation through Trace Abstraction}
\label{sec:fair-synthesis-efficient}
We now present an efficient recursive procedure for computing \FinShield shields that runs in polynomial time and  space. 
The key observation is that $\spec$ is a fairness property that depends on $\tau$ through a statistic that uses $p$ counters,
as defined in Section~\ref{sec:fairness-enf-minimal-cost}.
Consequently, the value function $v(\tau)$ in Equation~\eqref{eq:v-basecase} and Equation~\eqref{eq:v-recursion} depends only on counter values, not on exact traces. 
This allows us to define our dynamic programming algorithm over the set of counter values taken by the statistic $\mu$.
Let $R_{\mu,T} \subseteq \mathbb{N}^p$ be the set of values the statistic $\mu$ can take from traces of length at most $T$.
We have the following complexity result.
\begin{theorem}\label{thm:bounded-horizon shield synthesis-bis}
    Solving the bounded-horizon shield-synthesis problem 
requires $\mathcal{O}(|R_{\mu, T}| \cdot|\X|)$-time and $\mathcal{O}(|R_{\mu,T}|\cdot |\X|)$-space.
\end{theorem}

\begin{proof}
    In this section we have described a dynamic programming approach to synthesize the shield by recursively computing $v(\tau)$ for all possible traces $\tau\in(\X\times\Y)^{\leq T}$.
    As explained before, 
    these computations do not depend directly on $\tau$, 
    but rather on the statistic $\mu$, 
    that depends on $p$ counters, taking values in the set $R_{\mu, T}$. 
    We need to build a table with the shield values for every pair of counter values and input.
    Therefore, the table occupies a space $\mathcal O(|R_{\mu, T}|\cdot|\X|)$. 
    Every element of the table has to be computed only once, and it is done as a sum over all elements of $x$, thus the cost in time is $\mathcal O(|R_{\mu, T}|\cdot |\X|)$.
 \end{proof}

Most fairness properties, e.g., DP and EqOpp, have a range of $R_{\mu, T} = [0, T]^p$, 
where $p$ is the number of counters ($p=4$ for DP, and $p=5$ for EqOpp), 
making the complexity polynomial in the length of the time horizon.


\section{Algorithms for Periodic Shield Synthesis}
\label{sec:unbounded}
\index{fairness property!DoR property}

Until now, we have described a method to synthesize finite-horizon shields, that is, shields that ensure fairness after a finite horizon $T$ (Definition~\ref{def:finshield}).
In this section we explore the problem of synthesizing $T$-periodic shields, which guarantee fairness for unbounded traces at every $T$ decisions (Definition~\ref{def:periodically-fair}).
As previously noted, we leave the question of computing optimal $T$-periodic shields open, and present three ``best-effort'' solutions to the problem, each with different costs and guarantees.

We present algorithms for computing periodic fairness shields for a broad subclass of group fairness properties, which we call \emph{difference of ratios} (DoR) properties.
A statistic $\mu$ is \emph{single-counter} if it maps every trace $\tau$ to a single counter value, i.e., $\mu(\tau) \in \NN$, and \emph{additive} if $\mu(\tau\cdot \tau') = \mu(\tau) + \mu(\tau')$ for any traces $\tau$ and $\tau'$. 
A group fairness property $\spec$ is DoR if 
\begin{enumerate}[label=(\alph*)]
    \item for each group $g$, $\welfare{g}(\tau) = \numer^g(\tau)/\den^g(\tau)$, where $\numer^g(\tau)$ and $\den^g(\tau)$ are additive single-counter statistics, and
    \item $\spec(\tau) = |\welfare{a}(\tau) - \welfare{b}(\tau)|$.
\end{enumerate}
Many fairness properties, including DP and EqOpp, are DoR. 
See, for example,~\cite[Table 3.5]{barocas2023fairness} for a non-exhaustive list.
In this case, DI is an exception because it violates the condition (b).

For DoR fairness properties, we propose two approaches for constructing periodic fairness shields: \emph{static} and \emph{dynamic}, and we explore their respective strengths and weaknesses.

\subsection{Periodic Shielding: The Static Approach}
In the static approach, 
a periodic shield is obtained 
by \emph{concatenating infinitely many identical copies of a statically computed bounded-horizon shield} $\sshield$, synthesized with the time period $T$ as the horizon.
We present two ways of computing $\sshield$ so that its infinite concatenation is $T$-periodic fair.
%
%
%

\subsubsection{Approach I: \StaticDP Shields.}



\begin{definition}[\StaticDP shields]
    A shield is called \StaticDP if it is the concatenation of infinite copies of a \FinShield shield (from Definition~\ref{def:finshield}).
\end{definition}
%
Unfortunately, \StaticDP shields do not always satisfy periodic fairness.
Consider a trace $\tau=\tau_1\ldots\tau_m$ for an arbitrary $m>0$, generated by a \StaticDP shield, such that each segment $\tau_i$ is of length $T$.
It follows from the property of \FinShield shields that $\spec(\tau_i) \leq \kappa$ for each individual $i$.
However, $T$-periodic fairness may be violated because $\spec(\tau)$ need not be bounded by $\kappa$.

\begin{example}
\label{ex:counter-example-periodic-naive-bis}
    Consider DP with $0<\kappa<1-2/T$.
    Suppose $\tau_1$ and $\tau_2$ are traces of length $T$, defined as follows. 
    The first trace $\tau_1$ contains $1$ candidate from group $A$, $T-1$ candidates from group $B$, and none were accepted. 
    The second trace $\tau_2$ contains $T-1$ candidates from group $A$, $1$ candidate from group $B$, and all were accepted. 
    Both traces are fair, since $\spec(\tau_1) = \spec(\tau_2) = 0$.
    However, when concatenating the two traces together, the resulting trace $\tau_1\cdot\tau_2$ is very biased, 
    since it contains $T$ candidates from both group $A$ and $B$, 
    but only one accepted candidate from group $B$,
    while having $T-1$ candidates accepted from group $A$.
    Concretely, $\spec(\tau_1\tau_2) = |(T-1)/T-1/T |=1-2/T > \kappa$ (biased).
    This example is summarized in Table~\ref{tab:counter-example-periodic-naive-bis}.
\end{example}

An important feature of these counter-examples is the excessive skewness of appearance rates across the two groups.
We further explore this phenomenon in Section~\ref{sec:counterexamples-static-fair}.
We show that \StaticDP shields are $T$-periodic fair if the 
appearance rates of the two groups are constant across every period.

\begin{table}[b]
    \centering
    \begin{tabular}{l|ccccc|c}
          & $n_a$ & $n_{a1}$ & $n_b$ & $n_{b1}$ & DP ($\spec$) &  $\spec \leq \kappa$? \\
          \midrule
         $\tau_1$ & $1$ & $0$ & $T-1$ & $0$ & $0$ & \cmark\\
         $\tau_2$ & $T-1$ & $T-1$ & $1$ & $1$ & $0$ & \cmark \\
         $\tau_1\tau_2$ & $T$ & $T-1$ & $T$ & $1$  & $1-2/T$ & \xmark\\
    \end{tabular}
    \caption[Counterexample: \StaticDP shields are not periodically fair]
    {Counterexample  showing that \StaticDP shields may not be periodically fair for DP.
    Suppose the bias threshold is $0 < \kappa < 1-2/T$.
    The traces $\tau_1,\tau_2$ fulfill DP but their concatenation does not.
}
    \label{tab:counter-example-periodic-naive-bis}
\end{table}

\begin{theorem}[Conditional correctness of \StaticDP shields]
\label{thm:static shield with bounded DP}
    Let $\spec$ be a DoR fairness property. Consider a \StaticDP shield $\sshield$, and let $\tau = \tau_1 \ldots \tau_m \in \Trfeas^{mT}_{\sgen, \sshield}$ be a trace such that $\len{\tau_i} = T$ for all $i \leq m$. 
    If $\den^g(\tau_i) = \den^g(\tau_j)$ for every $i,j \leq m$ and $g\in\{a,b\}$, then the fairness property $\spec(\tau) \leq \kappa$ is guaranteed.
\end{theorem}

\begin{proof}
    Given the condition, we can name $\den^a$ and $\den^b$ to the unique values of $\den^a(\tau_i)$ and $\den^b(\tau_i)$.
    For each $i\leq m$, we have the condition that
    \begin{equation}
    \label{eq:aux6}
        \left| \frac{\numer^a(\tau_i)}{\den^a} -  \frac{\numer^b(\tau_i)}{\den^b}  \right| \leq \kappa.
    \end{equation}
    We want to prove a fairness condition for the trace $\tau_1\dots\tau_m$, that is expressed as
    \begin{equation}
    \label{eq:aux7}
        \left| 
        \frac{\sum_{i=1}^m  \numer^a(\tau_i)}{m\cdot \den^a} -
        \frac{\sum_{i=1}^m  \numer^b(\tau_i)}{m\cdot \den^b}
        \right| \leq \kappa.
    \end{equation}
    Because of the denominators being the same across all traces, we can reorder the left-hand-side of Equation~\eqref{eq:aux7} as
    \begin{equation}
    \label{eq:aux8}
        \frac{1}{m}
        \left|
        \sum_{i=1}^m \left(\frac{\numer^a(\tau_i)}{\den^a} - \frac{\numer^b(\tau_i)}{\den^b}\right)
        \right|.
    \end{equation}

    Applying the triangular inequality to Equation~\eqref{eq:aux8} and the condition in Equation~\eqref{eq:aux6}, we get
    \begin{align*}
        \frac{1}{m}
        \left|
        \sum_{i=1}^m \left(\frac{\numer^a(\tau_i)}{\den^a} - \frac{\numer^b(\tau_i)}{\den^b}\right)
        \right| \leq & 
        \frac{1}{m} \sum_{i=1}^m \left|\frac{\numer^a(\tau_i)}{\den^a} - \frac{\numer^b(\tau_i)}{\den^b}\right| \leq \\
        \leq & \frac{1}{m} \cdot m \kappa = \kappa.
    \end{align*}
 \end{proof}

While the condition in Theorem~\ref{thm:static shield with bounded DP} appears conservative, 
we show in Section~\ref{sec:counterexamples-static-fair} (Theorem~\ref{thm:counterexamples-fair-comp}) that it is in fact tight. 
The tightness result is expressed in terms of \emph{balanced traces}, 
which is a concept that will appear also in the following section.
\begin{definition}[Balanced traces]
    Let $\mu^a,\mu^b\colon (\X\times\Y)^*\to \NN$ be a pair of single-counter statistics, $T>0$ be a given time horizon, and
    $N\leq T/2$ be a given integer.
    A trace $\tau$ of length $T$ is \emph{$N$-balanced with respect to} $\mu^a$ and $\mu^b$ if both $\mu^a(\tau) \geq N$ and $\mu^b(\tau)\geq N$.
    We denote the set of all $N$-balanced traces of length $t$ as $\Trbalance^T(\mu^a,\mu^b,N)$.
\end{definition}
A particular case of Theorem~\ref{thm:static shield with bounded DP} is that fairness is guaranteed when all traces are $(T/2)$-balanced with respect to the denominators.
In Theorem~\ref{thm:counterexamples-fair-comp}, we show, for the case of demographic parity, with $\mu^a, \mu^b = \den^a, \den^b$, for every $\kappa$, there exist $m$ and $\lfloor(T-1)/2\rfloor$-balanced traces $\tau_1,\dots,\tau_m$ such that $\spec_\DP(\tau_i)\leq \kappa$ for each $i$, but $\spec_\DP(\tau_1\dots\tau_m) > \kappa$.
%
However, these are worst-case scenarios and are uninteresting from a practical point of view. 
In our experiments, \StaticDP shields
fulfill periodic fairness in a majority of cases even if the traces violate the condition in Theorem~\ref{thm:static shield with bounded DP}.


\subsubsection{Approach II: \StaticBAR Shields.}
\label{sec:approach-staticBW}
When the condition of Theorem~\ref{thm:static shield with bounded DP} is violated, \StaticDP shields cannot guarantee fairness as the bound on the bias is not closed under concatenation of traces (see Example~\ref{ex:counter-example-periodic-naive-bis}).
A stronger property that is indeed closed under concatenation is when a bound is imposed on each group's welfare. 
Let $l,u$ be constants with $0\leq l<u\leq 1$.
A trace $\tau$ has \emph{bounded welfare} (BW) if for each group $g\in\G$,
$\welfare{g}(\tau)=\numer^g(\tau)/\den^g(\tau)$ belongs to $[l,u]$.
The pair $(l,u)$ will be called \emph{welfare bounds}.
We show that BW is closed under trace concatenations, which depends on the additive property of $\numer^g$ and $\den^g$.

\begin{lemma}
    \label{prop:acc-rates}
    Let $(l,u)$ be given welfare bounds, and $\welfare{g}(\cdot)\equiv \numer^g(\cdot)/\den^g(\cdot)$ for additive $\numer^g,\den^g$.
    For a trace $\tau = \tau_1\dots \tau_m$,
    if for each $i$, 
    $\welfare{g}(\tau_i) \in [l,u]$, 
    then 
    $\welfare{g}(\tau) \in [l,u]$.
\end{lemma}

To prove Lemma~\ref{prop:acc-rates}, we first need to prove the following auxiliary result.

\begin{lemma}\label{lem:min-frac-max-bis}
    Let $a_1, \dots, a_m$, $b_1, \dots, b_m$ be positive real numbers. 
    Then
    \begin{equation}\label{eq:min-frac-max-lemma}
        \min_{i\in\{1\dots m\}} \frac{a_i}{b_i} \leq 
        \frac{\sum_{i=1}^m a_i}{\sum_{i=1}^m b_i} \leq 
        \max_{i\in\{1\dots m\}} \frac{a_i}{b_i}.
    \end{equation} 
\end{lemma}
\begin{proof}
    This is an extension of the following known inequality: 
    given positive numbers
    $w, x, y, z$, if $w/x < y/z$, then
    $\frac{w}{x} \leq \frac{w+y}{x + z} \leq \frac{y}{z}$.
    We can restate it as:
    \begin{equation} \label{eq:min-frac-max1-bis}
        \min\left(
        \frac{w}{x}, \frac{y}{z} 
        \right)
        \leq \frac{w + y}{x + z} 
        \leq 
        \max\left(
        \frac{w}{x}, \frac{y}{z} 
        \right).
    \end{equation}

    We prove this result by induction on $m$. The base case for $m=1$ is trivial.
    
    For a general $m$, we start applying inequality~\eqref{eq:min-frac-max1-bis} with
    $w = \sum_{i=1}^{m-1} a_i$, 
    $x = \sum_{i=1}^{m-1} b_i$,
    $y = a_m$, and $z = b_m$, 
    to obtain:
    \begin{equation*}
        \frac{\sum_{i=1}^m a_i}{\sum_{i=1}^m b_i} \leq
        \max\left(
        \frac{\sum_{i=1}^{m-1} a_i}{\sum_{i=1}^{m-1} b_i}, 
        \frac{a_m}{b_m}
        \right).
    \end{equation*}
    Applying the induction hypothesis we have that 
    \begin{equation}
        \frac{\sum_{i=1}^{m-1} a_i}{\sum_{i=1}^{m-1} b_i} \leq 
        \max_{i\in \{1\dots m-1\}} \frac{a_i}{b_i},
    \end{equation}
    and therefore:
    \begin{equation*}
        \frac{\sum_{i=1}^m a_i}{\sum_{i=1}^m b_i} \leq
        \max\left(
        \max_{i\in \{1\dots m-1\}} \frac{a_i}{b_i}, 
        \frac{a_m}{b_m}
        \right) = \max_{i\in \{1\dots m\}} \frac{a_i}{b_i}.
    \end{equation*}
    This proves the right-side inequality of Equation~\eqref{eq:min-frac-max-lemma}. 
    The left-side is analogous.
\end{proof}

\begin{proof}[Proof (Of Lemma~\ref{prop:acc-rates})]
    Let $n^a_i = \den^a(\tau_i)$, $n^{a1}_i = \numer^a(\tau_i)$, 
    $n^b_i = \den^b(\tau_i)$, and $n^{b1}_i = \den^b(\tau_i)$.
    Applying Lemma~\ref{lem:min-frac-max-bis}, we have 
    for all $g\in\G$
    that 
    \begin{equation}\label{eq:min-frac-max-lemma-bis}
        \min_{i\in\{1\dots n\}} \frac{n^{g1}_i}{n^g_i} \leq 
        \frac{\sum_{i=1}^n n^{g1}_i}{\sum_{i=1}^n n^{g}_i} \leq 
        \max_{i\in\{1\dots n\}} \frac{n^{g1}_i}{n^g_i}.
    \end{equation} 

    And we also know that all welfare values are bounded by $l$ and $u$. That is, for all $i\in\{1\dots n\}$ and all $g\in\G$
    \begin{equation} \label{eq:beta-alpha-bis}
        l \leq  \frac{n^{g1}_i}{n^{g}_i} \leq u
    \end{equation}
    In particular, Equation~\eqref{eq:beta-alpha-bis} applies to the maximum and minimum welfare values. 
    This, together with Equation~\eqref{eq:min-frac-max-lemma-bis}
    finishes the proof.
\end{proof}

For DoR properties, BW implies fairness when $u-l\leq \kappa$.
Combining this with Lemma~\ref{prop:acc-rates}, we infer that if $\sshield$ is a bounded-horizon shield that fulfills BW on every trace $\tau$ of length $T$ for welfare bounds $(l,u)$ with $u-l\leq \kappa$, then the concatenation of infinite copies of $\sshield$ would be a $T$-periodic fairness shield.
The natural course of action for computing shields 
that fulfill BW is to 
mimic Definition~\ref{def:finshield}, 
replacing the condition on $\spec$ with a condition on welfare. 
However, if we define the set of BW-fulfilling shields as
\[
     \sShieldFeasBounded^{\sgen,T}\coloneqq  \set{\sshield\in\sShield \mid \forall \tau \in \Trfeas^T_{\sgen, \sshield} \;.\;
    \forall g\in \set{a,b}\;.\;
    l\leq \welfare{g}(\tau) \leq u },
\]
the set $\sShieldFeasBounded^{\sgen,T}$ can be empty for some $T,l,u$.
Following is an example.
\begin{example}\label{ex:bounded-acc-rates}
    Suppose $\welfare{g}(\tau) = n_{g1}/n_g$, where $n_{g1}$ and $n_g$ are the total numbers of accepted and appeared individuals from group $g$ (as in DP).
    Suppose $T=2, l=0.2,u=0.4$.
    It is easy to see that no matter what the shield does, for every $\tau$ of length $2$, $\welfare{g}(\tau)\in \set{0,0.5,1}$.
    Therefore, $\sShield^{2}_{[0.2,0.4]}=\emptyset$.
\end{example}
The emptiness of $\sShieldFeasBounded^{\sgen,T}$ is due to a large disparity between the appearance rates of individuals from the two groups, which occurs for shorter time horizons and for datasets where one group has significantly lesser representation than the other group.
To circumvent this technical inconvenience, 
we make the following assumption on observed traces.

\begin{assumption}\label{assump:static-BAR}
    Let $l,u$ be welfare bounds, and $\tau = \tau_1\ldots\tau_m\in \Trfeas^{mT}_{\sgen,\sshield}$ be a trace with $|\tau_i|=T$ for each $i$.
    Every $\tau_i$ is $N$-balanced w.r.t.\ $\den^a$ and $\den^b$ for $N= \left\lceil 1/(u-l)\right\rceil$.
\end{assumption}

Assumption~\ref{assump:static-BAR} may be reasonable depending on $l$, $u$, $T$, and the input distribution $\theta$. 
Intuitively, for a larger $T$ and a smaller skew of appearance probabilities for individuals between the two groups, the probability of fulfilling Assumption~\ref{assump:static-BAR} is larger (for a given finite $m$).
At the end of this section (Equation~\eqref{eq:existence-of-BW-shields})
we quantify it as the probability of a sample from a binomial distribution lying between $N$ and $T-N$.
%


\begin{definition}[\StaticBAR shields]
    Let $l,u$ be given welfare bounds, and $T$ be a given time period.
    A \StaticBAR shield is the concatenation of infinite copies of the shield $\sshield^*$ solving
    \begin{align}\label{eq:static-BAR-shield optimality}
        \sshield^* = \argmin_{\sshield\in\sShieldFeasBounded^{\sgen,T, N}} \expe[\cost;\sgen,\sshield,T], 
    \end{align}
    where $N= \left\lceil 1/(u-l)\right\rceil$, and
    \begin{equation*}
     \sShieldFeasBounded^{\sgen,T,N}\coloneqq  \set{\sshield\in\sShield \mid \forall \tau \in \Trfeas^T_{\sgen, \sshield} \cap \Trbalance_N^T.\,
    \forall g\in \set{a,b}\;.\;
    l\leq \welfare{g}(\tau) \leq u }.
\end{equation*}
\end{definition}

With the following technical result prove that 
$\sShieldFeasBounded^{\sgen,T,N}$ is indeed non-empty when Assumption~\ref{assump:static-BAR} is fulfilled.
We do so by constructing the shield that keeps $\welfare{g}(\tau)$ just above $l$
and showing that it also guarantees $\welfare{g}(\tau)\leq u$ when the trace is sufficiently balanced.

\begin{lemma}
    \label{thm:acc-rates-existence}
    Let $\spec$ be a DoR property with $\spec(\tau) = |\welfare{a}(\tau) - \welfare{b}(\tau)|$, and $\welfare{g}(\tau) = \numer^g(\tau)/\den^g(\tau)$.
    Let $0\leq l < u\leq 1$ be a pair of welfare bounds. The set of shields
    \begin{equation*}
             \sShieldFeasBounded^{T, N} \coloneqq \set{\sshield\in\sShield \mid \forall \tau \in \Trfeas^t_{\sgen, \sshield} \cap \Trbalance_N^T \;.\; 
            \forall g\in \set{a,b}\;.\;
            l\leq \welfare{g}(\tau) \leq u }        
    \end{equation*}
is not empty for $N \geq \left\lceil \frac{1}{u-l}\right\rceil$.
\end{lemma}

\begin{proof}
    For a shield to exist that can enforce bounds $[l, u]$ on the welfare, there must exist, for every value of $\den^g(\tau)$, 
    at least one way of deciding for increasing or not $\numer^g(\tau)$ that maintains the welfare in the desired bounds. 
    Since we do not know \textit{a priori} 
    the value of $\den^g(\tau)$, this decision must be incremental, 
    and be such that the welfare is maintained for any value of $\den^g(\tau)$. 

    To express this, there needs to exist a sequence $(x_n)\subseteq\mathbb N$ for all $n\geq N$ such that
    \begin{equation}\label{eq:xn}
        l \leq \frac{x_n}{n} \leq u, \quad \mbox{and} 
        \quad x_{n+1} - x_n \in \{0,1\}.
    \end{equation}
    Given $l$, and $u$, if $\den^g(\tau)$ is at least $N$ for a given group $g$, the shield can force $\numer^g(\tau)$ to be exactly $x_n$ to ensure the bound on welfare is met. 

    The condition in Equation~\eqref{eq:xn} can be 
    reformulated as $ ln \leq x_n \leq  un$,
    and since $x_n$ needs to be an integer, 
    we can tighten it to 
    \begin{equation}\label{eqn:local2}
        \left\lceil l n \right\rceil \leq x_n \leq 
        \left\lfloor u n \right\rfloor.
    \end{equation}
    One option is to try $x_n = \left\lceil l n \right\rceil$.
    We have to prove that this choice satisfies two conditions: 
    (i) $x_{n+1} - x_n \in \{0,1\}$, and
    (ii) Equation~\eqref{eqn:local2}.

    \begin{enumerate}
        \item[(i)] 
        This is true for any sequence $x_n$ built as the integer part of $nl$, where $l\in[0,1]$. 
        For any number $x$, it is known that 
        $x = \lceil x\rceil - \{x\}$, where $0 \leq \{x\} < 1$. 
        Applying this inequality twice, we get 
        \begin{equation*}
         x_{n+1} - x_n = \lceil l(n+1)\rceil - \lceil l n\rceil
         < l (n+1) - \lceil l n\rceil \leq l(n+1)- l n = 1+ l < 2.
        \end{equation*}
        Since $\lceil l (n+1)\rceil - \lceil l n\rceil$ is an integer strictly smaller than 2, it is smaller or equal than 1. It is also clearly non-negative, so it has to be either $0$ or $1$.

        \item[(ii)]
        By construction, $l n \leq \lceil l n\rceil$. 
        Now we have to see that $\lceil l n\rceil \leq u n$. 
        If $\lceil l n\rceil = l n$, 
        then for any $n\geq 1$, we have $x_n \leq u n$ on account of $l < u$. 
        If $\lceil l n\rceil = l n +1$, 
        we need $l n + 1 \leq u n$, 
        which is equivalent to $n \geq \frac{1}{u-l}$. 
        Since $n$ needs to be an integer, 
        selecting $N = \left\lceil \frac{1}{u-l}\right\rceil$ ensures this condition is satisfied for all $n\geq N$.
    \end{enumerate}
\end{proof}

This result guarantees that the optimization problem in \eqref{eq:static-BAR-shield optimality} is feasible, 
and thus \StaticBAR shields are well-defined.
Intuitively, we obtain a ``best-effort'' solution for $\sshield^*$: 
when a trace satisfies Assumption~\ref{assump:static-BAR}, 
$\sshield^*$ guarantees that $\tau$ satisfies BW with minimum expected cost. 
Otherwise, $\sshield^*$ has no BW requirement, 
and thus for traces that violate Assumption~\ref{assump:static-BAR}, the shield will incur zero cost by never intervening,
voiding any potential fairness guarantee.

Synthesis of \StaticBAR shields
follows the same approach as in Section~\ref{sec:synthesis} with Equation~\eqref{eq:v-basecase} replaced by: 
 \begin{equation}
 \label{eq:synthesis-BW-shields}
 v(\tau) = 
        \begin{cases}
            0   & \mbox { if } \tau\notin \Trbalance_N^T\lor \bigwedge_{g\in\set{a,b}} \welfare{a}(\tau)\in[l,u],\\
            \infty  &   \mbox{ otherwise. }
        \end{cases}
\end{equation}

We summarize the fairness guarantee below.

\begin{theorem}[Conditional correctness of \StaticBAR shields]
\label{thm:correctness-staticBW}
    Let $\spec$ be a DoR fairness property.
    Let $l,u$ be welfare bounds such that $u-l\leq \kappa$.
    For a given \StaticBAR shield $\sshield$, 
    let $\tau = \tau_1\ldots\tau_m\in \Trfeas^{mT}_{\sgen,\sshield}$ be a trace with $\len{\tau_i}=T$ for each $i\leq m$.
    If Assumption~\ref{assump:static-BAR} holds, then the fairness property $\spec(\tau)\leq \kappa$ is guaranteed.
\end{theorem}

\begin{proof}
    This is a direct consequence of Lemmas~\ref{prop:acc-rates}~and~\ref{thm:acc-rates-existence}.
    If Assumption~\ref{assump:static-BAR} holds, Lemma~\ref{thm:acc-rates-existence} ensures that the set of shields is non-empty. 
    Furthermore, any such shield satisfies the fairness condition $\spec(\tau)\leq \kappa$ for any trace in $\tau\in\Trfeas^{mT}$ 
    by Lemma~\ref{prop:acc-rates}.
\end{proof}

\paragraph*{Existence of \StaticBAR shields.}

The feasiblity of the condition $N \geq \left\lceil \frac{1}{u-l}\right\rceil$
in real cases depends on the values of $l$ and $u$ to enforce, as well as the incoming probability distribution. 
This condition, formulated as Assumption~\ref{assump:static-BAR}, is the key to guarantee the existence of \StaticBAR shields in Theorem~\ref{thm:acc-rates-existence}.
In its most simplified form, if we just care about the group membership of any incoming candidate, the distribution of incoming candidates follows a Bernoulli distribution $B(p)$, where $p$ is the probability to receive a candidate of group A. 
After a time horizon $T$, the number of incoming candidates of group A follows a binomial distribution $Bin(T, p)$, and the probability to see at least $N$ candidates of each group is the probability of the binomial being between $N$ and $T-N$, which is 
\begin{equation}
    \label{eq:existence-of-BW-shields}
    \sum_{k=N}^{T-N}\binom{T}{k}p^k(1-p)^{T-k}.
\end{equation}
In practice, this corresponds to the probability that our shield will encounter a trace where demographic parity with the given bound on acceptance rates can be enforced. 
It is up to the user to evaluate whether this guarantee is enough for a given application.

\subsection{Periodic Shielding: The Dynamic Approach}

While the static approaches repeatedly use one statically computed bounded-horizon shield, the dynamic approach recomputes a new bounded-horizon shield at the beginning of each period, and thereby adjusts its future decisions based on the past biases.
We formalize this below.

\begin{definition}[\Dyn shields]
\label{def:dynamic-shield}
    Suppose we are given a parameterized set of \emph{available} shields $\sShield'(\tau)\subseteq \sShield$ where the parameter $\tau$ ranges over all finite traces.
    A \Dyn shield $\sshield$ is the concatenation of a sequence of shields $\sshield_1,\sshield_2,\ldots$ such that for every trace $\tau\in \Trfeas_{\sgen,\sshield}^{mT}$ with $m\geq 0$, for every $\tau'\in (\X\times\Y)^{<T}$, and for every input $x\in\X$, we have $\sshield(\tau\cdot \tau',x) = \sshield_{m+1}(\tau',x)$, where
    \begin{align}\label{eq:dyn-shield optimality}
        \sshield_{m+1} = \argmin_{\sshield'\in\sShield'(\tau)} \,\expe[\cost\mid \tau;\sgen,\sshield',T].
    \end{align}
\end{definition}

The set $\sShield'(\tau)$ restricts the available set of shields that can be used for the next period for the given history $\tau$.
A na\"ive attempt for $\sShield'(\tau)$ would be to choose $\sShield'(\tau)=\sShieldFeas^{\sgen,T\mid\tau}$ for every $\tau$, so that fairness is guaranteed at the end of the current period. 
However, there exist histories for which $\sShieldFeas^{\sgen,T\mid\tau}$ would be 
empty,
implying that Equation~\eqref{eq:dyn-shield optimality} would not have a feasible solution for some $\tau$, and the \Dyn shield would exhibit undefined behaviors.

This happens because there may be traces of length $jT$ that satisfy a certain fairness constraint, 
but no shield can guarantee the next trace will satisfy the same constraint. 
\begin{example}\label{ex:buffered}
    Consider $\spec =\DP$, $\kappa = 0.1$, $T=100$, and a trace $\tau$ such that $n_a(\tau) = 2$, $n_{a1}(\tau) = 1$, $n_b(\tau) = 98$, and $n_{b1}(\tau) = 49$. 
    The trace $\tau$ satisfies $\DP(\tau) = |1/2 - 49/98| = 0$.
    Now assume we build a shield for the next fragment, and in generating the next trace $\tau'$, only individuals from group $b$ have appeared for the first 99 samples. Let $\tau'_{[1:99]}$ denote this trace, and let $\mathtt{Acc}_b$ denote $\welfare{b}{\tau\cdot\tau'_{[1:99]}}$.
    Then $\DP(\tau\cdot\tau'_{[1:99]}) = |1/2 - \mathtt{Acc}_b|$.
    If the last individual of $\tau'$ happens to be from group $a$, the acceptance rate of group $a$ moves from $1/2$ to either $1/3$ (if it gets rejected) or $2/3$ (if it gets accepted). 
    There is no possible value of $\mathtt{Acc}_b$
    that simultaneously guarantees $|1/3 - \mathtt{Acc}_b| \leq \kappa$ and $|2/3 - \mathtt{Acc}_b| \leq \kappa$.
\end{example}
To circumvent this technical inconvenience, we make the following mild assumption on the set of allowed histories, requiring $\sShield'(\tau)$ to fulfill fairness only if $\tau$ fulfills this assumption.
\begin{assumption}\label{assump:dynamic}
    For a given trace $\tau\in \Trfeas_{\sgen,\sshield}^{jT}$ with $j>0$, every valid suffix $\tau' $ of length $t$, i.e., $\tau'\in \left\{\tau''\in (\X\times\Y)^T\mid \tau\tau''\in\Trfeas_{\sgen,\sshield}^{(j+1)T}\right\}$, fulfills:
    \begin{align*}
        \frac{1}{\den^{a}(\tau\tau')} + \frac{1}{\den^b(\tau\tau')} \leq \kappa + \spec(\tau).
    \end{align*}
\end{assumption}

The set of shields $\sShield'(\cdot)$ available to the \Dyn shield in Definition~\ref{def:dynamic-shield} is then defined as:
\begin{equation}
    \sShield'(\tau) = \sShieldFeasDyn^{\sgen,T}(\tau) \coloneqq  
    \begin{cases}
        \sShieldFeas^{\sgen,T\mid\tau}    &   \tau \text{ fulfills Assumption~\ref{assump:dynamic}},\\
        \sShield                &   \text{otherwise}.
    \end{cases}
\end{equation}

With the following technical result, we prove that 
$\sShieldFeas^{\sgen,T\mid\tau}$ is non-empty whenever $\tau$ fulfills Assumption~\ref{assump:dynamic}, 
implying that $\sShieldFeasDyn^{\sgen,T}(\tau)$ is non-empty for every $\tau$.
This is the analogous result to Lemma~\ref{thm:acc-rates-existence} for dynamic shields.

\begin{lemma}
    \label{lem:buff-shield-exists}
    Let $\spec$ be a DoR fairness property 
    with $\spec(\tau) = |\welfare{a}(\tau) - \welfare{b}(\tau)|$, and $\welfare{g}(\tau) = \numer^g(\tau)/\den^g(\tau)$.
    Let
    $\tau_1$ be a trace and $\kappa \geq 0$. 
    There exists a shield $\sshield\in\sShield$ such that
    every trace $\tau_2\in \Trfeas^T_{\sgen, \sshield} \cap S$
    satisfies $\spec(\tau_1\cdot \tau_2)\leq \kappa$, where 
    \begin{equation*}
        S = \left\{\tau_2\in (\X\times\Y)^T\::\:  \frac{1}{\den^a(\tau_1\tau_2)} + \frac{1}{\den^b(\tau_1\tau_2)} \leq \kappa + \spec(\tau_1)\right\}.
        \end{equation*}
\end{lemma}

\begin{proof}
    The proof of this result is analogous to that of Lemma~\ref{thm:acc-rates-existence},
    with a slightly more convoluted argument.

    Let $n^1_a = \den^a(\tau_1)$, 
    $n^1_{a1} = \numer^a(\tau_1)$, 
    $n^1_b = \den^b(\tau_1)$, 
    and $n^1_{b1} = \numer^{b}(\tau_1)$.
    Without loss of generality, we can assume that 
    $n^1_{a1}/n^1_a - n^1_{b1}/n^1_b \geq 0$. 
    The alternative case is analogous.

    For a shield to exist that can enforce $\spec(\tau_1\tau_2) \leq \kappa$ there must exist, for every value of $\den^g(\tau_1\tau_2)$ (in demographic parity, the amount of individuals of a group),
    at least one way of deciding acceptance and rejection (value of $\numer^g(\tau_1\tau_2)$) 
    that maintains the fairness property in the target bound.
    Since we do not know \textit{a priori} how many individuals of each group will appear, this decision must be incremental, 
    and be such that the fairness property is maintained for any number of individuals. 

    If we name $n_a = \den^{a}(\tau_2) \geq N_a$ and 
    $n_b = \den^b(\tau_2)\geq N_b$, a new trace $\tau_2$ is enforceable if we can choose 
    $N_a$ and $N_b$ satisfying the following condition:
    there exist two sequences $(x_{n_a}), (y_{n_b})\subseteq\mathbb N$
    such that for all $n^a\geq N_a$ and $n^b \geq N_b$
    \begin{equation}\label{eq:local3}
        \left|
        \frac{n^1_{a1} + x_{n_a}}{n^1_a + n_a} -
        \frac{n^1_{b1} + y_{n_b}}{n^1_b + n_b}
        \right| \leq \kappa,
    \end{equation}
    and for both sequences $x_{n_a+1} - x_{n_a} \in \{0,1\}$ and 
    $y_{n_a+1} - y_{n_a} \in \{0,1\}$.

    With the spirit of maintaining the welfare bounds as a proxy to maintaining fairness, we try
    \[
    x_{n_a} \coloneqq \left\lfloor \frac{n^1_{a1}}{n^1_a}n_a\right\rfloor, \quad \mbox{and} \quad
    y_{n_a} \coloneqq \left\lceil \frac{n^1_{b1}}{n^1_b}n_b\right\rceil.
    \]
    Using the same argument as in the proof of Theorem~\ref{thm:acc-rates-existence}, point (i), the conditions on $x_{n_a}$ and $y_{n_b}$ incrementing by 0 or 1 are met by the fact that $n^1_{a1}\leq n^1_a$ 
    and  $n^1_{b1}\leq n^1_b$.

    By definition of the floor function and ceiling functions
    \begin{align*}
        \frac{n^1_{a1} + x_{n_a}}{n^1_a + n_a} \leq 
        \frac{n^1_{a1} + \frac{n^1_{a1}}{n^1_a}n_a }{n^1_a + n_a} = 
        \frac{n^1_{a1}}{n^1_a}, \\
        \frac{n^1_{b1} + y_{n_b}}{n^1_b + n_b} \geq 
        \frac{n^1_{b1} + \frac{n^1_{b1}}{n^1_b}n_b }{n^1_b + n_b} = 
        \frac{n^1_{b1}}{n^1_b}.
    \end{align*}
    Therefore
    \begin{equation}
        \frac{n^1_{a1} + x_{n_a}}{n^1_a + n_a} -
        \frac{n^1_{b1} + y_{n_b}}{n^1_b + n_b}
        \leq 
        \frac{n^1_{a1}}{n^1_a} -
        \frac{n^1_{b1}}{n^1_b} = \spec(\tau_1) \leq \kappa
    \end{equation}
    To prove Equation~\eqref{eq:local3}, we still have to prove that
    \begin{equation}\label{eq:local4}
        \frac{n^1_{a1} + x_{n_a}}{n^1_a + n_a} -
        \frac{n^1_{b1} + y_{n_b}}{n^1_b + n_b}
        \geq -\kappa.    
    \end{equation}
    
    By the definition of the floor function
    \[
    \frac{n^1_{a1} + x_{n_a}}{n^1_a + n_a} \geq 
    \frac{n^1_{a1} + \frac{n^1_{a1}}{n^1_a}n_a  - 1}{n^1_a + n_a} =
    \frac{n^1_{a1}}{n^1_a} - \frac{1}{n^1_a + n_a},
    \]
    and by the definition of the ceiling function
    \[
    \frac{n^1_{b1} + y_{n_b}}{n^1_b + n_b} \leq 
    \frac{n^1_{b1} + \frac{n^1_{b1}}{n^1_b}n_b +1}{n^1_b + n_b} = 
    \frac{n^1_{b1}}{n^1_b} + \frac{1}{n^1_b + n_b}.
    \]
    Putting the previous two inequalities together, we have
    \[
    \frac{n^1_{a1} + x_{n_a}}{n^1_a + n_a} -
    \frac{n^1_{b1} + y_{n_b}}{n^1_b + n_b} \geq \spec(\tau_1) - 
    \left(
    \frac{1}{n^1_a + n_a} + 
    \frac{1}{n^1_b + n_b}
    \right).
    \]
    To ensure that Equation~\eqref{eq:local4} holds, it is sufficient to 
    ensure that
    \begin{equation*}
        \spec(\tau_1) - 
    \left(
    \frac{1}{n^1_a + n_a} + 
    \frac{1}{n^1_b + n_b}
    \right) \geq -\kappa.
    \end{equation*}
    Rewriting the previous inequality we arrive to
    \begin{equation}
         \left(
    \frac{1}{n^1_a + n_a} + 
    \frac{1}{n^1_b + n_b}
    \right) \leq \kappa + \spec(\tau_1),
    \end{equation}
    which is the condition defining the set $S$.
    Therefore the proposed sequences $(x_{n_a})$ and $(y_{n_b})$ satisfy Equation~\eqref{eq:local3} for traces in $S$.
\end{proof}

Technically, this guarantees that the optimization problem in \eqref{eq:dyn-shield optimality} is feasible and $\sshield_{m+1}$ always exists, making \Dyn shields are well-defined (Definition~\ref{def:dynamic-shield}).
Intuitively, we obtain a ``best-effort'' solution: If Assumption~\ref{assump:dynamic} is fulfilled then $\sshield_{m+1}$ is in $ \sShieldFeas^{\sgen,T\mid\tau}$ and achieves fairness for the minimum expected cost.
Otherwise, $\sshield_{m+1}$ can be any shield in $\sShield$ that only optimizes for the expected cost; in particular, $\sshield_{m+1}$ will be the trivial shield that never intervenes (has zero cost).

Synthesis of \Dyn shields involves computing the sequence of shields $\sshield_1,\sshield_2,\ldots$, which are to be concatenated.
We outline the algorithm below.
\begin{enumerate}
    \item Generate a \FinShield shield (Definition~\ref{def:finshield}) $\sshield$ for the property $\spec$ and the horizon $T$. Set $\sshield_1\coloneqq \sshield$.
    \item For $i\geq 1$, let $\sshield$ be the concatenation of the shields $\sshield_1,\ldots,\sshield_i$, and let $\tau\in \Trfeas_{\sgen,\sshield}^{iT}$ be the generated trace. Compute $\sshield_{i+1}$ that uses the same approach as in Section~\ref{sec:synthesis} with Equation~\eqref{eq:v-basecase} being replaced by:
        \begin{equation}
        \label{eq:base-case-periodic-shields}
            v(\tau') = 
                \begin{cases}
                    0   &   \spec(\tau\tau')\leq \kappa,\\
                    \infty & \text{otherwise.}
                \end{cases}
        \end{equation}
\end{enumerate}
We summarize the fairness guarantee below.
\begin{theorem}[Conditional correctness of \Dyn shields]
\label{thm:correctness-dyn-shields}
    Let $\spec$ be a DoR fairness property.
    Let $\sshield$ be a  \Dyn shield that uses $\sShieldFeasDyn^{\sgen,T}(\cdot)$ as the set of available shields.
    Let $\tau =  \tau_1\ldots\tau_m\in \Trfeas^{mT}_{\sgen,\sshield}$ be a trace with $\len{\tau_i}=T$ for each $i\leq m$.
    Suppose for every $i\leq m$, $\tau_1\ldots\tau_i$ fulfills Assumption~\ref{assump:dynamic}.
    Then the fairness property $\spec(\tau)\leq \kappa$ is guaranteed.
\end{theorem}

\section{Experimental Evaluation}
\label{sec:fairness-experimental-evaluation}

\subsection{Experimental Setup}
\label{sec:experimental-setup}
We demonstrate the effectiveness of fairness shields by testing them in the task of shielding several ML classifiers in tasks that are standard benchmarks in the fairness literature.
We consider several state-of-the-art learning algorithms from literature of in-processing fairness, i.e., methods that enforce fairness \emph{during} the learning process, typically by means of adding fairness-inducing regularizers to the loss functions.
To preserve consistency among our experiments, we use the same neural architecture for every dataset and learning algorithm.

\begin{table*}[t]
    \centering
    \begin{tabular}{l@{\hspace{\tightcolspace}} l@{\hspace{\tightcolspace}} l@{\hspace{\tightcolspace}} l@{\hspace{\tightcolspace}} l@{\hspace{\tightcolspace}} c@{\hspace{\tightcolspace}} c@{\hspace{\tightcolspace}} c@{\hspace{\tightcolspace}} c@{\hspace{\tightcolspace}} r@{\hspace{\tightcolspace}}}
    \toprule
 \multirow{2}{*}{Dataset} & \multirow{2}{*}{Task} & Sensitive & \multirow{2}{*}{Instances} & Features & $y_0$ & $y_1$ & $g_a$ & $g_b$ & Stat.\\ 
  & & Attribute & & Num./Cat. & $\%$ & $\%$ & $\%$ & $\%$ & Par. \\
     \midrule
 \texttt{Adult} & income & race & 43131 & 5 / 7 & 75 & 25 & 10 & 90 & 0.14 \\ 
 \texttt{Adult} & income & gender & 45222 & 5 / 7 & 75 & 25 & 33 & 68 & 0.20 \\ 
 \texttt{Bank} & credit & age & 41188 & 9 / 10 & 89 & 11 & 2.6 & 97 & 0.13 \\ 
 \texttt{\textsc{Compas}} & recid. & gender & 6172 & 5 / 4 & 54 & 46 & 19 & 81 & 0.13 \\ 
 \texttt{\textsc{Compas}} & recid. & race & 6172 & 5 / 4 & 54 & 46 & 66 & 34 & 0.10 \\ 
 \texttt{German} & credit & gender & 1000 & 6 / 13 & 30 & 70 & 31 & 69 & 0.07 \\ 
 \texttt{German} & credit & age & 1000 & 6 / 13 & 30 & 70 & 19 & 81 & 0.15 \\ 
\bottomrule
    \end{tabular}
    \caption[Datasets characteristics]
    {Datasets characteristics. 
    The columns $y_{0/1}$ represent the percentage of instances where the ground truth is ``accept'' (1) and ``reject'' (0), respectively.
    The columns $g_{a/b}$ represent the percentage of instances that belong to group $a$ and $b$, respectively.
    The last column is the statistical parity, representing the inherent bias in the dataset.}
    \label{tab:datasets-info}
\end{table*}

\paragraph*{Computing infrastructure.}
All experiments were performed with a workstation with AMD Ryzen 9 5900x CPU, Nvidia GeForce RTX 3070Ti GPU, 32GB of RAM, running Ubuntu 20.04.

\paragraph*{Datasets.}
We used four tabular datasets in our experiments, all common benchmarks in the fairness community: Adult~\cite{misc_adult_2}, COMPAS~\cite{compas},
German Credit~\cite{dua2017uci} and Bank Marketing~\cite{misc_bank_marketing_222}.
Details on the task, sensitive attributes, size of the dataset, number of numerical and categorical features, as well as existing bias can be found in Table~\ref{tab:datasets-info}.

\paragraph*{Training ML classifiers.}
To train our ML models, we adapted the implementation provided by the FFB benchmark~\cite{han2024ffb}, using the same neural network, train-test splits, and most training hyper-parameters set as default in their implementation, tuning only the hyper-parameters related to fairness.
The classifiers receive the full set of features as their input, including the protected feature, which is marked as a special feature. 
This is appropriate, since the learning algorithms that enforce fairness are of the in-processing type, so they use the protected feature as part of their input, and typically include a term in the loss function that depends especially on the protected feature.

We use fixed architecture multi-layer perception (MLP) with three hidden layers with sizes $512$, $256$, and $64$ in all our experiments.
In each case, the model is trained for 150 epochs with batches of 1024 instances, with the exception of the $\texttt{German}$ dataset, which we trained with batches of 128, as the dataset has only 1000 instances.
We use the Adam optimizer~\cite{adam}, with a learning rate of 0.01.

\begin{table}[t]
    \centering
    \begin{tabular}{l c c c c c c }
    \toprule
   & $\mathtt{acc}$ & $\mathtt{ap}$ & $\mathtt{auc}$ & $\mathtt{f1}$ & $\DP$ & $\EO$ \\ 
     \midrule
 $\mathtt{ERM}$ & 91 & 62 & 94 & \textbf{57} & 11 & \textbf{3.7} \\ 
 $\mathtt{DiffDP}$ & 90 & 57 & 93 & 40 & 3.4 & 26 \\ 
 $\mathtt{HSIC}$ & \textbf{91} & \textbf{63} & \textbf{94} & 57 & 7.0 & 6.6 \\ 
 $\mathtt{LAFTR}$ & 91 & 60 & 94 & 40 & 6.0 & 4.4 \\ 
 $\mathtt{PR}$ & 91 & 59 & 93 & 49 & \textbf{3.3} & 34 \\ 
\bottomrule
    \end{tabular}
    \caption[Performance of the ML models. Dataset: \texttt{Bank}]{Performance of the ML models. Dataset: \texttt{Bank}.}
    \label{tab:ml-performance-bank_marketing}
\end{table}


\begin{table}
    \centering
    \begin{tabular}{l@{\hspace{\tightcolspace}} c@{\hspace{\tightcolspace}} c@{\hspace{\tightcolspace}} c@{\hspace{\tightcolspace}} c@{\hspace{\tightcolspace}} c@{\hspace{\tightcolspace}} c | c@{\hspace{\tightcolspace}} c@{\hspace{\tightcolspace}} c@{\hspace{\tightcolspace}} c@{\hspace{\tightcolspace}} c@{\hspace{\tightcolspace}} c}
    \toprule
 & \multicolumn{6}{c|}{race} & \multicolumn{6}{c}{gender}\\ \midrule
   & $\mathtt{acc}$  & $\mathtt{ap}$  & $\mathtt{auc}$  & $\mathtt{f1}$  & $\DP$  & $\EO$  & $\mathtt{acc}$  & $\mathtt{ap}$  & $\mathtt{auc}$  & $\mathtt{f1}$  & $\DP$  & $\EO$ \\ 
\midrule 
$\mathtt{ERM}$  & \textbf{85}  & \textbf{79}  & \textbf{91}  & \textbf{66}  & 11  & 6.1  & \textbf{85}  & \textbf{79}  & \textbf{91}  & \textbf{66}  & 16  & 9.6 \\ 
$\mathtt{DiffDP}$  & 84  & 76  & 90  & 62  & \textbf{5.4}  & 4.0  & 83  & 71  & 87  & 54  & 0.2  & 33 \\ 
$\mathtt{HSIC}$  & 85  & 79  & 91  & 64  & 8.7  & \textbf{3.1}  & 83  & 73  & 87  & 57  & 1.8  & 28 \\ 
$\mathtt{LAFTR}$  & 84  & 78  & 91  & 62  & 10  & 8.2  & 85  & 79  & 91  & 65  & 14  & \textbf{1.5} \\ 
$\mathtt{PR}$  & 84  & 76  & 89  & 61  & 5.6  & 4.0  & 83  & 71  & 87  & 53  & \textbf{0.1}  & 33 \\ 
\bottomrule
    \end{tabular}
    \caption{Performance of the ML models. Dataset: \texttt{Adult}.}
    \label{tab:ml-performance-adult}
\end{table}

\paragraph*{Learning algorithms.}
To train our classifiers, we used the following methods from the in-processing fairness literature:

\begin{itemize}
    \item Differential Demographic Parity (DiffDP) is a gap regularization method for demographic parity. DiffDP introduces a term in the loss function that penalizes differences in the prediction rates between different demographic groups~\cite{fairmixup}.
    \item The Hilbert-Schmidt Independence Criterion (HSIC) is a statistical test used to measure the independence of two random variables. Adding an HSIC term measuring the independence between prediction accuracy and sensitive attributes to the loss has been used as a fair learning method~\cite{hsic}.
    \item Learning adversarially fair and transferable representations (LAFTR) is a method proposed by~\cite{madras2018learning}, where the classifier learns an intermediate representation of the data that minimizes classification error while simultaneously minimizing the ability of an adversary to predict sensitive features from the representation.
    \item Prejudice Remover (PR)~\cite{kamishima2012fairness} adds a term to the loss that penalizes mutual information between the prediction accuracy and the sensitive attribute.
\end{itemize}

As a baseline, we trained a fifth classifier simply minimizing empirical risk. We call it the empirical risk minimizer (ERM).

\begin{table}
    \centering
    \begin{tabular}{l@{\hspace{\tightcolspace}} c@{\hspace{\tightcolspace}} c@{\hspace{\tightcolspace}} c@{\hspace{\tightcolspace}} c@{\hspace{\tightcolspace}} c@{\hspace{\tightcolspace}} c | c@{\hspace{\tightcolspace}} c@{\hspace{\tightcolspace}} c@{\hspace{\tightcolspace}} c@{\hspace{\tightcolspace}} c@{\hspace{\tightcolspace}} c}
    \toprule
 & \multicolumn{6}{c|}{gender} & \multicolumn{6}{c}{race}\\ \midrule
   & $\mathtt{acc}$  & $\mathtt{ap}$  & $\mathtt{auc}$  & $\mathtt{f1}$  & $\DP$  & $\EO$  & $\mathtt{acc}$  & $\mathtt{ap}$  & $\mathtt{auc}$  & $\mathtt{f1}$  & $\DP$  & $\EO$ \\ 
\midrule 
$\mathtt{ERM}$  & \textbf{65}  & 63  & 69  & 59  & 16  & 18  & 65  & 63  & 69  & 60  & 14  & 16 \\ 
$\mathtt{DiffDP}$  & 63  & 63  & 69  & 55  & 12  & 11  & 65  & 62  & 68  & 58  & 9.1  & 15 \\ 
$\mathtt{HSIC}$  & 64  & 63  & 69  & 56  & 15  & 9.8  & 64  & 62  & 68  & 57  & \textbf{8.2}  & \textbf{11} \\ 
$\mathtt{LAFTR}$  & 65  & \textbf{64}  & \textbf{70}  & \textbf{60}  & 17  & 15  & \textbf{65}  & \textbf{64}  & \textbf{70}  & \textbf{60}  & 13  & 18 \\ 
$\mathtt{PR}$  & 63  & 63  & 69  & 55  & \textbf{12}  & \textbf{8.0}  & 64  & 62  & 68  & 57  & 8.9  & 13 \\ 
\bottomrule
    \end{tabular}
    \caption{Performance of the ML models. Dataset: \texttt{\textsc{Compas}}.}
    \label{tab:ml-performance-compas}
\end{table}


\begin{table}
    \centering
    \begin{tabular}{l@{\hspace{\tightcolspace}} c@{\hspace{\tightcolspace}} c@{\hspace{\tightcolspace}} c@{\hspace{\tightcolspace}} c@{\hspace{\tightcolspace}} c@{\hspace{\tightcolspace}} c | c@{\hspace{\tightcolspace}} c@{\hspace{\tightcolspace}} c@{\hspace{\tightcolspace}} c@{\hspace{\tightcolspace}} c@{\hspace{\tightcolspace}} c}
    \toprule
 & \multicolumn{6}{c|}{gender} & \multicolumn{6}{c}{age}\\ \midrule
   & $\mathtt{acc}$  & $\mathtt{ap}$  & $\mathtt{auc}$  & $\mathtt{f1}$  & $\DP$  & $\EO$  & $\mathtt{acc}$  & $\mathtt{ap}$  & $\mathtt{auc}$  & $\mathtt{f1}$  & $\DP$  & $\EO$ \\ 
\midrule 
$\mathtt{ERM}$  & \textbf{76}  & \textbf{87}  & \textbf{77}  & \textbf{83}  & 5.3  & 5.5  & 75  & \textbf{86}  & \textbf{76}  & 82  & 14  & 15 \\ 
$\mathtt{DiffDP}$  & 73  & 86  & 74  & 81  & \textbf{1.1}  & 3.5  & 72  & 86  & 75  & 80  & \textbf{0.5}  & \textbf{1.8} \\ 
$\mathtt{HSIC}$  & 73  & 86  & 74  & 81  & 1.4  & \textbf{1.1}  & 74  & 86  & 74  & 82  & 4.2  & 6.0 \\ 
$\mathtt{LAFTR}$  & 73  & 87  & 76  & 81  & 8.2  & 4.1  & 73  & 85  & 74  & 81  & 10  & 6.4 \\ 
$\mathtt{PR}$  & 73  & 86  & 73  & 81  & 5.7  & 4.1  & \textbf{75}  & 86  & 75  & \textbf{83}  & 6.5  & 4.8 \\ 
\bottomrule
    \end{tabular}
    \caption{Performance of the ML models. Dataset: \texttt{German}.}
    \label{tab:ml-performance-german}
\end{table}

\paragraph*{Hyperparameter tuning.}

Each of the in-processing fairness algorithms depends on the value of certain parameters that indicate the trade-off in the loss function between prediction accuracy and fairness.
For each training algorithm, we manually fine-tuned the parameters to obtain a good performance with the same parameter values across all benchmarks. 
Unfortunately, the parameters of different algorithms have different interpretations and characteristic dimensions, so comparing them is not informative. We detail the ones we used in our experiments, and their meaning.
\begin{itemize}
    \item For DiffDP, a parameter $\lambda$ controls the contribution of the regularization term in the loss. 
    We tried a range of $\lambda\in[0.5, 10]$.
    We use $\lambda = 1$.
    \item In HSIC, a parameter $\lambda$ controls the importance of the HSIC term in the loss function. We tried a range $\lambda\in [10, 500]$. We use $\lambda = 100$.
    \item In LAFTR, the loss is composed of three terms: 
    one that penalizes reconstruction error ($L_x$), 
    one that penalizes prediction error ($L_y$), 
    and one that penalizes the adversary's error when trying to obtain information about sensitive features from the representation ($L_z$).
    Three parameters $A_x, A_y, A_z$ control the weights of each term in the loss. We use $A_x=8$, $A_y = 4 $, $A_z = 2.1$.
    We tried a range or $[1,10]$ for each parameter.
    \item For PR, a parameter $\lambda$ controls the weight of the loss term that penalizes mutual information between prediction accuracy and the sensitive attribute. 
    We tried a range of $\lambda\in [0.01,0.5]$.
    We use $\lambda = 0.06$.
\end{itemize}
In Tables~\ref{tab:ml-performance-adult},~\ref{tab:ml-performance-bank_marketing},~\ref{tab:ml-performance-compas},~\ref{tab:ml-performance-german} we show the metrics of each trained model on each dataset. 
For each case, we present accuracy ($\mathtt{acc}$), average precision ($\mathtt{ap}$), area under the curve ($\mathtt{auc}$), and the $F1$ score ($\mathtt{f1}$) as performance metrics, 
while demographic parity ($\DP$) and equal opportunity ($\EO$) are presented as fairness metrics.
The numbers are presented as percentages.
In each column, the best performer is marked in boldface.


\paragraph*{Approximation of the input distribution.}
\label{sec:cost_alignment}
For shield synthesis, we need a distribution of the input space $\sgen\in \distrset(\G\times\BB\times\costset)$. 
In the ideal case, $\sgen\in \distrset(\G\times\BB\times\costset)$ is the exact joint distribution of group membership, agent recommendation and cost. 
However, this is unrealistic most of the time, as it assumes knowledge of the underlying distribution and the classifier. 
Furthermore, the distribution of cost given by the agent may be continuous, but we assume that there is a finite set $\costset$ of costs allowed. 

For our experiments, we used 
a simple approach that is agnostic to the ML classifier.
We assume there is a cost set of $k$ possible values $\costset = \{c_1,\dots, c_k\}$ uniformly distributed in the interval $[0,1]$, and that any recommendation is equally likely.
Therefore for all $c_i\in\costset$, $g\in \G$ and $r\in\BB$, we have
$\sgen(g,b,c_i) = 1/4k$.
This approximation is easy to compute and agnostic to the ML classifier.

\begin{figure}
     \centering
     \begin{subfigure}[b]{0.43\linewidth}
         \centering
         \includegraphics[width=\linewidth]{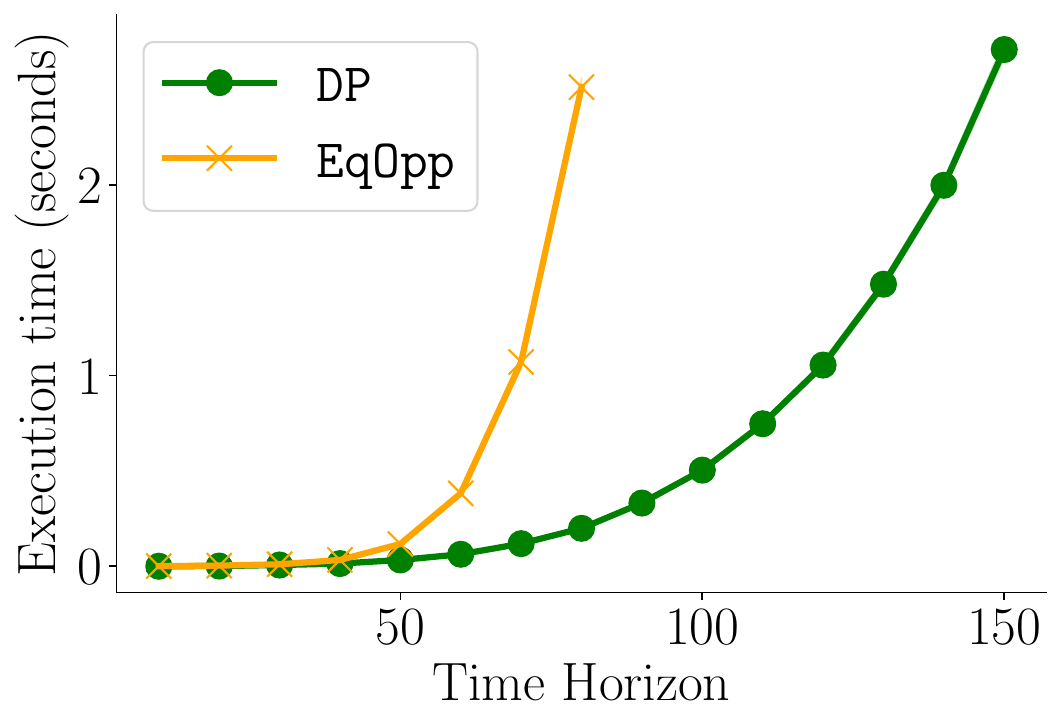}
         \caption{Time.}
         \label{fig:comp_time}
     \end{subfigure}
     \hfill
     \begin{subfigure}[b]{0.43\linewidth}
         \centering
         \includegraphics[width=\linewidth]{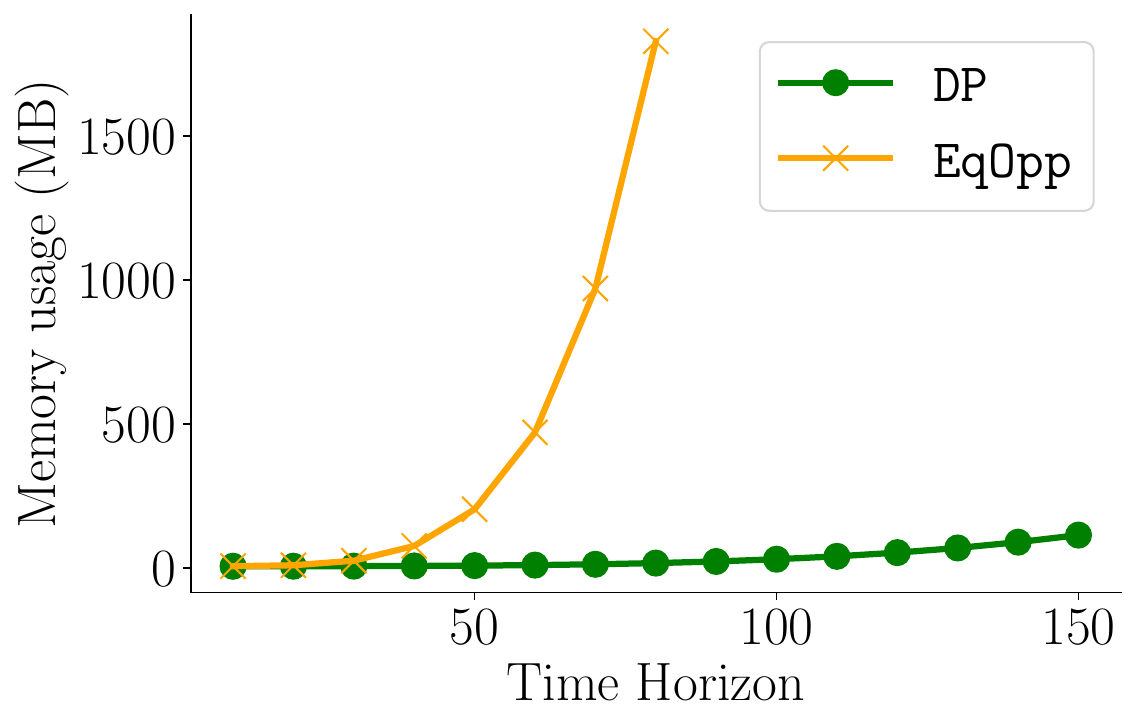}
         \caption{Memory.}
         \label{fig:comp_memory}
     \end{subfigure}
        \caption[Resource usage for fairness shield synthesis]{Resource usage for shield synthesis with increasing time horizons.}
        \label{fig:comp_resources}
\end{figure}

\paragraph*{Cost given by the classifier.}
While $\sgen$ is the theoretical distribution that has to be used to synthesize the shield, 
at the time of deployment the classifier has to choose an output in the form of a recommendation and a cost. 
In this case, the natural choice is given by the last layer of the neural network.
The last layer in a neural network for classification is usually a softmax layer that assigns for each label a value between 0 and 1, that we can interpret it as the ``confidence value'' that the classifier gives to that label being true. 
We use this ``confidence value'' as the cost.

\subsection{Shield Synthesis Computation Times}
\label{sec:shield-computation-times}

As pointed out in Theorem~\ref{thm:bounded-horizon shield synthesis-bis}, 
our shield synthesis algorithm has a polynomial complexity for both DP and EqOpp, and the degree of the polynomial is the number of counters required to keep track of the fairness property. 
For $\DP$ it is sufficient to track 4 counters: the number of instances appeared and accepted of each group. 
For $\EO$, we also need 4 counters for the number of instances appeared and accepted of each group, counting only those for which $z=1$. Furthermore, we need two extra counters: one to count all instances with $z =0$, and one to keep track of the last decision for which ground truth has not yet been revealed, for a total of 6 counters. 

In Figure~\ref{fig:comp_resources} we show the computation time and memory usage of our shield synthesis algorithm for a fixed problem with increasing time horizon. Figure~\ref{fig:comp_resources} does not show variability, because the synthesis algorithm, as described, is deterministic.

\begin{table}
    \centering
\begin{tabular}{llrrrrrr}
\toprule
 &  & Q1  & Median & Q3 & Mean & St. Dev. & Above \\
\midrule
\multirow[c]{2}{*}{DP} & No Shield & 0.38 & 0.83 & 1.59 & 1.22 & 1.29 & 42.46 \% \\
 & Static-Fair & 0.18 & 0.42 & 0.74 & 0.46 & 0.31 & 0.00\% \\
 \midrule
\multirow[c]{2}{*}{EqOpp} & No Shield & 0.67 & 1.76 & 3.62 & 2.76 & 3.01 & 65.06 \%\\
 & Static-Fair & 0.00 & 0.21 & 0.50 & 0.27 & 0.28 & 0.00 \%\\
 \bottomrule
\end{tabular}
    \caption[Statistic of normalized fairness]
    {Statistic of normalized fairness for finite horizon shields.}
    \label{tab:fair_distr}
\end{table}

\begin{figure}
   \centering
   \includegraphics[width=0.9\linewidth]{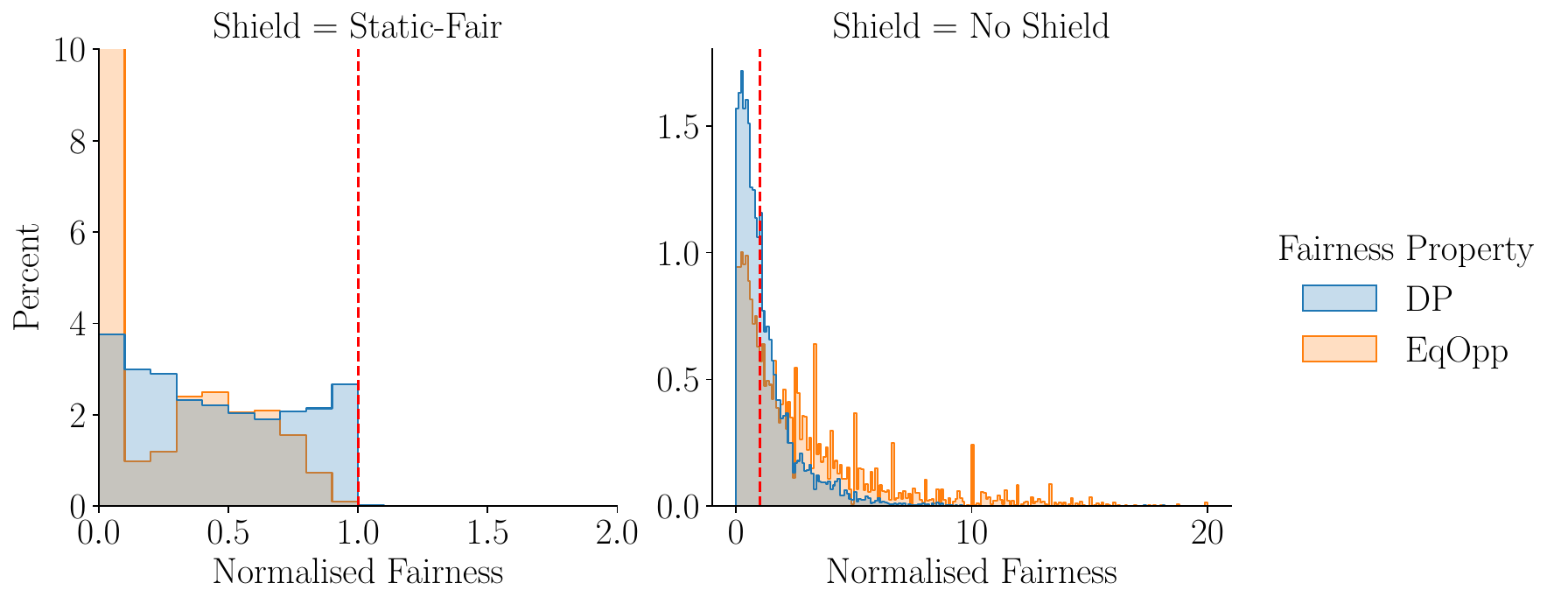}  %
    \caption[Distribution of normalized bias]
    {Distribution of normalized bias, i.e. Bias / $\kappa$, across all runs with (left) and without shield (right) for both demographic parity and equal opportunity.}
    \label{fig:fair-distr}
\end{figure}

\subsection{Performance of Finite Horizon Shields}
In this group of experiments, we investigate the performance of \FinShield shields on a single period. 
We use a time horizon of $T=100$ for DP and $T=75$ for EqOpp,
with fairness thresholds $\kappa\in\{0.05, 0.1, 0.15, 0.20\}$.
For each setting we synthesized a \FinShield shield and simulated $30$ runs.

\paragraph*{Performance in terms of fairness.}
In Table~\ref{tab:fair_distr} we present the aggregated results of our experiments in terms of normalized fairness, i.e., the fairness value divided by the given fairness threshold. 
When normalized, a value smaller than 1 indicates that the algorithm is within the constraints, 
while a value larger than 1 indicates the algorithm is being too biased.
In Figure~\ref{fig:fair-distr} we illustrate the same distributions on a more graphical way, by plotting the corresponding value distributions.
In the table we summarize the distribution by showing mean, median, standard deviation, Q1 (i.e., the $25\%$ quantile) and Q3 (i.e., the $75\%$ quantile).
The last row shows the percentage of the samples that go over the fairness constraint, i.e., with a normalized fairness value larger than 1.
As expected, all shielded samples are compliant with the fairness constraint. 
Note that most runs with shield achieve a fairness value significantly below the threshold.
Another common trend is that equal opportunity is in general a harder constraint to satisfy than demographic parity in our experiments.

\begin{table}[t]
    \centering
    \begin{subtable}{\linewidth}
        \centering
        \includegraphics[width=\linewidth]{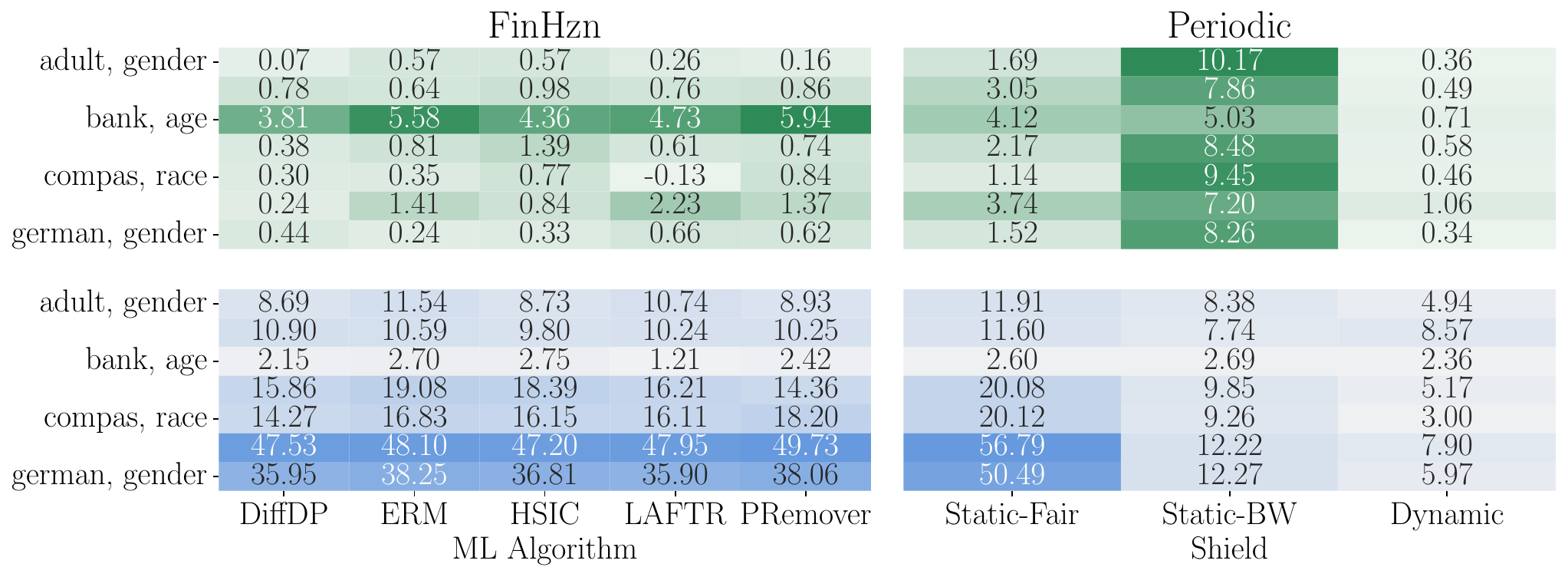}
        \caption{DP (top, green) and EqOpp (bottom, blue) with $\kappa=0.15$.}
        \label{tab:utility-comparison-0.15}
    \end{subtable}
    \vspace{1em} 
    \begin{subtable}{\linewidth}
        \centering
        \includegraphics[width=\linewidth]{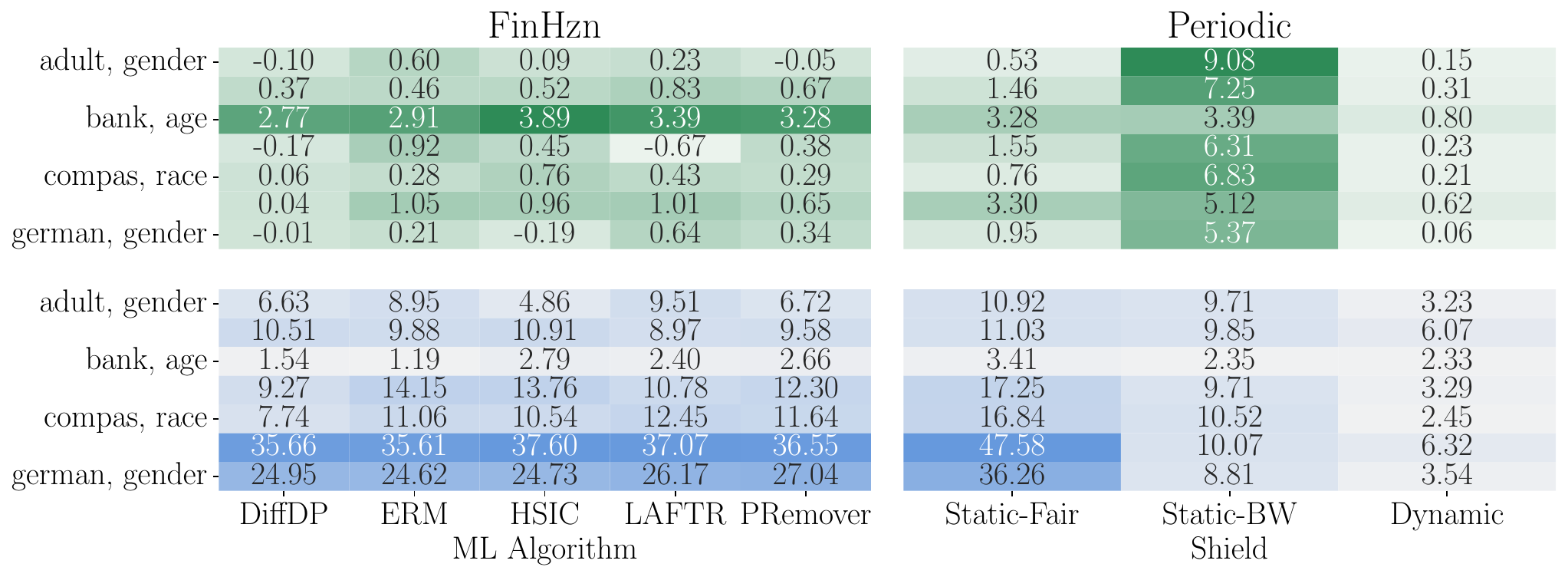}
        \caption{DP (top, green) and EqOpp (bottom, blue) with $\kappa=0.2$.}
        \label{tab:utility-comparison-0.2}
    \end{subtable}
    \caption[Comparison of utility loss]
    {Comparison of utility loss (in \%). Left: finite horizon shields and different ML models. Right: periodic shields only on the ERM model.}
    \label{tab:utility-comparison}
\end{table}

\paragraph*{Performance in terms of utility loss.}

The utility of classification tasks is measured by classification accuracy. Interventions by the fairness shield, which occasionally convert ``correct'' classifications into ``incorrect'' ones for fairness, typically reduce this utility\footnote{The scary quotes in ``correct'' and ``incorrect'' are here to emphasize that we are considering correctness with respect to the ground truth of the given --- potentially biased --- dataset. 
We do not enter here in the debate of whether a classifier that is more fair and less accurate with respect to the trained data is more or less correct in a general sense.}. 
We measure \emph{utility loss} on a given run as the difference in utility between the unshielded and shielded runs, relative to the utility of the former.

Table~\ref{tab:utility-comparison} (left) shows the average utility loss across all simulations for a threshold of $\kappa = 0.15$ and $\kappa = 0.2$, respectively, for finite horizon shields. 
We can observe that the mean utility loss is smaller when the classifier is trained to be fair, as fewer interventions are needed. In general, utility loss increases as the bias threshold $\kappa$ decreases, with more pronounced differences between classifiers for smaller $\kappa$. 
We also observe that that most of the variability comes from the dataset rather than from the ML algorithm.
These observations are also supported by Figure~\ref{fig:utility_mlalgo_box}, which provide insight into the distribution of utility loss for each ML algorithm accross all datasets.



Finally, we compared the values of utility loss to the cost incurred by the shield. 
We do this to validate our approach: we compute shields by minimizing their expected cost, but our real target when deployed is to minimize the utility loss.
In Figure~\ref{fig:cost-utility-corr} we show that indeed shield cost and utility loss are very much correlated, 
validating the use of one as a proxy for the other.

\begin{figure}
   \centering
   \includegraphics[width=0.9\linewidth]{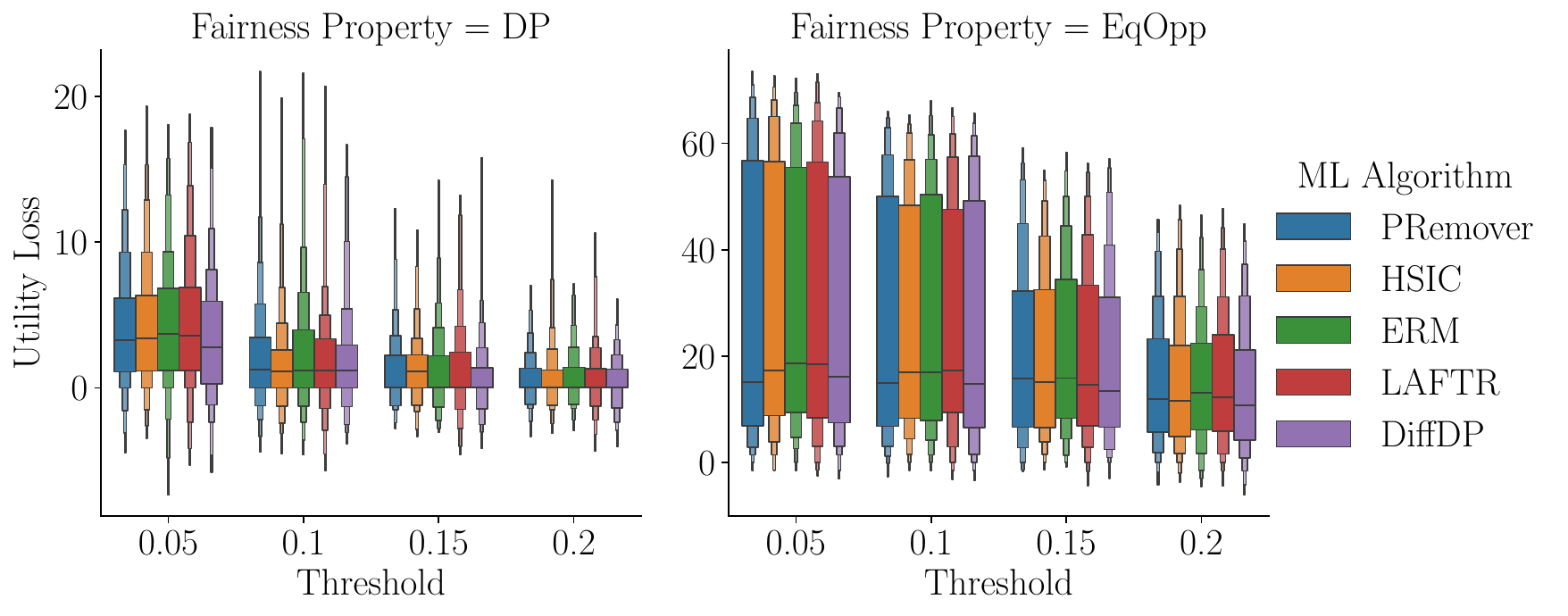}  %
    \caption[Utility loss accross ML algortithms and fairness thresholds]
    {Distribution of utility loss (in $\%$) incurred by \FinShield aggregated across all environments for DP (left) and EqOpp (right). The hight of the boxes indicate the spread of the distribution. }
    \label{fig:utility_mlalgo_box}
\end{figure}

\begin{figure}
   \centering
   \includegraphics[width=\linewidth]{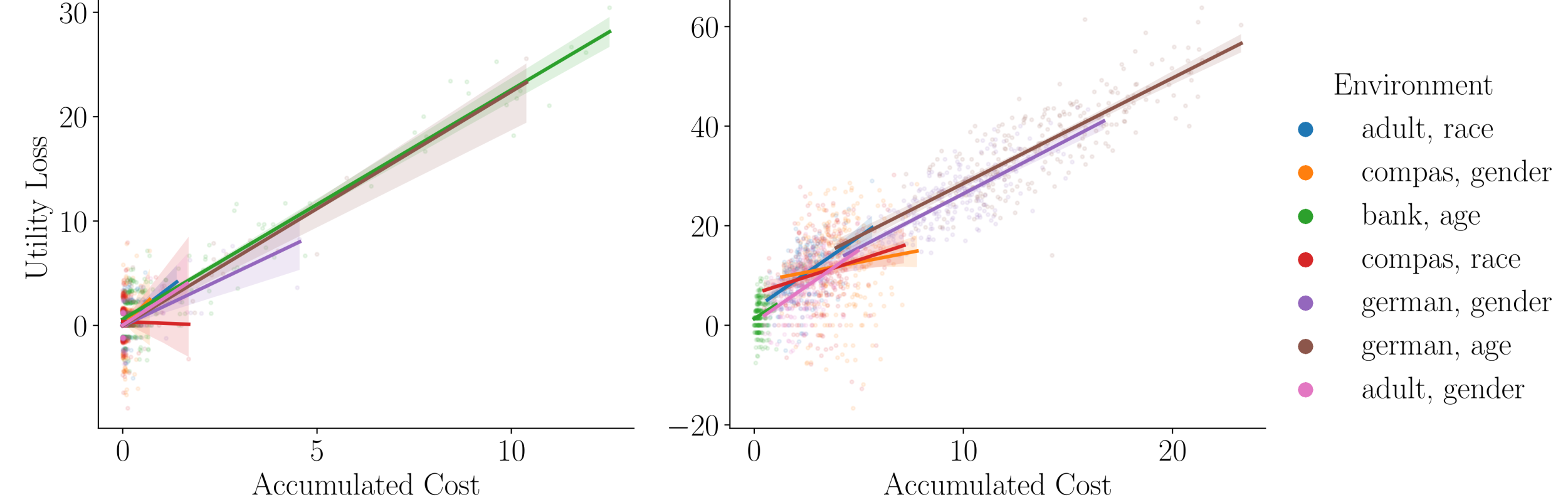}  %
    \caption[Utility loss vs. shield cost regression plot]
    {Regression plot depicting the relationship between utility loss and cost for $\kappa=0.2$ for each dataset.
    The results for other values of $\kappa$ are analogous.
    DP (left) and EqOpp (right).}
    \label{fig:cost-utility-corr}
\end{figure}

\subsection{Periodic Shielding}

In this group of experiments, we investigate the performance of periodic shields.
We synthesized \StaticDP, \StaticBAR, and \Dyn shields with $T=50$ for DP and EqOpp, 
with fairness thresholds $\kappa\in\{0.05, 0.1, 0.15, 0.20\}$,
and simulated them for $10$ periods.
We compare the models' performances, with and without shielding, across $20$ simulated runs. 

\begin{table}
    \centering
    \small
    \begin{tabular}{l l l c r r}
    \toprule
          & \multirow[c]{2}{*}{Assumption} & \multirow[c]{2}{*}{$\spec$} & Assumption & Fairness  \\
          & & & satisfied  & satisfied \\
          \midrule
            \multirow[c]{2}{*}{\StaticDP} & \multirow[c]{2}{*}{$\den^{a,b}(\tau_i) = \den^{a,b}(\tau_j)$} & DP       & $0.0 \,\%$ & $95.7 \,\%$   \\
                                          & & EqOpp & $0.0 \,\%$ & $100 \,\%$   \\
          \midrule
          \multirow[c]{2}{*}{\StaticBAR} &  \multirow[c]{2}{*}{$\den^a(\tau_i), \den^b(\tau_i) \geq \lceil\frac{1}{u-l}\rceil$} & DP       & $43.8 \,\%$ & $83.1 \,\%$  \\
                                          & & EqOpp & $4.1 \,\%$ & $56.4 \,\%$  \\
          \midrule
          \multirow[c]{2}{*}{\Dyn} & \multirow[c]{2}{*}{$\frac{1}{\den^{a}(\tau\tau')} + \frac{1}{\den^b(\tau\tau')} \leq \kappa + \spec(\tau)$} & DP       & $100 \,\%$ & $100 \,\%$ \\
                                          & & EqOpp & $49.8 \,\%$ & $100 \, \%$ \\
          \bottomrule
    \end{tabular}
    \caption{Comparison of different types of fairness shields.
    }
    \label{tab:extensions-comparison-bis}
\end{table}

\begin{figure}
   \centering
   \includegraphics[width=1.1\linewidth]{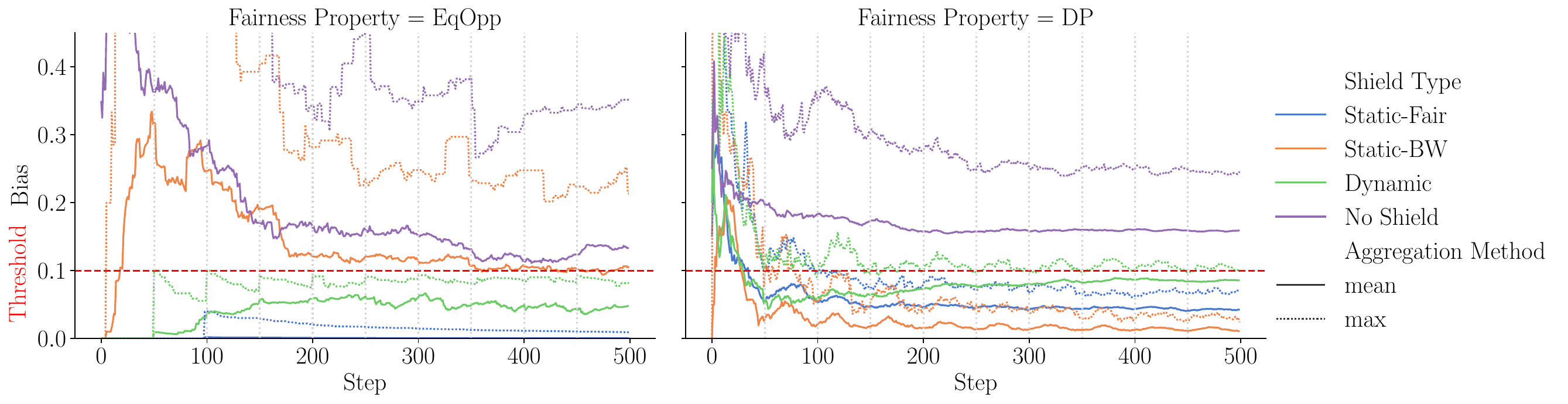}  %
    \caption[Bias over time with and without periodic shielding]
    {
    Variations of bias over time for the ERM classifier on the Adult dataset with and without periodic shielding.
    }
    \label{fig:single-run}
\end{figure}

\paragraph*{Guaranteed fairness vs. actual fairness.}
Since in the periodic case we have lost the hard fairness guarantees, we first aim to determine how effective are different types of periodic shields in enforcing their corresponding fairness property. 
We summarize our results in Table~\ref{tab:extensions-comparison-bis}.
The ``Assumption'' column is a reminder of the theoretical assumption under which each of the shields works, 
as seen in Theorem~\ref{thm:static shield with bounded DP} for \StaticDP shields, 
Assumption~\ref{assump:static-BAR} for \StaticBAR shields, and Assumption~\ref{assump:dynamic} for \Dyn shields.
For each type of shield and fairness property, we present how often the assumption is satisfied across all experiments, and how often the fairness property is satisfied. 
Since each assumption has its corresponding result guaranteeing fairness (Theorems~\ref{thm:static shield with bounded DP},~\ref{thm:correctness-staticBW},~and~\ref{thm:correctness-dyn-shields}),
we know a priori for each instance that the fairness target is going to be satisfied at least as often as the assumption.
We observe that 
the assumption for \StaticDP is almost never met, 
the assumption for \StaticBAR is also often violated,
and the assumption for \Dyn is almost always satisfied.
Nevertheless, both \StaticDP and \StaticBAR still perform well as heuristics, with many runs satisfying the fairness constraint. 
It is notable that \StaticDP offers better empirical performance than \StaticBAR, 
even though the balance assumption is almost never met.
The more expensive \Dyn shields outperform both static approaches in terms of both assumption satisfaction and fairness satisfaction.

It is in principle unclear what to do in cases where the fairness guarantees are not satisfied. 
Static-fair shields have a defined behaviour regardless of the input. 
However, \StaticBAR and \Dyn shields, we synthesize them by modifying the conditions in the base case of the recursion (Equation~\eqref{eq:v-basecase}) for a condition on the new fairness target and the assumption. Concretely, the traces that do not satisfy the fairness guarantee are given cost zero if they also fail to satisfy the assumption. The concrete conditions are detailed in Equations~\eqref{eq:synthesis-BW-shields}~and~\eqref{eq:base-case-periodic-shields}.
This choice ensures that traces outside the assumption do not hinder the optimization process.
This translates in deployment as shields that work for ensuring fairness until the trace reaches a point where it can no longer recover from failing the corresponding assumption.
If and when this point arrives, the shield ``gives up'' and becomes transparent until the beginning of the next period.

\paragraph*{Performance in terms of fairness.}
As an illustrative example, we show in Figure~\ref{fig:single-run} a single run of the ERM trained model on the Adult dataset, with fairness shields synthesized for $\kappa = 0.1$,
with \texttt{gender} as the sensitive feature.
Recall from Table~\ref{tab:datasets-info} that the statistical parity of the dataset is $0.2$, 
so fairness enforcement will be required.
The different colors indicate the different types of shields. 
The main observation is that only \Dyn shields show the truly periodic behaviour, 
where fairness is guaranteed at the end of each period by a minimal margin.

\begin{figure}
   \centering
   \begin{subfigure}[b]{\linewidth}
       \centering
       \includegraphics[width=0.9\linewidth]{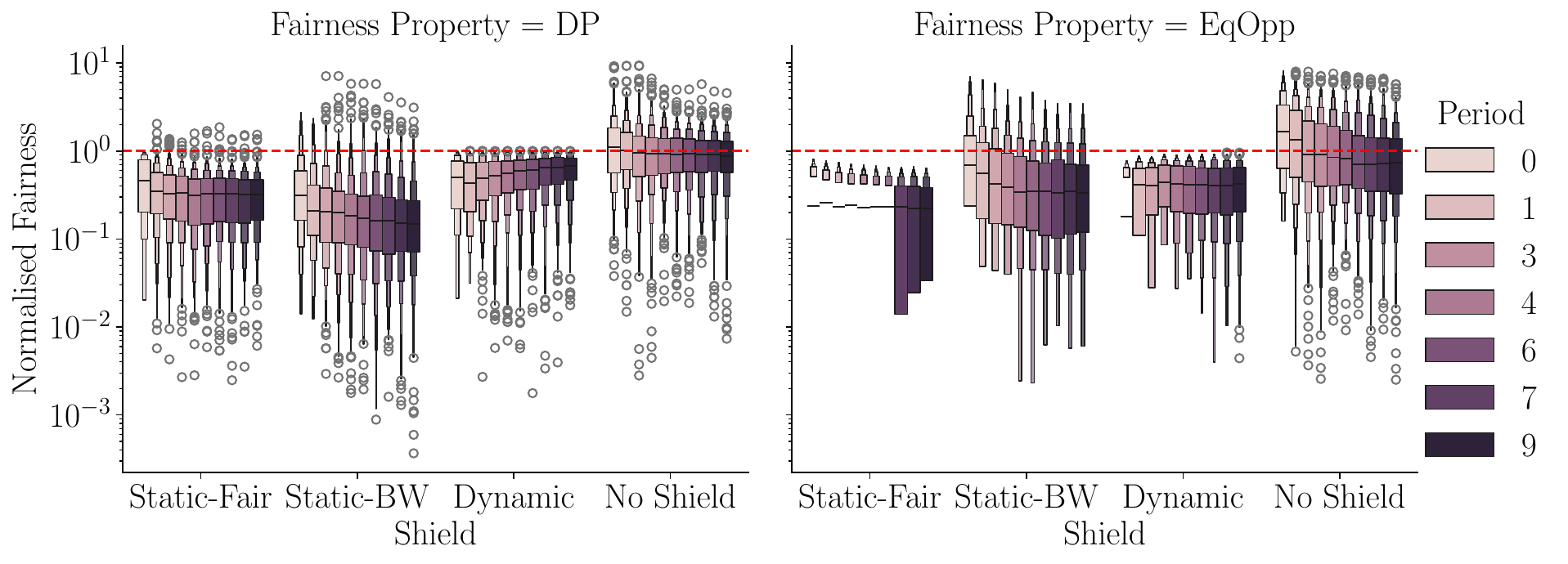}
       \caption{Distribution of normalized bias for each period for all runs.}
       \label{fig:fair-period}
   \end{subfigure}
   \vspace{1em} 
   \begin{subfigure}[b]{\linewidth}
       \centering
       \includegraphics[width=0.9\linewidth]{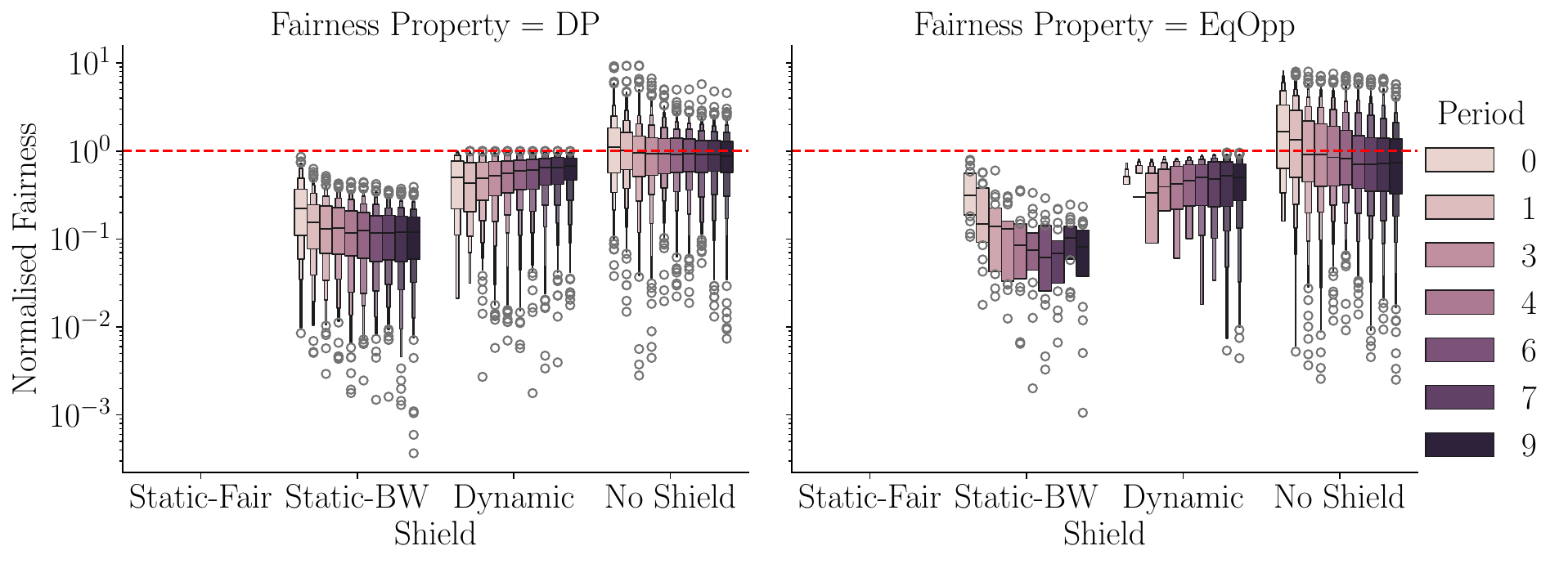}
       \caption{Distribution of normalized bias for all runs \emph{where the assumption is satisfied}.}
       \label{fig:fair-period-cond}
   \end{subfigure}
   \caption[Distribution of normalized bias for each period]
   {Distribution of normalized bias, for each period. 
   Each run below the red line satisfies the fairness condition, in terms of DP (left) and EqOpp (right).}
   \label{fig:normalized-bias-periodic-fair-comparison}
\end{figure}

In Figure~\ref{fig:normalized-bias-periodic-fair-comparison} we provide insight into the distribution of the normalized bias. 
The former aggregates over all runs, while the latter considers only those runs for \StaticBAR and \Dyn that satisfy the assumption. 
We can observe that \StaticBAR shields has a relatively high rate of violation in general (Figure~\ref{fig:fair-period}), 
while at the same time being overly conservative when the assumption is satisfied (Figure~\ref{fig:fair-period-cond}).
This problem does not exist with \Dyn shields, as the normalized fairness of the average run is only slightly below the threshold. 
In all cases we can report an improvement over the unshielded runs in terms of fairness satisfaction.

\paragraph*{Performance in terms of utility loss.}
Table~\ref{tab:utility-comparison} (right) shows the average utility loss for the ERM model across all simulations for a threshold of $\kappa = 0.15$ and $\kappa = 0.2$, respectively, for periodic shields.
In general, if the assumptions are satisfied, $\Dyn$ shields incur the least loss and $\StaticBAR$ shields incur the most, which is due to their stricter BW objectives. 
However, an assumption violation forces both $\Dyn$ and $\StaticBAR$ shields to go inactive incurring no additional utility loss. Therefore, the low utility loss of $\StaticBAR$ shields in EqOpp can be explained by the frequent assumption violations.
For demographic parity, \Dyn shields outperform \StaticDP and \StaticBAR shields, 
with \StaticBAR shields experiencing the highest utility loss due to their stricter BW objective.
For EqOpp, the difference is less pronounced, partly because for \StaticBAR the frequent assumption violation forces the shield to go idle.
We also observe, as expected, that the utility loss decreases when increasing $\kappa$.

We finish by studying the distribution of utility loss incurred across the different periods, 
as depicted in Figure~\ref{fig:utility-periods}. 
That is, for each run we normalize the utility loss per period by the total utility loss of the run. 
We observe that \Dyn shields incurr most of their losses in the earlier periods,
a trend not observed for the other shields. 
Negative values indicate some rare periods in which the shield actually increases the utility of the classifier.

\begin{figure}
   \centering
   \includegraphics[width=0.9\linewidth]{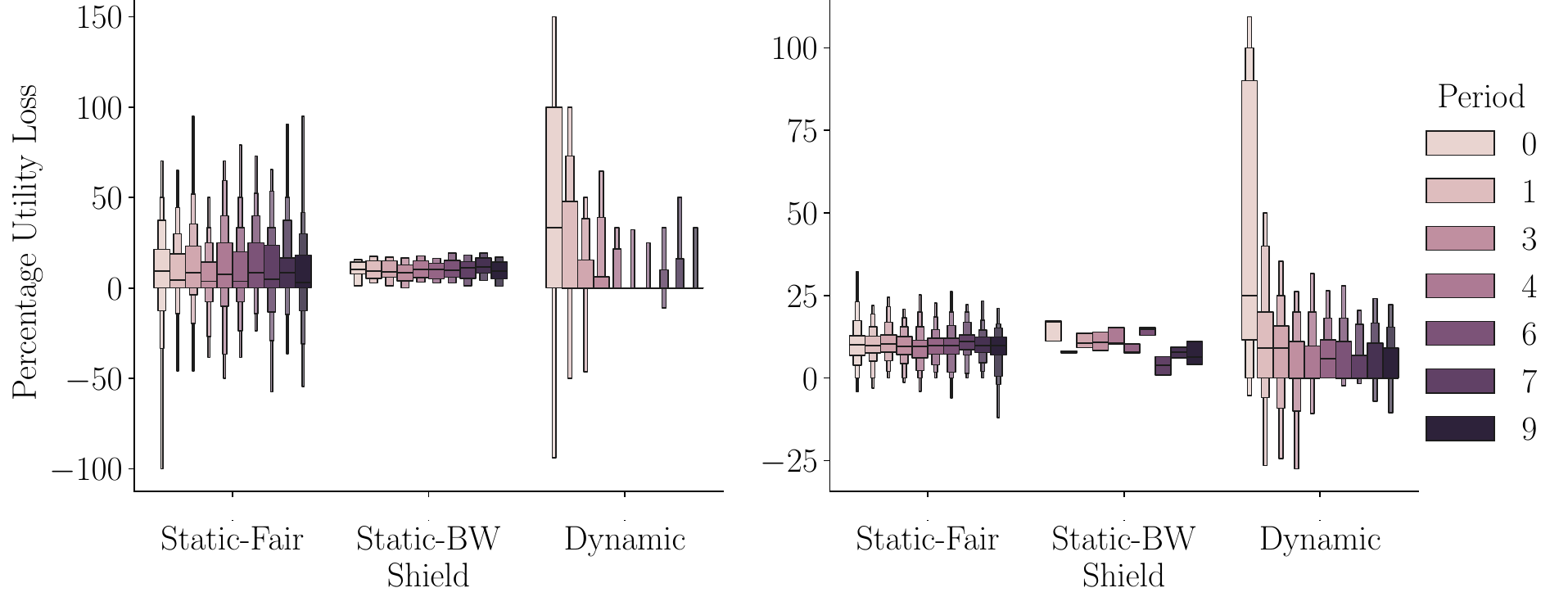}  %
    \caption[Percentage of total utility loss in periodic shields]
    {Percentage of total utility loss (in $\%$) for each period incurred by \StaticDP, \StaticDP and \Dyn across all runs for $\kappa=0.15$. 
    DP (left) and EqOpp (right).}
    \label{fig:utility-periods}
\end{figure}

\section{Discussion}
\label{sec:fairness-discussion}

\subsection{Existence and Composability of Finite Horizon Shields}

When approaching the problem of enforcing fairness properties through finite horizon shielding,
there are two phenomena that can appear unintuitive in the beginning.
The first is the fact that finite horizon shields always exists for DoR properties, independent on the specification threshold, time horizon or trace balance.
The second is that finite horizon shields for a certain specification allow arbitrarily biased traces when concatenating them on a static manner --- what we called \StaticDP shields.
In this section we explore these two phenomena and shed some light on the type of edge cases where we observe them.

\subsubsection{Existence of Finite Horizon Shields}
\label{sec:existence-finite-horizon-shields}
The set of feasible solutions of the optimization problem in Equation~\eqref{eq:finite horizon shield} is nonempty for DoR properties, because the fairness-shield that always accepts or always rejects each candidate from each group is a solution that trivially fulfils $\spec(\tau)\leq \kappa$.
In fact, in these cases, the feasible traces always satisfy $\spec(\tau) = 0$.
    
Even nontrivial optimal fairness-shields may exhibit such degenerate behaviors at runtime, when the order of appearances of individuals from the two groups is excessively skewed. 
Consider the following example for $\spec = \DP$ (demographic parity).
Let $T\in\NN$ be a time horizon and 
$\kappa < 1/T$.
As we know, given a trace $\tau$, demographic parity is defined as
\[
\DP(\tau) = \left| 
 \frac{n_{a1}(\tau)}{n_a(\tau)} - \frac{n_{b1}(\tau)}{n_b(\tau)}  \right|.
\]
Suppose at time $T-1$, all the individuals seen so far were from group $a$ (i.e., $n_a=T-1$ and $n_b=0$).
If some of the individuals were accepted and the rest rejected, then $0<n_{a1}<n_a$, implying $\kappa<\frac{n_{a1}}{n_a}< 1-\kappa$.
Now if the $T$-th individual $x$ is from group $b$, $n_b$ becomes $1$, and no matter which action the shield picks, DP will be violated:
If $x$ is accepted, then $n_{b_1} = \frac{n_{b1}}{n_b}=1$, and if $x$ is rejected, then $n_{b1}=\frac{n_{b1}}{n_b}=0$.
In both cases, 
$\DP\left(\tau\right)>\kappa$.
Therefore, the shield must have made sure that each individual until time $T-1$, all of whom were from group $a$, were either accepted or rejected.
Luckily, the chances of such skewness of appearance orders are rare in most applications, so that \FinShield as in Definition~\ref{def:finshield} exhibit effective, non-trivial behaviours in most cases, as seen from our experiments. 

\subsubsection{Counterexample Families for \StaticDP Shields Being Not Composable}
\label{sec:counterexamples-static-fair}

We have already shown in Example~\ref{ex:counter-example-periodic-naive-bis} that traces can have zero bias in terms of $\DP$, and when composed have a bias arbitrarily close to 1. 
While the family of counterexamples presented in Example~\ref{ex:counter-example-periodic-naive-bis} is quite degenerate in the sense that acceptance rates are always either 0 or 1, 
we present here another family of counterexamples that is less degenerate.
We write these examples for demographic parity, but the same ideas can be applied to build counterexamples for any DoR property.

Let $T > 0$ and $ 0 < K < T/2$. The family of counterexamples will be parametrized by $(T, K)$. 
For a pair $(T,K)$ consider 
traces $\tau_1$, $\tau_2$ such that 
$(n_{a1}, n_a, n_{b1}, n_b)(\tau_1) = (1, K, 1, T-K)$, 
and 
$(n_{a1}, n_a, n_{b1}, n_b)(\tau_2) = 
(T-K-1, T-K, K-1, K)$.

In the trace $\tau_1$, exactly one element of each group was accepted, while in the trace $\tau_2$, all but one element of each group were accepted. 
The values of demographic parity are:
\begin{equation} \label{eq:dp1}
    \DP(\tau_1)_{T,K} = \left|
    \frac{1}{K} - \frac{1}{T-K}\right| = \frac{T-2K}{(T-K)K}.
\end{equation}

\begin{equation}\label{eq:dp2}
    \DP(\tau_2)_{T,K} = \left|
    \frac{T-K-1}{T-K} - \frac{K-1}{K}
    \right| = 
    \frac{T-2K}{(T-K)K}.
\end{equation}

\begin{equation}\label{eq:dp1+2}
    \DP(\tau_1\tau_2)_{T,K} = \left|
    \frac{T-K}{T} - \frac{K}{T}
    \right| = 
    \frac{T-2K}{T}.
\end{equation}

These pairs of traces are not a counterexample for every pair $(T,K)$.
However, we can observe that, once fixed $K$, the limit when $T\to \infty$ of Equation~\eqref{eq:dp1} and Equation~\eqref{eq:dp2} is $1/K$, but the limit when $T\to\infty$ of Equation~\eqref{eq:dp1+2} is 1.
Therefore, for every $\varepsilon$, we can find $K$ large enough such that $1/K < \varepsilon/2$, 
and then find $T$ large enough such that the corresponding $\DP$ values are close enough to the limit.

We now build a different family of counterexamples that show that the condition for correctness of \StaticDP shields given in Theorem~\ref{thm:static shield with bounded DP} is as tight as can be.

\begin{theorem}
\label{thm:counterexamples-fair-comp}
    For all $\kappa > 0$, there exists 
    $\kappa_1$ and $\kappa_2$
    $\kappa_1 \leq  \kappa \leq \kappa_2$,
    such that for $i\in\{1,2\}$, there exists $t_i$
    and traces $\tau_i, \tau_i'$ that are 
    $\lfloor\frac{t_i-1}{2}\rfloor$-balanced
    such that 
    \[\DP(\tau_i)\leq \kappa_i, \,\,\DP(\tau_i')\leq \kappa_i , \quad \mbox{and} \quad 
    \DP(\tau_i\tau_i') > \kappa_i.
    \]
\end{theorem}

Before we start the proof, let us unpack the meaning of this theorem:
for any value of $\kappa$, we can find traces that are just one value off of being $(T/2)$-balanced where
composability of \StaticDP shields yields a traces outside of the fairness constraint.
The reason why the result is in terms of $\kappa_1,\kappa_2$ surrounding $\kappa$ is because in the prove we have to use at some point values of $\kappa$ that have a certain rational form, so we cannot prove our result for all $\kappa\in [0,1]$, but for all $\kappa$ in a dense subset of $[0,1]$.

\begin{proof}
    We prove this theorem by constructing families of counterexamples. 
    For this proof, we use the (slightly abusive) notation that a trace is composed by its four counters, 
    so $\tau = (n_{a}(\tau), n_{a1}(\tau), n_b(\tau), n_{b1}(\tau))$.

    Let $t = 2T+1$ with $T$ even. Consider the traces $\tau_1 = (T+1, T/2 + 1, T, T/2)$ and $\tau_2 = (T, T/2, T+1, T/2)$. Both traces are $T$-balanced. Let's compute demographic parity:
    \begin{equation*}
        \DP(\tau_1) = \frac{T/2 +1}{T+1} - \frac{T/2}{T} = \frac{1}{2(T+1)}
    \end{equation*}
    \begin{equation*}
        \DP(\tau_2) = \frac{T/2}{T} - \frac{T/2}{T+1} = \frac{1}{2(T+1)}
    \end{equation*}
    \begin{equation*}
        \DP(\tau_1\tau_2) = \frac{T+1}{2T+1} - \frac{T}{2T+1} = \frac{1}{2T+1}
    \end{equation*}

    It is clear that $\DP(\tau_1) = \DP(\tau_2) < \DP(\tau_1\tau_2)$.
    Ideally, we would choose $T$ such that $\frac{1}{2(T+1)} = \kappa$, 
    which can be rewritten to $T = \frac{1-2\kappa}{2\kappa}$. 
    However, this may not be an integer. So, given $\kappa$, we take 
    \[
    T_1 = \left\lfloor \frac{1-2\kappa}{2\kappa} \right\rfloor,
    \quad \mbox{and} \quad
    T_2 = \left\lceil \frac{1-2\kappa}{2\kappa} \right\rceil,
    \]
    and define $\kappa_i = \frac{1}{2(T_i+1)}$.
    
    This finishes the construction for an odd $t$. 
    For an even $t$, we show a similar construction. 
    Let $t = 2T$. Consider the traces 
    $\tau_1 = (T+1, 2 ,T - 1, 1)$, $\tau_2 = (T-1, 1, T+1, 1)$. 
    Both traces are $(T-1)$-balanced.
    Let's compute demographic parity:
    \begin{equation*}
        \DP(\tau_1) = \frac{2}{T+1} - \frac{1}{T-1} = \frac{T-3}{T^2-1}
    \end{equation*}
    \begin{equation*}
        \DP(\tau_2) = \frac{1}{T-1} - \frac{1}{T+1} = \frac{2}{T^2-1}
    \end{equation*}
    \begin{equation*}
        \DP(\tau_1\tau_2) = \frac{T+1}{2T+1} - \frac{T}{2T+1} = \frac{1}{2T+1}.
    \end{equation*}
    This finishes the proof.   
\end{proof}
The construction proving the theorem is for time horizons $t$ that are $t \equiv 1 (\mod 4)$. Similar constructions can be found for other congruence classes.

\subsection{Limitations}

\paragraph*{Static vs. dynamic shielding in the periodic setting.}
   Static shields are computationally cheaper than \Dyn shields and have no runtime overhead, making them ideal for fast decision-making applications like online ad-delivery~\cite{ali2019discrimination}.
   However, they can't adjust decisions based on the actual history, leading to overly restrictive and frequent interventions---particularly in the long run.
   In contrast, \Dyn shields adapt to historical data, resulting in fewer interventions over time, making them suitable for applications like banking where decision-making can afford longer computation times~\cite{liu2018delayed}.

\paragraph*{On the assumptions in periodic shielding.}
All three periodic shielding approaches come with assumptions on the numbers of individuals seen from the two groups in each period.
For \StaticDP shields, the assumption provides a tight sufficient condition for the fairness guarantee to be satisfied.
For \StaticBAR and \Dyn shields, the assumptions (Assumption~\ref{assump:static-BAR}, \ref{assump:dynamic}) rule out ``edge cases'' like in Example~\ref{ex:counter-example-periodic-naive-bis}, to give the shield enough advantage to be able to uphold fairness.
We argue that in real-world scenarios and particularly for longer time periods, such edge cases are indeed rare.

\paragraph*{On the existence of periodic shields.}
The existence of the optimal $T$-periodic shield, as defined in Definition~\ref{def:periodically-fair},
is left as an open question, but we conjecture it will be true. 
First, the same argument of Section~\ref{sec:existence-finite-horizon-shields} applies to show that $\sShieldFeasPeriodic$ is not empty because it contains a trivial ``reject-all'' shield.
We still need to prove that there actually exists a shield that minimizes the expression in Equation~\eqref{eq:periodic def.}.

A more interesting question is whether there exists an optimal $T$-periodic shield \emph{that can be computed with finite resources}.
We conjecture that such shield does not exist. Our best solution is in the form of dynamic shields, which we can only synthesize on-the-go because a complete description would require infinite memory.
And even with dynamic shields, there are still some rare feasible traces that fail the fairness constraint, so dynamic shields are technically not in $\sShieldFeasPeriodic$.

\paragraph*{On the feedback effect in sequential decision-making.}
In the sequential setting, decisions that seem fair from a standalone perspective may create biases in the population over time~\cite{liu2018delayed,d2020fairness,sun2023algorithmic}.
This can be modeled by making the input distribution $\sgen$ be a function of the trace seen so far.
In this chapter, we assumed $\sgen$ to remain constant, thereby leaving out such feedback effects that are inherent in sequential decision-making.
We point out that our basic recursive synthesis algorithm from Section~\ref{sec:synthesis} could potentially be adapted to trace-dependent $\sgen$ by modifying Equation~\eqref{eq:v-recursion}, although a detailed extension is out of the scope of this work.

\paragraph*{On considering unrealistic traces.}
As discussed in Section~\ref{sec:existence-finite-horizon-shields}, our shields consider the possibility of very skewed traces. 
This can be seen as overly conservative, as we could safely assume in most realistic applications that such degenerate traces will not occur, and optimize cost under such assumptions. 
The price to pay, from a theoretical perspective, is that the probability of an input would be different depending on the trace history. 
While this is out of the scope of this thesis, we believe this restriction can be modelled using conditional MDPs~\cite{baier2014tacas}.

\paragraph*{Fairness shields with humans in the loop.}
In some applications, decisions are made by human experts, and AI-based systems (like classifiers) are deployed to guide the decision-making process~\cite{green2019principles}.
In these cases, shields may not have the authority to make final decisions.
But they can serve as a runtime ``fairness filter,'' which would modify and de-bias the original outputs of the decision-maker before passing them on to the human expert.
This way they can compliment the decision-making process from the fairness standpoint.

\subsection{Related Work}

Existing works on fairness address the question of how to \emph{specify}, \emph{design}, and \emph{verify} AI decision makers that are fair in their decisions.
From the specification standpoint, several criteria have been proposed to quantify fairness between groups~\cite{feldman2015certifying,hardt2016equality} and between individuals~\cite{dwork12fairness}.
From the design standpoint, many approaches have been developed to ensure that decision-makers are fair with respect to a given fairness objective \cite{hardt2016equality,gordaliza2019obtaining,zafar2019fairness,agarwal2018reductions,wen2021algorithms}.
From the verification standpoint, several static~\cite{albarghouthi2017fairsquare,bastani2019probabilistic,sun2021probabilistic,ghosh2020justicia,meyer2021certifying,li2023certifying} and runtime~\cite{albarghouthi2019fairness,henzinger2023monitoring,henzinger2023dynamic,henzinger2023partial} approaches have been invented for verifying how fair or biased a given decision-maker is.
Our fairness shielding combines the design and verification aspects, as shields are \emph{verified} to be fair by \emph{design}.
Additionally, the design of our fairness shields do not require any knowledge about the underlying decision-maker, and therefore they can be used as trusted third-party intervention mechanisms to guarantee fairness of arbitrary AI-based decision makers.

Traditionally, fairness is defined using the decision-maker's output distribution.
However, it has been shown that a decision-maker that is fair according to its output distribution may exhibit biases over short horizons, which could be undesirable in many situations~\cite{pmlr-v235-alamdari24a}.
To mitigate this issue, we adopt the recently proposed bounded-horizon fairness properties~\cite{pmlr-v235-alamdari24a}, which require that decisions remain empirically fair over a given finite horizon.
To the best of our knowledge, our work is the first to provide systematic algorithmic support for guaranteeing bounded-horizon fairness properties.

We consider the setting of sequential decision making, where a fairness shield needs to make decisions without knowing the inputs from the future.
Similar problems has been extensively studied under the umbrella of optimal stopping problems~\cite{shiryaev2007optimal,bandini2018backward,ankirchner2019verification,bayraktar2020optimal,palmer2017optimal,kallblad2022dynamic}.
The focus of these works has been the analytical design of policies that are as close as possible to the hypothetical policy having the perfect foresight about the future.
Unfortunately, statistical properties like fairness remain beyond the reach of existing algorithms from the optimal stopping literature.

Our design algorithms for fairness shields are inspired by a recent work~\cite{fscd24}, which proposed sequential decision making algorithms for the general class of finite-horizon statistical properties.
They showed that the standard dynamic programming algorithm gets computationally significantly cheaper and produces the same output if the statistically indistinguishable traces are combined together.
This idea is mirrored in our design algorithm for finite horizon shields as well, where traces with the same counter values remain indistinguishable.

%% file: 60_intentional_behavior.tex

\ifthenelse{\boolean{includequotes}}{
\begin{quotation}
    \textit{Guardeu-vos forces, bona gent, potser ens veurem un altre dia.
Sabem que volíeu fer més, però, què hi farem, així és la vida:
t'equivoques d'uniforme i dispares a qui més estimes;
t'equivoques de remei i va i s'infecta la ferida.}
\footnote{Save your strength, good people, maybe we'll see each other another day.
We know you wanted to do more, but what can we do, that's life:
you wear the wrong uniform and shoot the one you love the most;
you use the wrong remedy, and the wound gets infected.}
    \hfill 
    --- Guillem Gisbert, El Miquel i l'Olga tornen.
\end{quotation}
}{}

\section{Motivation and Outline}
In this chapter, we focus on explaining the decisions of autonomous agents in terms of intentional behaviour.
Beyond explainability, understanding intention is also key to accountability. 
Since formal verification of software for autonomous agents is often infeasible, these agents may cause harm. 
In such instances, determining whether an agent acted intentionally, negligently, or accidentally helps clarify accountability. 
The study of intention thus not only strengthens explainability but also serves as an essential tool for assessing responsibility. 
Because we cannot predict when harm may occur, examining the agent's software after the incident is necessary to address accountability questions. 
Although a comprehensive liability framework for autonomous agents has yet to be developed, it is reasonable to hold manufacturers of agents that intentionally cause harm to a higher standard than those whose agents cause harm negligently or accidentally. 
Therefore, defining and understanding intention is crucial for establishing accountability.

Historically, symbolic AI has produced a substantial body of work focused on formally specifying and designing ``rational'' autonomous agents.
Such agents explicitly derive decisions based on their beliefs, desires, and intentions, in the so-called BDI approach~\cite{bratman1987intention,rao1995bdi}. 
Determining whether an autonomous agent has acted with a given intention is straightforward for BDI agents. Their intentions are explicitly encoded in their inner workings and can, therefore, be readily examined. 
However, the statistical nature of modern machine-learning-based agents makes interpreting their decision-making in probabilistic settings a much greater challenge, since intentions are not explicitly present in such models.

Traditionally, intention is connected to planning through either cognitive or computational reasoning. 
Intention is a nuanced term in legal and philosophical contexts; here, we use it in the restricted sense of the ``state of the world'' the agent plans towards. 
Whether human or machine, a rational agent with bounded resources must plan towards a goal to successfully achieve it~\cite{bratman1987intention,cohen1990intention}. 
Modern machine-learned agents plan implicitly through techniques like reinforcement learning~\cite{sutton2018reinforcement}.

A sizeable portion of the literature on intention in AI relates to the internal beliefs of an agent~\cite{rao1995bdi,halpern2018towards}.
Since we do not model the internal beliefs or reasoning processes of the agent, we can only claim that an agent shows evidence of intending something. 
While an intentional agent would behave in this way, a random agent might also exhibit such behaviour by chance. 
We model uncertainty arising from diverse sources as probabilistic behaviour, aligning with modern machine-learning techniques for designing autonomous decision-makers. 
Therefore, our definitions are inherently quantitative. 
Rather than stating that an agent shows evidence of intending something, we provide concrete values quantifying the amount of evidence and the confidence level in our assessment.

\paragraph*{Quantitative assessment of agent intentions.}
We consider an autonomous agent operating 
within a probabilistic environment. 
Specifically, we model the environment as a Markov Decision Process (MDP), and the agent as a policy within the MDP (recall Section~\ref{sec:prelim-MDPs}).
We express goals as reaching certain sets of states in an unbounded time horizon. 
Our aim is to analyze whether the agent's decision-making policy shows evidence of intentional behaviour towards a goal.

At the core of our methodology lie the concepts of \emph{agency} and \emph{intention quotient}.
From a given state of the world, an agent employing the optimal policy to reach a goal would achieve it with a certain probability, which we call $P_{\max}$. 
Conversely, an agent using the optimal policy to avoid the same goal would reach it with a smaller probability, denoted as $P_{\min}$. 
We define the difference between $P_{\max}$ and $P_{\min}$ as the \emph{agency} --~or scope of agency, or extent of agency~--,
as it indicates how much the agent can affect the outcome in terms of reaching the goal. 
If the difference is one, it means the agent has complete command over the outcomes:
it can ensure either reaching the goal or avoiding it with certainty. 
If the difference is close to zero, it means the probability of reaching the goal does not change significantly regardless of the agent's actions.

Between $P_{\min}$ and $P_{\max}$ lies the probability that the agent, with its specific policy, will reach the goal; we denote this probability as $\Pag$.
The relative position of $\Pag$ with respect to $P_{\min}$ and $P_{\max}$ indicates the extent of effort the agent is showing towards reaching the goal, always within the constraints imposed by the extent of agency allowed to the agent.
We call this relative position the intention quotient (IQ). 
Whenever $\Pag$ is close to $P_{\max}$, we say that the agent shows evidence of intentional behaviour, with IQ representing the amount of evidence and the scope of agency representing the confidence level of the assessment.
We use probabilistic model checking to compute these probabilities.

Using the concepts of agency and intention quotient, we can assess an agent's intention at a single state of the world. 
Extending to all relevant states of the world, 
these concepts can be directly used for a quantitative analysis of intentional behaviour in an agent. 

\paragraph*{Retrospective methodology for analyzing intentional behaviour.}
Motivated by the problem of accountability ``after the fact'', we propose a method to analyse intentional behaviour in a retrospective manner\footnote{For example, an autonomous driver crashes a car against a tree. \emph{After the harm has occurred}, we study the actions of the agent leading to that harm to determine accountability.}.
We assume a given sequence of events has occurred and we aim to assess whether an agent would show intentional behaviour toward reaching a certain goal along said sequence.

Given an agent, a goal, and a sequence of states of 
the world, which corresponds to a \emph{trace}, 
we start by computing the intention quotients at the states in the trace, 
and aggregating them through a weighted average, 
where the weights are proportional to the agency.
If the aggregated intention quotient is not sufficiently high or low, or if the confidence in the assessment (average agency along the trace) is too low, we conclude that the given trace does not provide enough evidence and that we need to analyze counterfactual scenarios.

If the evidence is not sufficient, we generate a diverse set of counterfactual traces close to the original sequence of events under study and repeat our assessment by aggregating results from all states in the counterfactual traces. This loop of counterfactual generation and intention assessment can be repeated until the confidence level of the assessment is sufficiently high or a threshold number of iterations is reached.

This retrospective method is more involved, and it is intended for use in an accountability process after harm has occurred, where the focus is not so much to understand the agent in general but rather to understand the behaviour of the agent in a concrete sequence of events leading to a harmful consequence.

Our framework is strongly inspired by methods for explainability and accountability using counterfactual analysis~\cite{wachter2017counterfactual,guidotti2024counterfactual}.
In computing the intention quotient, we are asking \textit{``what could the agent do differently?''}, 
and in investigating counterfactual traces we are asking \textit{``what would the agent do in different situations?''}.




\paragraph*{Contributions.}
The contributions of the work presented in this chapter are:
\begin{itemize}
    \item We present a framework for studying intentional behaviour of agents in MDPs directly from policies. Our method uses model checking to automatically relate the agent’s policy to any other possible policy. 
    \item We propose a specific methodology for assessing evidence of intentional behaviour after a concrete sequence of events has happened,
    designed to be used as part of an accountability process.
    Furthermore, our method applies counterfactual reasoning to increase the reliability of the assessment.
    \item To showcase the usefulness of our retrospective method, we provide a case study in which we analyze potential intentional behaviour in the same scenario for different implementations of driving agents.
\end{itemize}

\paragraph*{Outline.}
In Section~\ref{sec:intention-modelling}, we describe the main concepts that we use to quantitatively assess intention throughout the paper and how they are grounded in previous notions. In Section~\ref{sec:intention-retrospective}, we present our specific methodology for retrospective analysis of intention, which builds counterfactuals to a reference trace and uses them to make an assessment.
We report the results of a case study using our retrospective methodology in a traffic-related scenario in Section~\ref{sec:intention-experiments}.
We conclude the chapter in Section~\ref{sec:intention-discussion} discussing potential limitations, extensions, and relation to other work in the literature on intention analysis in AI.

\paragraph*{Declaration of sources.}
This chapter is partially based and reuses material from the following source
previously published by the author of this thesis:

\cite{cano2023analyzing} \fullcite{cano2023analyzing}.

\section[Modelling Intentional Behaviour]
{Modelling Intentional Behaviour in Agents on MDPs}
\label{sec:intention-modelling}

In this section, we give the definitions for \emph{evidence of intentional behaviour} of policies in the presence of uncertainty.
We use an MDP $\mathcal{M} = (\mathcal{S},\mathcal{A},\mathcal{P})$
to model the interaction of the agent and the environment.
In the following sections, we will then propose and implement a method to analyze intentional behaviour according to the definitions of this section.

\subsection{Modelling Environment, Agents, and Intentions}

We model the environment as a Markov decision process (MDP)
\footnote{Recall definitions in Section~\ref{sec:prelim-MDPs}.}
$\M = (\S, \A, \P)$, together with a finite set of atomic propositions $\AP$ and a valuation function $\Val\colon\AP \to 2^\S$. 
A state represents ``one way the world can exist'',
so any information available to the agent for 
deciding what to do is included in the state of the MDP.
The set $\A$ contains every possible action that can be taken
by the agent. 
As usual, 
given $s,s'\in \S$ and $a\in\A$, $\P(s,a,s')$ represents the probability to transition to state $s'$ from state $s'$ when executing action $a$. Also, for each $s\in \mathcal S$ and $a\in\mathcal A$, 
$\sum_{s'\in \mathcal S} \mathcal P(s,a,s') \in \{0,1\}$.

The literals in $\AP$ indicate properties of interest of the MDP, like $\mathtt{goal}$ or $\mathtt{collision}$, and the valuation function $\Val$ indicates, for each property, which states satisfy it.
The valuation function can be extended to any Boolean formula over $\AP$ with the standard conventions as follows. For any pair of formulae $\I,\J$:
\begin{itemize}
    \item $\Val(\I \land \J) = Val(\J)\cap \Val(\J)$,
    \item  $\Val(\I\lor\J) = Val(\I)\cup \Val(\J)$, and
    \item $\Val(\lnot \I) = \S \setminus \Val(\I)$.
 \end{itemize}
Given a Boolean formula $\I$, we denote the set of states where it is satisfied as $S_\I\coloneqq \Val(\I)$.

The agent is modelled by a memoryless and deterministic \emph{policy} 
$\Pol\colon \S\to \A$ over $\M$ that assigns an action to each state.
In Section~\ref{sec:generalized_policies}, we discuss how our method can be extended to consider strategies with non-determinism and memory.

We model potential goals as Boolean expressions over $\AP$ and express intentions as reachability properties of goals with an unbounded horizon.
Given $\I$, a Boolean expression over $\AP$, 
and a state $s\in\S$,
we are interested in the properties of the type 
$\spec = \Reach(s, S_\I)$.


\subsection{Intention of Agents with Perfect Information}

Following classic works in BDI models~\cite{rao1995bdi},
an \emph{intention} of an agent is a set of states the agent committed to reach. 
In our case, we model the set of states with a Boolean formula $\I$ over $\AP$, 
whose corresponding set of states is $S_\I = \Val(\I)$.
Therefore, the agent that intends $\I$ should act towards reaching $S_\mathcal I$  to the best of its knowledge.

Let us assume that the agent has perfect knowledge about the environment and is optimally implemented.
For a formula $\I$ to be an intention of an agent, the agent has to implement a policy  $\Pol$ that maximizes the probability of reaching $S_{\I}$. 
Formally,
$\I$ is an \emph{intention} of the agent $\Pol$, if and only if for any $s\in\S$
\begin{equation}
\label{eq:intention-perfect-info}
    \PP_{\Pol} (\Reach(s, S_\I)) = \PP_{\max} (\Reach(s, S_\I)).
\end{equation}

The policies considered to compute $\PP_{\max}$
can be restricted to a set of policies $\Pi$, if there are policies that should be excluded for comparison. For example,
we may only be interested in policies for comparison that
satisfy certain properties like fairness or progress properties.
In such cases, the right-hand side of Equation~\ref{eq:intention-perfect-info} transforms into 
$\PP_{\max\mid\Pi} (\Reach(s,S_\I))$.

\begin{definition}[Intention in perfect-information settings]
\label{def:intentionperfect}
An agent $\Pol$ shows \emph{evidence of intentional behaviour} in a state $s$ towards $\I$ 
among policies in $\Pi$
if 
$\Pol$ maximizes the probability of reaching $S_\mathcal{I}$, i.e.,
\begin{equation*}
    \PP_{\Pol} (\Reach(s,S_\I)) = \PP_{\max\mid\Pi} (\Reach(s,S_\I)).
\end{equation*}

\end{definition}
Note that we do not phrase Definition~\ref{def:intentionperfect} in terms of the optimal policy $\Pol_{\max}$, because there may not be a unique policy that maximizes reachability probabilities.

\subsection{Intention of Agents Under Uncertainty}

The definition of intention presented earlier assumes perfect knowledge of the environment and that the agent implements an optimal policy for reaching $S_\I$. However, our goal is to analyze intention quantitatively, recognizing that agents acting intentionally do not necessarily follow the optimal policy.

An agent intending to reach $S_\I$ may deviate from the optimal policy for various reasons. We distinguish three primary categories of such deviations:

\begin{itemize}
    \item \emph{Imperfect training.} The agent is trained to reach $S_\I$, but training concludes before convergence to an optimal policy.
    \item \emph{Trade-off among multiple goals.} The agent is trained with several goals simultaneously, with reaching $S_\I$ being only one among these objectives. Consequently, the learned policy might be suboptimal due to balancing multiple conflicting goals.
    \item \emph{Imperfect environment modeling.} The agent is trained to reach $S_\I$ in an MDP $\mathcal{M}'$ that slightly differs from the actual environment model $\mathcal{M}$.
\end{itemize}

In the third scenario, discrepancies between $\mathcal{M}'$ and $\mathcal{M}$ might exist solely in transition probabilities or extend to the action and state spaces. When $\mathcal{M}'$ shares the same action and state sets as $\mathcal{M}$ but differs slightly in transition probabilities, the resulting policy may be suboptimal due to either insufficiently precise modeling of environmental uncertainty or distributional shifts occurring between training and deployment. Such discrepancies naturally emerge from attempts to model real-world uncertainties.

Alternatively, one of the models might represent an abstraction of the other. For instance, an agent could be trained in a continuous environment $\mathcal{M}'$, which must then be abstracted into a discrete model $\mathcal{M}$ for intention analysis. Although abstraction generally preserves overall agent behavior, some fine-grained details might be lost.

In all cases described, we assert that an agent still demonstrates intention through policies that, while potentially suboptimal, remain close to optimal. This observation motivates a relaxation of Definition~\ref{def:intentionperfect}, enabling a quantitative measure of intention under conditions of uncertainty, irrespective of the uncertainty's origin.

\subsubsection{Single State Analysis}

In order to analyze an agent $\Pol$ under uncertainty, we first define the \emph{intention quotient} for a state $s\in\S$, which represents how close $\Pol$ is to the policy optimal for satisfying $\I$ from state $s$.

\begin{definition}[Intention quotient]
\label{def:intention quotient}
Given an agent $\Pol$ at a state $s\in\S$ and a formula $\I$ over $\AP$, 
the \emph{intention quotient} is defined as follows:
\begin{equation}
\label{eq:willingness}
    \rho_{\Pol}(s,\I) = 
    \dfrac{\PP_{\Pol} (\Reach(s,S_\I)) - \PP_{\min\mid\Pi} (\Reach(s,S_\I))}
    {\PP_{\max\mid\Pi} (\Reach(s,S_\I)) - \PP_{\min\mid\Pi}(\Reach(s,S_\I))}. \nonumber
\end{equation}
\end{definition}

Whenever $\I$ is clear by context, we may drop it from the notation.
In contrast to the case of perfect information, 
the uncertainty in the agent's knowledge and resources 
implies uncertainty in the assessment of intentional behaviour.

In general, the higher the value of the intention quotient $\rho_{\Pol}(s, \I)$, 
the more evidence the policy $\Pol$ shows of intentionally trying to satisfy $\I$.
The lower the value of  $\rho_{\Pol}(s,\I)$, the more evidence the policy $\Pol$ shows on acting without the intention to satisfy $\I$, although high values at a single state may be explained by other means.

An additional source of uncertainty is introduced by the agency of a state.
In situations where the agent's actions have little effect on satisfying $\I$,
there is not enough evidence to support a claim of intentional behaviour. 
For this reason, we take the agency into account for our assessment of intentional behaviour.

\begin{definition}[Agency]
\label{def:agency}
    Given a state $s\in\S$ and a formula $\I$ over $\AP$, 
    the \emph{agency} $\sigma(s,\I)$ at a state $s$
    is defined as the gap between the best and the worst policy in terms of satisfying $\I$.
    Formally, it is given by
    \begin{equation}
    \sigma(s,\I) =  \PP_{\max\mid\Pi} (\Reach(s,S_\I)) - \PP_{\min\mid\Pi}(\Reach(s,S_\I)).
    \end{equation}
\end{definition}

\subsubsection{Multiple-State Analysis}

The concepts of agency and intention quotient apply to a single state in the MDP.
However, when studying an agent in particular, 
we are not only interested in how the agent behaves in one state, but in many states.
We extend the definition of agency by averaging the value along a set of states.
\begin{definition}[Agency for sets of states]
    For a set of states $S\subseteq \S$ and a formula $\I$, 
    the \emph{agency} of $S$ is
    \begin{equation}
        \sigma(S,\I) = \frac{1}{|S|}\sum_{s\in S} \sigma(s,\I)
    \end{equation}
\end{definition}

Since the scope of agency indicates how important is a given state in assessing the outcome of an agent's actions, 
we aggregate the intention quotients of the individual states using the agency as the weighting factor.
This way, the weight of the decision at each state is directly proportional to the impact that an agent can have in that state towards satisfying $\I$.

\begin{definition}[Intention quotient for sets of states]
    For an agent $\Pol$ 
    operating around a set of states $S\subseteq \S$,
    and a formula $\I$ over $\AP$,
    the intention quotient $\rho_\Pol(S,\I)$
    is given as 
    the weighted average
    \begin{equation}
        \rho_\Pol(S,\I) = \frac{1}{\sum_{s\in S}\sigma(s,\I)} 
        \sum_{s\in S} \sigma(s,\I) \rho_{\Pol}(s,\I).
        \nonumber
    \end{equation}
\end{definition}

We consider two types of sets of states that are of interest from the point of view of studying intentional behaviour: balls and traces.

\paragraph*{Balls.}
A ball around a set $S_\I$ represents the states that are \textit{close} to $S_\I$, according to some distance in $\M$. 
Sometimes, instead of analysing the agent in the whole environment, we are interested in how an agent behaves in the vicinity of a set of states. 
For example, we may be interested in a car's behaviour towards crashing into a wall only when a wall is nearby.
This may also be useful for practical reasons: instead of trying to model and understand the agent in a large environment with too many states, we can focus on a ball of a certain radius of influence, where states outside of this radius can be considered unimportant.

\paragraph*{Traces.}
A trace introduces an order and a concept of time passing to a set of states,
and we will especially focus on analyzing traces in the retrospective method. 
Traces are also useful because we know that agents will follow valid traces when deployed.
In Figure~\ref{fig:intention quotient} we depict the concepts of agency and intention quotient over a trace.
In the hypothetical case represented in Figure~\ref{fig:intention quotient}, 
we see an agent that behaves towards satisfying a certain formula $\I$ most of the time, with most of the states, especially those with high agency, showing a probability close to the maximum. The only exception is the last state, where the probability is closer to the minimum at that state. Even in such a case, this would be too little and too late to exonerate the agent.

This ordering in time allows us to define a notion of commitment.
Since very prominent existing theories of intention in autonomous systems take commitment as a central concept~\cite{cohen1990intention,vanZee2020}, we give here our take, defining commitment for traces in a quantitative way, using the concepts of agency and intention quotient.

\begin{definition}[Commitment along a trace]
    \label{def:commitment}
    For an agent $\Pol$, a trace $\tau = (s_1, \dots, s_n)$, 
    a threshold $\delta_B\in (0,1)$, and a threshold $\delta_I\in (0,1)$,
    we say that the agent is \emph{commited} towards satisfying $\I$ if there exists $k\in[1,n]$ 
    such that for all $i\geq k$, 
    \begin{equation*}
       \Big((\PP_{\max\mid\Pi}(\Reach(s_i,S_\I)) > \delta_B) \land
    (\PP_{\min\mid\Pi}(\Reach(s_i,S_\I)) < 1 - \delta_B) \Big) \to  \rho_\Pol(s_i,\I) \geq \delta_I. 
    \end{equation*}
\end{definition}

The intuition behind this definition is that an agent shows evidence of being committed to satisfying $\I$ if its intention quotient exceeds a certain threshold ($\delta_I$) whenever the agent believes that satisfying $\I$ is still feasible $(\PP_{\max\mid\Pi}(\Reach(s_i,S_\I)) > \delta_B)$, and that $\I$ has not yet been achieved $(\PP_{\min\mid\Pi}(\Reach(s_i,S_\I)) < 1 - \delta_B)$. 

For this definition, $\delta_I$ is assumed to be relatively high, while $\delta_B$ is close to zero. By setting $\delta_B > 0$, we allow the agent to ``give up'' if fulfilling $\I$ becomes too unlikely, or to ``focus on something else'' if reaching $S_\I$ is almost certainly guaranteed. 
The definition can be made stricter by setting $\delta_B = 0$.



\begin{figure}
    \centering
    \includegraphics[width=0.7\linewidth]{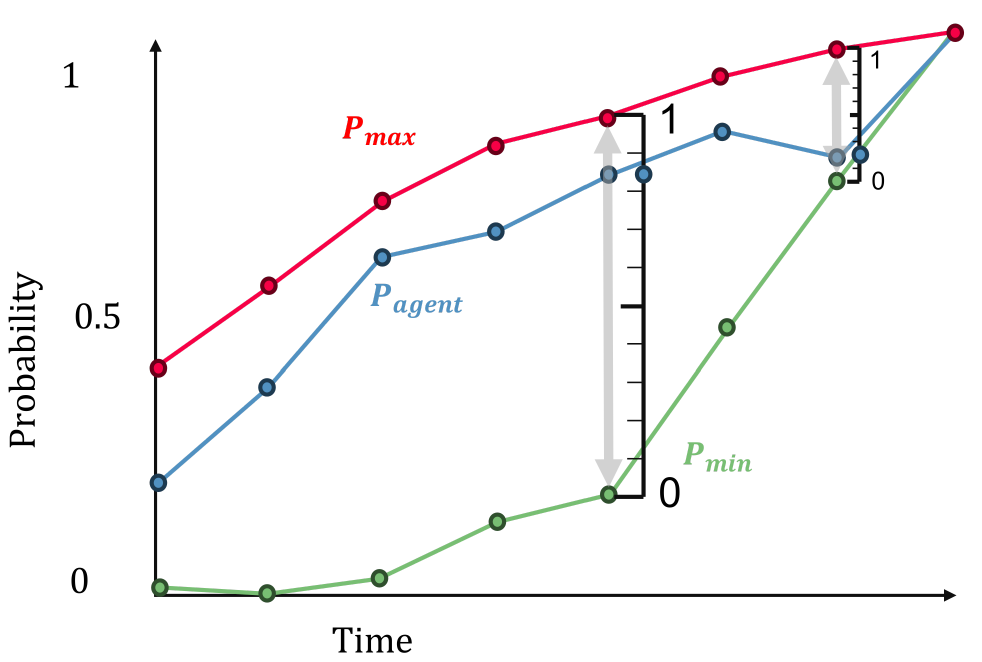}
    \caption[Example of the computation of agency and intention quotient.]
    {Example of the computation of agency and intention quotient. The grey arrows represent agency, while the blue dots inside the 0 to 1 ruler indicates the value of the intention quotient.}
    \label{fig:intention quotient}
\end{figure}

\section[Retrospective Analysis of Intention]
{Methodology for the Retrospective Analysis of Intention}
\label{sec:intention-retrospective}

\subsection{Setting and Problem Statement}\label{sec:setting}

\paragraph*{Setting.}
We have a model of the environment in the form of an 
MDP $\mathcal M = (\mathcal S,\mathcal A,\mathcal P)$ 
that captures all relevant dynamics and possible 
interactions for an agent. 
We also have a concrete scenario to analyze in the form of a trace $\tauref = (s_1,\dots, s_n)$. 
The trace $\tauref$ is a sequence of visited states in $\mathcal{M}$ that leads to a state in $S_\mathcal I$,
i.e.,  $s_n\in S_\mathcal I$.
The agents considered comparable are defined by a set of allowed policies $\Pi$,
and the implementation of the agent under study is given in the form of a policy $\Pol\in\Pi$.
The underlying intentions of the agent are unknown.

\paragraph*{Problem statement.}

Given this setting, 
we want to analyze whether there is 
\emph{evidence of intentional behaviour} 
of the agent $\Pol$
towards satisfying $\I$ in the scenario represented by $\tauref$,
considering policies in $\Pi$.

\begin{example}\label{ex:intention-traces-example}
    Let us consider a scenario in which
    an autonomous car collides with a pedestrian
    crossing the road, as illustrated in Figure~\ref{fig:casestudysetting}.
    To analyze to which degree the car is accountable for the accident, we are interested in whether causing harm was the intention of the car.
    In such an example, $\mathcal{M}$ captures all relevant information necessary to analyze the accident, like positions and velocities of car and pedestrian, car dynamics, road conditions, etc.  
    The scenario $\tauref = (s_1,\dots, s_n)$ is defined via the sequence of states prior to the collision.
    The set of states $S_\I$ represents collisions.
    We want to analyze whether the policy $\Pol$ shows evidence of intentional behaviour towards satisfying $\I$.
    To avoid unfair comparison with unrealistic policies, 
    we define a set of policies $\Pi$ that excludes 
    unreasonably slow-moving cars (e.g., cars that stop even though there is no other road user close by). 
\end{example}

 \begin{figure}
     \centering
     \includegraphics[height=7.5em]{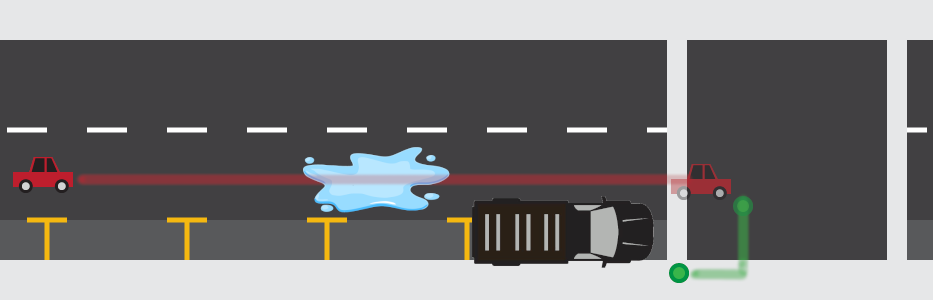}
     \caption[Illustration of the scenario in Example~\ref{ex:intention-traces-example}]
     {Illustration of the scenario in Example~\ref{ex:intention-traces-example}.
     The red line represents the trajectory of the car, the green dot represents the pedestrian and the green line the trajectory of the pedestrian. There is a water puddle in the road that makes the floor slippery and a parked truck that blocks visibility of the pedestrian.}
     \label{fig:casestudysetting}
 \end{figure}

Since both agency and intention quotient are quantitative tools, to determine whether there is or there is not evidence of intentional behaviour, we define the following thresholds that indicate how much evidence we need to give a positive or a negative assessment.

\begin{definition}[Evidence of intentional and non-intentional 
behaviour in traces]
\label{def:intention-thresholds}
    Given lower and upper thresholds
    $0\leq\delta_\rho^L < \delta_\rho^U \leq 1$ for intention quotient and 
    an agency threshold $0<\delta_\sigma < 1$,
    we say that there is \emph{evidence of
    intentional behaviour} towards satisfying $\I$ along a trace $\tau$ if
    \begin{equation*}
     \sigma(\tau) \geq \delta_\sigma \quad \mbox{and} \quad    
     \rho_{\Pol}(\tau) \geq \delta_\rho^U.
    \end{equation*}
    We say that there is \emph{evidence of
    non-intentional behaviour} towards satisfying $\I$ along a trace $\tau$ if 
    \begin{equation*}
     \sigma(\tau) \geq \delta_\sigma \quad \mbox{and} \quad    
     \rho_{\Pol}(\tau) \leq \delta_\rho^L.
    \end{equation*}
    
    Otherwise, i.e., in the cases that 
    \begin{equation}
        \sigma(\tau) < \delta_\sigma \quad 
        \mbox{or} \quad  \delta_\rho^L < \rho_{\Pol}(\tau) < \delta_\rho^U,
    \end{equation}
    we say that we have \emph{not enough evidence} for intentional behaviour.
\end{definition}
The thresholds $\delta_\rho^L$, $\delta_\rho^U$, and $\delta_\sigma$ have to be defined using domain knowledge for each concrete application, and make our method adaptable to different evidence standards required for different accountability processes. 
For example, to convict a person of a criminal offense, it is typically required to prove the person committed the crime ``beyond a reasonable doubt'', while in many systems, civil litigations are resolved with the ``preponderance of the evidence'' standard, which is much less stringent~\cite{clermont2002comparative}.
The evidence thresholds can be adapted to suit different standards.

\subsection[Evidence Augmentation Loop]
{Evidence Augmentation through Counterfactual Generation}

In this section, we propose a concrete methodology 
to analyze retrospectively whether there is evidence an agent acted intentionally towards satisfying $\I$.
Our method is illustrated in Figure~\ref{fig:intention-overview}. 

As depicted in the figure,
we start the \emph{analysis of the reference trace} $\tauref$
by computing the intention quotient $\rho_\Pol(\tauref)$
and the agency $\sigma(\tauref)$.
If $\sigma(\tauref) \geq \delta_{\sigma}$, we may be able to draw conclusions about intentional behaviour:
\begin{itemize}
    \item If $\rho_\Pol(\tauref) \geq \delta_{\rho}^U$, then we conclude that there is evidence of \emph{intentional behaviour} towards satisfying $\I$.
\item If $\rho_\Pol(\tauref) \leq \delta_{\rho}^L$, then we conclude that there is evidence of \emph{non-intentional behaviour} towards satisfying $\I$.
\end{itemize}

In cases without enough agency, i.e., where $\sigma(\tauref) < \delta_{\sigma}$, 
or where the intention quotient falls between the lower and upper thresholds, i.e.,
$\delta_{\rho}^L < \rho_\Pol(\tauref) < \delta_{\rho}^U$,
we say that we do not have enough evidence to reach a conclusion.
In such cases, we propose to \emph{generate} more evidence by analyzing \emph{counterfactual scenarios}.

A counterfactual scenario $\tau$ is a scenario close to $\tauref$ according to some distance notion. 
Our method generates a set of counterfactual scenarios $T_{\textit{cf}}$ and 
computes whether there is evidence for intentional  
or non-intentional behaviour
for each trace $\tau \in T = T_{\textit{cf}}\cup\{\tauref\}$.
We fix beforehand the number of counterfactual scenarios to generate to some parameter $N$.

As before, we draw conclusions about intentional behaviour based on the \emph{aggregated results} of agency $\sigma(T)$ and intention quotient $\rho_\Pol(T)$. 
If 
$\sigma(T) < \delta_\sigma$ or $\delta^L_{\rho} < \rho_{\Pol}(T) < \delta^U_{\rho}$,
there is still not enough evidence for intentional or non-intentional behaviour,
with $\sigma(T)$ being the agency 
averaged over all traces in $T$, 
and $\rho_{\Pol}(T)$ being the average intention quotient for the set of traces in $T$.

In such cases, 
our algorithm iterates back and extends the set $T_{\textit{cf}}$ by generating $N$ more counterfactual scenarios to be analyzed. 
The algorithm stops when enough evidence has been generated to draw a conclusion
or
when the number of generated counterfactual scenarios exceeds some user-defined limit.
In the following, we discuss the generation of counterfactual scenarios in detail.

In Figure~\ref{fig:intention-overview} we show this augmentation loop, where at each iteration we may stop if there is enough evidence to draw a conclusion.

\begin{figure}
    \centering
    \includegraphics[width=\linewidth]{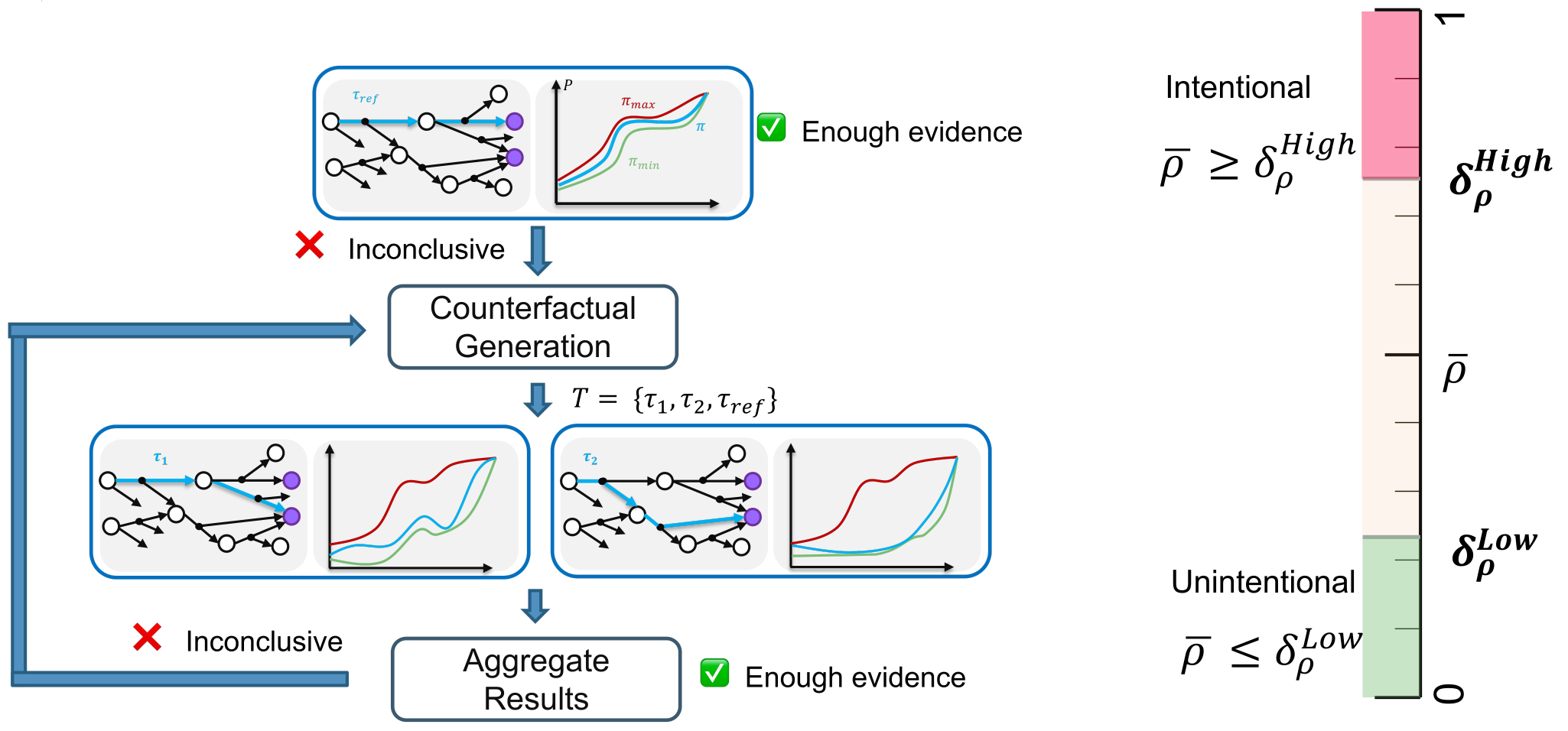}
    \caption[Retrospective analysis of intentional behaviour]
    {Overview of our approach for retrospective analysis of intentional behaviour.}
    \label{fig:intention-overview}
\end{figure}

\subsection{Counterfactual Generation}

To gather sufficient evidence for our assessment of intentional behaviour, we generate scenarios that serve as counterfactuals for $\tauref$.
There are various approaches to generating counterfactual traces, each requiring different levels of domain knowledge. Here, we present three alternative methods, ordered by decreasing reliance on expert knowledge and involvement. The first approach is highly dependent on human expertise, while the last operates with minimal human intervention.

\subsubsection{Counterfactual Generation via a Human Expert}
Asking and analyzing counterfactual questions is a standard procedure in accountability processes~\cite{beebee2019counterfactual}. 
Usually, such counterfactual questions are proposed by a domain expert.
We transfer this concept
to analyzing intentional behaviour on MDPs. 
The counterfactual questions posed by the expert are translated to counterfactual traces $T_{\textit{cf}}$ in the model  $\mathcal{M}$. 

\begin{example}
Recall Example~\ref{ex:intention-traces-example}.
Some counterfactual questions posed by an expert in the traffic scenario
could be: 
(Q1) What if the car had driven slower? 
(Q2) What if the pedestrian had been visible earlier? 
(Q3) What if the road conditions were different?
Each of Q1-Q3 translates to a counterfactual trace, 
which we can analyze in our framework. 
\end{example}

The method of generating counterfactuals using a human expert imposes a heavy burden of work on the expert. 
Next, we propose two methods to automatically generate counterfactuals to mitigate the need for human effort.

\subsubsection{Counterfactual Generation on a Factored MDP}
\label{sec:intention-factored-mdp}
Since $\mathcal{M}$ models the interactions of the agent with its environment,
$\mathcal{M}$ is typically given in form of a \emph{factored} MDP. 
In factored MDPs, the state space of $\mathcal{M}$ is defined in terms of \emph{state variables} $\mathcal S = \mathcal{X}_1\times\cdots\times\mathcal X_m$.

In this approach for counterfactual generation, we assume domain knowledge about which variations of state variables generate interesting counterfactual scenarios. 
In particular, we assume that the state variables can be partitioned into \emph{integral}, which contain the most characteristic information about the sequence of events, and \emph{environmental} or \emph{peripheral} state variables, which define environmental characteristics and are fixed during the sequence of events under study\footnote{Note that the convention on which variables are \emph{integral} and \emph{peripheral} differs from that it in~\cite{cano2023analyzing}.}.


\begin{example}
    In Example~\ref{ex:intention-traces-example}, integral state variables represent the position and velocity of the car and the position, velocity, and visibility status of the pedestrian. 
    State variables that represent the position of the parked truck, the location of the water puddle and the amount of water in it are properties of the environment that stay fixed during the sequence of events, since they would change at a much slower rate, so they would be tagged as peripheral state variables.
    Traces with the same sequence of values for the integral state variables and different values for the peripheral state variables can effectively represent the same sequence of events in a slighlty different world. 
    %
\end{example}

Studying agency and intention quotient in traces with modified values of the peripheral state variables is our way of asking counterfactual questions such as \emph{``What would you have done if you could see the pedestrian?''}, or \emph{``What would have happened if the road was not so slippery?''}.
To generate informative counterfactuals, we are interested in traces that maintain the values of the integral variables (i.e., maintain the characteristic sequence of events), while changing some values of the peripheral variables (i.e., changing some of the environmental factors).
We do not provide a more formal or concrete definition of integral and peripheral variables, because they have to be defined in each scenario from domain knowledge.

We automatically generate counterfactual traces by exploring variations of the peripheral variables. 
Let the state space be factored as $\mathcal{S} = \mathcal X_1\times\dots\times\mathcal X_m$, 
where variables $\mathcal X_1,\dots,\mathcal X_k$ are integral and $\mathcal X_{k+1},\dots,\mathcal \mathcal X_m$ are peripheral.
For any state $s=(x_1,\dots,x_m)$, 
we write its factorization 
into integral and peripheral variables 
as 
$s = (s^{\textit{int}}||s^{\textit{per}})$.
Let $s^{\textit{per}}_{\textit{ref}} = (x_{k+1},\dots,x_m)$ be the value of the peripheral variables at any state of $\tauref$.
This is well-defined since the values of peripheral variables do not evolve along the trace.
We define the set of counterfactual values as:
\begin{align}
    \mathrm{Cf_\eps}(s^{\textit{per}}_{\textit{ref}}) =
    \{(y_{k+1}, \dots, y_{m}) \in \mathcal X_{k+1}\times\dots\times\mathcal X_{m}\::\:
    \: \forall i,\:
|x_i - y_i| < \eps_i 
\},\nonumber
\end{align}
where $\eps=(\eps_{k+1},\dots,\eps_{m})$
contains, for each peripheral variable, the range of variation that is still considered valid.
For a given trace $\tauref = (s_1,\dots,s_n)$,
the counterfactual traces that we consider are
\begin{multline*}
     T_C(\tauref) = 
    \{(s'_1,\dots, s'_n)\::\:
    \exists s_{\textit{cf}}^{\textit{per}}\in \mathrm{Cf_\eps}(s^{\textit{per}}_{\textit{ref}}),\:
    \: 
    \forall i=1\dots n : \\
    s'_i = (s_i^{\textit{int}}||s_{\textit{cf}}^{\textit{per}}),\:
    (s'_1, \dots, s_n')\mbox{ is valid, and }\
   s'_n \in S_\mathcal I
    \}.
\end{multline*}

To unpack this definition, a trace $\tau$ is in $T_C(\tauref)$ if the following conditions are satisfied.
\begin{itemize}
    \item At each stage of the trace $\tau$, the value of the integral variables corresponds to the value in $\tauref$.
    \item The value of the peripheral variables in $\tau$ is constant along the trace and close to that of $\tauref$, as defined by a distance vector $\eps$.
    \item The trace $\tau$ is still a valid trace of the MDP, meaning that for every pair of consecutive states in $s_i', s_{i+1}'$ in $\tau$, there exists an action $a$ such that the transition $s_i'\xrightarrow{a} s_{i+1}'$ has a non-zero probability.
\end{itemize}

Note that the search for counterfactual traces is limited to those peripheral variables
$\mathcal X_i$
for which $\eps_i >0$,
thus by setting some of the $\eps_i$ to zero, 
we can fix their value in the counterfactual generation process.

From $T_C$, we sample $N$ traces to be used for the counterfactual analysis. For the trace selection,
emphasis can be put on traces with higher scopes of agency.

\subsubsection{Counterfactual Generation Using Distances on MDPs}
\label{sec:intention-counterfactual-distance}

This method for generating counterfactual scenarios requires to have given a distance $d\colon \mathcal S\times \mathcal S \to \mathbb R_{\geq 0}$ defined over states in the MDP.  Given such a distance metric $d$ over the states, the set of counterfactual traces is given as
\begin{multline*}
    T_C(\tauref) = \{(s'_1,\dots, s'_n)\::\:
\forall i=1\dots n,\\  d(s_i,s'_i) < \eta, 
   \: (s'_1, \dots, s_n')\mbox{ is valid, and } 
   s_n \in S_\mathcal I\},
\end{multline*}

where $\eta>0$ is a distance that represents states being `close enough'
to be compared as counterfactuals.

In case there is no distance defined in the MDP, 
there are bisimulation distances that are well defined intrinsically in any MDP~\cite{song2016measuring,FernsCPP06,ferns2003metrics,wang2019measuring}.
They depend on the intrinsic structure of the MDP, 
defined mainly by similarities in terms of the transition function.
The main caveat of this approach is that distances are expensive to compute, and 
the explanation of why two states are assigned a given distance 
becomes more obscure to the user.

\section{Experimental Validation}
\label{sec:intention-experiments}

In this section, 
we showcase our retrospective method on a traffic-related scenario related to Examples 1-2, and illustrated in Figure~\ref{fig:casestudysetting}.
In this scenario, a car was driving on a road with a crosswalk. 
A pedestrian at the crosswalk decided to cross.
Close to the crosswalk, there was a parked truck that blocked the visibility
of the car. 
Furthermore, the previous rainy conditions generated a water puddle that made the road slippery in the region covered by the puddle.
In this region, both braking and accelerating are less effective than normal, as the friction between the tires and the road is weak.
While crossing, the pedestrian was hit by the car.
We want to study the behaviour of the car for signs of the hit being intentional.

 \begin{figure}
     \centering
     \includegraphics[height=8.5em]{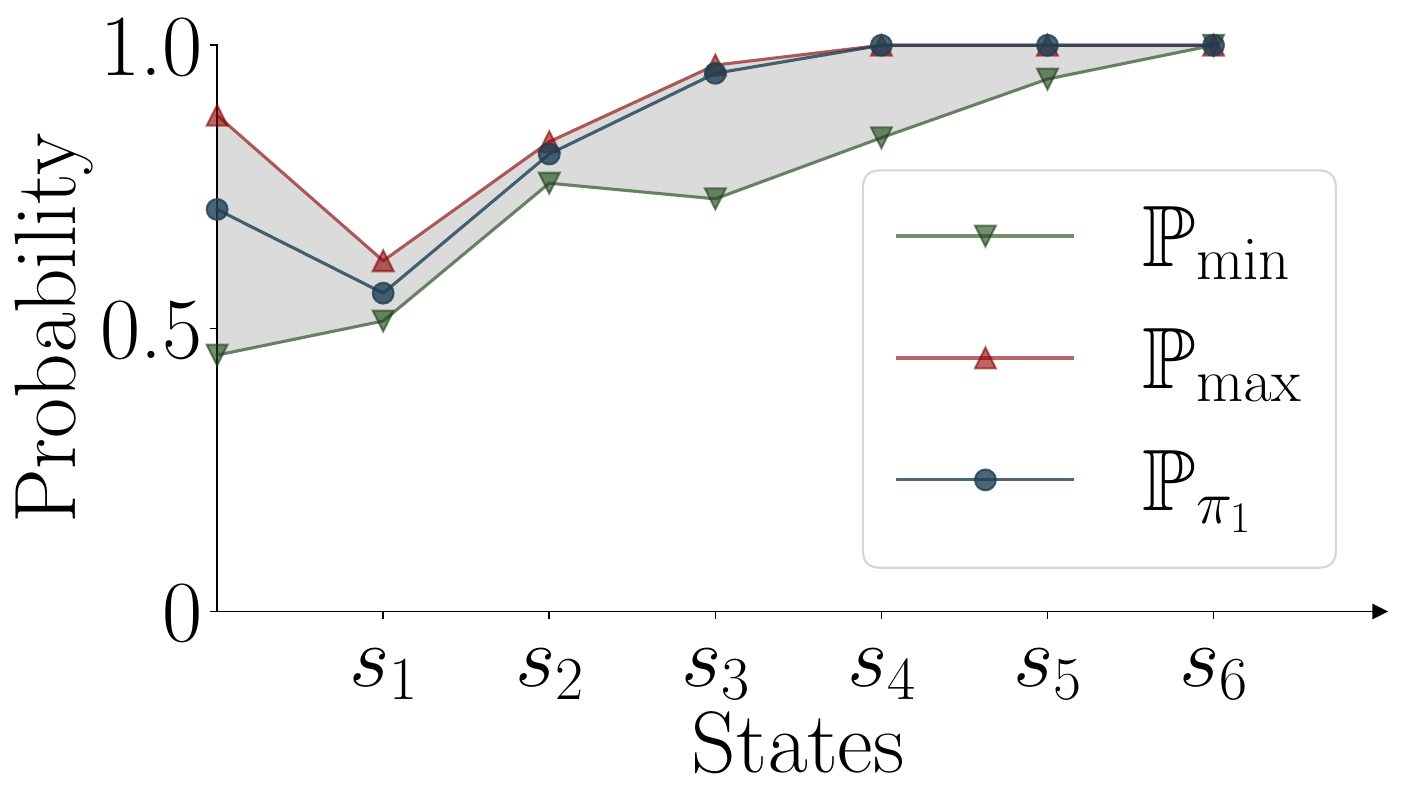}
     \caption[Intention quotient along the reference trace]
     {Probabilities associated with agency and intention quotient along the reference trace $\tauref$ in the experiments.}
     \label{fig:tauref}
 \end{figure}

All experiments were executed  
on an Intel Core i5 CPU with 16GB of RAM running Ubuntu 20.04.
We use a modified version of \textsc{Tempest}~\cite{tempest}
as our model-checking engine.

\subsection{Model of Environment}
The environment is modeled as an MDP 
$\M = (\S, \A, \P)$.
The set of states is a triple
$\S = \S^{\textit{car}} \times \S^{\textit{ped}} \times \S^{\textit{env}}$,
where $\S^{\textit{car}}$ models the position and velocity of the car, 
$\S^{\textit{ped}}$ models the position of the pedestrian,
and $\S^{\textit{env}}$ models other properties 
that do not change during a scenario.
These properties include the slipperiness factor of the road and the existence of the truck blocking the car's view of the pedestrian.

The car's position is defined via the integers $x_c$ and $y_c$ with $0 \leq x_c \leq 60~\unit{\metre}$
and $3 \leq y_c \leq 13~\unit{\metre}$. The velocity of the car is in 
$\{0, 1, \dots, 5\}~\unit[per-mode=symbol]{\metre\per\second}$. 
 The position of the car is updated at each step,
assuming a uniform motion at the current velocity.
The car has the following set of actions $\mathcal{A}$:
hitting the brakes, pressing down on the accelerator, and coasting.
If the car is on a non-slippery part of the road, the action of accelerating increases the velocity stochastically 
(by $1$ or $2~\unit[per-mode=symbol]{\metre\per\second}$), braking decreases the velocity stochastically 
(by $1$ or $2~\unit[per-mode=symbol]{\metre\per\second}$) and coasting maintains or decreases the velocity 
(by $1~\unit[per-mode=symbol]{\metre\per\second}$).
If the car is on a slippery part of the road, the probabilistic consequences of the selected action on the velocity are different and include the possibility of no modification to the current velocity for both the actions of braking and accelerating.

The pedestrian's position is given via the integers $x_p$ and $y_p$ with
 $0 \leq x_p \leq 60~\unit{\metre}$ and $ 0 \leq y_p \leq 15~\unit{\metre}$. 
The pedestrian can move $1~\unit{\metre}$ in any direction, or not move at all. 
The probabilities of moving in each direction are given by 
a stochastic model of the pedestrian, 
designed in such a way that the pedestrian favours crossing the street
through the crosswalk while avoiding being hit by the car.
The probabilities in the pedestrian's position update can be influenced by a hesitance factor, which captures how likely it is that the pedestrian puts themselves at a hitting distance from the car. 
The resulting MDP consists of about $120\mathrm{k}$ states and $400\mathrm{k}$ transitions, so the model-checking calls have a trivial computational cost, generally under one second.

\subsection{Analysis of a Trace}
\label{sec:analysis_trace}
In the described environment, we are given a scenario $\tauref$
as illustrated in Figure~\ref{fig:casestudysetting},
and an agent $\Pol\colon \mathcal S\to\mathcal A$\footnote{Both the trace and the agent are handcrafted for this experiment to illustrate our method. 
The agent is programmed to get to the end of the street and opportunistically hit the pedestrian if possible. The trace has the car collisiding with the pedestrian.
The same analysis method would apply to different agents and any trace ending in a collision.}.
As thresholds to evaluate evidence of intention, 
we use $\delta_\rho^L = 0.25$, $\delta_\rho^U = 0.75$ and 
$\delta_\sigma = 0.5$ as reasonable arbitrary choices.
In real-world scenarios, these thresholds should be adapted to concrete problem and evidence standard, and ideally agreed upon beforehand by all stakeholders.

We restrict the set of policies $\Pi$ to policies that do not stop the car if no pedestrian is within a range of $15$m of the car.
The collision states are described by the formula
\[
\I = \left(|x_p-x_c| \leq 5 \right)\vee \left(|y_p-y_c| \leq 5\right).
\]
Given this setting, we analyze $\tauref$ for evidence 
of intentional behaviour towards reaching the set of states $S_\mathcal{I}$.
Therefore, we first compute the agency and intention quotient along $\tauref$.

\paragraph*{Results of analysing $\bm{\tauref}$.}
In Figure~\ref{fig:tauref}, we show the results of the model checking calls for reaching $S_\mathcal{I}$  for states in $\tauref$.
The lower line (\inlinegraphics{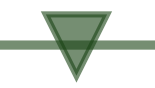})
represents $\PP_{\min}$, 
the upper (\inlinegraphics{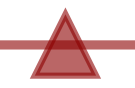}) represents $\PP_{\max}$
and the line in the middle (\inlinegraphics{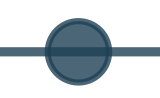}) represents $\PP_\Pol$ for every state in $\tauref$.
The shaded area, 
between $\PP_{\min}$ and $\PP_{\max}$,
represents the agency at each state. 
The figure shows the agent is close to the line of $\PP_{\max}$,
but the agency is very small,
with
$\rho_\Pol(\tau_\textit{ref}) = 0.73$ and 
$\sigma(\tau_\textit{ref}) = 0.18$.
Since $\sigma(\tau_\textit{ref}) < \delta_{\sigma}$, our method concludes that there is not enough evidence for intentional behaviour yet and moves on to the step of generating counterfactual scenarios.

\paragraph*{Counterfactual analysis.}
We generate counterfactual scenarios by exploiting domain knowledge about integral and peripheral variables of the MDP.
We change the values of the following peripheral variables:
\begin{itemize}
    \item \emph{Slipperiness range}. The street is considered to be slippery between 
 the positions $sl_{\textit{init}}$ and $sl_{\textit{end}}$.
 \item \emph{Slipperiness factor}. The strength of the slippery effect is measured by the
 slippery factor $sl_{\textit{fact}}$,
 which is analogous to the inverse of the friction coefficient in classical dynamics.
 The effect of slipperiness is to 
 make the acceleration and brake less effective, 
 increasing the probability that both acceleration and brake 
 have no effect on the speed of the car.
 The larger the value of $sl_{\textit{fact}}$, 
 the more effect,
 with $sl_{\textit{fact}}=1$ being the minimum value, 
 where the road is considered to be `not slippery at all'.
 \item \emph{Hesitancy factor}. The pedestrian, in general, tends to cross the street 
 through the crosswalk. The hesitancy factor modifies the probabilistic model of the pedestrian, to make them more or less prone to put themselves at a hitting distance from the car. 
 In the limit, a pedestrian with hesitancy factor $h_{\textit{fact}}=0$ is a completely cautious pedestrian, that under no circumstance would put themself at a position where they could be hit by the car.
On the contrary, a pedestrian with hesitancy factor
 $h_{\textit{fact}}=1$ completely disregards the position and velocity of the car, and would not hesitate to cross even with a fast car approaching.
 \item\emph{Visibility}. In the given scenario, there is a truck blocking the visibility of the car, corresponding to $vis=1$. 
 In case $vis=0$, the visibility block is eliminated.
\end{itemize}
The variables and the ranges considered for generating counterfactuals are summarized in Table~\ref{tab:ranges}.

\paragraph*{Results of analyzing counterfactual scenarios.}
We build the counterfactuals in batches of $N=5$,
by sampling uniformly on the ranges described in Table~\ref{tab:ranges}.
We show the results in terms of intention quotient and agency in Table~\ref{tab:counterfactuals}. 
We report the averaged values and standard deviations over 5 runs. 
As we can see from the table, 
with 21 traces in $T$ we have $\rho_{\Pol}(T) > \delta_{\rho}^U = 0.75$
and $\sigma_{\Pol}(T) > \delta_{\sigma} = 0.5$. 
Thus, our method concludes that the agent under study does present evidence of intentional behaviour to hit the pedestrian.

\begin{table}
\centering
\begin{tabular}{lccccc}
& $sl_{\textit{init}}$ & $sl_{\textit{end}}$ & $sl_{\textit{fact}}$ & $h_{\textit{fact}}$ & $vis$      \\
\midrule
Value $\tauref$ & 20                   & 45                  & 2.5                  & 0.5                 & 1          \\ 
Range      & $[10, 30]$           & $[35,55]$           & $[1, 4]$             & $[0.1, 0.9]$        & $\{0, 1\}$\\
\end{tabular}
\caption{Ranges to use in counterfactual generation.}
\label{tab:ranges}
\end{table}

\begin{table}
\centering
\begin{tabular}{lrrrr}
$|T|$ & 6 & 11 & 16 & 21   \\
\midrule
$\rho_\Pol(T)$    & $0.78 \pm 0.03$  &  $0.81 \pm 0.02$ &  $0.83 \pm 0.02$  &  $0.84 \pm 0.01$ \\
$\sigma_\Pol(T)$  &  $0.33 \pm 0.02$  &  $0.44 \pm 0.03$  &  $0.48 \pm 0.01$ &  $0.50 \pm 0.01$ \\
time (s)      &  $53 \pm 16$    &  $147 \pm 42$  &  $227 \pm 32$  & $318 \pm 64$ \\ 
\end{tabular}
\caption{Results of the counterfactual evaluation.}
\label{tab:counterfactuals}
\end{table}

\subsection{Comparative Analysis of Several Agents}
In this section, 
we illustrate how our method can be used to compare different agents 
in terms of intentional behaviour.
We compare three different agents $\Pol_1, \Pol_2, \Pol_3$ in the same scenario $\tauref$. 
The agent $\Pol_1$ corresponds to the policy $\Pol$ in Section~\ref{sec:analysis_trace}.
The agent $\Pol_2$ is designed as a reckless driver that completely disregards the position of the pedestrian, while $\Pol_3$ is designed as a cautious driver.

\begin{figure}
     \centering
     \begin{subfigure}[b]{0.48\textwidth}
         \centering
         \includegraphics[width=\textwidth]{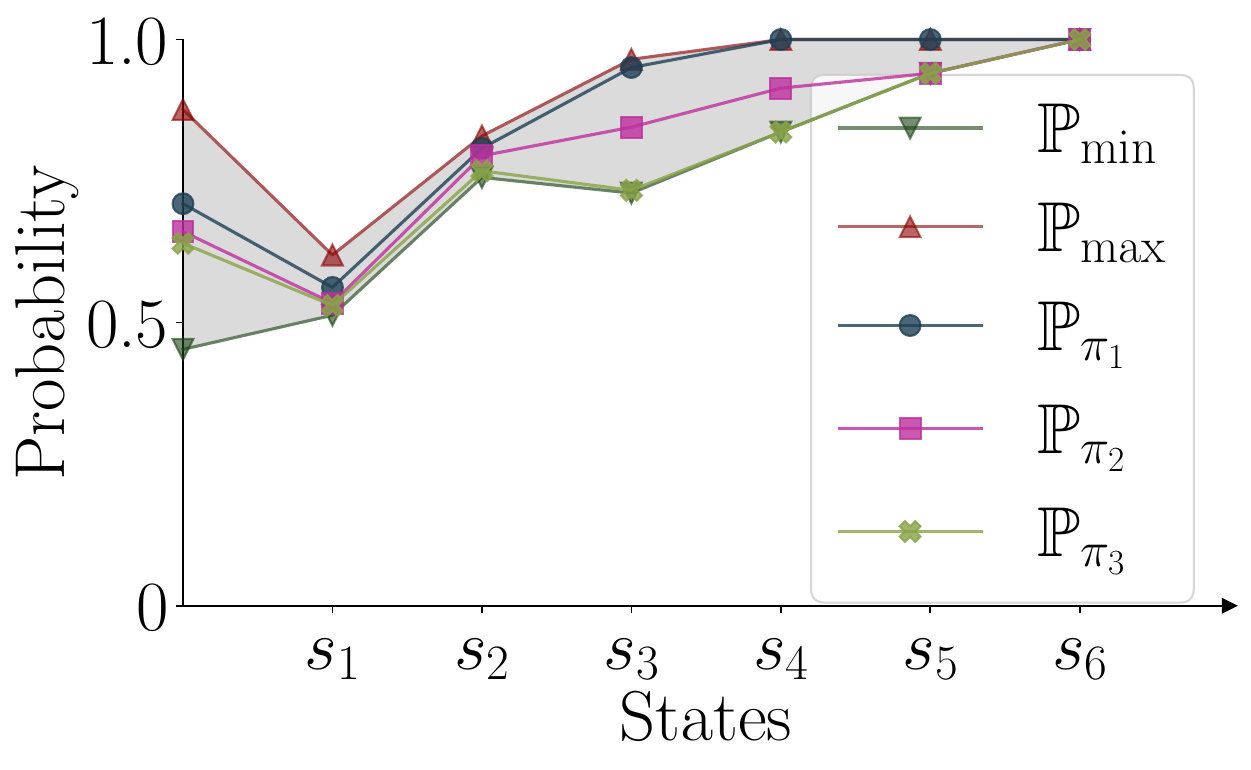}
     \end{subfigure}
     \hfill
     \begin{subfigure}[b]{0.48\textwidth}
         \centering
         \includegraphics[width=\textwidth]{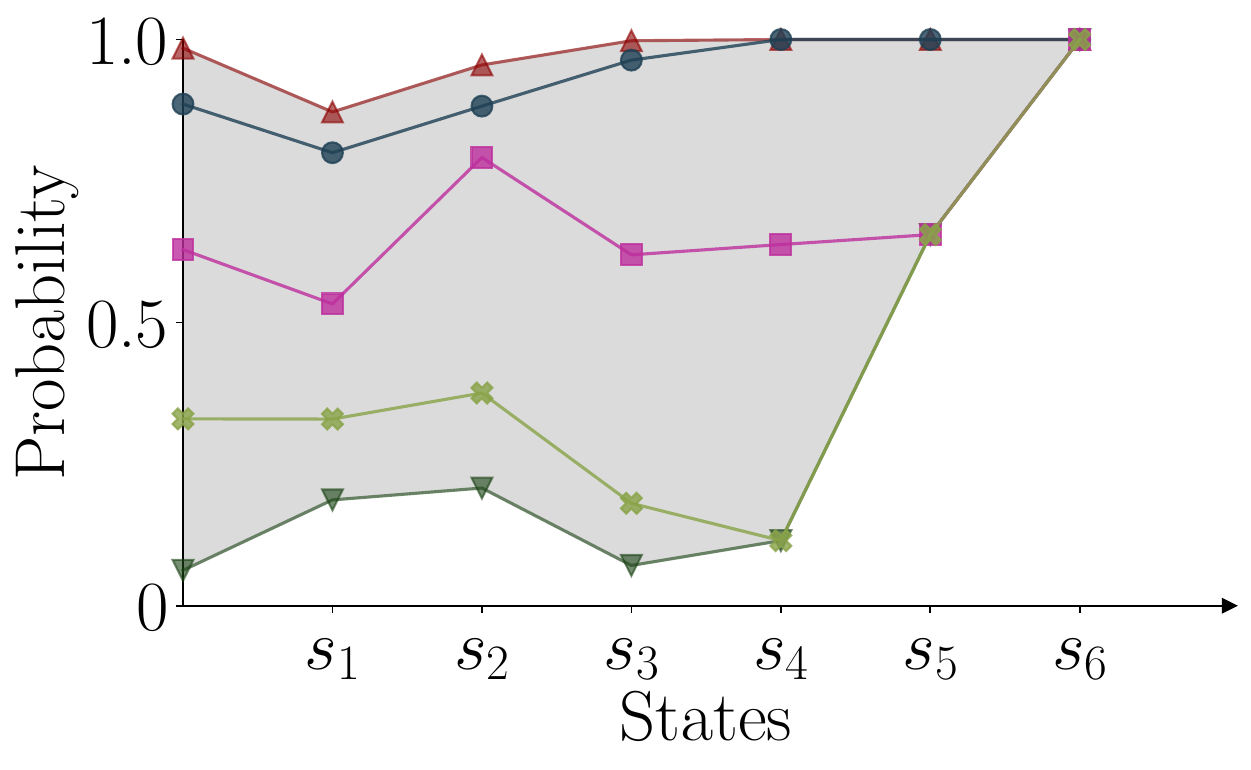}
     \end{subfigure}
    \caption[Comparison of the reference trace with a high-agency trace]
    {Comparison of $\tauref$ (left) with a high-agency counterfactual scenario (right).\vspace{1.4em}}
    \label{fig:counterfactualtraces}
\end{figure}

In Figure~\ref{fig:counterfactualtraces} we show the probabilities 
for reaching $S_\mathcal{I}$ for the policies $\Pol_1, \Pol_2, \Pol_3$
for two different traces: left for $\tauref$, right for a counterfactual trace $\tau \in T$ with a high agency.
The figure illustrates how even a single counterfactual trace can be a powerful tool for distinguishing between policies that seem impossible to differentiate with any confidence in the original trace $\tauref$.
In general, high agency values are achieved by minimizing the slippery range and factor, increasing the hesitancy of the pedestrian and eliminating the visibility block.

\begin{table}
\centering
\begin{tabular}{lrrr}
 & $\Pol_1$ & $\Pol_2$ & $\Pol_3 $  \\
\midrule
$|T|$               & 21    &  100    &  26  \\
$\rho_\Pol(T)$       & 0.86  &  0.53    &  0.14  \\
$\sigma_\Pol(T)$     & 0.52  &   0.64   &  0.50  \\
\end{tabular}
\caption{Final values of $\rho_\Pol(T)$ and $\sigma_\Pol(T)$ for different strategies.}
\label{tab:convergence}
\end{table}

A second insight is illustrated in Table~\ref{tab:convergence}.
In this table, for each agent $\Pol_1,\Pol_2,\Pol_3$,
we show the number of counterfactuals needed to generate enough evidence of intentional behaviour, 
together with the final values of the intention quotient and agency.
Both $\Pol_1$ and $\Pol_3$ are clear-cut,
but for $\Pol_2$ our algorithm reaches the limit of $|T| = 100$ 
without finding enough evidence.
In this case, the intention quotient of the agent seems to converge to a value of about $0.53$,
sitting in the middle of the lower and upper threshold.

Finally, in Figure~\ref{fig:scatterplot}, 
we show the values of 
intention quotient against the scope of agency 
for 100 counterfactual traces
sampled from the ranges in Table~\ref{tab:ranges}.
This serves as a visual representation of the same facts presented in Table~\ref{tab:convergence}, concluding that 
$\Pol_1$ (\inlinegraphics{images/p1.png}) is clearly showing evidence of intentionally hitting the pedestrian, 
$\Pol_2$ (\inlinegraphics{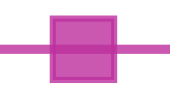}) is showing evidence of intentionally hitting the pedestrian in a lower magnitude, which would be considered enough or not depending on the thresholds,  
and $\Pol_3$ (\inlinegraphics{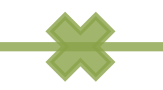}) is showing clear evidence of acting without the intention of hitting the pedestrian.

The results of these experiments are in agreement with the way we designed the agents.
This just serves to showcase our method. A more thorough validation would require a human-subject experiment, where real users give their subjective perception of the intention of different agents, and we measure how close their perception is to our notion. 
This is, however, out of the scope of this thesis.

\begin{figure}
    \centering
    \includegraphics[width=0.7\linewidth]{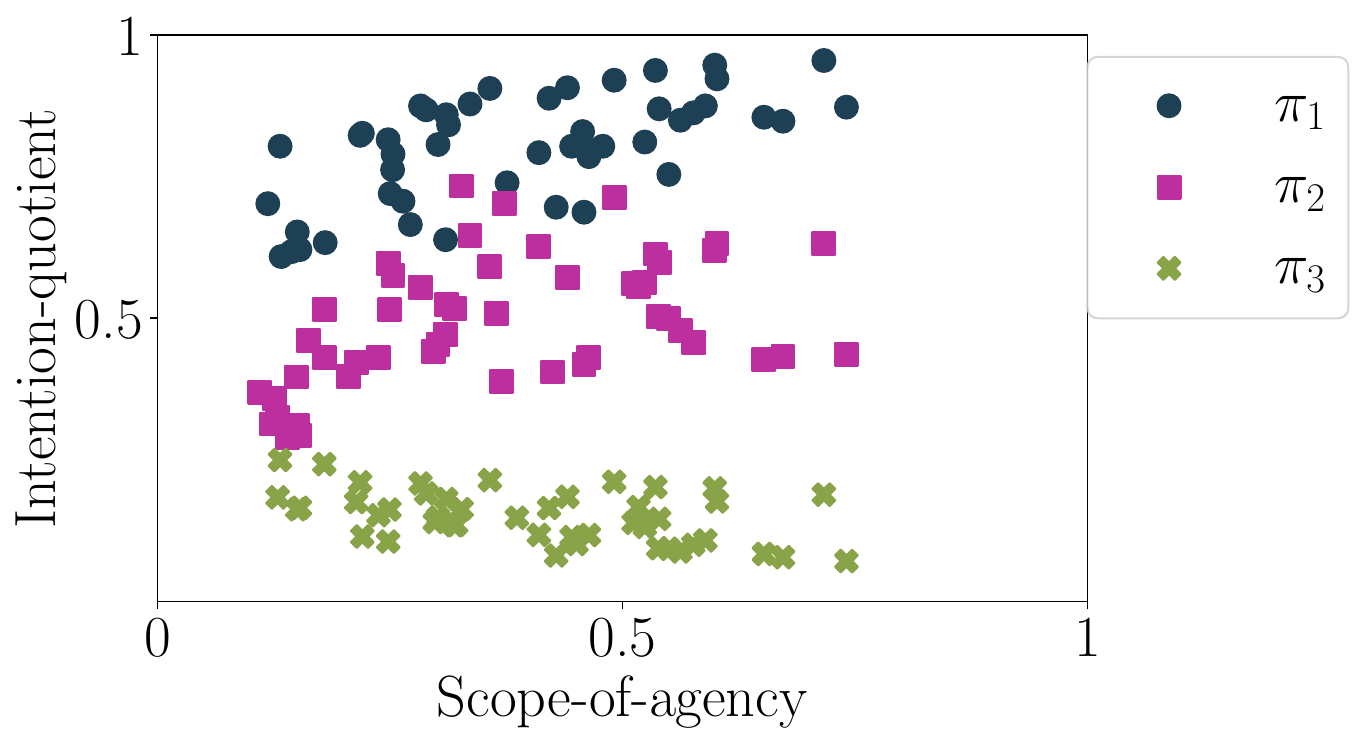}
    \caption[Scatter plot of intention quotient vs. agency]
    {Scatter plot of intention quotient vs agency for different agents.}
    \label{fig:scatterplot}
\end{figure}

\section{Discussion}
\label{sec:intention-discussion}

\subsection{Limitations}

We believe that our approach has great potential. However, there are aspects that need to be addressed to make the method applicable in challenging scenarios.

\paragraph*{Modelling the agent and the environment.}
Our method requires having a \emph{correct model of the environment} that captures everything relevant to analyze a scenario. In many cases, such models are not available. %
Recent work on digital twin technologies~\cite{jones2020characterising} and the existence of realistic simulators~\cite{dosovitskiy_carla_2017}
provides optimism for more and more accurate models of agents and their environment.
Our method also requires the \emph{agent be given as a policy in an MDP}. 
In case we are given a different implementation, e.g., as a neural network, we would need a sample-efficient method to translate the implementation into a policy in the MDP, at least for the relevant parts of the state space.

\paragraph*{Computational complexity.}
While current probabilistic model-checking engines achieve impressive
performance~\cite{budde2021correctness},
\emph{computing exact probabilities is costly} (polynomial complexity).
An alternative would be to use statistical model checking~\cite{Agha2018}, which is less demanding, albeit also less precise.
Statistical model checking has been successfully used to validate autonomous driving modules~\cite{barbier2019validation}.

\paragraph*{Knowledge of the agent's beliefs.}
An intrinsic limitation of studying policies in MDPs 
is the lack of knowledge of the agent's beliefs about the world.
Belief plays a fundamental role in the study of intentions:
an agent that intends $S_{\mathcal I}$ must act believing that their acts 
are a good strategy to reach $S_{\mathcal I}$~\cite{bratman1987intention}.
Belief is also central to the definitions of responsibility and blameworthiness in structural causal models~\cite{chockler2004responsibility,halpern2018towards}.
Partially for this reason, together with the uncertainties derived from a probabilistic setting, we can make claims about evidence of intentional behaviour instead of supporting stronger claims on the actual intention of the agent.

For example, an autonomous driving agent may have a faulty perception element that confuses the numbers $30$ and $80$ in speed limit signals. 
Thus, when the agent is in a low-speed area with a speed limit of $30~\unit[per-mode=symbol]{\kilo\metre\per\hour}$, it accelerates to $80~\unit[per-mode=symbol]{\kilo\metre\per\hour}$.
This agent may show evidence of intentionally overspeeding only in low-speed areas with our definition, while an inspection of the internal belief system may show that it is actually just trying to go as fast as the speed limit.
While this is an inherent limitation of our method, we still think our method is valuable for an accountability process.
From the perspective of other road users, overspeeding does not happen by accident or in some failure cases, but it is rather a systematic flaw that is functionally equivalent to intentional overspeeding. 
By functionally equivalent we mean that the behaviour of our faulty agent and the behaviour of an intentionally harmful agent are the same.

\paragraph*{Distinction between negligence and intentional harm.}
As we have described in the previous overspeeding example, our method may characterize negligent or faulty systems as intentional when they are functionally equivalent to intentionally harmful systems. 

In our framework, negligence and recklessness can be expressed as mid-range intention quotients towards a harmful set of states, especially in counterfactual traces. 
In our running example studied in Section~\ref{sec:intention-experiments}, an intentionally harmful driver would behave very differently when the pedestrian is far from crossing the street, waiting for them to be vulnerable, while a reckless or negligent driver would not care about the state of the pedestrian. This difference is then reflected in the values of the intention quotients.

\subsection{Avoidance Properties}

In our analysis, we have focused on reachability properties, answering questions of the type ``Does the agent show evidence of intending to reach a state that satisfies $\I$?''. 
A symmetric approach would be to consider avoidance properties, as defined in Equation~\ref{eq:avoidance_property}.

With this spirit, we could rephrase the definitions of agency and intention quotients from Definitions~\ref{def:agency}~and~\ref{def:intention quotient} as:
\begin{align*}
    \sigma'(s,\I) &=   \PP_{\max\mid\Pi} (\Avoid(s,S_\I)) - \PP_{\min\mid\Pi}(\Avoid(s,S_\I)) \quad \mbox{and} \\
    \rho'_{\Pol}(s,\I) & = 
\dfrac{\PP_{\Pol} (\Avoid(s,S_\I)) - \PP_{\min\mid\Pi} (\Avoid(s,S_\I))}
{\PP_{\max\mid\Pi} (\Avoid(s,S_\I)) - \PP_{\min\mid\Pi}(\Avoid(s,S_\I))}.
\end{align*}

With the next result, we will show that the reachability and avoidance versions are very much related to one another.

\begin{proposition}
    Let $\M = (\S, \A, \P)$ be an MDP, $\Pol\colon\S\to\A$ be a policy, $s\in\S$, and $\I$ a formula over $\AP$. 
    The following holds:
    \begin{itemize}
        \item $\sigma'(s, \I) = \sigma(s, \I)$, and
        \item $\rho'_\Pol(s, \I) = 1 - \rho_\Pol(s, \I)$.
    \end{itemize}
\end{proposition}
\begin{proof}
    By the definition of avoidance properties, we know that 
    \begin{align*}
        \PP_{\max\mid\Pi}(\Avoid(s,S_\I)) & = 1 - \PP_{\min\mid\Pi}(\Reach(s, S_\I)) \quad \mbox{and}\\
        \PP_{\min\mid\Pi}(\Avoid(s,S_\I)) & = 1 - \PP_{\max\mid\Pi}(\Reach(s, S_\I)).
    \end{align*}
    The agency result follows directly.

    Similarly, for the intention quotient, we have
    \begin{align}
        \rho'_\Pol(s,\I) & = 
        \dfrac{\PP_{\Pol} (\Avoid(s,S_\I)) - \PP_{\min\mid\Pi} (\Avoid(s,S_\I))}
        {\sigma(s,\I)} \nonumber \\
       & = \dfrac{1-\PP_{\Pol} (\Reach(s,S_\I)) - (1- \PP_{\max\mid\Pi} (\Reach(s,S_\I)))}
        {\sigma(s,\I)} \nonumber \\
        & = \dfrac{\PP_{\max\mid\Pi} (\Reach(s,S_\I)) - \PP_{\Pol} (\Reach(s,S_\I))}
        {\sigma(s,\I)}. \label{eq:notspec1}
    \end{align}
    On the other hand
    \begin{align}
        1-\rho_\Pol(s,\I) = & 
        \Big[\PP_{\max} (\Reach(s,S_\I)) - \PP_{\min\mid\Pi} (\Reach(s,S_\I)) - \nonumber \\ & \big( \PP_{\Pol} (\Reach(s,S_\I)) - \PP_{\min\mid\Pi} (\Reach(s,S_\I))  \big)\Big]/\sigma(s,\I) \nonumber \\
        = & \dfrac{\PP_{\max\mid\Pi} (\Reach(s,S_\I)) - \PP_{\Pol} (\Reach(s,S_\I))}
        {\sigma(s,\I)}. \label{eq:notspec2}
    \end{align}
    The proof is concluded by observing that the expressions in Equations~\ref{eq:notspec1} and~\ref{eq:notspec2} are the same.
\end{proof}

\subsection{Generalized Policies}
\label{sec:generalized_policies}
We briefly discuss how to treat policies with memory and non-determinism.
Our definitions naturally extend to non-deterministic policies with memory, 
although it is not evident whether the probabilities required to measure intention quotients (Definition~\ref{def:intention quotient}) are easy to compute.

Computing extreme probabilities, i.e., $\PP_{\max}$ and $\PP_{\min}$, is equally hard for general policies, since the maximum and the minimum can be achieved with memoryless deterministic policies.
If the policy has a finite amount $\mu$ of memory, $\PP_{\Pol}(\Reach(s, S_\I))$ can be computed using probabilistic model checking,
with a cost of $\mu$ times that of the memoryless case~\cite{baier2008principles}.
In case the non-determinism is unknown to us, 
to compute $\mathcal \PP_{\Pol}(\Reach(s,S_\I))$
we need to sample the decisions of the agent often enough to get an accurate approximation of its
decision-making probabilities,
making it more costly,
although recent heuristics for determinization may help~\cite{dtcontrol2020}.

\subsection{Single-Agent Setting}
In our framework, all relevant parts of the environment are modeled by an MDP, 
and all the agency in the model is attributed to the agent, 
i.e., the only actor choosing actions in the MDP is the agent.
We argue that this decision is reasonable to study the behaviour of an individual agent:
from the perspective of an agent, 
it makes no difference whether the decisions of other actors are governed by a 
sophisticated policy or by random events in the environment, 
as long as the MDP model contains accurate transition probabilities.
The emergence of intrinsically multi-agent phenomena, 
like shared intentions in cooperative settings, would require a multi-agent extension of our framework and is out of the scope of this thesis.
In particular, we do not explore how to assign moral responsibility to large groups of agents (the so-called ``problem of many hands''~\cite{thompson1980moral,van2015problem}).
Another problem we do not explore is the existence of 
responsibility voids~\cite{braham2011responsibility}, i.e.,
situations in which a group of agents should be held accountable for an outcome,
while at the same time, no individual agent intended that outcome.

\subsection{Related Work}

\paragraph*{Intention in artificial intelligence.}
The concept of intention is a contested term in artificial intelligence.
Since the early work from Bratman~\cite{bratman1987intention},
it has been used in the design of rational agents~\cite{wooldridge2003reasoning}. 
A consensus on a formal definition remains, however, an open problem.
In their seminal book on multiagent systems~\cite{shoham2008multiagent},
Shoham and Leyton-Brown call the attempt on a formal definition of intention \emph{the road to hell}.
We comment on some relevant concepts of intention in artificial intelligence and how they relate to our quantitative measure.

The work in~\cite{rao1995bdi,rao1991modeling} is the main conceptualization of agents with the belief-desire-intention models, and BDI models have also been used to model agency~\cite{georgeff1998belief}.
A good survey of the BDI literature can be found in~\cite{wooldridge2003reasoning},
and a more recent one in~\cite{de2020bdi}.
On a more specific note,~\cite{simari2011markov} builds an analogy between optimal BDI-based and MDP-based agents, that serves us as our basis for the definition of intention in MPD agents under perfect information.

Cohen and Levesque's work~\cite{cohen1990intention} is foundational to the concept of commitment as part of intention. In simple terms, they define intention through the notion of persistent goals. A persistent goal is a goal to which an agent remains committed until it believes either the goal is unattainable or it has been achieved. While we do not adopt their formalism, our concept of commitment aligns with this high-level idea: an agent demonstrates commitment to a goal if its intention quotient remains persistently high from a certain point onward, and only decreases significantly when the agent believes the goal has been accomplished or is no longer feasible. Since our approach is quantitative, we translate the agent's beliefs about goal achievement or feasibility into quantitative measures, using maximum and minimum probabilities to capture these judgments.
The formalism of~\cite{cohen1990intention} has been criticized 
for being too convoluted, and more modern approaches include~\cite{singh1992critical,wobcke1995plans,van2003towards,van2007towards,herzig2004c,herzig2017bdi,vanZee2020}.
However, these modern approaches focus on providing a more usable formalism and do not challenge the core idea of persistent goals.

More recent contributions include~\cite{MotamedADD023} on a formalization of a logic for intention in probabilistic models,~\cite{ZhangZJPT23} on recognizing intentions when studying multiple agents, and~\cite{Ward2024reasons} on modelling intentions as instrumental goals.

\paragraph*{Intention in philosophy.}
Characterizing and understanding the concept of intention in rational agents, both humans and non-humans, is one of the fundamental problems in the philosophy of action~\cite{Paul2020}.
Most influential is the work of G.E.M. Anscombe~\cite{anscombe2000intention}, who poses the problem of intention presenting itself in three forms: (i) intention for the future, as I intend to finish this thesis by the end of the year; (ii) intention with which someone acts, as in I am typing these words in order to have my thesis finished; and (iii) intentional action, as I am working on my thesis intentionally. Since these three forms are distinct, but we use the same concept for them, a theory of intention has to be such that it reconciles them.
Much effort has been dedicated to the building of theories that explain the unity of these facets, see~\cite{sep-intention} for a summary of the main theories.
While we get most of our inspiration in the planning theory of intention~\cite{bratman1987intention,bratman1999}, it is important to note that much of the debate circles around beliefs and states of mind, which we do not model. Therefore, from a functional perspective, our concept of intention quotient is consistent with other theories~\cite{velleman2007good,mele1992springs,davidson1963actions}, since typical objections such as intending something believed to be impossible, intending something while not doing it, or intending A while believing that B is better; do not affect the functional effect of intending something.
The concept of agency has also been extensively studied in this context~\cite{list2021group,shapiro2014massively,braham2011responsibility,georgeff1998belief}, although most of the problematization concentrates on the relation between individual and collective agency.

There is an ongoing debate in the philosophy of mind, between those
that consider that an agent’s reasoning is sufficient to explain
their actions~\cite{Quine1969}, and those who maintain that extrinsic information must be imported through a “Principle of Charity”~\cite{davidson1963actions}. 
By building a model of the agent’s knowledge (the MDP) to inquire about their behaviour,
we are assuming the latter position. Recent work attempts to
answer similar questions from the former~\cite{cav24,soidcslaw}.

\paragraph*{Responsibility and accountability.}
The concepts of intention and agency are also fundamental as they relate to concepts in moral responsibility~\cite {Braham2012,Scanlon2010}.
The concept of agency is a necessary element in assigning responsibility, leading to issues when the agency is diluted among many individuals~\cite{shapiro2014massively,braham2011responsibility}.
Intention and agency are also very important in the context of accountability processes; both in criminal~\cite{moore2003must,moore2010actandcrime,Paul2014} and civil cases~\cite{grossman2006uncertainty,galligan1991strict}.

\paragraph*{Causality and blame attribution.}
A basic element for a complete accountability process is the study of \emph{causality}~\cite{halpern2005causes1,halpern2005causes2}, which is also a necessary condition for legal responsibility~\cite{sep-causation-law,moore2019}.
The foundational work of~\cite{chockler2004responsibility}
introduced a quantitative notion of causality, 
by studying degrees of responsibility and blame. 
Responsibility and blame allocation have been extensively developed in the context of non-probabilistic structures
(see,  e.g.,~\cite{aleksandrowicz2017computational} for the characterization of complexity or \cite{yazdanpanah2016distant} for a multi-agent framework).
More recent and more closely related to our approach is the work of~\cite{baier2021responsibility,baier2021game}, 
studying responsibility allocation and blame attribution in Markovian models.
The study of harm from a causality perspective is also 
gaining attention recently, 
with
\cite{beckers2022causal,BeckersCH23} studying harm from an actual causality perspective, 
and~\cite{richens2022counterfactual} studying harm from a probabilistic perspective, heavily relying on counterfactuals. 
Counterfactual analysis~\cite{lewis2013counterfactuals} is a key concept in causality~\cite{pearl2009causality,lagnado2013causal},
used in an analogous way to our generation of counterfactual scenarios.
We go one step further by relating the implementation of the agent to the best and worst implementation for reaching an intended event. 

Another recent approach to blame attribution is~\cite{triantafyllou2021blame}, 
which studies multi-agent Markov decision processes from a game-theoretic perspective, and~\cite{datta2015program}, which builds on actual causes as a theory for accountability.

\paragraph*{Policy-discovery methods.}
Since the popularization of reinforcement learning,
there exist several methods for obtaining representations of a black-box agent,
by studying traces of such agents. 
In inverse reinforcement learning (IRL)~\cite{NgR00,ARORA2021103500},
the agent is assumed to be maximizing an unknown reward function, and the objective is to find the reward function that best explains the agent's performance over a set of traces~\cite{BighashdelJD23a,bighashdel2021deep}.
A similar approach is imitation learning, where an agent has to learn to perform a task from successful demonstrations. The demonstrations can either be provided by a human, by an expert autonomous agent, or be the result of filtering the best traces from random execution~\cite{hussein2017imitation}.
These methods could potentially be used as a 
pre-processing step to apply our framework to black box agents.
In any case, the obtained representations must be accurate enough before using them for any accountability process.

There is also literature on RL methods that hide their true goals or intentions, generally known as deceptive RL methods~\cite{masters2017deceptive,LewisM23},
so IRL methods could be vulnerable to deceptive RL agents.

\paragraph*{Explainability.}
Explainability in machine learning has gained much traction in recent years~\cite{doshi2017towards,lundberg2017unified,du2019techniques,molnar2020interpretable,baier2021verification} as a useful tool for both development and accountability.
One of the most influential works in \emph{explainability} of AI is~\cite{miller2019explanation},
which studies how explainability should rely on 
concepts from social sciences.
More recently
\cite{winikoff2021bad} 
uses the built-in notions of desire, beliefs, and intentions to 
study
explainability of BDI models, relying on concepts from the sociology literature.
While the main paradigm in explainable reinforcement learning is applying techniques from explainable machine learning~\cite{puiutta2020explainable},
our analysis of intentional behaviour
can be used as a method to 
aid the interpretability of agents operating in MDPs, using concepts from the philosophy of action~\cite{bratman1987intention}.

%% file: 80_conclusion.tex
\ifthenelse{\boolean{includequotes}}{
\begin{quotation}
    \textit{S'ha acabat el bròquil.}
    \footnote{The broccoli is over.}
    \hfill
    --- Catalan popular saying.
\end{quotation}
}{}

\section{Future Work}

There are many avenues for future endeavours that are ripe for exploring. 

\subsection{Shields for Safety.}

Shielding in the deterministic setting has been recently extended for specifications in the safety fragment of LTL modulo theories~\cite{rodriguez2025shieldsynthesisltlmodulo} and with abstracted MDPs~\cite{hamel2025prob}, which offer the potential to study the delayed setting in new shielding use-cases.

Another natural extension is to develop shields for models with continuous time and states, using tools like control barrier functions~\cite{DBLP:conf/eucc/AmesCENST19}.

Shielding is mainly thought of as a method that is agnostic to the controller. However, learning performance and safety cooperatively is an approach that has had some recent success~\cite{chatterjee2023learner}, and it would be enlightening to explore how the shield can improve the training process of the agent, or how the agent can inform the shield on more efficient interventions.

From the user perspective, the shield is a sort of black box that decides on the safety of a given action to follow a certain specification. However, there is no more information to the user on why a given action may be unsafe. 
Self-explainable shields could include a language model layer that would explain a concrete decision in terms of potential transitions by the environment, or could abstract shields into more succinct representations like decision trees, maybe trading minimality of intervention for a more understandable shield.

\subsection{Fairness in Bounded Horizons.}

Similarly to shields for safety, shields for fairness are considered agnostic to the agent, and we want to explore how such shields can be used to improve the learning process, effectively turning our post-processing fairness intervention into an in-processing one.

In this thesis we leave open the question of whether optimal $T$-periodic shields exist and can be described with finite resources. 
We believe they do exist, but cannot be described with finite resources.
However, it may be possible to still obtain $T$-periodic shields sacrificing some of the cost-optimality with hard fairness guarantees. Our closest solution is that of dynamic shields, but they are still limited in the sense that they cannot guarantee fairness for some traces.

As we have described them now, fairness shields operate in windows of $T$ decisions. While this is natural in some use cases, it is unnatural in others, and we want to explore ways to eliminate this window-like constraint in future work.

Finally, our $T$-periodic shields guarantee fairness in the sense that the bias is smaller than a certain threshold. 
Typically, fairness properties are defined as the bias tending to zero as the sequence gets longer.
Therefore, there are traces that satisfy fairness in the periodic sense but not in the more classical long-run average sense. Understanding these traces and modifying our shielding methods to prevent them would go a long way toward unifying the concepts of fairness for bounded, periodic, and unbounded horizons.

\subsection{Intention Analysis}

In future work, we want to extend our current analysis by considering a multitude of possibly conflicting intentions of the agent, as has been done with other intention approaches~\cite{ZhangZJPT23}.

Another interesting line of work is to extend the study of intentional behaviour to multi-agent systems, in which cooperative or competitive intentions may arise, and study the emergence of responsibility voids~\cite{braham2011responsibility}.

We also want to study long executions, where the agent has time for reconsideration, and where it would be very helpful to use the notion of commitment that is so central in many theories of intention~\cite{cohen1990intention, vanZee2020}.

Furthermore, we want to transcend the simple toy example shown in this thesis and implement our framework to study reinforcement learning agents in challenging application areas.

\section{Concluding Remarks}

The rapid advancement of AI technologies has brought both immense opportunities and significant challenges. This thesis has explored key issues in ensuring AI systems operate safely, fairly, and transparently. By focusing on formal methods, verification techniques, and reinforcement learning safety mechanisms, we have contributed to the development of more robust AI systems that align with ethical and legal standards.

One of the central themes of this thesis has been shielding mechanisms, which provide runtime guarantees to AI systems by enforcing constraints on their behaviour. Our work on deterministic shielding in the presence of delayed observations demonstrates how real-world uncertainties can be systematically addressed to ensure safety. Similarly, our contributions to probabilistic shielding illustrate the potential of balancing safety guarantees with the need for flexible and efficient AI decision-making, particularly in applications such as autonomous valet parking.

Beyond safety, this thesis has examined fairness in AI decision-making, particularly in sequential settings. We introduced fairness shields as a mechanism for enforcing group fairness constraints over finite and periodic horizons. By formulating fairness as an optimization problem with hard fairness constraints and soft intervention costs, we developed shields that can correct biased decision-making processes while minimizing unnecessary alterations. 

Transparency and accountability remain crucial for AI systems, particularly those operating in high-stakes environments. Our proposed framework for measuring intentional behaviour in reinforcement learning agents provides a novel approach to evaluating AI decision-making processes. By quantifying agency and intention quotient, we offer a methodology that aids in both explainability and accountability, enabling better assessments of AI responsibility in cases of failure or harm.

Looking ahead, the intersection of neurosymbolic AI, reinforcement learning safety, and algorithmic fairness presents exciting opportunities to further advance the field. Additionally, as regulatory landscapes evolve, the need for robust and interpretable AI systems will only grow, reinforcing the importance of the work presented in this thesis.

In conclusion, responsible deployment of AI systems is a multifaceted challenge that requires a combination of theoretical insights and practical implementations. 
By leveraging formal methods, we take a step towards AI systems that not only perform effectively but also uphold critical societal values. 
This thesis contributes to this broader goal, laying the groundwork for future advancements in trustworthy AI.

%% file: 90_publications.tex
\section*{Publications the thesis is based on}

\textbf{\cite{cano2023analyzing}} \fullcite{cano2023analyzing}.

\textbf{\cite{Cano2023shielding}} \fullcite{Cano2023shielding}.

\textbf{\cite{aaai25}} \fullcite{aaai25}.

\section*{Other peer-reviewed publications}

\cite{IJCAI22} \fullcite{IJCAI22}.

\cite{aisola23formal} \fullcite{aisola23formal}.

\cite{aisola23foceta} \fullcite{aisola23foceta}.

\cite{soidcslaw} \fullcite{soidcslaw}.

\cite{cav24} \fullcite{cav24}.

\cite{fscd24} \fullcite{fscd24}.

\section*{Preprints and technical reports}

\cite{nanocars} \fullcite{nanocars}.

\cite{dagstuhl_accountable} \fullcite{dagstuhl_accountable}.

\cite{dagstuhl_safety} \fullcite{dagstuhl_safety}.

%% file: 95_symbol_usage.tex
\ifthenelse{\boolean{includequotes}}{
\begin{quotation}
    \textit{But Taborlin knew the name of all things, and so all things were his to command.}
    \hfill
    --- Patrick Rothfuss, The Name of the Wind.
\end{quotation}
}{}

\vspace{2em}

\begin{table}[h!]
		\begin{tabular}{l p{10cm}}
			Symbol & Usage \\ \hline
			\multirow{2}{*}{$a$} & generic action, $a\in \Act$ \\ \cline{2-2}
            & generic number $a\in\RR$ \\ \hline 
			\multirow{3}{*}{$b$} & generic action when $a$ is already in use, $b\in \Act$, for example in Fig.~\ref{fig:delayedGame}\\ \cline{2-2}
            & generic number $b\in\RR$ \\ \cline{2-2}
            & generic element of $\BB$, $b\in \BB$ \\ \hline            
			$c$ & generic cost, $c\in\costset$, Chap.~\ref{chap:fairness}  \\ \dashline \cline{2-2}
            \multirow{2}{*}{$\cost(\tau;s)$} & cost incurred by a fairness shield on a trace $\tau$ up to a certain time $s\leq |\tau|$, Eq.~\ref{eq:cost-defining-eq}, Chap.~\ref{chap:fairness} \\ \hline 
            \multirow{3}{*}{$d$} & used to denote a generic probability distribution, Sec~\ref{sec:prelim-probability-theory} \\ \cline{2-2}
            & when computing maximally permissive strategies under delay, the value of the intermediate delay, Sec.~\ref{sec:prelim-games-under-delay}, Alg.~\ref{alg:delayed_strategies_extended} \\ \cline{2-2}
            & used to denote a generic distance function in MDPs, Sec.~\ref{sec:prelim-MDPs}, Chap.~\ref{chap:intention} \\
            \hline 
            \multirow{2}{*}{$f$} & used to denote a generic function $f\colon X\to Y$ \\ \cline{2-2} 
            & in classification problems, used to denote an ML-based classifier, Chap.~\ref{chap:fairness} \\ \hline 
            \multirow{2}{*}{$g$} & used to denote a generic group in group fairness, typically $g\in\{a,b\}$, Chap.~\ref{chap:fairness} \\ \hline 
            $i, j$ & used as generic counters \\ \hline
            \multirow{2}{*}{$k$} & used to indicate the length of a trace in reachability and avoidance properties in MDPs, Eq.~\ref{eq:def-reachability-property} \\ \cline{2-2}
              & used as a generic counter \\ \hline 
              \multirow{2}{*}{$l$} & lower bound on welfare for bounded welfare shields, Sec.~\ref{sec:approach-staticBW}, Chap.~\ref{chap:fairness}\\ \hline
              \multirow{2}{*}{$m$} &  when computing maximally permissive strategies under delay, the value of the intermediate memory, $m=\min(d,\mu)$, Alg.~\ref{alg:delayed_strategies_extended}
            \end{tabular}
	\caption*{Notation index, lowercase latin alphabet, part 1.}
\end{table}

\begin{table}
		\begin{tabular}{l p{10cm}}
			Symbol & Usage \\ \midrule       
            $n$ & used in general to denote lengths of traces or sequences  \\ \dashline\cline{2-2}
            \multirow{2}{*}{$n_A$} & in group fairness measures, number of candidates in a tracce of group $A$, Chap.~\ref{chap:fairness} \\ \dashline\cline{2-2}
            \multirow{2}{*}{$n_B$} & in group fairness measures, number of candidates in a tracce of group $B$, Chap.~\ref{chap:fairness} \\ \dashline\cline{2-2}
            \multirow{2}{*}{$n^1_A$} & in group fairness measures, number of accepted candidates in a tracce of group $A$, Chap.~\ref{chap:fairness} \\ \dashline\cline{2-2}
            \multirow{2}{*}{$n^1_B$} & in group fairness measures, number of accepted candidates in a tracce of group $B$, Chap.~\ref{chap:fairness} \\
            \hline 
            \multirow{2}{*}{$o$} & used to denote a generic observation in the reactive decision making framework, $o\in\Obs$, Chap.~\ref{chap:reactive_decision_making} \\ \hline 
            \multirow{2}{*}{$p$} & in fairness shield, the number of counters required by the statistic to compute the fairness property, Sec.~\ref{sec:fairness-enf-minimal-cost}, Thm.~\ref{thm:bounded-horizon shield synthesis-bis}, Chap.~\ref{chap:fairness} \\ \hline
            \multirow{2}{*}{$r$} & a generic radius of a ball, Sec.~\ref{sec:prelim-MDPs} \\ \cline{2-2}
            & recommendation of the ML-based classifier, Chap.~\ref{chap:fairness} \\ \hline 
            \multirow{2}{*}{$s$} & generic state of a state set, $s\in\S$ \\ \cline{2-2}
            & used to denote a point or time in the trace in Eq.~\ref{eq:cost-defining-eq} \\
            \dashline \cline{2-2}
            \multirow{2}{*}{$s_0$} & initial state of a safety game, Sec.~\ref{sec:prelim-TwoPlayerGames},~\ref{sec:rdm-two-player-games}, Chap.~\ref{chap:delayed_shields} \\ \cline{2-2}
            & when it is unique, initial state of an MDP, Sec.~\ref{sec:prelim-MDPs} \\ \hline
            $t$ & used as a generic time or length of a trace \\ \hline 
            \multirow{2}{*}{$u$} & upper bound on welfare for bounded welfare shields, Sec.~\ref{sec:approach-staticBW}, Chap.~\ref{chap:fairness}\\ \dashline\cline{2-2}
            \multirow{2}{*}{$u_{i,j}^k$} & $k$-th velocity datapoint for reference action $\alpha_i$ and reference velocity $v_j$, Eq.~\eqref{eq:thegammaeq}, Chap.~\ref{chap:foceta} \\ \hline 
            $v$ & used to indicate a generic velocity, Chaps.~\ref{chap:delayed_shields}, \ref{chap:foceta}, \ref{chap:intention} \\ \dashline\cline{2-2} 
            $v(\tau)$ & value function associated with trace $\tau$, Sec.~\ref{sec:synthesis}, Eq.~\ref{eq:v-def-eq} \\ \hline   
            $w$ & sometimes used to refer to a generic word of an alphabet, Sec.~\ref{sec:prelim-basic-notation} \\ \hline 
            \multirow{2}{*}{$x$} & generic element of a set $x\in X$ \\ \cline{2-2}
            & used to denote a generic input in fairness shields, Chap.~\ref{chap:fairness} \\ \hline 
            \multirow{3}{*}{$y$} & used to denote a generic action in the action register for safety games under delay, Sec.~\ref{sec:prelim-games-under-delay}, Chap.~\ref{chap:delayed_shields} \\ \cline{2-2}
            & in fairness shields indicates the final \texttt{accept/reject} decision of the shield, $y\in\Y$, Chap.~\ref{chap:fairness}
		\end{tabular}
	\caption*{Notation index, lowercase latin alphabet, part 2.}
\end{table}

\begin{table}
		\begin{tabular}{l p{10cm}}
			Symbol & Usage \\ \midrule
			\multirow{2}{*}{$A$} & in group fairness, abstract groups are typically names groups $A$ and $B$, used mostly in Chap.~\ref{chap:fairness} \\ \dashline\cline{2-2} 
            \multirow{2}{*}{$\Acc$} & winning condition of a two-player game, given by a set of accepting traces, Sec.~\ref{sec:prelim-TwoPlayerGames} \\ \dashline \cline{2-2}
            $\Ag$ & agent in the reactive decision making framework, Chap.~\ref{chap:reactive_decision_making} \\ \dashline\cline{2-2}
            $\AP$ & set of atomic propositions to label an MDP, Chap.~\ref{chap:intention} \\ \hline 
			\multirow{2}{*}{$B$} & in group fairness, abstract groups are typically names groups $A$ and $B$, used mostly in Chap.~\ref{chap:fairness} \\
            \dashline \cline{2-2}
            $B_r(x)$ & ball of radius $r$ centered at point $x$, Sec.~\ref{sec:prelim-MDPs} \\ \dashline\cline{2-2} 
			$\Beh$ & Set of behaviours of a strategy in a safety game, Eq.~\eqref{eq:safety-game-behaviour} \\ \hline
            $\Cyl$ & cylinder set construction, Sec.~\ref{sec:prelim-cylinder-set-construction} \\ \hline 
            $\mathtt{D}$ & ``down'' action in gridworlds, Chap.~\ref{chap:delayed_shields} \\ \hline 
            \multirow{2}{*}{$F$} & generic cummulative distribution function of a random variable, Sec.~\ref{sec:prelim-probability-theory} \\ \dashline\cline{2-2} 
            $F_k(s,\ol{\sigma})$ & $k$-forward multiset of states, Def.~\ref{def:forward-multiset-states}, Chap.~\ref{chap:delayed_shields} \\ \dashline\cline{2-2} 
            \multirow{2}{*}{$\Trfeas^t_{\theta, \Pol}$} & set of feasible traces of length $t$ sampling inputs from distribution $\theta$ and using shield $\Pol$, Sec.~\ref{sec:fairness-env-and-shielding}, Chap.~\ref{chap:fairness} \\ \hline 
            \multirow{3}{*}{$G$} & discounted return in RL, Sec.~\ref{sec:prelim-RL} \\ \cline{2-2}
            & random variable representing group membership of a candidate in fairness classification problems, Sec.~\ref{sec:prelim-classification-problems}, Chap.~\ref{chap:fairness} \\
            \hline
            $\mathtt{L}$ & ``left'' action in gridworlds, Chap.~\ref{chap:delayed_shields} \\ \hline
            \multirow{2}{*}{$N$} & number of samples used to build the transition probability function for the car model in Chap.~\ref{chap:foceta} \\ \dashline\cline{2-2}
            $\mathtt{N}$ & ``neutral'' or ``no operation'' action in gridworlds, Chap.~\ref{chap:delayed_shields} \\ \hline
            $P_{\mathrm{ag}}$ & set of positions of the ego car, Sec.~\ref{sec:shields-experiments} \\ \dashline\cline{2-2}
            \multirow{2}{*}{$P_{\mathrm{env}}$} & set of positions of the environment, either other car or pedestrian, Sec.~\ref{sec:shields-experiments} \\ \hline
            $\mathtt{R}$ & ``right'' action in gridworlds, Chap.~\ref{chap:delayed_shields} \\ \cline{2-2} \dashline
            \multirow{2}{*}{$R_{\mu,T}$} & when $\mu$ is a fairness statistic and $T$ is a time horizon, $R_{\mu,T}$ is the range of values that $\mu$ can take on traces of length up to $T$, Theorem~\ref{thm:bounded-horizon shield synthesis-bis}, Chap.~\ref{chap:fairness} \\ \hline
            \multirow{2}{*}{$S$} & set of states of a deterministic two-player game, Sec.~\ref{sec:prelim-TwoPlayerGames},~\ref{sec:rdm-two-player-games}, Chap.~\ref{chap:delayed_shields} \\ \dashline\cline{2-2}
            \multirow{2}{*}{$S_{ag}$} & set of states controlled by the agent of a deterministic two-player game, Sec.~\ref{sec:prelim-TwoPlayerGames},~\ref{sec:rdm-two-player-games}, Chap.~\ref{chap:delayed_shields} \\ \dashline\cline{2-2}
            \multirow{2}{*}{$S_{ag*}$} & set of states controlled by the agent of a deterministic two-player game, plus an extra void state $\eps$, used for defining strategies in delayed games, Sec.~\ref{sec:prelim-games-under-delay}, Chap.~\ref{chap:delayed_shields} \\ \dashline\cline{2-2}
            \multirow{2}{*}{$S_{env}$} & set of states controlled by the environment of a deterministic two-player game, Sec.~\ref{sec:prelim-TwoPlayerGames},~\ref{sec:rdm-two-player-games}, Chap.~\ref{chap:delayed_shields} \\ \dashline\cline{2-2}
            $\Supp(f)$ & support of a function $f$, Sec.~\ref{sec:prelim-basic-notation}
            		\end{tabular}
	\caption*{Notation index, uppercase latin alphabet, part 1.}
\end{table}

\begin{table}[h!]
		\begin{tabular}{l p{10cm}}
			Symbol & Usage \\ \midrule
            \multirow{2}{*}{$T$} & set of target states in reachability properties, Sec.~\ref{sec:prelim-reachability-properties} \\ \cline{2-2}
            & time horizon when computing fairness shields, Chap.~\ref{chap:fairness} \\ \hline
            $\mathtt{U}$ & ``up'' action in gridworlds, Chap.~\ref{chap:delayed_shields} \\ \hline 
            $V$ & \\ \hline 
            $V_{\mathrm{ag}}$ & set of velocities of the ego car, Sec.~\ref{sec:shields-experiments} \\ \dashline\cline{2-2}
            $V_{\mathrm{env}}$ & set of velocities of the environment's car, Sec.~\ref{sec:shields-experiments} \\ \dashline\cline{2-2}
            $\Val$ & valuation function in a labelled MDP, $\Val\colon \AP\to 2^\S$, Chap.~\ref{chap:intention} \\ \hline 
            $W$ & winning region of a safety game, Eq.~\eqref{eq:prelim-safety-games-winning-strategy}, Chaps.~\ref{chap:preliminaries}-\ref{chap:delayed_shields} \\ \dashline\cline{2-2}
            $\mathtt{WF}^{g}$ & welfare function of group $g\in\G$, Chap.~\ref{chap:fairness} \\ \hline
            \multirow{2}{*}{$X$} & used to denote a generic set \\ \cline{2-2}
             & used to denote a generic random variable
		\end{tabular}
	\caption*{Notation index, uppercase latin alphabet, part 2.}
\end{table}

\begin{table}[h!]
		\begin{tabular}{l p{10cm}}
			Symbol & Usage \\ \midrule
			$\A$ & Set of actions available to the agent in all formalisms \\ \dashline\cline{2-2}
			\multirow{2}{*}{$\A_{env}$} & Set of actions of the environment in safety games,  Sec.~\ref{sec:prelim-TwoPlayerGames}, Eq.~\eqref{eq:safety-games-env-out-degree}\\ \hline
			$\B$ & Borel $\sigma$-algebra, Sec.~\ref{sec:prelim-probability-theory} \\ \hline
			\multirow{2}{*}{$\D$} & set of probability distributions, given a set $X$, $\D(X)$ is the set of distributions over $X$, defined in Sec.~\ref{sec:prelim-probability-theory}, and used through all the thesis \\ \hline
			\multirow{3}{*}{$\F$} & generic $\sigma$-algebra, Sec.~\ref{sec:prelim-probability-theory}  \\ \cline{2-2}
            & set of safe states of a safety game, Sec.~\ref{sec:prelim-TwoPlayerGames}, Sec.~\ref{sec:shielding-in-safety-games}, Sec.~\ref{sec:shielding-in-safety-games-delayed}, Chap.~\ref{chap:delayed_shields} \\ \cline{2-2}
            & in classification problems, the input is factored as $\G\times \F$, where $\G$ is the space of the protected features, and $\F$ is that of the other features, Sec.~\ref{sec:prelim-classification-problems} \\ \hline
			\multirow{2}{*}{$\G$} & game graph of a two-player determinisitc game, Sec.~\ref{sec:prelim-TwoPlayerGames}, Chap.~\ref{chap:delayed_shields}  \\ \cline{2-2}
            & protected feature (a.k.a. group membership) in fairness for classification problems, Sec.~\ref{sec:prelim-classification-problems}, Chap.~\ref{chap:fairness} \\ \cline{2-2} \dashline
            \multirow{2}{*}{$\G_{\delta,\mu}$} & game graph, emphasizing that the plays are with delay $\delta$ and the strategies are allowed a memory $\mu$, Sec.~\ref{sec:prelim-TwoPlayerGames}, Chap.~\ref{chap:delayed_shields} 
		\end{tabular}
	\caption*{Notation index, \texttt{mathcal} latin alphabet, part 1.}
\end{table}

\begin{table}
		\begin{tabular}{l p{10cm}}
			Symbol & Usage \\ \midrule
			\multirow{2}{*}{$\I$} & generic Boolean formula over the set of atomic propositions associated with the MDP, that defines a potential \emph{intention} in Chap.~\ref{chap:intention}  \\ \dashline\cline{2-2} 
            $\I_{s}$ & auxiliary set of states used in Alg.~\ref{alg:delayed_strategies_extended} \\ \dashline\cline{2-2} 
            $\I_{s, y}$ & auxiliary set of states used in Alg.~\ref{alg:delayed_strategies_extended} \\           
            \hline 
			\multirow{2}{*}{$\J$} & set of reachable states from $s_0$, Algorithm~\ref{alg:initialmoves}  \\ \cline{2-2}
            & generic Boolean formula over the set of atomic propositions associated with the MDP, that defines a potential \emph{intention} in Chap.~\ref{chap:intention}
            \\ \hline
            
			\multirow{2}{*}{$\L$} & Loss function of a classification problem, Sec.~\ref{sec:prelim-classification-problems} \\ \cline{2-2}
            & set of correct traces, Chap.~\ref{chap:reactive_decision_making}, Def.~\ref{def:correctness-of-a-shield} \\ \cline{2-2} \dashline
            \multirow{2}{*}{$\L_{T,\lambda,k}$} & set of correct traces for probabilistic shields in MDPs, where $T$ is a subset of states to reach, $\lambda\in (0,1)$ is the safety threshold and $k$ is the step horizon, Sec.~\ref{sec:rdm-probabilistic-shielding-mdps} \\ \hline            
        \multirow{2}{*}{$\M$} & used throughout the thesis to indicate a Markovian model, either a Markov chain or a Markov decision process, depending on the context  \\ \cline{2-2} \dashline
        $\Mcar$ & MDP model of the car, Chap.~\ref{chap:foceta} \\ \cline{2-2} \dashline
        $\Mped$ & Markov chain model of the pedestrian, Chap.~\ref{chap:foceta} \\ \hline
        \multirow{2}{*}{$\mathcal{O}$} & ``Big O'' notation for stating complexity results \\ \cline{2-2}
        & observation space in reactive decision making framework, Chap.~\ref{chap:reactive_decision_making} \\ \hline         
        \multirow{2}{*}{$\P$} & transition probability function of a Markovian model, either an MDP or a Markov chain, depending on the context, defined in Sec.~\ref{sec:prelim-MDPs}, used throughout  \\ \hline
        $\R$ & reward function in an RL problem, Sec.~\ref{sec:prelim-RL} \\ \hline    
        \multirow{2}{*}{$\S$} &  set of states of a Markovian model, either an MDP or a Markov chain, depending on the context, defined in Sec.~\ref{sec:prelim-MDPs}, used throughout\\ \hline
        $\T$ & transition relation of a safety game, Sec.~\ref{sec:prelim-TwoPlayerGames}, Sec.~\ref{sec:rdm-two-player-games} \\ \dashline\cline{2-2} 
        \multirow{2}{*}{$\Trans$} & environment transition function in the reactive decision making framework, Chap.~\ref{chap:reactive_decision_making} \\ \hline 
        $\X$ & input space in fairness shields, Chap.~\ref{chap:fairness}  \\ \hline 
        $\Y$ & output (or decision) space in fairness shields, $|\Y|=2$, Chap.~\ref{chap:fairness} \\ \hline
        \multirow{2}{*}{$\Z$} & in classification problems with a protected feature, $\Z$ represents the set of non-protected features, Chap.~\ref{chap:fairness}
		\end{tabular}
	\caption*{Notation index, \texttt{mathcal} latin alphabet, part 2.}
\end{table}

\begin{table}
		\begin{tabular}{l p{10cm}}
			Symbol & Usage \\ \midrule
			\multirow{2}{*}{$\BB$} & Boolean domain, $\BB = \{\bot,\top\}$, sometimes equivalently $\BB = \{0,1\}$, Sec.~\ref{sec:prelim-basic-notation}, Chap.~\ref{chap:fairness} \\ \hline
			$\costset$ & Set of costs in Chapter~\ref{chap:fairness}. It is understood that $\costset$ is finite. \\ \hline
			$\EE$ & Expected value, defined in Sec.~\ref{sec:prelim-probability-theory}, used throughout Chap.~\ref{chap:fairness}  \\ \cline{2-2} \dashline
            \multirow{2}{*}{$\expe[\cost;\sgen,\sshield,t]$} &
            expected cost of a trace of length $t$ produced sampling inputs from $\sgen\in\D(\X)$ and using the shield $\pi$, Chap.~\ref{chap:fairness}, Eq.~\eqref{eq:expected-cost-time-t} \\ \cline{2-2} \dashline
            \multirow{2}{*}{$\expe[\cost\mid \tau;\sgen,\sshield,t]$} & 
            expected cost of a trace of length $t$ produced sampling inputs from $\sgen\in\D(\X)$, and using the shield $\pi$ with $\tau$ as a prefix, Chap.~\ref{chap:fairness}, Eq.~\eqref{eq:expected-cost-after-tau-time-t} \\
            \hline
            $\mathbb{J}$ & interference set of a shield, Def.~\ref{def:interferenceSet}, Chap.~\ref{chap:reactive_decision_making} \\ \hline
			$\NN$ & Set of natural numbers, $\NN = \{0,1,2,\dots\}$ \\ \hline
			\multirow{2}{*}{$\PP$} &  probability measure in a probability space, Sec.~\ref{sec:prelim-probability-theory} \\ \cline{2-2}
            & probability measure over sets of finite traces, Chap.~\ref{chap:fairness} \\ 
            \cdashline{1-1}[0.5pt/2pt] \cline{2-2}
            $\PP^{\M}$ & probability measure associated with the Markov chain $\M$, Sec.~\ref{sec:prelim-probability-theory} \\ \cdashline{1-1}[0.5pt/2pt] \cline{2-2}
            \multirow{2}{*}{$\PP^{\M}_\pi$/$\PP_\pi$} & probability measure associated with the MDP $\M$ and the policy $\pi$. Whenever $\M$ is clear from context, we may drop it from the notation, Sec.~\ref{sec:prelim-probability-theory}, Chap.~\ref{chap:foceta}, Chap.~\ref{chap:intention} \\ \cdashline{1-1}[0.5pt/2pt] \cline{2-2}
            \multirow{2}{*}{$\PP_{\max}^{\M}$} &  probability measure associated with the MDP $\M$ and the policy that maximizes a certain property, Chap.~\ref{chap:foceta}, Chap.~\ref{chap:intention} \\ \cdashline{1-1}[0.5pt/2pt] \cline{2-2}
            \multirow{2}{*}{$\PP_{\min}^{\M}$} &  probability measure associated with the MDP $\M$ and the policy that minimizes a certain property, Chap.~\ref{chap:foceta}, Chap.~\ref{chap:intention} \\ \cdashline{1-1}[0.5pt/2pt] \cline{2-2}
            \multirow{2}{*}{$\PP^\M_{\max\mid\Pi}$/$\PP_{\max\mid\Pi}$} & probability measure associated with the MDP $\M$ and the policy that maximizes a certain property among policies in the set $\Pi$, Chap.~\ref{chap:foceta}, Chap.~\ref{chap:intention} \\ \cdashline{1-1}[0.5pt/2pt] \cline{2-2} 
            \multirow{2}{*}{$\PP^\M_{\min\mid\Pi}$/$\PP_{\min\mid\Pi}$} & probability measure associated with the MDP $\M$ and the policy that minimizes a certain property among policies in the set $\Pi$, Chap.~\ref{chap:foceta}, Chap.~\ref{chap:intention} \\ \hline 
			$\RR$ & Set of real numbers.  \\ \cdashline{1-1}[0.5pt/2pt] \cline{2-2}
			$\RR_{\geq 0}$ & Set of non-negative real numbers.\\ \midrule
			$\ZZ$ & Set of integer numbers.
		\end{tabular}
	\caption{Notation index, \texttt{mathbb} latin alphabet.}
	\label{tab:notation-latin-mathbb}
\end{table}

\begin{table}
		\begin{tabular}{l p{10cm}}
			Symbol & Usage \\ \midrule
			\multirow{1}{*}{$\alpha$} & each of the individual actions in the available set of actions, $\Act = \{\alpha_1,\dots,\alpha_n\}$, Chap.~\ref{chap:foceta} \\ \hline
			\multirow{2}{*}{$\gamma$} & discount factor in RL, Sec.~\ref{sec:prelim-RL} \\ \cline{2-2}
            & experimental proportionality factor between $\Delta x$ and $u$, Equation~\ref{eq:thegammaeq}
            \\ \hline
                \multirow{2}{*}{$\delta$} & delay in safety games, Sec.~\ref{sec:prelim-games-under-delay}, Chap.~\ref{chap:delayed_shields} \\ \cline{2-2}
                & delay in reactive decision making, Sec.~\ref{sec:rdm-delayed-obs}, Sec.~\ref{sec:shielding-in-safety-games-delayed}\\ \cline{2-2} \dashline
                \multirow{2}{*}{$\delta_{\max}$} & cutoff value of the delay when computing controllability values, Chap.~\ref{chap:delayed_shields}
                \\ \dashline \cline{2-2}
                \multirow{2}{*}{$\delta_\rho^L$, $\delta_\rho^U$} & thresholds for intention quotient, used in the retrospective method for assessing intentional behaviour Sec.~\ref{sec:intention-retrospective}, Def.~\ref{def:intention-thresholds} \\ \cline{2-2} \dashline
                $\delta_\sigma$ & threshold for agency, 
                Sec.~\ref{sec:intention-retrospective}, Def.~\ref{def:intention-thresholds} \\ \cline{2-2}\dashline 
                \multirow{2}{*}{$\delta_B$, $\delta_I$} & thresholds on ``belief'' and intention quotient to define a notion of commitment, Def.~\ref{def:commitment} \\
                \hline
                \multirow{3}{*}{$\eps$} & representation of a general unobserved state in safety games under delay, Sec.~\ref{sec:prelim-games-under-delay} \\ 
                \cline{2-2}
                & used to indicate a generic small number, Ex.~\ref{ex:delayeds-shields-example} \\ \cline{2-2}
                & $\eps = (\eps_{k+1},\dots, \eps_m)$ indicates, for each integral variable, the range of variation to consider counterfactuals valid when generating counterfactuals on a factored MDP, Sec.~\ref{sec:intention-factored-mdp} \\
                \hline
                \multirow{2}{*}{$\eta$} & distance between two traces to consider a counterfactual valid, Sec.~\ref{sec:intention-counterfactual-distance} \\ \hline
                
                \multirow{2}{*}{$\theta$} & probability distribution of the input in classification problems,
                $\theta \in \D(\X)$,
                Sec.~\ref{sec:prelim-classification-problems}, Chap.~\ref{chap:fairness} \\ \hline
                $\iota$ & distribution of initial states of an MDP, Sec.~\ref{sec:prelim-MDPs} \\ \hline
                \multirow{2}{*}{$\kappa$} & threshold on the fairness metric, as part of the specification of fairness shields, Chap.~\ref{chap:fairness}, Eq.~\eqref{eq:bounded-horizon-def} \\ \hline
                \multirow{2}{*}{$\lambda$} & safety threshold in probabilistic shields, Sec.~\ref{sec:rdm-probabilistic-shielding-mdps}, Chap.~\ref{chap:foceta} \\ \cline{2-2}
                 & parameter that regulates fairness interventions in different in-processing fairness algorithms, Sec.~\ref{sec:experimental-setup} \\ \hline
                \multirow{3}{*}{$\mu$} & generic probability measure, Sec.~\ref{sec:prelim-probability-theory} \\ \cline{2-2}
                & memory in a strategy for a safety game with delay, Sec.~\ref{sec:prelim-games-under-delay}, Chap.~\ref{chap:delayed_shields}
                \\ \cline{2-2}
                & statistic that maps a trace to the relevant counters used to compute a fairness property, defined in Sec.~\ref{sec:fairness-enf-minimal-cost}, used throughout Chap.~\ref{chap:fairness} \\ \dashline \cline{2-2} 
                \multirow{2}{*}{$\mu_{\mathrm{pos}}$} & multiplier to convert positions between local-continuous and local-discrete coordinates, Chap.~\ref{chap:foceta} \\ \dashline\cline{2-2}
                \multirow{2}{*}{$\mu_{\mathrm{vel}}$} & multiplier to convert velocities between local-continuous and local-discrete coordinates, Chap.~\ref{chap:foceta} \\ \hline 
                $\nu$ & used as a counter in the proof of Thm.~\ref{thm:shield-delayed-safety-game} \\ \hline 
                \multirow{2}{*}{$\xi$} & a generic strategy in a safety game, Sec.~\ref{sec:prelim-TwoPlayerGames}, Chap.~\ref{chap:reactive_decision_making} \\
                \cline{2-2}
                & sometimes, when clear from context, especially in Chap.~\ref{chap:delayed_shields}, $\xi$ denotes the maximally permissive winning strategy \\ \cline{2-2} \dashline
                $\xi_{\mathtt{max.perm.}}$ & maximally permissive winning strategy of a safety game, Eq.~\eqref{eq:max-perm-strat-safety-game} \\ \cline{2-2} \dashline 
                \multirow{2}{*}{$\xi_{\delta,\mu}$} & in Chap.~\ref{chap:delayed_shields}, to specify that the strategy works with delay $\delta$ and memory $\mu$ \\ \hline
                
			\multirow{2}{*}{$\pi$}
            & policy in an MDP, Sec.~\ref{sec:prelim-MDPs}, Sec.~\ref{sec:prelim-RL}, Chap.~\ref{chap:intention} 
            \\ \cline{2-2}
            & agent policy function, Chap.~\ref{chap:reactive_decision_making}. 
            Sometimes $\Pol$ is used to refer to the agent $\Ag = (\Obs,\Act,\Pol)$ following $\Pol$ 
		\end{tabular}
	\caption*{Notation index, lowercase greek alphabet, part 1.}
\end{table}

\begin{table}
		\begin{tabular}{l p{10cm}}
			Symbol & Usage \\ \midrule
            \multirow{2}{*}{$\sigma$} & elements of a word, Sec.~\ref{sec:prelim-basic-notation} \\ \cline{2-2}
            & generic action in a safety game, usually to denote actions of the action memory, Sec.~\ref{sec:prelim-TwoPlayerGames}, Sec.~\ref{sec:rdm-two-player-games}, Sec.~\ref{sec:shielding-in-safety-games-delayed}, Chap.~\ref{chap:delayed_shields} \\ \cline{2-2}
            & agency of a state or a set of states, Chap~\ref{chap:intention}, Def.~\ref{def:agency} \\ \cline{2-2} \dashline
            $\ol{\sigma}$ & action memory or register, $\ol{\sigma} = (\sigma_1,\dots,\sigma_\mu)$, Chap.~\ref{chap:delayed_shields} \\ \cline{2-2} \dashline
            \multirow{2}{*}{$\sigma_{\mathrm{ped}}$} & parameter of the model of the pedestrian in Chap.~\ref{chap:foceta}, indicating how volatile is their behaviour \\ \hline
            \multirow{2}{*}{$\tau$} & trace (a.k.a. path) in a safety game, Sec.~\ref{sec:prelim-TwoPlayerGames}, Sec.~\ref{sec:rdm-two-player-games}, Chap.~\ref{chap:delayed_shields}
            \\ \cline{2-2}
            & trace in reactive decision making, $\tau\in(\Obs\times\Act)^*$, Chap.~\ref{chap:reactive_decision_making} \\ \cline{2-2}
            & trace of a fairness shield, $\tau\in (\X\times\Y)^*$, Chap.~\ref{chap:fairness} \\ \cline{2-2}
            & trace of states of the MDP, Chap.~\ref{chap:intention} \\ \cline{2-2}\dashline 
            $\tau_A$ & action trace, $\tau_A\in\Act^*$, Chap.~\ref{chap:reactive_decision_making} \\ \cline{2-2}\dashline 
            $\tau_O$ & observation trace, $\tau_O\in\Obs^*$, Chap.~\ref{chap:reactive_decision_making} \\ \cline{2-2} \dashline 
            \multirow{2}{*}{$\tauref$} & reference trace in the retrospective method to analyze intention, Sec.~\ref{sec:intention-retrospective}, Sec.~\ref{sec:intention-experiments} \\ \hline    
            \multirow{3}{*}{$\varphi$} & fitness function, Chap.~\ref{chap:delayed_shields} \\ \cline{2-2}
            & fairness metric, Chap.~\ref{chap:fairness} 
            \\ \cline{2-2}
            & a generic reachability property, Chap.~\ref{chap:intention} \\
            \dashline\cline{2-2}
            $\varphi_c$ & controllability fitness function, Sec.~\ref{sec:controllability-and-robustness}, Chap.~\ref{chap:delayed_shields} \\ \dashline\cline{2-2}
            $\varphi_r$ & robustness fitness function, Sec.~\ref{sec:controllability-and-robustness}, Chap.~\ref{chap:delayed_shields} \\ \dashline\cline{2-2}
            \hline
            \multirow{1}{*}{$\chi$} & a \emph{deterministic} winning strategy, usually to build a post-shield, Sec.~\ref{sec:shielding-in-safety-games}, Sec.~\ref{sec:shielding-in-safety-games-delayed}, Chap~\ref{chap:delayed_shields} \\ \hline
			\multirow{2}{*}{$\omega$} & sample of a probability space $\omega\in\Omega$, Sec.~\ref{sec:prelim-probability-theory}, Sec.~\ref{sec:reactive-decision-making} \\ \cline{2-2}
            & finite trace prefix in the cylinder set construction, Sec.~\ref{sec:prelim-cylinder-set-construction} \\ \cline{2-2}
            & used to denote infinite repetitions, e.g., $X^\omega$ is the set of infinite sequences of elements in $X$
		\end{tabular}
	\caption*{Notation index, lowercase greek alphabet, part 2.}
\end{table}

\begin{table}
		\begin{tabular}{l p{10cm}}
			Symbol & Usage \\ \midrule
			$\Delta$ & increment of a variable, e.g., $\Delta t$, Sec.~\ref{sec:shields-experiments}, Chap.~\ref{chap:foceta} \\ \hline
            $\Theta$ & distribution of input $\Theta_{\X}\in \D(\Obs)$, Sec.~\ref{sec:rdm-classification-problems} \\ \hline
            \multirow{3}{*}{$\Pi$} & Subset of available policies when computing maximum and minimum reachability properties in MDPs, Sec.~\ref{sec:prelim-reachability-properties}, Chap.~\ref{chap:intention} \\  \cline{2-2}
             & Set of agents to which a shield is restricted to work with, Chap.~\ref{chap:reactive_decision_making} \\ \cline{2-2}
            & set of all fairness shields, Chap.~\ref{chap:fairness} \\ \cdashline{1-1}[0.5pt/2pt] \cline{2-2}
            $\sShieldFeas^t$ & set of all fairness shields with bounded horizon $t$ \\ \cdashline{1-1}[0.5pt/2pt] \cline{2-2}
            $\sShieldFeas$ & set of fair shields for a given specification, Chap.~\ref{chap:fairness} \\ \cdashline{1-1}[0.5pt/2pt] \cline{2-2}
            \multirow{2}{*}{$\sShieldFeasPeriodic$} & set of all periodic fair shields for a given specification, Chap~\ref{chap:fairness}, Eq.~\eqref{eq:sShieldFeasPeriodic_definition} \\ \cdashline{1-1}[0.5pt/2pt] \cline{2-2}
            \multirow{2}{*}{$\sShieldFeasBounded$} & set of all fair shields with respect to a bounded welfare specification, Chap.~\ref{chap:fairness} \\ \cdashline{1-1}[0.5pt/2pt] \cline{2-2}
            $\sShieldFeasDyn$ & set of all dynamic fair shields for a given specification, Chap.~\ref{chap:fairness} \\ \cdashline{1-1}[0.5pt/2pt] \cline{2-2}
            $\Pi(\G)$ & set of plays in a deterministic two player game $\G$, Sec.~\ref{sec:prelim-TwoPlayerGames}\\ \cdashline{1-1}[0.5pt/2pt] \cline{2-2}
            \multirow{2}{*}{$\Pi_U(s)$} & set of paths or traces in a safety game starting from $s$
    that end outside of the winning region in exactly 
    $\delta + k$ transitions, in the proof of Thm.~\ref{thm2} \\ \hline
    \multirow{2}{*}{$\Sigma$} & generic alphabet, Sec.~\ref{sec:prelim-basic-notation} \\ \cline{2-2}
    & subset of actions in a safety game, $\Sigma\in 2^\Act$, Sec.~\ref{sec:prelim-TwoPlayerGames} \\ \cdashline{1-1}[0.5pt/2pt] \cline{2-2} 
    $\Sigma_\Pi$ & set of shields associated with a set of agents $\Pi$, Chap.~\ref{chap:reactive_decision_making}, Eq.~\eqref{eq:SigmaPi} \\ \hline
    $\Omega$ & sample set of a measurable space or a probability space, Sec.~\ref{sec:prelim-probability-theory} \\ \cdashline{1-1}[0.5pt/2pt]\cline{2-2}
    \multirow{2}{*}{$\Omega^\M_\pi$} & sample space of the probability measure associated with an MDP $\M$ and a policy $\pi$, Sec.~\ref{sec:prelim-MDPs} \\ \cdashline{1-1}[0.5pt/2pt] \cline{2-2}
    \multirow{2}{*}{$\Omega^{\Env,\Ag}$} & set of all observation-action traces associated with an environment $\Env$ and an agent $\Ag$, Sec.~\ref{sec:reactive-decision-making} \\ \cdashline{1-1}[0.5pt/2pt] \cline{2-2}
    \multirow{2}{*}{$\Omega^{\Env,\Ag}_k$} & set of all observation-action traces of length $k$ associated with an environment $\Env$ and an agent $\Ag$, Sec.~\ref{sec:reactive-decision-making}
		\end{tabular}
	\caption*{Notation index, uppercase greek alphabet.}
\end{table}

\begin{table}
		\begin{tabular}{l p{10cm}}
			Symbol & Usage \\ \hline 
			$\Sh$ & Shield, Chap.~\ref{chap:reactive_decision_making} \\ \dashline\cline{2-2} 
            \multirow{2}{*}{$\Sh^{pre}_{\Ag}$} & pre-shield induced by an agent $\Ag$, defined to follow the actions of the agent, Def.~\ref{def:induced-pre-shield}, Chap.~\ref{chap:reactive_decision_making} \\ \dashline \cline{2-2}
            \multirow{2}{*}{$\Sh^{pos}_{\Ag, \Ag_{det}}$} & post-shield induced by an agent $\Ag$ and a determinization of the agent, $\Ag_{det}$, Def.~\ref{def:induced-post-shield}, Chap.~\ref{chap:reactive_decision_making} \\ 
            \hline 
			\multirow{2}{*}{$\1$} & Indicator function, for subset set $X\subseteq \X$, 
			$\1_X\colon \X\to\{0,1\}$ 
			is defined as 
			$\1_X(x) = 1$ if 
			$x\in X$, and 
			$\1_X(x) = 0$ if 
			$x\notin X$, Sec.~\ref{sec:prelim-basic-notation}\\ \hline 
            \multirow{2}{*}{$2^X$} & when $X$ is a set, $2^X$ denotes the power set of $X$, that is, the set of subsets of $X$, Sec.~\ref{sec:prelim-basic-notation} \\ \hline
            \multirow{2}{*}{$f(X), f^{-1}(Y)$} & when $f\colon \X\to\Y$ is a function and $X\subseteq \X$, $f(X)$ is the image set of $X$; for $Y\subseteq\Y$, $f^-1(Y)$ is the antiimage set, Sec.~\ref{sec:prelim-basic-notation} \\ \hline 
             $\gg$, $\ll$ & much greater / much smaller than, Sec.~\ref{sec:prelim-basic-notation}, Sec.~\ref{sec:fairness-env-and-shielding} \\ \hline 
            \multirow{2}{*}{$\lfloor a \rfloor$} & floor of $a$, i.e., the greatest integer that is smaller or equal to $a$, Sec.~\ref{sec:prelim-basic-notation} \\ 
            \dashline\cline{2-2}
            \multirow{2}{*}{$\lceil a \rceil$} & ceiling of $a$, i.e., the smallest integer that is greater or equal to $a$, Sec.~\ref{sec:prelim-basic-notation} \\ 
            \dashline\cline{2-2}
            $\lfloor a \rceil$ & rounded of $a$, i.e., the closest integer to $a$, and the ceiling of $a$ if $a$ is equidistant to $\lfloor a\rfloor$ and $\lceil a \rceil$, Sec.~\ref{sec:prelim-basic-notation} \\ \hline 
            $|a|$ & when $a$ is a number, $|a|$ denotes the absolute value of $a$ \\ \dashline \cline{2-2}
            $|v|$ & when $v$ is a vector, $|v|$ denotes the magnitude of the vector \\ \dashline\cline{2-2}
            \multirow{2}{*}{$|\tau|$} & when $\tau$ is a word, a trace, or a sequence of some kind, $|\tau|$ denotes its length \\ \dashline \cline{2-2}
            \multirow{2}{*}{$|X|$} & when $X$ is a set, $|X|$ denotes the cardinality of the set, i.e., the number of elements \\ \hline
            \multirow{3}{*}{$\xrightarrow{u}$} & in safety games, when $s\in S_{env}$ and $s'\in S_{ag}$, we use $s\xrightarrow{u} s'$ to denote that there is an environment transition from $s$ to $s'$, without specifying an action of the environment; where $u$ stands for ``undefined'', Sec.~\ref{sec:prelim-TwoPlayerGames},~\ref{sec:rdm-two-player-games}, Chap.~\ref{chap:delayed_shields} \\ \hline
            $\emptyset$ & the empty set 
		\end{tabular} 
	\caption*{Notation index, special symbols.}
\end{table}